%% file: main.tex
\documentclass{article}

\usepackage[accepted]{icml2025}

\usepackage{mymath}

\icmltitlerunning{Implicit Bias of Gradient Descent for Non-Homogeneous Deep Networks}

\begin{document}

\twocolumn[
\icmltitle{Implicit Bias of Gradient Descent for Non-Homogeneous Deep Networks}



\icmlsetsymbol{equal}{*}

\begin{icmlauthorlist}
\icmlauthor{Yuhang Cai}{equal,yyy}
\icmlauthor{Kangjie Zhou}{equal,comp}
\icmlauthor{Jingfeng Wu}{yyy}
\icmlauthor{Song Mei}{yyy}
\icmlauthor{Michael Lindsey}{yyy,sch}
\icmlauthor{Peter L. Bartlett}{yyy,gddd}
\end{icmlauthorlist}

\icmlaffiliation{yyy}{University of California, Berkeley}
\icmlaffiliation{comp}{Columbia University}
\icmlaffiliation{sch}{Lawrence Berkeley National Laboratory}
\icmlaffiliation{gddd}{Google DeepMind}

\icmlcorrespondingauthor{Yuhang Cai}{willcai@berkeley.edu}
\icmlcorrespondingauthor{Kangjie Zhou}{kz2326@columbia.edu}
\icmlcorrespondingauthor{Peter L. Bartlett}{peter@berkeley.edu}

\icmlkeywords{Machine Learning, ICML}

\vskip 0.3in
]



\printAffiliationsAndNotice{\icmlEqualContribution} 

\input{sections/maintext.tex}

\section*{Impact Statement}


This paper presents work whose goal is to advance the field of 
Machine Learning. There are many potential societal consequences 
of our work, none of which we feel must be specifically highlighted here.



\bibliography{ref}
\bibliographystyle{icml2025}

\newpage
\appendix
\onecolumn



\clearpage

\input{sections/appendix.tex}

\end{document}

%% file: sections/maintext.tex
\begin{abstract}
We establish the asymptotic implicit bias of gradient descent (GD) for generic non-homogeneous deep networks under exponential loss. Specifically, we characterize three key properties of GD iterates starting from a sufficiently small empirical risk, where the threshold is determined by a measure of the network's non-homogeneity. First, we show that a normalized margin induced by the GD iterates increases nearly monotonically. Second, we prove that while the norm of the GD iterates diverges to infinity, the iterates themselves converge in direction. Finally, we establish that this directional limit satisfies the Karush–Kuhn–Tucker (KKT) conditions of a margin maximization problem. Prior works on implicit bias have focused exclusively on homogeneous networks; in contrast, our results apply to a broad class of non-homogeneous networks satisfying a mild near-homogeneity condition. In particular, our results apply to networks with residual connections and non-homogeneous activation functions, thereby resolving an open problem posed by \citet{ji2020directional}.
\end{abstract}

\section{Introduction} \label{sec:intro}

Deep networks often have an enormous amount of parameters and are theoretically capable of \emph{overfitting} the training data. 
However, in practice, deep networks trained via \emph{gradient descent} (GD) or its variants often generalize well. 
This is commonly attributed to the \emph{implicit bias} of GD, in which GD finds a certain solution that prevents overfitting \citep{zhang2021understanding,neyshabur2017pac,bartlett2021deep}.
Understanding the implicit bias of GD is one of the central topics in deep learning theory.

The implicit bias of GD is relatively well-understood when the network is \emph{homogeneous} \citep[see][and references therein]{soudry2018implicit,ji2018risk,lyu2020gradient,ji2020directional,wu2023implicit}. 
For linear networks trained on linearly separable data, GD diverges in norm while converging in direction to the maximum margin solution \citep{soudry2018implicit,ji2018risk,wu2023implicit}.
Similar results have been established for generic homogeneous networks that include a class of deep networks, assuming that the network at initialization can separate the training data. 
Specifically, \citet{lyu2020gradient} showed that the normalized margin induced by GD increases nearly monotonically, and that the limiting direction of a subsequence of the GD iterates satisfies the Karush-Kuhn-Tucker (KKT) condition of a margin maximization problem.
Moreover, if the network is definable in an o-minimal structure (see \Cref{sec:prelim}, which is satisfied for most networks), \citet{ji2020directional} showed that \emph{gradient flow} (GF, that is, GD with infinitesimal stepsizes) converges in direction and that the limiting direction aligns with the direction of the gradient.

However, the implicit bias of GD remains largely unknown when the network is \emph{non-homogeneous}, an arguably more common case in deep learning (there are a few exceptions, which will be discussed later in \Cref{sec:related}). 
For instance, networks with residual connections or non-homogeneous activation functions are inherently non-homogeneous. 
As posted as an open problem by \citet{ji2020directional}, it is unclear whether the implicit bias results developed for homogeneous networks can be extended to non-homogeneous networks.

\paragraph{Contributions.} 
In this work, we establish the implicit bias of GD for a broad set of non-homogeneous but definable networks under exponential loss. 
Starting from the simpler case of \emph{gradient flow} (GF), we identify two natural conditions of \emph{near-homogeneity} and \emph{strong separability}, respectively.  
The former condition requires the homogeneous error to grow slower than the output of the network, and the latter condition requires GF to attain a sufficiently small empirical risk depending on the homogeneous error of the network. 
Under these two conditions, we prove the following implicit bias results of GF for non-homogeneous networks:
\begin{enumerate}[leftmargin=*]
\item GF induces a normalized margin that increases nearly monotonically. 
\item GF converges in its direction while diverging in its norm. 
\item The limiting direction of GF satisfies the KKT conditions of a margin maximization problem.
\end{enumerate}

Our near-homogeneity condition covers many commonly used deep networks. In particular, it applies to networks with residual connections and nearly homogeneous activation functions. In addition, we provide structural rules for computing the near-homogeneity order of a network based on that of each layer in the network. 

Our strong separability condition is a generalization of the separability condition used in the prior homogeneous analysis \citep{lyu2020gradient,ji2020directional}. In particular, it reduces to the separability condition when the network is homogeneous.
Later in \Cref{sec:example}, we demonstrate that this condition is satisfiable by any non-degenerate near-homogeneous network. 
Moreover, we show that GF reaches this strong separability condition from zero initialization for training a two-layer network with residual connections.

Finally, we extend the above results from GF to GD with an arbitrarily large stepsize under additional technical conditions. 
Altogether, we extend the implicit bias of GD from homogeneous cases \citep{lyu2020gradient,ji2020directional} to non-homogeneous cases, thereby addressing the open problem posted by \citet{ji2020directional}.

\paragraph{Notation.} 
For two positive-valued functions $f(x)$ and $g(x)$, we write  $f(x)\lesssim g(x)$ (or $f(x)\gtrsim g(x)$)
if $f(x) \le cg(x)$ (or $f(x) \ge cg(x)$) for some constant 
$c \in (0, +\infty)$. 
We write 
$f(x) \eqsim g(x) $ if $f(x) \lesssim g(x) \lesssim f(x)$.
We use $\Rbb_{\ge 0}[x]$ to denote the set of all univariate polynomials with non-negative real coefficients.
For a polynomial $\homop$, we use $\deg \homop$ to denote its degree.
We use $\nabla$ to denote the gradient of a scalar function or the Jacobian of a vector-valued function. We use $\|\cdot \|$ to denote the $\ell^2$-norm of a vector or the operator norm of a matrix. We use $[n]$ to denote the set $\{1,2,\ldots,n\}$. 
We use $\text{conv} \mathcal{A}$ to denote the convex hull of a set $\mathcal{A}$ in Euclidean space. 
Throughout the paper, we define $\phi(x) := \log (1 / (nx))$ 
and $\Phi(x) := \log \phi(x) - 2 / \phi(x)$, where $n$ is the number of samples. 

\subsection{Related Works} \label{sec:related} 
We discuss related papers in the remainder of this section.

\paragraph{Homogeneous Networks.}
We first review prior implicit bias results for deep homogeneous networks. 
In this case, \citet{lyu2020gradient} showed that GD induces a nearly increasing normalized margin, and the direction of a subsequence of GD iterates converges, the limit of which can be characterized by the KKT conditions of a margin maximization problem. 
Part of these results are generalized to \emph{steepest descent} by  \citet{tsilivis2024flavors}.
Using the notion of definability, \citet{ji2020directional} further showed the directional convergence and alignment for gradient flow. 
In the special case of two-layer homogeneous networks, \citet{chizat2020implicit} characterized the limiting direction of (Wasserstein) gradient flow in the large-width limit.
For a two-layer Leaky ReLU network with symmetric data, \citet{lyu2021gradient} showed that gradient flow eventually converges to a linear classifier. 
Different from these works, we aim to establish implicit bias of GD for non-homogeneous deep networks. 

\paragraph{Non-homogeneous Networks.}
Before our work, there are a few papers that extend the implicit bias results from homogeneous cases to certain special non-homogeneous cases \citep{nacson2019lexicographic,chatterji2021does,kunin2023asymmetric,cai2024large}.
The works by \citet{nacson2019lexicographic,kunin2023asymmetric} considered a special non-homogeneous network, which is homogeneous when viewed as a function of a subset of the trainable parameters while other parameters are fixed. The homogeneity orders for different subsets of parameters might be different. However, their results cannot cover networks with many commonly used non-homogeneous activation functions.
The work by \citet{chatterji2021does} showed the margin improvement for GD with small stepsizes for MLPs with a special type of near-homogeneous activation functions. As a consequence, their results do not allow networks that use general non-homogeneous activation functions or residual connections. 
In comparison, we handle a large class of non-homogeneous networks satisfying a natural definition of near-homogeneity, covering far more commonly used deep networks.

The work by \citet{cai2024large} is most relevant to us, in which they proved the margin improvement of GD for near-$1$-homogeneous networks (see their Assumption 1 and our \Cref{def:nearhomo} in \Cref{sec:prelim}).
Our work can be viewed as an extension of theirs by handling near-$M$-homogeneous networks for general $M \ge 1$, as well as proving that GF and large-stepsize GD converge in direction to the KKT point of a margin maximization problem. 







\section{Preliminaries} \label{sec:prelim}

In this section, we set up the problem and introduce basic mathematical tools used in our analysis.

\paragraph{Locally Lipschitz Functions and Clarke Subdifferential.} 
For a function $f: D \to \Rbb$ defined on an open set $D$, we say $f$ is \emph{locally Lipschitz} if, for every $x\in D$, there exists a neighborhood $U$ of $x$ such that $f|_U$ is Lipschitz continuous.  
By Rademacher's theorem, a locally Lipschitz function is differentiable almost everywhere \citep{borwein2000convex}. 
The {\it Clarke subdifferential} of a locally Lipschitz function $f$ at $x\in D$ is defined as 
\[
\begin{aligned}
\partial f(x) \coloneqq \text{conv} \bigg\{ 
 \lim_{i\to \infty} \nabla f(x_i) : x= \lim_{i\to \infty} x_i, \\
\text{where}\ x_i \in D \ \text{and}\ \nabla f(x_i) \ \text{exists}
  \bigg\},
\end{aligned}
\]
which is nonempty, convex, and compact \citep{clarke1975generalized}. 
In particular, if $f$ is continuously differentiable at $x$, then $\partial f(x) = \{\nabla f(x)\}$. 
Elements of $\partial f(x)$ are called {\it subgradients}. 


\paragraph{Gradient Flow.} 
Let $(\xB_i, y_i)_{i=1}^n$ be a binary classification dataset, where $\xB_i \in \Rbb^d$ and $y_i \in \{\pm 1\}$ for all $i\in [n]$.  We denote a network by $f(\param; \cdot) : \Rbb^d \to \Rbb$, where $\param \in \Rbb^D$ are the trainable parameters.   
Thoughout the paper, we assume $f(\param;\xB_i)$ is locally Lipschitz with respect to $\param$ for every $i \in [n]$.
We focus on the empirical risk under the exponential loss defined as 
\begin{equation}
\label{eq:loss}
        \Loss \big( \param \big) \coloneqq \frac{1}{n} \sum_{i=1}^n \ell \big( y_i f(\param; \xB_i) \big),\quad  \ell(x) := e^{-x}.
\end{equation}
A curve $z$ from an interval $I$ to some Euclidean space $\Rbb^m$ is called an {\it arc} if it is absolutely continuous on any compact subinterval of $I$. 
Clearly, the composition of an arc and a locally Lipschitz function is still an arc. 
Following \citet{lyu2020gradient,ji2020directional},
we define \emph{gradient flow} as an arc from $[0, + \infty)$ to $\Rbb^D$ that satisfies
\begin{equation}
    \label{eq:GF}
    \tag{GF}
\frac{\mathrm{d} \param_t}{ \mathrm{d} t} \in - \partial \Loss \big( \param_t\big), \quad \text{for almost every}\ t\ge 0. 
\end{equation}

\paragraph{Homogeneity and Near-Homogeneity.} 
Let $M\ge 1$ be an integer. 
Recall that a locally Lipschitz network $f(\param; \xB)$ is \emph{$M$-homogeneous} \citep{lyu2020gradient,ji2020directional} if for every $\xB\in(\xB_i)_{i=1}^n$,
\begin{equation}\label{eq:homogeneous}
\text{for all $a>0$ and $\param\in\Rbb^D$},\ f(a \param; \xB) = a^M f(\param; \xB).
\end{equation}
One can verify that the above is equivalent to:
for every $\xB\in(\xB_i)_{i=1}^n$, $\thetaB\in\Rbb^D$, and $\boldsymbol{h} \in  \partial_{\param} f(\param ;\xB)$,
\begin{equation}
\label{eq:homo-2}
\langle \boldsymbol{h}, \param \rangle - M  f(\param ;\xB) = 0.
\end{equation}
We generalize this definition by introducing the following near-homogeneous condition, which provides a natural quantification of the homogeneity error of the network $f$: 
\begin{definition}[Near-$M$-homogeneity]
    \label{def:nearhomo}
Let $M\ge 1$ be an integer. A network $f(\param; \xB)$ is called \emph{near-$M$-homogeneous}, if there exist $\homop, \homoq \in \Rbb_{\ge 0} [x]$ with $\deg \homop, \deg \homoq \le M$ such that the following holds for every $\xB\in(\xB_i)_{i=1}^n$, $\thetaB\in\Rbb^D$, and $\boldsymbol{h} \in  \partial_{\param} f(\param ;\xB)$:
\begin{assumpenum}
\item [(A1).]
$|\langle \boldsymbol{h}, \param \rangle - M  f(\param ;\xB)| \le \homop^\prime (\|\param\|)$;
    
\item [(A2).]
$\| \boldsymbol{h} \| \le \homoq^\prime (\|\param\|)$;

\item [(A3).]
$|f(\param;\xB)| \le \homoq(\|\param\|)$.  
\end{assumpenum}
\end{definition}

We make a few remarks on \Cref{def:nearhomo}. 
First, our near-homogeneity condition is modified from \Cref{eq:homo-2} instead of \Cref{eq:homogeneous}. In this way, our near-homogeneity condition implicitly puts regularity conditions on the subgradient of the network, which will be useful in our analysis.  

Second, every $M$-homogeneous (see \Cref{eq:homogeneous}) network for $M\ge 1$ is also near-$M$-homogeneous. 
We see this by setting $\homop(x)=0$ and $\homoq(x) = C (x^M+1)$ for a sufficiently large constant $C>1$ in \Cref{def:nearhomo}. 

Third, a near-$M$-homogeneous network is also near-$(M+1)$-homogeneous according to \Cref{def:nearhomo}. Thus when we say a network is near-$M$-homogeneous, the degree $M$ should be interpreted as the minimum degree $M$ such that \Cref{def:nearhomo} is satisfied. 
We will see that this interpretation is necessary when we introduce the strong separability condition in \Cref{sec:margin-direct}.

Finally, we point out that many commonly used deep networks are near-homogeneous but not homogeneous. Examples include networks using residual connections or non-homogeneous activation functions.
This will be further elaborated in \Cref{sec:near-homo-nn}.


\paragraph{O-minimal Structure and Definable Functions.}
The o-minimal structure and definable functions 
were introduced to deep learning theory by \citet{davis2020stochastic,ji2020directional} for studying the convergence and implicit bias of subgradient methods. 
We briefly review these notions below.

An o-minimal structure is a collection $\Scal = (\Scal_n)_{n=1}^\infty$, where each $\Scal_n$ is a collection of subsets of $\Rbb^n$ that satisfies the following properties. First, $\Scal_n$ contains all algebraic sets in $\Rbb^n$, i.e., zero sets of real-coefficient polynomials on $\Rbb^n$. Second, $\Scal_n$ is closed under finite union, finite intersection, complement, Cartesian product, and projection. Third, $\Scal_1$ consists of finite unions of open intervals and points in $\Rbb^1$. A set $A \subset \Rbb^n$ is {\it definable} if $A \in \Scal_n$. A function $f: D \to \Rbb^m$ with $D \subset \Rbb^n$ is {\it definable} if its graph is in $\Scal_{n+m}$. See \Cref{sec:o-minimal} for a more detailed introduction.

By the definition of o-minimal structure, the set of definable functions is closed under algebraic operations, composition, inversion, and taking maxima or minima. 
The work by \citet{wilkie1996model} established the existence of an o-minimal structure in which polynomials and the exponential function are definable. 
By the closure property, commonly used mappings in deep learning are all definable with respect to this o-minimal structure, including fully connected layers, convolutional layers, ReLU activation, max-pooling layers, residual connections, and cross-entropy loss. We refer the readers to \citet{ji2020directional} for more discussion.


Throughout this paper, we assume that 
there exists an o-minimal structure such that $t\mapsto\exp(t)$ and the network $\param\mapsto f(\param; \xB_i)$, $i=1,\dots,n$, are all definable. 
This assumption, together with the local Lipschitzness, allows us to apply the chain rule (see \Cref{lem:chain-rule-clark} in \Cref{sec:o-minimal}) to GF defined by subgradients. 
Moreover, this assumption allows us to leverage the desingularizing function (see \Cref{def: designularizing function} in \Cref{sec:proof:direct}, and \citet{ji2020directional}) to show the directional convergence. Specifically, we apply two Kurdyka–Łojasiewicz inequalities (See \Cref{lem:KL1,lem:KL2}) under the o-minimal structure to establish the existence of the desingularizing function.



\section{Implicit Bias of Gradient Flow} \label{sec:margin-direct}
In this section, we establish the implicit bias of gradient flow for near-homogeneous networks. 
Our first assumption is that the network is near-$M$-homogeneous (see \Cref{def:nearhomo}). 
\begin{assumption}[Near-homogeneous network]
\label{asp:nearhomo}
Let $M\ge 1$.
We assume that the network $f$ is near-$M$-homogeneous with polynomials $\homop(\cdot)$ and $\homoq(\cdot)$, as described in \Cref{def:nearhomo}.
\end{assumption}


Under \Cref{asp:nearhomo}, let
$\homop(x) := \sum_{i=0}^M a_i x^i$. The following function is handy for our presentation: 

\begin{equation}\label{eq:def-pa}
\homop_a (x) := \sum_{i=1}^{M-1} \frac{(i+1)a_{i+1}}{M-i} x^{i} + \frac{a_1}{M-1/2}.
\end{equation}
One intuition behind the choice of $\homop_a$ is that, for any $f$ satisfying \Cref{asp:nearhomo}, $g\coloneqq f-p_a$ satisfies a one-sided inequality of the homogeneous condition \eqref{eq:homo-2}. This is demonstrated in \Cref{sec:proof: sanity check gf}.

\paragraph{Normalized and Modified Margins.}
Under \Cref{asp:nearhomo}, we define the \emph{normalized margin} \citep{lyu2020gradient} as 
\begin{equation}
\label{eq: normalized margin}
    \Normalmargin(\param) \coloneqq  \min_{i \in [n]}\frac{ y_i f(\param;\xB_i)}{\paramNorm^M}.
\end{equation}
The normalized margin is hard to analyze directly, due to the hard minimum in its definition. Instead, we analyze a \emph{modified margin}, which increases monotonically and approximates the normalized margin well. Specifically, the \emph{modified margin} for gradient flow is defined as
\begin{equation}
\label{eq: modified margin}
        \GFmargin(\param) \coloneqq \frac{\phi\big(\Loss(\param)\big) - \homop_a(\paramNorm)}{\paramNorm^M},
\end{equation}
where $\phi(x) := \log\big(1/(nx)\big)$ and $\homop_a$ is given by \Cref{eq:def-pa}. 

In the definition of the modified margin, the $\phi(\Loss(\param))$ term produces a soft minimum of $(y_i f(\param; \xB_i))_{i=1}^n$, which approximates the hard margin $\min_{i\in [n]} y_i f(\param;\xB_i)$. This idea appears in prior analysis for homogeneous networks \citep{lyu2020gradient,ji2020directional}. 
The offset term $\homop_a(\paramNorm)$ is our innovation, which controls the homogeneous error when the network is non-homogeneous. 
Note that as $\|\param\|$ grows, the offset term $\homop_a(\paramNorm)$ grows slower than the main term $\phi(\Loss(\param))$ according to \Cref{def:nearhomo}. Therefore our modified margin is a good approximation of the normalized margin. 

Our second assumption ensures GF can reach a state in which the network strongly separates the data.
\begin{assumption}[Strong separability condition]
\label{asp:initial-cond-gf}
Let $f(\param;\xB)$ be a network satisfying \Cref{asp:nearhomo}.
Assume that there exists a time $s>0$ such that $\param_s$ given by \Cref{eq:GF} satisfies 
\begin{equation}
\label{eq: initial-cond-gf}
    \Loss(\param_s) < 
        e^{-\homop_a(\|\param_s\|)}/n,
\end{equation}
where $\homop_a$ is defined in \Cref{eq:def-pa} and $n$ is the number of samples.
\end{assumption}

Our \Cref{asp:initial-cond-gf} is a natural extension of the separability condition (that is, \(\Loss(\param_s) < 1/n\) for some $s$) used in the analysis of homogeneous networks \citep{lyu2020gradient,ji2020directional}. Specifically, for homogeneous networks, we can set $\homop=0$, in which $\homop_a=0$ by its definition in \Cref{eq:def-pa}. Then \Cref{asp:initial-cond-gf} reduces to the separability condition.

For non-homogeneous networks, \Cref{asp:initial-cond-gf} requires GF to attain an empirical risk that is sufficiently small compared to a function of the homogeneous error. Note that this condition can be satisfied by any non-degenerate near-homogeneous network that is able to separate the training data, which will be discussed further in \Cref{sec:example}.

In detail, \Cref{asp:nearhomo,asp:initial-cond-gf}, and the choice of $\homop_a$ together force the ``leading term'' in the network to be exactly $M$-homogeneous. Therefore, the near-homogeneity order in \Cref{def:nearhomo} must be understood as the minimum possible one. 
This is precisely explained in the following lemma, whose proof is deferred to \Cref{sec:proof: sanity check gf}.

\begin{lemma}
[Near-homogeneity order]
\label{lem: Sanity check for gradient flow}
Let $f$ be such that
\[
    f(\param; \xB) = \sum_{i=0}^\infty f^{(i)}(\param;\xB),
\]
where $f^{(i)}(\param;\xB)$ is $i$-homogeneous with respect to $\param$. 
If $f$ satisfies \Cref{asp:nearhomo,asp:initial-cond-gf}, 
then for every $j\in[n]$, we must have 
\begin{align*}
    f^{(i)}(\cdot;\xB_j) \begin{cases}
        = 0, & \text{if } i>M, \\
        \ne 0, & \text{if } i=M.
    \end{cases}
\end{align*}
Furthermore, we have $f^{(M)}(\param_s; \xB_j) >0$ for all $j\in [n]$. 
\end{lemma}

\paragraph{Margin Improvement.}
We are ready to present our first main theorem on the margin improvement of GF.
The proof is deferred to \Cref{sec:proof:margin}. 
\begin{theorem}
[Risk convergence and margin improvement]
\label{thm: Margin improving and convergence}
Suppose that \Cref{asp:nearhomo,asp:initial-cond-gf} hold. 
For $(\param_t)_{t>s}$ given by  \Cref{eq:GF}, we have:
\begin{itemize}[leftmargin=*]
\item 
For all $t>s$, the risk and the parameter norm satisfy
\begin{equation*}
     \Loss(\param_t) < e^{-\homop_a(\paramNormt)}/n.
\end{equation*}
Furthermore, 
we have
\begin{align*}
\Loss(\param_t) \eqsim  \frac{1}{t(\log t)^{2-2/M}},\quad 
\paramNormt \eqsim  (\log t)^{\frac{1}{M}} ,
\end{align*}
where $\eqsim$ hides constants that depend on $M$, $\GFmargin(\param_s)$, and coefficients of $\homoq$.
\item 
The modified margin $\GFmargin(\param_t)$ is positive, increasing, and upper bounded. Moreover,  
$\GFmargin(\param_t)$ is an $\epsilon_t$-multiplicative approximation of $\Normalmargin(\param_t)$, that is, 
\[
\GFmargin(\param_t) \le \Normalmargin(\param_t) \le \big(1+\epsilon_t\big) \cdot \GFmargin(\param_t),\quad  \text{for all}\ t>s, 
\]
where 
\[\epsilon_t :=  \frac{\log n + \homop_a(\paramNormt)}{\LinkFun(\Loss(\param_t))-\homop_a(\paramNormt)}  =\Ocal \big((\log t)^{-1/M}\big)\to 0.\] 
\end{itemize}
\end{theorem}

This result generalizes Theorem 4.1 in \citet{lyu2020gradient} from homogeneous networks to near-$M$-homogeneous networks for $M\ge 1$.
In particular, we recover their results when the network is $M$-homogeneous, in which we set $\homop_a(x) = 0$. When the network is non-homogeneous, our \Cref{asp:initial-cond-gf} is stronger than the separability condition in \citet{lyu2020gradient}. This is one of the key conditions that enables our analysis for non-homogeneous networks. In fact, \Cref{asp:initial-cond-gf} is necessary for generic near-homogeneous model to exhibit implicit bias. We illustrate this point via the following example, whose proof is deferred to \Cref{sec:proof:counter-eg}.

\begin{example}[Necessity of \Cref{asp:initial-cond-gf}]
\label{eg:counter-eg init cond}
Assume we only have one sample: $(x,y)=(1,1) \in \mathbb{R}^2$ and our model is $f(\param)=\param^M+M|\param|^{M-1}$ for $\param \in \mathbb{R}$ and some odd integer $M\ge 3$. Then, $f$ satisfies \Cref{asp:nearhomo} with $\homop(\param)=|\param|^M$. If $f$ does not satisfy \Cref{asp:initial-cond-gf} at $t=s$, we have
$$
\param_t\le 0, \quad y\cdot \param_t x\le 0 \quad \text{and} \quad \Loss(\param_t)\ge 1,
$$
for all $t\ge s$ and $\param_t$ following \cref{eq:GF}.
     
\end{example}


\paragraph{Directional Convergence.}
Our next theorem establishes the directional convergence of GF for non-homogeneous networks. The proof is deferred to \Cref{sec:proof:direct}.

\begin{theorem}[Directional convergence]
\label{thm: directional convergence}
Under the setting of \Cref{thm: Margin improving and convergence}, let $\tilde{\param}_t \coloneqq \param_t / \|\param_t\|$ be the direction of  \Cref{eq:GF}. Then the curve swept by $\tilde{\param}_t$ has finite length. Therefore, 
the directional limit $\param_* \coloneqq \lim_{t\to \infty} \tilde{\param}_t$ exists.
\end{theorem}

Our \Cref{thm: directional convergence} extends Theorem 3.1 in \citet{ji2020directional} from $M$-homogeneous networks to near-$M$-homogeneous networks for $M\ge 1$. 
Our \Cref{asp:nearhomo,asp:initial-cond-gf} allow the application of 
tools (specifically the desingularizing functions) from \citet{ji2020directional} for showing the direction convergence in the non-homogeneous cases. 


\paragraph{KKT Convergence.}
Provided with the directional convergence of GF,  
our next step is to characterize the limiting direction.
To this end, we need to understand the asymptotic behavior of the near-homogeneous network $f$ as $\|\param\| \to \infty$. 
Since $f$ is near-$M$-homogeneous, one can expect that $f$ is close to an $M$-homogeneous function for large $\param$. Motivated by this, we define the \emph{homogenization} of $f$ as 

\begin{equation}
\label{eq: homogenized f}
\homoPredictor (\param; \xB) \coloneqq \lim_{r \to +\infty} \frac{f(r \param; \xB)}{r^M}.
\end{equation}
The well-definedness, continuity, and differentiability of $\homoPredictor$ are guaranteed by the near-homogeneity of $f$ (see \Cref{thm:homogenization} in \Cref{sec:example}). 
We will show that the limiting direction of GF satisfies the KKT conditions of the following margin maximization problem:
\begin{equation}
\label{eq: KKT}
    \mbox{minimize} \, \| \param\|^2, \quad \text{s.t.} \, \min_{i \in [n]} y_i \homoPredictor(\param;\xB_i) \ge 1. \tag{P}
\end{equation}
This generalizes the margin maximization problem in \citep{lyu2020gradient} from homogeneous to non-homogeneous cases.
It is worth noting that when $f$ is already homogeneous, \Cref{eq: KKT} reduces to the same optimization problem as that in \citep{lyu2020gradient}.

To ensure the limiting direction satisfies the KKT conditions of \Cref{eq: KKT}, we need to compare the gradients of $f$ and $\homoPredictor$. This comparison requires an additional regularity assumption.

\begin{assumption}[Weak-homogeneous gradient]
\label{asp:strongerhomo}
Assume that the network $f(\param;\xB)$ is continuously differentiable with respect to $\param$ for $\xB \in (\xB_i)_{i=1}^n$ and that $\lim_{r \to \infty} {\nabla f (r \param; \xB)}/{r^{M-1}}$ exists for all $\param$ and $\xB \in (\xB_i)_{i=1}^n$. Assume that the limit
\begin{equation*}
    (\nabla f)_{\homo} (\param; \xB) := \lim_{r \to \infty} \frac{\nabla f (r \param; \xB)}{r^{M-1}}
\end{equation*}
satisfies 
\begin{equation*}
    \left\vert \nabla f(\param; \xB) - (\nabla f)_\homo (\param; \xB) \right\vert \le \homor \left( \| \param \| \right),
\end{equation*}
where $\homor: \Rbb_{\ge 0} \to \Rbb_{\ge 0}$ is a function such that
\begin{equation*}
    \lim_{x \to + \infty} \frac{\homor(x)}{x^{M - 1}} = 0.
\end{equation*}
\end{assumption}



It is worth noting that for $M \ge 2$, \Cref{asp:nearhomo,asp:strongerhomo} are satisfied as long as $\nabla_{\param} f(\param;\xB)$ is component-wise near-$(M-1)$-homogeneous. See \Cref{lem:near-homo gradients lead to near-homo functions} for a precise statement. In comparison, our \Cref{asp:strongerhomo} is less restrictive and allows for $M=1$.
The next theorem shows that the limiting direction of GF satisfies the KKT conditions of \eqref{eq: KKT}, with its proof included in \Cref{sec:proof:KKT}. 

\begin{theorem}[KKT convergence]
\label{thm: KKT convergence} 
Under the same setting of \Cref{thm: directional convergence}, and additionally assume \Cref{asp:strongerhomo} hold.
Then the rescaled limiting direction 
$$
\param_* / \big(\min_{i \in [n]} y_i f(\param_* ;\xB_i)\big)^{1/M}
$$ 
satisfies the KKT conditions of \Cref{eq: KKT}, where $\param_*$ is the directional limit in \Cref{thm: directional convergence}.
\end{theorem}

This theorem is a generalization of 
Theorem 4.4 in \citet{lyu2020gradient} from homogeneous networks to non-homogeneous networks. 
Note that our margin maximization problem \Cref{eq: KKT} is defined for the homogenization of the non-homogeneous network. 
This is because, asymptotically, every near-homogeneous network $f$ can be approximated by its homogenization $f_{\homo}$ (see \Cref{thm:homogenization} in \Cref{sec:example}). As an important implication, understanding the implicit bias of a near homogeneous network can be reduced to understanding that of its homogenization.
It is worth noting that \citet{ji2020directional} also established the asymptotic alignment between parameters and gradients, which we leave as future work.





We have presented our results on the implicit bias for non-homogeneous networks. In the next two sections, we discuss the satisfiability of \Cref{asp:nearhomo,asp:initial-cond-gf}, respectively.




\section{Near-Homogeneity Condition} \label{sec:near-homo-nn}

In this section, we verify that a large class of building blocks used in deep learning are near-homogeneous, and, by a composition rule, networks constructed using these blocks are also near-homogeneous. 

We denote a block by $s(\param;\xB)$, where $\param$ are the trainable parameters in this block and $\xB$ is the input (and the output of the preceding block). 
We use $s_{\param}(\xB)$ as a shorthand for $s(\param;\xB)$. 
Then, a network is defined as
\begin{equation}\label{eq:network}
    f(\param;\xB) := s^1_{\param_1} \circ s^2_{\param_2} \circ \cdots \circ s^L_{\param_L}(\xB),
\end{equation}
where $\param = (\param_i)_{i=1}^L$ and $s^i$ is the $i$-th block. Here and in sequel, we assume all the blocks are locally Lipschitz and definable with respect to some o-minimal structure.

To deal with the compositional structure of networks, we need to introduce the following generalized definition of near-homogeneity, which takes both trainable parameters and input into account.

\begin{definition}
[Near-$(M, N)$-homogeneity]
\label{def:dual-homo}
Let $M, N$ be two non-negative integers such that $M + N \ge 1$.
A function $s(\param; \xB)$ is called \emph{near-$(M, N)$-homogeneous},
if there exist $\homop_s, \homoq_s,\homor_s, \homot_s \in \Rbb_{\ge 0} [x]$ with $\deg \homop_s, \deg\homoq_s \le M$ and $ \deg \homor_s, \deg \homot_s \le N $ such that the following holds
for every $\xB\in(\xB_i)_{i=1}^n$, $\thetaB$, and $(\hB_{\param}, \hB_{\xB}) \in  \partial s(\param ;\xB)$:
\begin{assumpenum}
    \item [(B1).] 
  \(
     | \la\hB_{\param}, \param \ra  - M s(\param;\xB) | \le \homop_s^\prime (\|\param\|) \homor_s(\|\xB\|) \),  
     
     \(  
     | \la \hB_{\xB}, \xB \ra - N  s(\param;\xB) | \le \homop_s (\|\param\|) \homor_s^\prime(\|\xB\|);
  \)
  

  \item [(B2).]
    \(\| \hB_{\param} \| \le \homoq_s^\prime (\|\param\|) \homot_s(\|\xB\|)\), 
    \(\| \hB_{\xB} \| \le \homoq_s (\|\param\|) \homot_s^\prime(\|\xB\|);\)
 \item [(B3).] 
    \(
        \|s(\param;\xB)\| \le \homoq_s (\|\param\|) \homot_s(\|\xB\|).
    \) 
\end{assumpenum}
A block $s(\param;\xB) : \Rbb^{d_1} \times \Rbb^{d_2} \to \Rbb^{d_3} $ is called near-$(M, N)$-homogeneous if all of its components,
$ s(\param;\xB)_{i}$ for $i \in [d_3]$, 
are near-$(M,N)$-homogeneous with the same polynomials $\homop_s, \homoq_s,\homor_s, \homot_s$.
\end{definition}

From \Cref{def:dual-homo}, it is easily seen that a near-$(0, N)$-homogeneous block $s(\param; \xB)$ must be independent of $\param$ and near-$N$-homogeneous in $\xB$, and vice versa.
Below are a few examples of near-homogeneous blocks used in practice.
\begin{example}
\label{eg:Examples of dual homogeneous blocks}
The following blocks are near-homogeneous:
\begin{assumpenum}
    \item For $\param = (A, b)$, the linear mapping $s(\param;\xB) = A\xB + b$ is near-$(1,1)$-homogeneous.
    \item Let $M\ge 1$ be an integer. Then for $\param = (A, b)$, the perceptron layer $s(\param; \xB) = \phi^M(A\xB + b)$ is near-$(M, M)$-homogeneous, where the activation function $\phi$ is one of the following: ReLU, Softplus, GELU, Swish, SiLU, and Leaky ReLU.
    \item Max pooling layer, average pooling layer, convolution layer, and residual connection are near-$(0,1)$-homogeneous. 
    \item The SwiGLU activation \citep{shazeer2020glu} is near-$(2,2)$-homogeneous.
    \item The linear self-attention \citep{zhang2024trained} is near-$(2,3)$-homogeneous and the ReLU attention \citep{wortsman2023replacing} is near-$(4,3)$-homogeneous.
\end{assumpenum}
\end{example}

The following lemma suggests that near-homogeneity is preserved under functional composition and tensor product. Note the residual connection mentioned in \Cref{eg:Examples of dual homogeneous blocks} is not a specific block mapping, but a way to enhance an existing block by adding the input from the previous block to the output of this block. Part C of \Cref{prop: Composition of block mappings} can thus be applied to compute the near-homogeneous order of networks with residual connections.
The proof of  \Cref{prop: Composition of block mappings} is deferred to \Cref{sec:proof:42}. 
\begin{lemma}
[Composition and multiplication rules]
\label{prop: Composition of block mappings}
Let the blocks $s^1(\param_1;\xB): \Rbb^{d_1} \times \Rbb^{d_2} \to \Rbb^{d_3}$ and $s^2(\param_2;\xB): \Rbb^{d_4} \times \Rbb^{d_5} \to \Rbb^{d_6}$ be near-$(M_1, M_2)$-homogeneous and near-$(M_3, M_4)$-homogeneous, respectively.
Then, the followings hold:
\begin{assumpenum}
\item Let $d_2 = d_6$, then $s^1_{\param_1} \circ s^2_{\param_2}(\xB)$ is near-$(M_1 + M_2M_3, M_2M_4)$-homogeneous. 
\item The tensor product $ s^1(\param_1;\xB) \otimes s^2(\param_2;\xB)$ is 
near-$(M_1 + M_3, M_2 + M_4)$-homogeneous. 
\item Let $(d_1, d_2, d_3) = (d_4, d_5, d_6)$, then $ s^1(\param_1;\xB) + s^2(\param_2;\xB)$ is
near-$(\max(M_1, M_3), \max(M_2, M_4))$-homogeneous.
\item Let $T: \Rbb^{d_3} \to \Rbb^{d_6}$ be a linear mapping, then $T(s^1(\param_1;\xB))$ is 
near-$(M_1, M_2)$-homogeneous. 
\item Let $f: \Rbb^{d_4} \times \Rbb^{d_5} \to \Rbb^{d_2}$ be near-$M_5$-homogeneous as a function of $\param_3$, then
$s^1_{\param_1} \circ f(\param_2;\xB)$ is near-$(M_1+ M_2 M_5)$-homogeneous. 
\end{assumpenum}
\end{lemma}
The following corollary establishes the near-homogeneity of a network composed of near-homogeneous blocks. We defer its proof to \Cref{sec:proof-43}. 

\begin{corollary}
[Near-homogeneous networks]
\label{cor: Near homogeneity order of networks}
Consider a network defined as \Cref{eq:network}.
If the block $s^i(\param_i;\xB) $ is near-$(M_1^i, M_2^i)$-homogeneous for $i\in [L]$, then the network $f$ is near-$(M_1, M_2)$-homogeneous with
\begin{equation}
\label{eq:composition_network_order}
M_1 := \sum_{j=1}^{L} M_1^j \cdot \prod_{i=1}^{j-1} M_2^i, \quad M_2 := \prod_{j=1}^{L} M_2^j.
\end{equation}
In particular, $f$ is near-$M_1$-homogeneous as a function of $\param$ for any fixed $\xB$.
\end{corollary}

Based on \Cref{eg:Examples of dual homogeneous blocks}, \Cref{prop: Composition of block mappings} and \Cref{cor: Near homogeneity order of networks}, we can show that a broad class of commonly used networks are near-homogeneous, including the following examples:


\begin{example}
\label{eg: Near homogeneous networks}
The following networks are near-homogeneous:
\begin{assumpenum}
    \item An $L$-layer MLP with $k$-th power of ReLU activation is near-$(k^L-1)/(k-1)$-homogeneous, or near-$L$-homogeneous when $k=1$. The same holds when ReLU is replaced by other activation functions in \Cref{eg:Examples of dual homogeneous blocks} (B).
    \item VGG-$L$ \citep{simonyan2015very} is near-$L$-homogeneous where $L\in\{11, 13, 16, 19\}$. 
    \item Without batch normalization, ResNet-$L$ \citep{he2016deep} is near-$L$-homogeneous where $L \in \{18,34,50,101,152\}$, and DenseNet-$L$ \citep{huang2017densely} is near-$L$-homogeneous where $L \in \{121, 169, 201, 264\}$.
\end{assumpenum}
\end{example}

To conclude this section, we comment that  normalization layers and softmax map violate our near-homogeneity definitions. Intuitively, these blocks should be ``near-$0$-homogeneous'' as their outputs are bounded. However, our near-homogeneity definitions are only non-trivial for $M\ge 1$. 
As a consequence, the softmax attention architecture also violates our definitions of near-homogeneity, thus requiring a different treatment.
We believe that our notion of near-homogeneity can be generalized to include those components, which we leave as future work.


\section{Strong Separability Condition} \label{sec:example}

In this section, we discuss the satisfiability of \Cref{asp:initial-cond-gf}. 
We first give an intuitive explanation of why \Cref{asp:initial-cond-gf} should be expected to hold.
Note that \Cref{asp:initial-cond-gf} is equivalent to 
\[
\Loss(\param_s) < e^{-\homop_a(\|\param_s\|)}/n\ \Leftrightarrow \ 
\frac{\log\big(1/(n\Loss(\param_s))\big)}{\homop_a(\|\param_s\|)} > 1.
\]
Recall that $\deg \homop_a \le M-1$ by \Cref{asp:nearhomo}. Thus \Cref{asp:initial-cond-gf} can be understood as requiring GF to induce a ``lower-order'' smoothed margin which is at least $1$.
If the normalized margin \Cref{eq: normalized margin} is asymptotically positive (as $\|\param_t\|$ grows), then the ``lower-order'' smoothed margin 
must diverge, hence \Cref{asp:initial-cond-gf} must hold at some point. 



In what follows, we establish a sufficient condition under which a non-homogeneous network satisfies \Cref{asp:initial-cond-gf}.
To this end, we establish below several important properties of the \emph{homogenization} of a non-homogeneous network.

\paragraph{Homogenization.} 
The idea of homogenization appears in \Cref{sec:margin-direct} for constructing the margin maximization problem, where its KKT conditions characterize the limiting direction of GF. Our next theorem rigorously controls the approximation error between a non-homogeneous network and its homogenization.

\begin{proposition}[Homogenization]
\label{thm:homogenization}
Suppose that $f$ satisfies \Cref{asp:nearhomo}.
Then for every $\xB\in(\xB_i)_{i=1}^n$, the homogenization of $f(\param; \xB)$:
\begin{equation*}
    \homoPredictor (\param; \xB) := \lim_{r \to +\infty} \frac{f(r \param; \xB)}{r^M}
\end{equation*}
exists and is well-defined.
Moreover, as a function of $\param$, $\homoPredictor (\param; \xB)$ is continuous, differentiable almost everywhere, and $M$-homogeneous. 
We also have
\[
\text{for every } \param,\quad 
    \left\vert f(\param; \xB) - \homoPredictor (\param; \xB) \right\vert \le \homop_a( \left\| \param \right\| ),
\]
where $\homop_a$ is given by \Cref{eq:def-pa}.

If, in addition, $f(\param;\xB)$ satisfies \Cref{asp:strongerhomo},
then $\homoPredictor(\param;\xB)$ is continuously differentiable for every nonzero $\param$, and that $(\grad f)_\homo (\param; \xB) = \grad \homoPredictor (\param; \xB)$. 
\end{proposition}
\Cref{thm:homogenization} ensures the well-definedness, continuity, and differentiability of $\homoPredictor$.
Consequently, the following theorem guarantees that the strong separability condition can be satisfied by a near-homogeneous network, as long as its homogenization satisfies the (weak) separability condition of \cite{lyu2020gradient}.

\begin{proposition}
[A sufficient condition]
\label{cor:Initial condition via homogeneization}
Suppose that $f$ satisfies \Cref{asp:nearhomo}. Then $f$ admits a homogenization, denoted by $\homoPredictor$.
Assume that $\homoPredictor$ satisfies the weak separability condition, that is, 
\[
\sum_{i} \ell\big(- y_i \homoPredictor(\param^\prime ;\xB_i) \big) < 1\   \text{for some}\ \ \param^\prime.
\]
Then, there exists a constant $c>0$ such that $f$ with $\param_s := c \param^\prime $ satisfies  \Cref{asp:initial-cond-gf}.
\end{proposition}




To further elaborate the idea of homogenization, we provide the following compositional rule. 
\begin{proposition}[Homogenization of networks]\label{thm:block_compo_homogenization}
Consider a network given by \Cref{eq:network}, where each block $s^i (\param_i; \xB)$ is near-$(M_1^i, M_2^i)$-homogeneous (see \Cref{def:dual-homo}) for $i=1,\dots,L$.
Then each block $s^i (\param_i; \xB)$ admits a well-defined 
homogenization 
\begin{equation*}
s_{\homo}^i (\param_i; \xB) := \, \lim_{r_1, r_2 \to \infty} \frac{s^i (r_1 \param_i; r_2 \xB)}{r_1^{M_1^i} r_2^{M_2^i}}.
\end{equation*}
Moreover, the homogenization network \Cref{eq:network} is well-defined and satisfies 
\begin{equation*}
\homoPredictor (\param; \xB) = s_{\homo, \param_1}^1 \circ s_{\homo, \param_2}^2 \circ \cdots \circ s_{\homo, \param_L}^L (\xB).
\end{equation*}
\end{proposition}

\paragraph{A Two-Layer Network.}
We next provide a two-layer network example where GF with zero initialization can provably reach a point that satisfies \Cref{asp:initial-cond-gf}. The two-layer network is defined as 
\begin{equation}
\label{eq: toy-resnet}
    f(\param;\xB) := \wB_1^\top \xB + \wB_2^\top \xB + a_1 \varphi( \wB_1^\top \xB) - a_2 \varphi( - \wB_2^\top \xB),
\end{equation}
where $\param := (\wB_1, \wB_2, a_1, a_2)$ are the trainable parameters, and  $\varphi$ is the leaky ReLU activation function, 
\[
\varphi(x) := \max\{x, \alpha_L x\},\quad 0< \alpha_L<1.
\]
Our example \Cref{eq: toy-resnet} is motived by the two-layer network considered in \citet{lyu2021gradient}. However, their network is homogeneous, while ours is non-homogeneous (but near-$2$-homogeneous) due to the residual connections.
Similar to \citet{lyu2021gradient}, we consider a symmetric and linearly separable dataset.
\begin{assumption}[Dataset conditions] 
\label{asp: symmetric}
Assume that the dataset $(\xB_i, y_i)_{i=1}^n$ satisfies $\max_{i}\|\xB_i\|\le 1$ and $\min_{i} y_i\xB_i^\top \wB_* \ge \gamma$ for some unit vector $\wB_*$ and margin $\gamma>0$; moreover, 
$n$ is even and $(y_{i+n/2},\xB_{i+n/2}) = -(y_i,\xB_i) $ for $i=1,\ldots ,n/2$. 
\end{assumption}

Note that the network \Cref{eq: toy-resnet} is not Lipschitz continuous, thus \Cref{eq:GF} is not well-defined. Instead, we study the small stepsize limit of gradient descent for \Cref{eq: toy-resnet}, which is well-defined under our assumptions.
In the next theorem, we establish that the small-stepsize limit of gradient descent with zero initialization for \Cref{eq: toy-resnet} can satisfy \Cref{asp:initial-cond-gf}.
\begin{theorem}
[A two-layer network example]
\label{thm: Near-two homogeneous example achieves init bound}
Consider the network $f$ given by \eqref{eq: toy-resnet} and a dataset satisfying \Cref{asp: symmetric}.
Let 
$(\thetaB_t)_{t\ge 0}$ be the limit of GD iterates with initialization $\thetaB_0=\boldsymbol 0$ as the stepsize tends to zero. 
Then there exists $s>0$ such that $f$ with $\thetaB_s$ satisfies \Cref{asp:initial-cond-gf}.
\end{theorem}

\section{Implicit Bias of Gradient Descent} \label{sec:margin-direct-gd}
In this section, we extend our previous results for GF to gradient descent (GD), which is defined as 
\begin{equation}
\label{eq:GD}
    \tag{GD}
    \param_{t+1} := \param_t - \eta \nabla \Loss (\param_t),
\end{equation}
where $\eta > 0$ is a fixed stepsize.
Similarly, we will assume the network $f$ is near-$M$-homogeneous with $\homop$ and $\homoq$ (see \Cref{asp:nearhomo}).
For GD, the definition of $\homop_a$ needs to be slightly modified. Let $\homop(x) :=\sum_{i=0}^M a_i x^i$. For $M\ge 2$, we 
define
\begin{equation}
\label{eq:def-pa-gd}
\homop_a (x) := 
\begin{dcases}
\sum_{i=1}^{M-1} \frac{(i+1)a_{i+1}}{M-i} x^{i} + \frac{a_1}{M-1/2},& \text{if } x \ge 1, \\
\sum_{i=2}^{M-1} \frac{(i+1)a_{i+1}}{M-i} x^{i} \\
+ \frac{2 a_2}{M - 1} \frac{x^2 + 1}{2} + \frac{a_1}{M-1/2},& \text{if } 0 \le x < 1.
\end{dcases}
\end{equation}
For $M=1$, we define $\homop_a(x) \coloneqq a_1/(M-1/2)$. 
Compared to \Cref{eq:def-pa}, \Cref{eq:def-pa-gd} replaces the linear term $x$ with a quadratic polynomial $(x^2 + 1)/2$. This is purely for circumventing a technical issue when $\| \param \|$ is small (see the proof of \Cref{lem:gradient_hessian_bound} in \Cref{sec:proof:gd:prelim}). 

To handle the discretization error introduced by GD, we need a stronger version of the strong separability condition. 
\begin{assumption}
[Strong separability condition for GD]
\label{asp:initial-cond-gd}
Let $f(\param;\xB)$ be a network satisfying \Cref{asp:nearhomo}.
Assume that $f$ is twicely differentiable for $\param$ and 
that for some constant $A > 0$, we have
\begin{equation*}
    \left\| \nabla_{\param}^2 f(\param; \xB) \right\| \le
    \begin{dcases}
        A \left( \| \param \|^{M - 2} + 1 \right), & \text{if $M \ge 2$} \\
        A, & \text{if $M = 1$} \\
    \end{dcases} 
\end{equation*}
for every $\xB \in (\xB_i)_{i=1}^n$. 
For $\homop_a$ given by \cref{eq:def-pa-gd},
assume that $\deg \homop_a \ge 1$ if $M \ge 2$, and that there exists a time $s>0$ such that
\begin{equation*}
    \Loss(\param_s) < \min \left\{ \frac{1}{n e^2}, \ \frac{1}{B \eta} \right\} e^{-\homop_a(\|\param_s\|)},
\end{equation*}
where $B$ is a constant depending only on $(M, \homop, \homoq)$ and $A$. 
\end{assumption}
Note that in \Cref{asp:initial-cond-gd}, the stepsize $\eta$ can be arbitrarily large as long as the empirical risk $\Loss(\param_s)$ is of the order of $\Ocal(1/\eta)$.
We note that \Cref{cor:Initial condition via homogeneization} can also be adapted to guarantee the satisfiability of \Cref{asp:initial-cond-gd}.

We consider the following modified margin for GD:
\begin{equation}\label{eq:modified_margin_GD}
    \GDmargin(\param) \coloneqq \frac{\exp \left( \Phi \left( e^{\homop_a (\| \param \|)} \Loss(\param) \right) \right)}{\|\param\|^M},
\end{equation}
where $\Phi(x) := \log (\log \frac{1}{nx}) + \frac{2}{\log (nx)}$. 

In the following two theorems, we extend our results for GF to GD. The proofs are included in \Cref{sec:margin_improve_gd}.
\begin{theorem}
    [Risk convergence and margin improvement for GD]
    \label{thm: Margin improving and convergence-gd} 
    Suppose that \Cref{asp:nearhomo,asp:initial-cond-gd} hold. For $(\param_t)_{t\ge 0}$ given by \Cref{eq:GD} with any stepsize $\eta>0$, we have:
    \begin{itemize}[leftmargin=*]
        \item For all $t > s$, the risk and parameter norm satisfy
        \begin{equation*}
            \Loss (\param_t) < \min \left\{ \frac{1}{ne^2}, \frac{1}{B \eta} \right\} \cdot  e^{-\homop_a(\|\param_t \|)}.
        \end{equation*}
        Furthermore, as $t \to \infty$, we have
        \begin{align*}
            \Loss (\param_t) \eqsim \frac{1}{\eta t(\log \eta t)^{2-2/M}}, \quad
            \| \param_t \| \eqsim (\log \eta t)^{\frac{1}{M}},
        \end{align*}
        where $\eqsim$ hides constants that depend on $M$, $\GDmargin(\param_s)$, and coefficients of $\homoq$, but not $\eta$.
        
        \item The modified margin $\GDmargin(\param_t)$ in \Cref{eq:modified_margin_GD} is increasing and bounded. 
        Moreover, $\GDmargin(\param_t)$ is an $\epsilon_t$-multiplicative approximation of $\gamma(\param_t)$, that is, for all $t > s$,
            \begin{equation*}
                \GDmargin\big(\param_t\big) \le \gamma\big(\param_t\big) \le \big(1+\epsilon_t\big) \GDmargin\big(\param_t\big), 
            \end{equation*}
            where $\epsilon_t \to 0$ as $t \to \infty$. 
    \end{itemize}
\end{theorem}



\begin{theorem}[Directional and KKT convergence for GD]
\label{thm: directional convergence-gd}
Under \Cref{asp:nearhomo,asp:initial-cond-gd}, the same results in \Cref{thm: directional convergence} hold for \Cref{eq:GD} with any stepsize $\eta>0$. 
Under \Cref{asp:strongerhomo,asp:initial-cond-gd}, the same results in \Cref{thm: KKT convergence} hold for \Cref{eq:GD} with any stepsize $\eta>0$.
\end{theorem}


\section{Conclusion} \label{sec:conclusion}
We show the implicit bias of gradient descent (GD) for training generic non-homogeneous deep networks under exponential loss.
We show that the normalized margin induced by GD increases nearly monotonically.
Moreover, GD converges in direction, with the limiting direction satisfying the KKT conditions of a margin maximization problem.
Our results rely on a near-homogeneity condition and a strong separability condition, both of which are natural generalizations of the conditions used in prior implicit bias analysis for homogeneous networks. 
In particular, our results apply to networks with residual connections and non-homogeneous activation functions.



\section*{Acknowledgements}
We thank Fabian Pedregosa, Nathan Srebro, and Matus Telgarsky for their helpful comments and discussion.
We gratefully acknowledge the support of the NSF for FODSI through grant DMS-2023505, of the Founder's Postdoctoral Fellowship in Statistics at Columbia University, of the Sloan Fellowship by the U.S. Department of Energy, Office of Science, Office of Advanced Scientific Computing Research’s Applied Mathematics Competitive Portfolios program under Contract No. AC02-05CH11231, of the NSF and the Simons Foundation for the Collaboration on the Theoretical Foundations of Deep Learning through awards DMS-2031883 and \#814639, DMS-2210827, CCF-2315725, and of the ONR through MURI award N000142112431 and N00014-24-S-B001.

%% file: sections/appendix.tex
\section{O-minimal Structure} \label{sec:o-minimal}

In this section, we give a brief introduction to the o-minimal structure. We mainly follow the notation and definitions in \citet{ji2020directional}.

\begin{definition}
[O-minimal structure]
\label{def:O-minimal structure}
An O-minimal structure is a collection $\Scal = (\Scal_n)_{n=1}^\infty$ where $\Scal_n$ is a subset of $\Rbb^n$ that satisfies the following properties:
\begin{itemize}
    \item [1.] $ \Scal_1 $ is the collection of all finite unions of open intervals and points. 
    \item [2.] $ \Scal_n $ includes the zero sets of all polynomials (algebraic sets) on $\Rbb^n$: if $\homop \in \Rbb[x_1,\ldots ,x_n] $, then $\{ x \in \Rbb^n | \homop(x) = 0 \} \in \Scal_n$. 
    \item [3.] $ \Scal_n $ is closed under finite union, finite intersection,  and complement.  
    \item [4.] $ \Scal $ is closed under Cartesian product: if $A \in \Scal_m$ and $B \in \Scal_n$, then $A \times B \in \Scal_{m+n}$.
    \item [5.] $ \Scal $ is closed under projection  $ \prod_n $ onto the first $ n $ coordinates: if $ A \in \Scal_{n+1} $, then $ \prod_n(A) \in \Scal_n $. 
\end{itemize}
\end{definition}

Then we can define the definable functions.

\begin{definition}
[Definable sets and functions]
\label{def:Definable sets and functions}
Given an o-minimal structure $\Scal$, a set $A \subset \Rbb^n$ is \emph{definable} if $A \in \Scal_n$. A function $f: D \to \Rbb^m$ with $D \subset \Rbb^n$ is \emph{definable} if its graph is in $\Scal_{n+m}$.
\end{definition}

The following properties of definable sets and functions can then be derived (check \citet{coste2000introduction,loi2010lecture,van1996geometric}).

\begin{proposition}
[Property of definable functions]
\label{prop:Property of definable functions} We fix an arbitrary o-minimal structure $\Scal$. Then the following properties hold:
\begin{itemize}
    \item [1.] Given any $\alpha , \beta \in \Rbb$ and any definable functions $f,g : D \to \Rbb$, we have $\alpha f + \beta g $ and $fg$ are definable. If $g \neq 0$ on $D$, then $f/g$ is definable. If $f \ge 0$ on $D$, then $f^{\frac{1}{\ell}}$ is definable for any positive integer $\ell$. 
    \item [2.] Given a function $f:D \to \Rbb^m$, let $f_i$ denote the $i$-th coordinate of its output. Then $f$ is definable if and only if all $f_i$ are definable. 
    \item [3.] Any composition of definable functions is definable. 
    \item [4.] Any coordinate permutation of a definable set is definable. Consequently, if the  inverse of a definable function exists, it's also definable. 
    \item [4.] The image and preimage of a defianble set by a definable function is definable. Particulalry, given any real-valued definable function $f$, all of $f^{-1}(0), f^{-1} ( -\infty, 0 )$ and $ f^{-1} ( (0,\infty) )  $ are definable. 
    \item [5.] Any combination of finitely many definable functions with disjoint domains is definable. For example, the pointwise maximum and minimum of definable functions are definable. 
\end{itemize}
\end{proposition}

Once we have these properties, we can observe that almost all the network structures are definable.

\begin{lemma}
[Definable network structures, Lemma B.2 in \citet{ji2020directional}]
\label{lem:Definable network structures}
Suppose there exist $k, d_0, d_1, \ldots ,d_k >0$ and $L$ definable functions $ (g_1,\ldots ,g_L) $ where $g_j: \Rbb^{d_0} \times \cdots \times \Rbb^{d_j-1} \times \Rbb^{k} \to \Rbb^{d_j}$. Let $h_1(x;\wB) \coloneqq g_1 (x, \wB)$ and for $2\le j \le L$, 
\[
h_j(x, \wB) \coloneqq g_j( x, h_1(x,\wB) ,\ldots ,h_{j-1}(x,\wB), \wB),
\]
then all $h_j$ are definable. It suffices if each output coordinate of $g_j$ is the minimum or maximum over some finite set of polynomials, which allows for linear, convolutional, ReLU, max-pooling layers, and skip connections.
\end{lemma}

One important corollary is that this allows us to composite and multiply the near-homogeous blocks.

\begin{corollary}
[Composition and multiplication of definable blocks]
\label{cor:Composition and multiplication of definable blocks}
Given two definable block mappings: $s_1(\param_1; \xB): \Rbb^{d_1} \times \Rbb^{d_2} \to \Rbb^{d_3}$ and $s_2(\param_2; \xB): \Rbb^{d_3} \times \Rbb^{d_4} \to \Rbb^{d_5}$, then, 
\begin{itemize}
    \item If $d_2= d_5$, the composition $s(\param; \xB) = s_1(\param_1; s_2(\param_2;\xB)) $ is definable.
    \item $s(\param; \xB) = s_1(\param_1; \xB) \otimes s_2(\param_2; \xB)$ is definable.
    \item If $d_1 = d_3$, $s(\param_1; \xB) = s_1(\param_1; \xB) \otimes s_2(\param_1; \xB)$ is definable.
\end{itemize}
\end{corollary}
\begin{proof}
    This is a direct consequence of Lemma \ref{lem:Definable network structures} and Proposition \ref{prop:Property of definable functions}.
\end{proof}

Definable functions have good properties, especially there is some stratification which makes them piecewise differentiable. However, the gradient, even the Clarke subdifferential, of a definable function, is not well-defined everywhere. To overcome this issue, we need the local Lipschitz continuity.

\begin{definition}
[Local Lipschitz continuity]
\label{def:Local Lipschitz continuity}
Given a function $f: D \to \Rbb$ with an open domain $D$, we say $f$ is locally Lipschitz continuous if for any $x \in D$, there exists a neighborhood $U$ of $x$ such that $f$ is Lipschitz continuous on $U$.
\end{definition}

Similar to the definability, the local Lipschitz continuity is preserved under composition and multiplication.
It's worth noting that the Clarke subdifferential of a locally Lipschitz function is well-defined everywhere.

\begin{lemma}
[Clarke subdifferential of locally Lipschitz functions, Corollary 6.1.2 in \citet{borwein2000convex}]
\label{lem:Clarke subdifferential of locally Lipschitz functions}
Given a locally Lipschitz function $f: D \to \Rbb$ with an open domain $D$, the Clarke subdifferential $\partial f(x)$ is nonempty, compact, and convex for all $x \in D$.
\end{lemma}

Now we will assume all the models and block mappings are definable and locally Lipschitz. Since block mappings are multivalue functions, we can use the Clarke general Jacobian for them. Similarly, for locally Lipschitz functions to $\Rbb^n$, the general Jacobian is well-defined for all points.

\begin{corollary}
[Clarke Jacobian of locally Lipschitz functions, Proposition 2.6.2 in \citet{clarke1990optimization}]
\label{cor:Clarke Jacobian of locally Lipschitz functions}
Given a locally Lipschitz function $f: D \to \Rbb^n$ with an open domain $D \subseteq \Rbb^m$, the Clarke Jacobian $\partial f(x)$ is nonempty, compact, and convex for all $x \in D$.
\end{corollary}

Following the notation in \citet{ji2020directional}, we let $\bar \partial f(x)$ denote the unique minimum-norm subgradient: 
\begin{equation}
\label{eq:min-norm-sub-grad}
 \MinGrad f(x) \coloneqq \arg\min_{x^* \in \partial f(x)} \|x^*\|.   
\end{equation}
Another important property is that for definable functions, it admits a chain rule almost everywhere. 
\begin{lemma}
[Chain rule, Lemma B.9 in \citet{ji2020directional}]
\label{lem:chain-rule-clark}
Given a locally Lipschitz definable $f: D \to \Rbb$ with an open domain $D$, for any interval $I$ and any arc $z: I \to D$, it holds that for almost every $t \in I$ that 
\[
    \frac{\mathrm{d} f(z_t)}{ \mathrm{d}  t} = \bigg\langle z_t^*, \frac{\mathrm{d} z_t}{\mathrm{d} t} \bigg\rangle, \quad \text{ for all } z_t^* \in \partial f(z_t).  
\]
Moreover, for \eqref{eq:GF}, we have for almost every $t\ge 0$, 
\begin{align*}
&\frac{\mathrm{d} \param_t}{\mathrm{d}  t} = - \bar \partial \Loss_t, \quad \frac{\mathrm{d}  \Loss_t}{\mathrm{d} t} = - \| \bar \partial \Loss_t \|^2
\\
&\frac{\mathrm{d} \rho_t^2}{\mathrm{d}  t} = -2 \langle \param_t, \bar \partial \Loss_t  \rangle, \quad \frac{\mathrm{d} \GFmargin (\param_t)}{\mathrm{d} t} =    -2 \langle \bar \partial \GFmargin(\param_t), \bar \partial \Loss_t  \rangle.
\end{align*}
\end{lemma}
Clarke has the following lemma to characterize the subgradient of the composition of functions.

\begin{theorem}
[Clark chain rule, Theorems 2.3.9 and 2.3.0 in \citet{clarke1990optimization}]
\label{thm:Clark chain rule}
Let $z_1,\ldots ,z_n: \Rbb^d \to \Rbb$ and $f: \Rbb^n \to \Rbb$ be locally Lipschitz functions. Let $(f \circ z) (x) = f(z_1(x),\ldots ,z_n(x))$ be the composition of $f$ and $z$. Then, 
\[
\partial (f \circ z) (x) 
\subseteq 
\text{conv} \bigg\{\sum_{i=1}^n \alpha _i \hB_i : \alpha \in \partial f(z_1(x),\ldots, z_n(x)), \hB_i \in \partial z_i(x) \bigg\}.
\]
\end{theorem}

Recall that: 
\[
    \Loss(\param) = \frac{1}{n} \sum_{i=1}^n e^{- y_i f(\param; \xB_i)}. 
\]
We have the following corollary which characterizes the min-norm subgradient of the exponential loss.

\begin{corollary}
[Subgradient of the exponential loss]
\label{cor:subgrad-exp-loss}
For the exponential loss $\Loss(\param)$, there exists $\hB_i \in \partial f(\param; \xB_i)$ such that
\[
\bar \partial \Loss(\param) = \frac{1}{n} \sum_{i=1}^n -  e^{- y_i f(\param; \xB_i)} y_i\hB_i.
\]
\end{corollary}
\begin{proof}[Proof of \Cref{cor:subgrad-exp-loss}]
    This is a direct consequence of Theorem \ref{thm:Clark chain rule} and the definition of the exponential loss.
\end{proof}

Similarly, we can characterize the subgradient of the composition of block functions. 
\begin{corollary}
[Subgradient of the composition of block functions]
\label{cor:subgrad-comp-block}
Given two locally Lipschitz block functions 
$s_1(\param_1; \xB): \Rbb^{d_1} \times \Rbb^{d_2} \to \Rbb^{d_3}$ and $s_2(\param_2; \xB): \Rbb^{d_3} \times \Rbb^{d_4} \to \Rbb^{d_2}$, 
we have 
\begin{align*}
\partial_{\param} s_1(\param_1, s_2(\param_2;\xB)) 
&\subset \text{conv} \big\{ (\alphaB_1, \alphaB_2 \cdot\hB_1) : (\alphaB_1, \alphaB_2) \in \partial_{\param_1, \xB} s_1 (\param_1;\xB), \hB_1 \in \partial_{\param_2} s_2(\param_2;\xB)  \big\},\\   
\partial_{\xB} s_1(\param_1, s_2(\param_2;\xB)) 
&\subset \text{conv} \big\{ \alphaB_2 \cdot\hB_2 :  \alphaB_2 \in \partial_{\xB} s_1 (\param_1;\xB), \hB_2 \in \partial_{\xB} s_2(\param_2;\xB)  \big\}.  
\end{align*}
\end{corollary}

\begin{proof}[Proof of \Cref{cor:subgrad-comp-block}]
    It follows from Theorem \ref{thm:Clark chain rule} and the definition of the block functions.
\end{proof}
Here we use $\partial_{\param_1, \xB} s_1 (\param_1;\xB)$ to denote the corresponding Jacobian. Note that $\alphaB_1 \in \Rbb^{d_3 \times d_1}, \alphaB_2 \in \Rbb^{d_3 \times d_2}, \hB_1 \in \Rbb^{d_2 \times d_3}, \hB_2 \in \Rbb^{d_2 \times d_4}$. Once we have this, to verify some properties of Jacobian of $s_1\circ s_2$, we can focus on the Jacobian of $s_1$ and $s_2$. 

The o-minimal structure eliminates many bad geometry, allowing us to focus on the good part of the optimization landscape. 



\section{Proof Overview}\label{sec:proof_overview}
In the main text, we present several results on the implicit bias of GF/GD for near-homogeneous networks. Below we briefly describe the approaches we use to prove these results, with actual proofs deferred to the subsequent appendices.

\paragraph{Margin Improvement and Convergence Rates.}
We first sketch the proof of \Cref{thm: Margin improving and convergence} (GF), and then highlight the major technical innovations in the proof of \Cref{thm: Margin improving and convergence-gd} (GD), as compared to the case of GF. The key ingredient for proving \Cref{thm: Margin improving and convergence} is the following lemma, which establishes the monotonicity of the modified margin under the strong separability condition.

\begin{lemma}[Restatement of \Cref{thm:gamma-a-increase}]\label{lem:restate_gamma_gf_increase}
    Denote $\Loss_t = \Loss (\param_t)$ and $\rho_t = \| \param_t \|$. Under \Cref{asp:nearhomo,asp:initial-cond-gf}, we have  
    $\Loss_t < e^{-\homop_a(\rho_t)}/n$ for all $t\ge s$, and
    \begin{equation}\label{eq:restate_gamma_gf_lowerbd}
        \frac{\mathrm{d} \log \GFmargin \big(\param_t\big)}{\mathrm{d} t} > \frac{\| \bar \partial \Loss_t \|^2 \rho_t^2 - \langle  \bar \partial \Loss_t , \param_t \rangle^2  }{\rho_t^2\Loss_t \big (\LinkFun(\Loss_t) - \homop_a(\rho_t)\big)} \ge 0.
    \end{equation}
\end{lemma}

The proof of \Cref{lem:restate_gamma_gf_increase} is similar to the proof of \citet[Lemma 5.1]{lyu2020gradient}. Since $\GFmargin \big(\param_t\big)$ only depends on $\Loss_t$ and $\rho_t$, its growth rate can be attributed to two quantities: $\mathrm{d} \Loss_t / \mathrm{d} t$ and $\mathrm{d} \rho_t / \mathrm{d} t$. We use the same argument as that in \citet[Proof sketch of Lemma 5.1]{lyu2020gradient} to estimate each of these two quantities, which finally leads to the lower bound \eqref{eq:restate_gamma_gf_lowerbd} on $\mathrm{d} \GFmargin ( \param_t ) / \mathrm{d} t$. Further since
\begin{equation*}
    \LinkFun(\Loss_t) - \homop_a(\rho_t) > 0 \Longleftrightarrow \Loss_t < \frac{1}{n} e^{- \homop_a (\rho_t)},
\end{equation*}
we know that the strong separability condition is necessary for the lower bound \eqref{eq:restate_gamma_gf_lowerbd} at $t = s$, and can be established for $t \ge s$ using continuous induction. It is noteworthy that while we follow a similar proof scheme as \cite{lyu2020gradient}, the design of the modified margin $\GFmargin$ is completely novel and highly non-trivial.

From this lower bound, the ``margin improvement" part of \Cref{thm: Margin improving and convergence} directly follows. For the ``convergence rates" part, we use the monotonicity of modified margin to upper and lower bound $- \mathrm{d} \Loss_t / \mathrm{d} t$, thus establishing convergence rates for $\Loss_t$. The claim for $\rho_t$ can be proved similarly.

For the proof of \Cref{thm: Margin improving and convergence-gd}, we need to establish a lower bound on $\log \GDmargin ( \param_{t+1} ) - \log \GDmargin ( \param_{t} )$ similar to \Cref{eq:restate_gamma_gf_lowerbd}. Due to the discrete nature of GD, the Hessian of $\log \GDmargin$ should also be taken into account when dealing with $\log \GDmargin ( \param_{t} )$ as a function of $t$. To address this challenge, we analyze a modified loss $\ModifiedLoss (\param) = \exp(\homop_a(\| \param \|)) \Loss(\param)$ that is closely related to $\log \GDmargin$. We establish tight upper bounds on the Hessian of $\ModifiedLoss (\param_t)$ in terms of $\Loss_t$ and $\rho_t$, thus leading to a tight lower bound on $\log \GDmargin ( \param_{t+1} ) - \log \GDmargin ( \param_{t} )$. Notably, our lower bound allows for arbitrarily large step size $\eta$, as long as $\ModifiedLoss (\param_s) = O(1 / \eta)$, which is guaranteed by \Cref{asp:initial-cond-gd}. In contrast, \citet{lyu2020gradient} assumes that the step size is upper bounded by a function of the initial margin, preventing it to be large. The proof of other parts in  \Cref{thm: Margin improving and convergence-gd} is completely analogous to the case of GF.

\paragraph{Convergence to the KKT Direction.}
The proofs of \Cref{thm: directional convergence,thm: directional convergence-gd} largely rely on the techniques developed in \citet{ji2020directional}. As before, we will first sketch the proof for GF, and then explain how to adapt it to establish KKT convergence for GD. For GF, define the arc length swept by the direction of $\param_t$:
\begin{equation*}
	\zeta_t = \, \int_{s}^{t} \left\| \frac{\mathrm{d} \tilde{\param}_u}{\mathrm{d} u} \right\| \mathrm{d} u.
\end{equation*}
Similar to \citet{ji2020directional}, we construct a desingularizing function $\Psi$ that controls the growth rate of $\zeta_t$ using that of $\Psi (\gamma_* - \GFmargin (\param_t))$, where $\gamma_* \coloneqq \lim_{t \to \infty} \GFmargin (\param_t)$:
\begin{lemma}[Restatement of \Cref{lem: Existence of desingularizing function}]
\label{lem: restate of Existence of desingularizing function}
There exist $R>0, \nu>0$ and a definable desingularizing function $\Psi$ on $[0, \nu)$, such that for a.e. large enough $t$ with $\left\|\param_t\right\|>R$ and $\GFmargin(\param_t)>\gamma_*-\nu$, it holds that
\begin{equation}\label{eq:desingularizing_inequality}
	\frac{\mathrm{d} \zeta_t}{\mathrm{~d} t} \leq-c \frac{\mathrm{d} \Psi\left(\gamma_*-\GFmargin(\param_t)\right)}{\mathrm{d} t}
\end{equation}
for some constant $c>0$.
\end{lemma}
The proof of the above lemma is similar to that of \citet[Lemma 3.1]{ji2020directional}. Integrating both sides of \Cref{eq:desingularizing_inequality}, we deduce that $\lim_{t \to \infty} \zeta_t < \infty$. Therefore, $\tilde{\param}_t$ must converge to some $\param_*$. This establishes directional convergence of GF path.

To go further and show that $\param_*$ satisfies the KKT conditions \Cref{eq: KKT}, the key ingredient is to show the asymptotic alignment between $\param_t$ and $\hB_M(\param_t)$ for subsequence of $t$, i.e., 
\[
\lim_{t_m\to \infty}\beta(t_m) = \lim_{t_m\to \infty} \frac{\langle\param_{t_m}, \hB_M(\param_{t_m}) \rangle}{ \|\param_{t_m}\| \cdot \|\hB_M(\param_{t_m})\|} = 1,
\]
where we define
\begin{equation*}
    \hB_M(\param_t) \coloneqq  \frac{1}{n}\sum_{i=1}^n e^{-y_i f(\param_t;\xB_i)} y_i\nabla \homoPredictor(\param_t;\xB_i)
\end{equation*}
as a proxy of $\grad f (\param_t; \xB)$.
We establish the following bound by looking further into the first inequality in \cref{eq:restate_gamma_gf_lowerbd}. 
\begin{lemma}[Restatement of \Cref{cor: beta bound}] For any $t_2>t_1$ large enough, there exists $t_*\in (t_1,t_2)$ such that 
\[
    \frac{1 - p_1(t_*)}{\big(\beta(t_*) + p_2(t_*) \big)^2}-1 \le \frac{1}{M} \cdot \frac{\log \GFmargin(\param_{t_2}) - \log\GFmargin(\param_{t_1})}{\log \|\param_{t_2}\| - \log \|\param_{t_1}\|}, 
\]
where $p_1(t), p_2(t) \to 0$ as $t\to \infty$. 
\end{lemma}
As $\GFmargin (\param_t)$ converges and $\|\param_t\|$ diverges, the right-hand side of the above inequality converges to $0$. Hence, there must exist a subsequence $\{ \beta(t_m) \}$ converging to $1$. This shows that $\param_t$ and $\hB_M(\param_t)$ asymptotically aligns along this subsequence, and consequently verifies that $\param_*$ is a KKT point.

To establish directional convergence for GD, we need to show that the discrete arc length swept by $\tilde{\param}_t$ is finite, i.e., 
\begin{equation}\label{eq:gd_sketch_finite_arc}
    \sum_{t=s}^{\infty} \left\| \tilde{\param}_{t+1} - \tilde{\param}_t \right\| < \infty.
\end{equation}
Similar to the proof of GF, we construct a desingularizing function $\Psi$, and show that there exists a constant $c > 0$, such that for all large enough $t$,
\begin{equation}\label{eq:gd_sketch_desingular}
     \left\| \tilde{\param}_{t+1} - \tilde{\param}_t \right\| \leq c \left( \Psi \left(\gamma_*-\GDmargin (\param_t) \right) - \Psi \left(\gamma_*-\GDmargin (\param_{t+1}) \right) \right),
\end{equation}
where $\GDmargin$ is the modified margin for GD and $\gamma_* = \lim_{t \to \infty} \GDmargin (\param_{t})$. To this end, we construct an arc by connecting all GD iterates using line segments, and carefully estimate the spherical and the radial parts of this arc, addressing several new technical challenges that arise in the analysis of GD.

Finally, summing up both sides of \Cref{eq:gd_sketch_desingular} immediately leads to \Cref{eq:gd_sketch_finite_arc}, thus proving directional convergence of $\param_t$ along the GD path. The proof of KKT convergence is completely similar to the case of GF.

\section{Proofs for Section \ref{sec:margin-direct}}  \label{sec:proof:margin-direct}
In this section, we will provide the proofs of the results for \Cref{eq:GF} in Section \ref{sec:margin-direct}. Recall that we have the following notation: 
\begin{align*}
    \LinkFun(x) \coloneqq  \log \frac{1}{nx}, \quad \nabla  \coloneqq \nabla_{\param} , \quad   
    \rho_t \coloneqq  \paramNormt,\quad  \Loss_t \coloneqq \Loss (\param_t), \quad \\ 
    f_i(\param) \coloneqq f(\param; \xB_i), \quad \bar f_i(\param) \coloneqq y_i f(\param, \xB_i), \quad \bar f_{\min}(\param) \coloneqq \min_{i \in [n]} \bar f_i(\param).
\end{align*}
For a function of the time, $h(t)$, we use $h'(t)$ as a shorthand of $\frac{\dif }{\dif t}h(t)$.

\subsection{Margin Improvement}  \label{sec:proof:margin}

    
 

In this section, we always assume $f(\param; \xB)$ satisfies \Cref{asp:nearhomo} with $\homop_a$ defined in \Cref{eq:def-pa}. 
We can directly verify the following property of $\homop$ and $\homop_a$.
\begin{lemma}
[Property of $\homop_a$]
\label{lem:property-pa}
For $\homop$ and $\homop_a$ defined in \Cref{eq:def-pa}, we have
\[
 \text{for every $x$},\quad    \homop_a^\prime (x) x+ \homop^{\prime}(x) \le  M  \homop_a(x). 
\]
\end{lemma}
\begin{proof}[Proof of \Cref{lem:property-pa}] 

Recall that $\homop(x) = \sum_{i=0}^{M}a_i x^i$ and 
\begin{align*}
 \homop_a (x) := \sum_{i=1}^{M-1} \frac{(i+1)a_{i+1}}{M-i} x^{i} + \frac{a_1}{M-1/2}. 
\end{align*}
Hence, 
\begin{align*}
\homop_a^\prime (x) + \homop^\prime (x) - M \homop_a(x)  
&= \, \sum_{i=1}^{M-1}  \big[\frac{i(i+1)a_{i+1}}{M-i} x^{i} + (i+1)a_{i+1} x^i - \frac{M a_{i+1}}{M-i}x^i  \big] + a_1 -  \frac{Ma_1}{M-1/2} \\
&= -\frac{a_1/2}{M-1/2} \le 0. 
\end{align*}
This completes the proof of \Cref{lem:property-pa}.
\end{proof}

The following bounds from \citet{lyu2020gradient} connect $\bar  f_{\min}(\param_t)$ and $\Loss_t$. 
\begin{lemma}
    [property of $\bar f_{\min}(\param)$]
    \label{lem:f-min}
For the Loss function $\Loss$ defined in \Cref{eq:loss}, we have 
    \[
 \text{for every $\param$,}\quad        \log \frac{1}{n \Loss(\param)}\le  \bar f_{\min}(\param) \le \log \frac{1}{\Loss(\param)}.
    \]
\end{lemma}
\begin{proof}[Proof of \Cref{lem:f-min}]
Recall that we consider the exponential loss
\[
    \Loss(\param) = \frac{1}{n} \sum_{i=1}^n e^{-\bar f_i(\param)}.
\]
Therefore, we can get that
\[
    \frac{1}{n} \Loss(\param) \le \frac{1}{n} e^{-\bar f_{\min}(\param)} \le \Loss(\param) \implies \log \frac{1}{n \Loss(\param)}\le  \bar f_{\min}(\param) \le \log \frac{1}{\Loss(\param)}.
\]
This completes the proof of \Cref{lem:f-min}. 
\end{proof}

We have the following fundamental lemma from \citet{lyu2020gradient} about the dynamics of the GF. 
\begin{lemma}
[Decrease of the risk]
\label{lem:GF-risk-decrease}
For \Cref{eq:GF}, we have for a.e. $t\ge s$,
\[
    - \frac{\mathrm{d} \Loss_t}{\mathrm{d} t}  = \bigg(\frac{1}{2} \frac{\mathrm{d} \rho_t^2}{\mathrm{d}  t} \bigg) \cdot \frac{\mathrm{d}  \log \rho_t}{\mathrm{d}  t} + \frac{\| \bar \partial \Loss_t \|^2 \rho_t^2 - \langle  \bar \partial \Loss_t , \param_t \rangle^2 }{\rho_t^2},\quad t\ge 0. 
\]
\end{lemma}
\begin{proof}[Proof of \Cref{lem:GF-risk-decrease}]
By \cref{lem:chain-rule-clark}, we have for a.e. $t\ge s$, 
\begin{align*}
    -  \frac{\mathrm{d} \Loss_t}{\mathrm{d} t} 
    & = \|  \bar \partial \Loss_t\|^2 &&\explain{ by \cref{lem:chain-rule-clark} } \\ 
    & = \frac{\langle  \bar \partial \Loss_t, \param_t \rangle ^2 }{\rho_t^2} + \frac{\|  \bar \partial \Loss_t\|^2 \rho_t^2 - \langle  \bar \partial \Loss_t, \param_t \rangle^2 }{\rho_t^2} &&\explain{ Projecting $\param$ to spherical and radical parts}  \\ 
    & = \frac{1}{\rho_t^2} \cdot \bigg(\frac{1}{2} \frac{\mathrm{d}\rho_t^2}{\mathrm{d} t} \bigg)^2 + \frac{\|  \bar \partial \Loss_t\|^2 \rho_t^2 - \langle  \bar \partial \Loss_t, \param_t \rangle^2 }{\rho_t^2} 
&&\explain{ Since $\frac{1}{2} \frac{\mathrm{d}\rho_t^2}{\mathrm{d}t} = -\langle  \bar \partial \Loss_t, \param_t \rangle$ } \\ 
    & = \bigg(\frac{1}{2} \frac{\mathrm{d}\rho_t^2}{\mathrm{d} t} \bigg) \cdot \frac{\mathrm{d}  \log \rho_t}{\mathrm{d}  t} + \frac{\|  \bar \partial \Loss_t\|^2 \rho_t^2 - \langle \bar \partial \Loss_t, \param_t \rangle^2 }{\rho_t^2}. &&\explain{ Since $ \frac{\mathrm{d} \log \rho_t}{\mathrm{d} t} = \frac{1}{\rho_t}\rho^{\prime}_t= \frac{1}{2\rho_t^2} \frac{\mathrm{d}\rho_t^2}{\mathrm{d}t}$ } 
\end{align*}
We have completed the proof of \Cref{lem:GF-risk-decrease}. 
\end{proof}

The next lemma shows the parameter norm is increasing under the strong separability condition.
\begin{lemma}
[Increase of the parameter norm]
\label{lem:NH-rho-increase}
Under \Cref{asp:nearhomo}, for almost every $t\ge s$, 
we have 
\[
\Loss_t < e^{-\homop_a(\rho_t)}/n \quad 
\Rightarrow\quad 
    \frac{1}{2} \frac{\mathrm{d}\rho_t^2}{\mathrm{d}t} \ge  \big( M  \LinkFun(\Loss_t) - \homop^\prime (\rho_t)\big) \Loss_t > 0.
\]
\end{lemma}
\begin{proof}[Proof of \Cref{lem:NH-rho-increase}]
By \Cref{lem:f-min}, we have
\begin{equation*}
        \bar f_i(\param_t) \ge \bar f_{\min}(\param_t) \ge \log \frac{1}{n \Loss_t} = \LinkFun(\Loss_t). 
\end{equation*}
Then we have for almost every $t\ge s$,
\begin{align*}
\frac{1}{2} \frac{\mathrm{d}\rho_t^2}{\mathrm{d}t}
& = \langle -\bar \partial \Loss_t, \param_t \rangle  &&\explain{ by \Cref{lem:chain-rule-clark} } \\ 
& = \frac{1}{n}\sum_{i=1}^n e^{-\bar f_i(\param_t)} y_i  \langle \hB_i, \param_t \rangle 
&&\explain{ By \Cref{cor:subgrad-exp-loss}} \\
& \ge M \cdot \frac{1}{n}\sum_{i=1}^n e^{-\bar f_i(\param_t)} \bar f_i(\param_t) - \Loss_t \cdot \homop^{\prime}(\rho_t) &&\explain{ By  \Cref{asp:nearhomo}} \\
& \ge M \cdot\frac{1}{n}\sum_{i=1}^n e^{-\bar f_i(\param_t)} \LinkFun(\Loss_t) - \Loss_t \cdot \homop^{\prime}(\rho_t) &&\explain{ Since $\bar f_i(\param_t) \ge\LinkFun(\Loss_t)$} \\ 
& =   M \Loss_t \LinkFun(\Loss_t) - \homop^{\prime}(\rho_t)\Loss_t. 
\end{align*}
Note that
\begin{equation*}
    \Loss_t < e^{-\homop_a(\rho_t)}/n \quad 
    \Leftrightarrow \quad 
    M \LinkFun(\Loss_t)  >   M \homop_a \big(\rho_t \big).
\end{equation*}
Additionally,  \Cref{lem:property-pa} implies $M \homop_a(x) \ge \homop_a^\prime (x) x+ \homop^\prime(x)\ge \homop^\prime(x)$. 
As a consequence of these two inequalities, we have
\[
    \frac{1}{2} \frac{\mathrm{d}\rho_t^2}{\mathrm{d}t}  \ge  M\LinkFun(\Loss_t) \Loss_t - \homop^{\prime}(\rho_t)\Loss_t  >   M \homop_a (\rho_t) \Loss_t -\homop^{\prime}(\rho_t)\Loss_t \ge 0.
\]
 We have completed the proof of \Cref{lem:NH-rho-increase}. 
\end{proof}

\begin{theorem}
[Monotonicity of the modified margin]
\label{thm:gamma-a-increase}
Under \Cref{asp:nearhomo,asp:initial-cond-gf}, we have  
$\Loss_t < e^{-\homop_a(\rho_t)}/n$ for $t\ge s$. Moreover, we have
\begin{equation}
\label{eq: gamma-c-increase}
\frac{\mathrm{d} \log \GFmargin \big(\param_t\big)}{\mathrm{d} t} > \frac{\| \bar \partial \Loss_t \|^2 \rho_t^2 - \langle  \bar \partial \Loss_t , \param_t \rangle^2  }{\rho_t^2\Loss_t \big (\LinkFun(\Loss_t) - \homop_a(\rho_t)\big)} \ge 0,\quad \text{ for almost every } t\ge s.
\end{equation}
As a consequence, the modified margin $\GFmargin(\param_t)$ is increasing for $t\ge s$. 
\end{theorem}
\begin{proof}[Proof of \Cref{thm:gamma-a-increase}]
We first show that $\Loss_t < e^{-\homop_a(\rho_t)}/n$ implies $\GFmargin(\param_t)$ is positive and increasing. Then we show $\Loss_t < e^{-\homop_a(\rho_t)}/n$ for all $t\ge s$ by contradiction. 

\paragraph{Step 1: the condition that $\Loss_t < e^{-\homop_a(\rho_t)}/n$ implies that $\GFmargin(\param_t)$ is positive and  increasing.}
Recall that
\begin{equation*}
    \Loss_t < e^{-\homop_a(\rho_t)}/n \quad 
    \Leftrightarrow \quad 
    M \LinkFun(\Loss_t)  >   M \homop_a \big(\rho_t \big).
\end{equation*}
Under this condition, the modified margin $\GFmargin(\param_t)$ defined in \Cref{eq: modified margin} is positive. 
Moreover, we have for almost every $t \ge s$, 
\begin{align*}
    \frac{\mathrm{d} \log \GFmargin(\param_t)}{\mathrm{d} t} 
    &= \frac{\mathrm{d} \log\big( \LinkFun(\Loss_t) - \homop_a(\rho_t)\big)}{\mathrm{d} t} - M \frac{\mathrm{d} \log \rho_t}{\mathrm{d} t}\\ 
    &= \frac{-\frac{\mathrm{d} \Loss_t}{\mathrm{d} t}  - \Loss_t \homop_a^\prime (\rho_t)  \rho_t}{\Loss_t \big(\LinkFun(\Loss_t) - \homop_a(\rho_t)\big)} - M \frac{\mathrm{d} \log \rho_t}{\mathrm{d} t}.
\end{align*}
By \Cref{lem:NH-rho-increase} we have $\rho_t$ is increasing, which implies $\mathrm{d}  \log \rho_t / \mathrm{d} t >0$ almost everywhere.  
Then we have for almost every $t\ge s$,
\begin{align*}
        -  \frac{\mathrm{d} \Loss_t}{\mathrm{d} t}  
        &= \bigg(\frac{1}{2} \frac{\mathrm{d} \rho_t^2}{\mathrm{d}  t} \bigg) \cdot \frac{\mathrm{d}  \log \rho_t}{\mathrm{d}  t} + \frac{\| \bar \partial \Loss_t \|^2 \rho_t^2 - \langle  \bar \partial \Loss_t , \param_t \rangle^2 }{\rho_t^2} && \explain{By \Cref{lem:GF-risk-decrease}}\\
        &>\Loss_t \bigg(   \LinkFun(\Loss_t) - \frac{\homop^{\prime}(\rho_t)}{M}  \bigg) \cdot M\frac{\mathrm{d} \log \rho_t}{\mathrm{d} t}  + \frac{\|\bar \partial \Loss_t\|^2 \rho_t^2 - \langle \bar \partial \Loss_t, \param_t \rangle^2 }{\rho_t^2}. && \explain{By \Cref{lem:NH-rho-increase}}
\end{align*}
Furthermore, 
\begin{equation*}
        \Loss_t \homop_a^\prime (\rho_t)  \rho_t^{\prime} = \Loss_t \frac{\homop_a^{\prime}(\rho_t) \rho_t}{M} \cdot M\frac{\mathrm{d} \log \rho_t}{\mathrm{d} t}. \qquad\explain{ By $\frac{\mathrm{d} \log \rho_t}{\mathrm{d} t} = \frac{\rho_t^{\prime}}{\rho_t}$ } 
\end{equation*}
Applying the above two bounds in the derivative of $\log \GFmargin(\param_t)$, we get for almost every $t\ge s$,
\begin{align*}
   &\lefteqn{ \frac{\mathrm{d} \log \GFmargin(\param_t)}{\mathrm{d} t} = \frac{-\frac{\mathrm{d} \Loss_t}{\mathrm{d} t} - \Loss_t \homop_a^\prime (\rho_t)  \rho_t^{\prime}}{\Loss_t \big(\LinkFun(\Loss_t) - \homop_a(\rho_t)\big)} - M \frac{\mathrm{d} \log \rho_t}{\mathrm{d} t} }\\
     & > \frac{\Loss_t \Big(\LinkFun(\Loss_t) - \frac{p^\prime _a(\rho_t)\rho_t}{M} - \frac{\homop^{\prime}(\rho_t)}{M} \Big)}{\Loss_t \big(\LinkFun(\Loss_t) - \homop_a(\rho_t)\big)}\cdot M \frac{\mathrm{d} \log \rho_t}{\mathrm{d} t}  + \frac{\|\bar \partial \Loss_t\|^2 \rho_t^2 - \langle \bar \partial \Loss_t, \param_t \rangle^2 }{\rho_t^2\Loss_t \big(\LinkFun(\Loss_t) - \homop_a(\rho_t)\big)} - M \frac{\mathrm{d} \log \rho_t}{\mathrm{d} t} \\ 
    & \ge \frac{\Loss_t (\LinkFun(\Loss_t) - \homop_a(\rho_t) )}{\Loss_t \big(\LinkFun(\Loss_t) - \homop_a(\rho_t)\big)}\cdot M \frac{\mathrm{d} \log \rho_t}{\mathrm{d} t} - M \frac{\mathrm{d} \log \rho_t}{\mathrm{d} t}  + \frac{\|\bar \partial \Loss_t\|^2 \rho_t^2 - \langle \bar \partial \Loss_t, \param_t \rangle^2 }{\rho_t^2\Loss_t \big(\LinkFun(\Loss_t) - \homop_a(\rho_t)\big)}  \qquad \explain{ By \Cref{lem:property-pa}} \\ 
    & = \frac{\|\bar \partial \Loss_t\|^2 \rho_t^2 - \langle \bar \partial \Loss_t, \param_t \rangle^2 }{\rho_t^2\Loss_t \big(\LinkFun(\Loss_t) - \homop_a(\rho_t)\big)} \ge 0. 
\end{align*}
That is, $\GFmargin(\param_t)$ is increasing. 

\paragraph{Step 2: proving $\Loss_t < e^{-\homop_a(\rho_t)}/n$ for all $t\ge s$ by contradiction.} 
For $t=s$, $\Loss_s$ satisfies the condition by \Cref{asp:initial-cond-gf}.
Let $s_1$ be the first time where $L_{s_1}$ violates the condition,
that is,
\[
    s_1 \coloneqq  \sup \{ s^\prime | \Loss_t < e^{-\homop_a(\rho_t)}/n, \text{ for all } s\le t<s^\prime \}.
\]
We need to show $s_1 = + \infty$.
If not, we have $s_1 < \infty$. 
Note that 
\[
    \frac{e^{-\homop_a(\rho_t)}}{n\Loss_t} = \exp\Big( \log\big(1/n\Loss_t \big)   - \homop_a(\rho_t) \Big)  = \exp(\GFmargin(\param_t)\cdot \rho_t^M). 
\]
Since $\Loss_t < e^{-\homop_a(\rho_t)}/n$ for $s\le t<s_1$, by \textbf{Step 1} we have $\GFmargin(\param_t)$ and $\rho_t$ are both positive and increasing.  
Hence $\frac{e^{-\homop_a(\rho_t)}}{n\Loss_t}$ is increasing for $s\le t<s_1$. 
So we have
\[
    \frac{e^{-\homop_a(\rho_t)}}{n\Loss_t}\ge \frac{e^{-\homop_a(\rho_s)}}{n\Loss_s},\quad s\le t <s_1. 
\]
Since $\param_t$ forms an arc and the left-hand side is a continuous function of $\param_t$, it is continuous as a function of $t$. 
As a consequence, we have
\[
    \frac{e^{-\homop_a(\rho_{s_1})}}{n\Loss_{s_1}} \ge \frac{e^{-\homop_a(\rho_s)}}{n\Loss_s}> 1 \quad \implies \quad \Loss_{s_1} < e^{-\homop_a(\rho_{s_1})} /n. 
\]
Since $\Loss_t$ and $\homop_a(\rho_{t})$ are both continuous functions with respect to $t$ ($\param_t$ is an arc), this leads to that there exists $s_2>s_1$ and for all $s\le t\le s_2$, $\Loss_t < e^{-\homop_a(\rho_t)}/n$. This is a contradiction. So we have $\Loss_t < e^{-\homop_a(\rho_t)}/n$ for all $t\ge s$.  This completes the proof of \Cref{thm:gamma-a-increase}. 
\end{proof}


\subsection{Proof of Lemma~\ref{lem: Sanity check for gradient flow}} \label{sec:proof: sanity check gf}
\begin{proof}[Proof of Lemma~\ref{lem: Sanity check for gradient flow}]
    Recall that we have  the      decomposition: 
    \[
        f(\param; \xB) = \sum_{i=0}^\infty f^{(i)}(\param;\xB),\quad \xB\in(\xB_i)_{i=1}^n,
    \]
    where $f^{(i)}(\param;\xB)$ is $i$-homogeneous with respect to $\param$. Then \Cref{asp:nearhomo} implies that $f^{(i)}(\param;\xB)\equiv 0$ for all $i>M$. 
    By \Cref{lem:property-pa} we have 
    \[
        M \homop_a(x) - \homop_a^\prime (x) x \ge \homop^\prime (x). 
    \]
 Let $g(\param;\xB)\coloneqq f(\param;\xB) - \homop_a(\paramNorm)$. Then we have 
\begin{align*}
&\la \param, \nabla_{\param} g(\param;\xB) \rangle  - M g(\param;\xB)\\ 
= \, & \la \param, \nabla_{\param} f(\param;\xB) \rangle - M f(\param;\xB) - \la \param, \nabla_{\param} \homop_a(\paramNorm) \rangle + M \homop_a(\paramNorm)\\
= \, & \la \param, \nabla_{\param} f(\param;\xB) \rangle - M f(\param;\xB) -  \homop_a^\prime (\paramNorm) \|\param\|  + M \homop_a(\paramNorm) \\ 
\ge \, & -\homop^\prime (\paramNorm) - \homop_a^\prime (\paramNorm) \|\param\|  + M \homop_a(\paramNorm) \ge 0. 
\end{align*}
Let $h(\alpha;\xB) := g(\alpha \param_s;\xB)/(\alpha \|\param_s\|)^M$. 
Then we have 
\begin{align*}
\frac{\mathrm{d} h(\alpha ;\xB)}{\mathrm{d} \alpha }  
&= \frac{1}{(\alpha \|\param_s\|)^M} \frac{\mathrm{d} g(\alpha \param_s;\xB)}{\mathrm{d} \alpha} - M \frac{g(\alpha \param_s;\xB) \|\param_s\|}{(\alpha \|\param_s\|)^{M+1}}\\
&= \frac{\langle \alpha \param_s, \nabla g(\alpha \param_s ; \xB) \rangle }{\alpha^{M+1} \|\param_s\|^M}  - M \frac{g(\alpha \param_s;\xB)}{\alpha^{M+1} \|\param_s\|^{M}}\\
& = \frac{\langle \alpha \param_s, \nabla g(\alpha \param_s ; \xB) \rangle  - Mg(\alpha  \param_s;\xB)}{\alpha^{M+1} \|\param_s\|^M}\\ 
&\ge 0. 
\end{align*}
Since $\Loss(\param_s) <  e^{-\homop_a(\|\param_s\|)}/n$, we have $h(1;\xB)=g(\param_s;\xB)/(\alpha \|\param_s\|)^M >0$. As $\alpha \to \infty$, $h(\alpha;\xB)>0$. Note that 
\[
    \lim_{\alpha \to \infty} h(\alpha;\xB) = \lim_{\alpha \to \infty} \frac{ g(\alpha \param_s;\xB)}{(\alpha \|\param_s\|)^M}  = \lim_{\alpha \to \infty} \frac{ f(\alpha \param_s;\xB) - \homop_a(\alpha \|\param_s\|)}{(\alpha\|\param_s\|)^M} = \frac{ f^{(M)}( \param_s;\xB) }{ \|\param_s\|^M} >0. 
\] 
This leads to $f^{(M)}(\param;\xB) \not \equiv 0$.  This completes the proof of \Cref{lem: Sanity check for gradient flow}.
\end{proof}
\subsection{Convergence Rates} \label{sec:proof:rates}

We will show that $\GFmargin(\param_t)$ is a good approximation of $\gamma(\param_t)$. Before we proceed, we need one auxiliary margin. 
\begin{equation}
\label{eq:smooth-margin}
        \Smoothmargin(\param_t) \coloneqq \frac{\LinkFun(\Loss_t)}{\rho_t^M}.
\end{equation}

\begin{lemma}
    [Modified margin is a good approximation]
    \label{lem:gamma-approx}
    Under \Cref{asp:nearhomo} and $\Loss_t < e^{-\homop_a(\rho_t)}/n$, for the \Cref{eq:GF}, we have 
    \[
        \GFmargin (\param_t) \le \Smoothmargin(\param_t) \le \Normalmargin(\param_t)\le \bigg ( 1+ \frac{\log n + \homop_a(\rho_t)}{\LinkFun(\Loss_t) -\homop_a(\rho_t)} \bigg) \GFmargin(\param_t).
    \]
\end{lemma}
\begin{proof}[Proof of \Cref{lem:gamma-approx}]
Note that we have $\bar f_{\min}(\param_t) \ge \LinkFun(\Loss_t)$ from the proof of Lemma~\ref{lem:NH-rho-increase}. Then, we have 
\[
    \GFmargin (\param_t) = \frac{\LinkFun(\Loss_t) - \homop_a(\rho_t)}{\rho_t^M} \le \frac{\LinkFun(\Loss_t)}{\rho_t^M} = \Smoothmargin(\param_t) \le \frac{\bar f_{\min}(\param_t)}{\rho_t^M} = \Normalmargin (\param_t).
\]
For the last inequality, note that $\bar f_{\min}(\param_t) \le \log \frac{1}{\Loss_t}$ from Lemma~\ref{lem:f-min}, then using the definition of $\phi(\cdot)$ we have
\begin{align*}
    \frac{\gamma(\param_t)}{\GFmargin(\param_t)} 
    = \frac{\bar f_{\min}(\param_t)}{\LinkFun(\Loss_t) - \homop_a(\rho_t)} 
    \le \frac{ \log \frac{1}{\Loss_t}}{\LinkFun(\Loss_t) - \homop_a(\rho_t)}
    = 1 + \frac{\log n + \homop_a(\rho_t)}{\LinkFun(\Loss_t) -\homop_a(\rho_t)}.
\end{align*}
This completes the proof of \Cref{lem:gamma-approx}. 
\end{proof}
We will later show that 
\[
    \frac{\log n + \homop_a(\rho_t)}{\LinkFun(\Loss_t) -\homop_a(\rho_t)} \approx \frac{\log n + \homop_a(\rho_t)}{\log \frac{1}{\Loss_t}} \to 0.
\] 
This implies that $\GFmargin(\param_t)$ is a good multiplicative approximation of $\Normalmargin(\param_t)$. We need to characterize the behaviors of $\Loss_t$ and $\rho_t$.  

\begin{lemma}[Upper bound of the risk]
\label{lem:L-upperbound} 
Under \Cref{asp:nearhomo,asp:initial-cond-gf}, for \cref{eq:GF}, we have
\[
    \Loss_t = \Ocal \bigg( \frac{1}{t(\log t)^{2-2/M}} \bigg),
\]
where the hidden constant depends $\GFmargin(\param_s)^{\frac{2}{M}}$. 
\end{lemma}
\begin{proof}[Proof of \Cref{lem:L-upperbound}]
Recall that in \Cref{lem:GF-risk-decrease}, we have shown that 
\[
    - \frac{\mathrm{d} \Loss_t}{\mathrm{d} t} \ge \frac{1}{\rho_t^2} \cdot\bigg(\frac{1}{2} \frac{\mathrm{d}\rho_t^2}{\mathrm{d} t} \bigg)^2.
\]
By \Cref{thm:gamma-a-increase}, $\Loss_t < e^{-\homop_a(\rho_t)}/n$ for all $t\ge s$.  Note that when $\Loss_t < e^{-\homop_a(\rho_t)}/n$, we have for $M\ge 2$, 
\begin{align*}
    (M-1/2) \LinkFun(\Loss_t) &> (M-1/2) \homop_a(\rho_t) \\ 
    & = \sum_{i=1}^{M-1} \frac{(M-1/2)(i+1)a_{i+1}}{M-i}\rho_t^i + \frac{(M-1/2)a_1}{M-1/2}\\ 
    &\ge \sum_{i=1}^{M-1} (i+1)a_{i+1}\rho_t^i +a_1 = \homop^{\prime}(\rho_t).
\end{align*}
This leads to 
\begin{align}
    \label{eq:v_t log lb}
    M \LinkFun(\Loss_t) - \homop^{\prime}(\rho_t)   \ge \frac{1}{2}\LinkFun(\Loss_t) >0.
\end{align}

By Lemma~\ref{lem:NH-rho-increase}, we have
\begin{align*}
- \frac{\mathrm{d} \Loss_t}{\mathrm{d} t} 
& \ge \frac{1}{\rho_t^2} \cdot \big(  M \Loss_t \LinkFun(\Loss_t) - p^\prime (\rho_t)\Loss_t \big)^2\\ 
& \ge \frac{1}{4\rho_t^2} \cdot \big( \Loss_t \LinkFun(\Loss_t)\big)^2 &&\explain{ By the inequality above } \\ 
& = \frac{\Smoothmargin(\param_t)^{\frac{2}{M}}}{4\big(\LinkFun(\Loss_t)\big)^{\frac{2}{M}}} \cdot \big( \Loss_t \LinkFun(\Loss_t)\big)^2 &&\explain{ By the definition of $\Smoothmargin(\param_t)$ } \\
& \ge \frac{\GFmargin(\param_t)^{\frac{2}{M}}}{4} \cdot  \Loss_t^2\big( \LinkFun(\Loss_t)\big)^{2-2/M} &&\explain{ Since $\Smoothmargin(\cdot) \ge \GFmargin(\cdot)$ } \\ 
&\ge \frac{\GFmargin(\param_s)^{\frac{2}{M}}}{4}  \cdot  \Loss_t^2\big( \LinkFun(\Loss_t)\big)^{2-2/M} &&\explain{ By Theorem~\ref{thm:gamma-a-increase} }\\ 
& = c \Loss_t^2\big( \LinkFun(\Loss_t)\big)^{2-2/M},
\end{align*}
where $c = \frac{\GFmargin(\param_s)^{\frac{2}{M}}}{4}  $. Equivalently, we have
\[
    \frac{1}{\big( \LinkFun(\Loss_t)\big)^{2-2/M} } \cdot \frac{\mathrm{d}}{\mathrm{d}t} \frac{1}{n\Loss_t} \ge \frac{c}{n}. 
\]
Let $\Gcal_t \coloneqq \frac{1}{n\Loss_t}$ and $S(x) = \int_{\Gcal_s}^x \frac{1}{(\log t)^{2-2/M}} dt$. Then we have 
\[
    \frac{\Gcal_t^\prime }{(\log \Gcal_t)^{2-\frac{2}{L}}} \ge \frac{c}{n} \quad \Rightarrow \quad S(\Gcal_t) \ge \frac{c}{n}(t-s).
\]
Note that
\begin{align*}
    \Gcal_s = \frac{1}{n \Loss_s} \ge \frac{\exp\big(-\homop_a(\rho_s)\big)}{n\Loss_s} >1. 
\end{align*}
We can invoke Lemma~\ref{lem:int-log} to get 
\[
    \Gcal_t \ge S^{-1} \bigg( \frac{c}{n}(t-s) \bigg) = \Omega(t(\log t) ^{2-2/M}) \implies \Loss_t = \Ocal \bigg( \frac{1}{t(\log t)^{2-2/M}} \bigg).
\]
This completes the proof of \Cref{lem:L-upperbound}.  
\end{proof}

Then we control the rate of the parameter norm with respect to the risk. The following lemma characterizes the M-near homogeneity.

\begin{lemma}
[Near homogeneity]
\label{lem:near-homogeneity}
Let $f(\param;\xB_i)$ be locally Lipschitz and $M$-near-homogeneous, then for $\|\param\|\ge r$, there exists a constant $B_r$ such that 
\[
    |f(\param; \xB_i)| \le B_r \cdot \|\param_t\|^M \quad \text{ for all } i \in [n].
\]
\end{lemma}
\begin{proof}[Proof of \Cref{lem:near-homogeneity}]
We can fix an index $i$ and let $c_1 = \max_{\|\param\|= r} |f(\param; \xB_i)|/r^M$. For any $\|\param\|\ge r$, we let $\bar \param_t= r \param/\|\param\|$. Consider  $g(t) = f(t \bar \param; \xB_i)/t^M$. Since $f$ is locally Lipschitz, $ g(t)$ is also Locally Lipschitz and its derivative exists almost everywhere. Then we have for almost every $t>0$,
\[
    g^\prime (t) = \frac{\langle \nabla_{\theta}f(t\bar \param;\xB_i) , t\bar \param_t\rangle - M f(t\bar \param; \xB_i) }{t^{M+1}} \le \frac{a \|\bar \param\|^{M-1}}{t^2} = \frac{ar^{M-1}}{t^2}.  
\]
We can also show that $g^\prime (t) \ge -ar^{M-1}/t^2$. Note that $|g(1)|\le c_1r^M$.  Since $g(t)$ is locally Lipschitz, it's also absolutely continuous. Hence we can apply the fundamental theorem of calculus to get
\begin{align*}
\frac{|f(\param; \xB_i)|r^M}{\|\param\|^M} 
&=  g(\|\param\|/r) \\
& = g(1) + \int_1^{\|\param\|/r} g^\prime (t) dt \\ 
&\le c_1r^M + \int_1^{\|\param\|/r} |g^\prime (t)| dt  \\ 
&\le  c_1 r^M + ar^{M-1}.
\end{align*}
This is equivalent to 
\[
    |f(\param; \xB_i)| \le ( a/r + c_1) \cdot \|\param\|^M.
\]
Let $B_r = a/r + c_1$. We complete the proof of \Cref{lem:near-homogeneity}. 
\end{proof}

With this Lemma, we show that $\Normalmargin$ and $\GFmargin$ are both bounded. 

\begin{lemma}
    [Boundedness of margins]
    \label{lem:margin-bounded}
    Under \Cref{asp:nearhomo,asp:initial-cond-gf}, we have
    \[
    \GFmargin(\param_t) \le \Normalmargin(\param_t)\le B_{\|\rho_s\|}. 
    \] 
\end{lemma}
\begin{proof}[Proof of \Cref{lem:margin-bounded}]
    Note that $\rho_t$ is increasing. Therefore, $\rho_t \ge \rho_s$. Then we can apply \Cref{lem:near-homogeneity} to get 
    \[
        \Normalmargin(\param_t) = \frac{\min y_i f(\param_t;\xB_i)}{\rho_t^M} \le B_{\|\rho_s\|} . 
    \]
    Therefore, $\GFmargin(\param_t) \le \Normalmargin(\param_t)\le B_{\|\rho_s\|}$.
\end{proof}

\begin{lemma}
[Rate of the parameter norm]
\label{lem:rho-rate}
Under \Cref{asp:nearhomo,asp:initial-cond-gf}, for \Cref{eq:GF}, we have 
\[
    \rho_t^M = \Theta \Big( \log \frac{1}{\Loss_t} \Big), 
\]
where the hidden constant depends on $\GFmargin(\param_s)$ and $\homoq$. 
\end{lemma}
 \begin{proof}[Proof of \Cref{lem:rho-rate}]
Recall that 
 \[
    \Smoothmargin(\param_t) \ge \GFmargin(\param_t) \ge \GFmargin(\param_s). 
 \]
 Then we have 
 \[
    \rho_t^M \le \frac{1}{\GFmargin(\param_s)} \cdot \LinkFun(\Loss_t) = \Ocal \Big( \log \frac{1}{\Loss_t} \Big).
 \]
 On the other hand, by \Cref{lem:near-homogeneity,lem:f-min} we have
 \[
    B \rho_t^M \ge \bar f_{\min}(\param_t) \ge  \log \frac{1}{n\Loss_t} \implies \rho_t^M = \Omega \Big( \log \frac{1}{\Loss_t} \Big).
 \]
 Combining these two bounds, we have
 \[
    \rho_t^M = \Theta \Big( \log \frac{1}{\Loss_t} \Big).
 \]
We complete the proof of \Cref{lem:rho-rate}. 
 \end{proof}

 \begin{lemma}
 [Lower bound of the risk]
 \label{lem:L-lowerbound}
 Under \Cref{asp:initial-cond-gf,asp:nearhomo}, for \Cref{eq:GF}, we have
 \[
        \Loss_t = \Omega \bigg( \frac{1}{t(\log t)^{2-2/M}} \bigg), 
 \]
where the hidden constant depends on $\GFmargin(\param_s)$ and $\homoq(x)$.
 \end{lemma}
 \begin{proof}[Proof of \Cref{lem:L-lowerbound}]
 Note that by \Cref{lem:chain-rule-clark,cor:subgrad-exp-loss},
 \[
    \frac{\mathrm{d} \param_t}{\mathrm{d} t} = - \bar \partial \Loss (\param_t) = \frac{1}{n}\sum_{i=1}^n e^{-y_i f(\param;\xB_i)} y_i \hB_i,
 \]
where $\bar \partial $ \Cref{eq:min-norm-sub-grad} denotes the minimal norm subgradient and $\hB_i \in \partial f(\param;\xB_i)$. Therefore under \Cref{asp:nearhomo}, we have 
 \[
    \| \bar \partial \Loss_t \| \le \frac{1}{n}\sum_{i=1}^n e^{-y_i f(\param;\xB_i)} \| \hB_i\| \le \Loss_t\cdot  \homoq^\prime (\rho_t).
 \]
By \Cref{asp:nearhomo} and increase of $\rho_t$ , we can bound $\homoq^\prime (\rho_t)$ by $b \|\rho_t\|^{M-1}$ for some constant $b$. 
Then we have 
\[
    - \frac{\mathrm{d} \Loss_t}{\mathrm{d} t} = \| \bar \partial \Loss_t\|^2 \le \Loss_t^2 b^2 \rho_t^{2M-2} \le \Loss_t^2\cdot \Ocal \bigg( \Big(\log \frac{1}{\Loss_t}\Big)^{2-2/M}\bigg). 
\]
By the definition of $S(\cdot)$ in the proof of \Cref{lem:L-upperbound}, this implies that 
\[
    \frac{\mathrm{d}}{\mathrm{d}t} S\bigg(\frac{1}{\Loss_t}\bigg)\le c. 
\]
Therefore we have 
\[
    \frac{1}{\Loss_t} \le S^{-1}(c(t-s)) = \Ocal(t(\log t)^{2-2/M})\implies \Loss_t = \Omega \bigg( \frac{1}{t(\log t)^{2-2/M}} \bigg).
\]
This completes the proof of \Cref{lem:L-lowerbound}.
 \end{proof}

 \begin{theorem}
 [Rates of the risk and the parameter norm]
 \label{thm:rate-rho-L}
 Under \Cref{asp:initial-cond-gf,asp:nearhomo}, for \Cref{eq:GF}, we have
 \[
    \Loss_t = \Theta \bigg( \frac{1}{t(\log t)^{2-2/M}} \bigg), \quad \rho_t = \Theta \big(     (\log t)^{\frac{1}{M}} \big),
 \]
where the hidden constant depends on $\GFmargin(\param_s)$ and $\homoq(x)$.
 \end{theorem}
 \begin{proof}[Proof of \Cref{thm:rate-rho-L}]
    By \Cref{lem:L-upperbound,lem:L-lowerbound}, we have
    \[
        \Loss_t = \Theta \bigg( \frac{1}{t(\log t)^{2-2/M}} \bigg).
    \] 
    By Lemma~\ref{lem:rho-rate}, we have 
    \[
        \rho_t^M = \Theta \Big( \log \frac{1}{\Loss_t} \Big) = \Theta ( \log t ) \implies \rho_t = \Theta \Big( (\log t)^{\frac{1}{M}} \Big).
    \]
    We complete the proof of \Cref{thm:rate-rho-L}. 
 \end{proof}

\begin{lemma}
[Multiplicative error]
\label{lem:eps}
The multiplicative error of $\GFmargin(\param_t)$ satisfies
\[
    \frac{\log n + \homop_a(\rho_t)}{\LinkFun(\Loss_t) -\homop_a(\rho_t)} 
    \to 0.
\]
\end{lemma}
\begin{proof}[Proof of \Cref{lem:eps}]
Applying \Cref{thm:rate-rho-L}, we have 
\[
    \log \frac{1}{\Loss_t} = \Theta(\log t), \quad \homop_a(\rho_t) = \homop_a(\rho_t) = \Theta( \rho_t^{M-1}) = \Theta( (\log t)^{1-1/M}).
\]
Therefore we have 
\[
    \frac{\log n + \homop_a(\rho_t)}{\log \frac{1}{\Loss_t} -\log n -\homop_a(\rho_t)} = \Theta \bigg( \frac{(\log t)^{1-1/M}}{\log t} \bigg) = \Theta \bigg( \frac{1}{(\log t)^{\frac{1}{M}}}\bigg)\to 0.
\]
This completes the proof of \Cref{lem:eps}.
\end{proof}

\subsection{Proof of Theorem~\ref{thm: Margin improving and convergence}} \label{sec:proof:thm1}
\begin{proof}[Proof of Theorem~\ref{thm: Margin improving and convergence}]

    By \Cref{thm:gamma-a-increase}, we know that $\GFmargin(\param_t)$ is increasing and $\Loss_t < e^{-\homop_a(\rho_t)}/n$ for all $t\ge s$.  By \Cref{thm:rate-rho-L}, we get the rates of $\rho$ and $\Loss_t$. By \Cref{lem:eps}, we know $\GFmargin$ is a good approximator of $\Normalmargin$. This completes the proof of Theorem~\ref{thm: Margin improving and convergence}. 
\end{proof}

\subsection{Proof of Example~\ref{eg:counter-eg init cond}}\label{sec:proof:counter-eg}
\begin{proof}[Proof of Example~\ref{eg:counter-eg init cond}]
Note that for this example we have $\homop(\param)=|\param|^M$, $\homop_a(\param)=M|\param|^{M-1}$, and $f(\param)=\param^M+\homop_a(|\param|)$. This means \Cref{asp:initial-cond-gf} is equivalent to:
$$
\Loss(\param_s) = \exp (-f(\param_s))<\exp(-p_a(|\param_s|)) \Longleftrightarrow \param_s^M>0 \Longleftrightarrow \param_s >0.
$$
If the above condition does not hold, then $\param_s\le 0$. Note that $0$ is a stationary point for $\Loss(\param)$. So if \Cref{eq:GF} is initialized from $\param_s$, it cannot produce positive parameters in the future, that is, $\param_t\le 0$ for every $t\ge s$. Hence, \Cref{eq:GF} cannot minimize the loss or exhibit any implicit bias. This completes the proof of \Cref{eg:counter-eg init cond}.
\end{proof}

\subsection{Directional Convergence} \label{sec:proof:direct}
The main idea of the proof is to show that the curve swept by $\tilde{\param}_t$ has a finite length, thus $\tilde{\param}_t$ converges. We use $\zeta_t$ to denote the curve swept by $\tilde{\param}_t$. Another important quantity in our proof is the modified margin $\GFmargin$ defined in \eqref{eq: modified margin}. 

In \Cref{thm: Margin improving and convergence}, we show that $\GFmargin(\param_t)$ is nondecreasing with some limit $\gamma_* \in (0, +\infty)$ and $\|\param_t\|\to \infty$. 

Motivated by \citet{ji2020directional}, we will invoke a sophisticated but standard tool in the analysis of definable functions, the {\it desingularizing function}.  A desingularizing function will witness the flow is well-behaved and the curve swept by $\tilde{\param}_t$ has a finite length.

\begin{definition}
[Desingularizing function]
\label{def: designularizing function}
A function $\Psi:[0, \nu) \rightarrow \mathbb{R}$ is called a desingularizing function when $\Psi$ is continuous on $[0, \nu)$ with $\Psi(0)=0$, and continuously differentiable on $(0, \nu)$ with $\Psi^{\prime}>0$. 
\end{definition}

The following Lemma plays a key role in proving the directional convergence of params. 
\begin{lemma}
[Existence of desingularizing function]
\label{lem: Existence of desingularizing function}
There exist $R>0, \nu>0$ and a definable desingularizing function $\Psi$ on $[0, \nu)$, such that for a.e. large enough $t$ with $\left\|\param_t\right\|>R$ and $\GFmargin(\param_t)>\gamma_*-\nu$, it holds that

$$
\frac{\mathrm{d} \zeta_t}{\mathrm{~d} t} \leq-c \frac{\mathrm{d} \Psi\left(\gamma_*-\GFmargin(\param_t)\right)}{\mathrm{d} t}. 
$$
for some constant $c>0$.
\end{lemma}
The central part of this subsection is to prove \Cref{lem: Existence of desingularizing function}, and then we will use this lemma to prove \Cref{thm: directional convergence} in \Cref{sec:direct-conv:direct-proof}. To prove \Cref{lem: Existence of desingularizing function}, we need many valuable tools. We present the tools first.

Following the notation in \citet{ji2020directional},
given any $f$ that is locally Lipschitz around a nonzero $\param$, let 
\[
    \bar  \partial_r f(\param) \coloneqq \langle \bar \partial f(\param), \tilde{\param} \rangle  \tilde{\param}  \quad \text{ and } \quad  \bar \partial_\perp f(\param) \coloneqq \bar \partial f(\param) - \bar  \partial_r f(\param)
\]
be the radial and spherical parts of $\bar \partial f(\param)$ respectively. We use 
\[a_t\coloneqq \LinkFun(\Loss_t) - \homop_a(\rho_t)\] to denote the numerator of $\GFmargin(\param_t)$. 
 Before we dive into the proof of \Cref{lem: Existence of desingularizing function}, we need to characterize the Clarke subdifferential of $\GFmargin(\param_t)$.

\begin{lemma}
[Subdifferential of $\GFmargin$]
\label{lem:subdifferential of GFmargin}
For \eqref{eq:GF}, we have 
\begin{align*}
\partial a_t
&= \bigg \{- \frac{ \nabla \Loss_t}{\Loss_t } - {\homop_a^\prime (\rho_t) \tilde \param_t}  \bigg | \, \text{ for any } \nabla \Loss_t \in \partial \Loss_t  \bigg\},
\\
\partial \GFmargin(\param_t) 
&= \bigg \{- \frac{ \nabla \Loss_t}{\Loss_t \rho_t^M} - \frac{\homop_a^\prime (\rho_t) \tilde \param_t}{\rho_t^M} - \frac{Ma_t \tilde \param_t}{\rho_t^{M+1}} \bigg | \,  \text{ for any }  \nabla \Loss_t \in \partial \Loss_t  \bigg\}. 
\end{align*}
\end{lemma}
\begin{proof}[Proof of \Cref{lem:subdifferential of GFmargin}]
The proof of these two results is a direct consequence of the definition of Clarke subdifferentials. Recall that 
\[
\begin{aligned}
\partial f(x) \coloneqq \text{conv} \bigg\{ 
 \lim_{i\to \infty} \nabla f(x_i) : x= \lim_{i\to \infty} x_i, 
\text{where}\ x_i \in D \ \text{and}\ \nabla f(x_i) \ \text{exists}
  \bigg\}.
\end{aligned}
\]
Recall that $a_t$ and $\GFmargin(\param_t)$ are differentiable functions of $\Loss_t$. 
Next, we prove that:
for a locally Lipschitz function $a(x):\Rbb^d \to \Rbb$ and a differentiable function $b(\cdot):\Rbb \to \Rbb$, we have 
\[
\partial_x [b(a(x))] = \big \{ b^\prime|_{a(x)} \cdot \hB \; \big | \text{ for any } \hB \in \partial_x a(x) \big \}.
\]
By \Cref{thm:Clark chain rule},  we know that 
\[
\partial_x [b(a(x))] \subset \big \{ b^\prime|_{a(x)} \cdot \hB \; \big | \text{ for any } \hB \in \partial_x a(x) \big \}. 
\]
It remains to show the other direction. 
For every $\hB \in \partial_x a(x)$, there exists $\lim_{i\to \infty} x_i = x$ such that $\hB  = \lim_{i\to \infty} \nabla a(x_i)$. Since $\lim_{i \to \infty} a(x_i) = a(x)$, we have 
\[
 \lim_{i \to \infty}  b^\prime (a(x_i)) = b^\prime (a(x)) \Rightarrow \lim_{i \to \infty}  b^\prime (a(x_i)) \nabla a(x_i) = b^\prime (a(x)) \hB.
\] 
Here, we use the facts that $b$ is differentiable at $a(x_i)$ and that $a$ is differentiable at $x_i$, which  
implies that $b(a(\cdot))$ is also differentiable at $x_i$. Hence we have 
\[
\big \{ b^\prime|_{a(x)} \cdot \hB \; \big | \text{ for any } \hB \in \partial_x a(x) \big \} \subset\partial_x [b(a(x))].
\]
The above result completes the proof of \Cref{lem:subdifferential of GFmargin}.
\end{proof}

Consider the following two special elements within the set of subdifferentials,
\begin{align}
\label{eq:tilde-partial-margin-GD}
\tilde \partial a_t 
&\coloneqq - \frac{ \bar \partial \Loss_t}{\Loss_t } - {\homop_a^\prime (\rho_t) \tilde \param_t}, \\
\tilde \partial \GFmargin(\param_t) 
&\coloneqq - \frac{ \bar \partial \Loss_t}{\Loss_t \rho_t^M} - \frac{\homop_a^\prime (\rho_t) \tilde \param_t}{\rho_t^M} - \frac{Ma_t \tilde \param_t}{\rho_t^{M+1}}. 
\end{align}
These two elements are crucial for our analysis, enabling us to deal with the non-homogeneity. 
Specifically, this together with \Cref{lem:chain-rule-clark} leads to \Cref{lem:decomp-radial-spherical}. Another remarkable point is that \citet{ji2020directional} did not consider these special elements, instead, they considered $\bar \partial a_t$ and $\bar \partial \GFmargin(\param_t)$. 
This is because their model is homogeneous so $\homop_a=0$ in their analysis. 
It is straightforward to check that $\bar \partial a_t$, and $ - \bar \partial \Loss_t$ are all in the same direction for homogeneous models. 
However, for non-homogeneous models, we need to consider the above elements. 
Thanks to \Cref{lem:chain-rule-clark}, we can use any element of $\partial a_t$ and $\partial \GFmargin(\param_t)$ with the chain rule, that is, we have
\[
   \frac{\mathrm{d}  a(\param_t)}{\mathrm{d}  t} = \langle \tilde \partial a(\param_t), - \bar \partial \Loss_t \rangle, \quad \frac{\mathrm{d}  \GFmargin(\param_t)}{\mathrm{d}  t} = \langle \tilde \partial \GFmargin(\param_t), - \bar \partial \Loss_t \rangle.
\]
This helps us to understand the increase of $\GFmargin(\param_t)$ in terms of the spherical and the radial change of $\bar \partial \Loss_t$.  To further elaborate this, we define: 
\begin{align}
\label{eq:tilde-radial-perp-margin-GD}
\tilde \partial_r a_t \coloneqq \langle \tilde \partial a_t, \tilde \param_t \rangle \tilde \param_t, 
&\quad \quad \tilde \partial_\perp a_t \coloneqq \tilde \partial a_t - \tilde \partial_r a_t,
\\ 
\tilde \partial_r \GFmargin(\param_t) \coloneqq \langle \tilde \partial \GFmargin(\param_t), \tilde \param_t \rangle \tilde \param_t, 
&\quad \quad \tilde \partial_\perp \GFmargin(\param_t) \coloneqq \tilde \partial \GFmargin(\param_t) - \tilde \partial_r \GFmargin(\param_t).
\end{align}
With simple calculations, we get the following inequalities.

\begin{lemma}
[Decomposition of tilde subdifferential]
\label{lem:Decomposition of tilde subdifferential}
Under \Cref{asp:nearhomo,asp:initial-cond-gf}, for $\tilde \partial_r \GFmargin(\param_t)$ and $\tilde \partial_\perp \GFmargin(\param_t)$ in \eqref{eq:tilde-radial-perp-margin-GD}, we have 
\begin{align}
\label{eq:tilde-radial-perp-margin-GD-2}
\langle \tilde \partial_r a_t, \tilde \param_t \rangle \ge \frac{M a_t}{ \rho_t}  \ge 0 , &\quad \tilde \partial_\perp a_t  = - \frac{ \bar \partial_\perp \Loss_t}{\Loss_t}, \\ 
\langle \tilde \partial_r \GFmargin(\param_t), \tilde \param_t \rangle  \ge 0 , &\quad \tilde \partial_\perp \GFmargin(\param_t)  = - \frac{ \bar \partial_\perp \Loss_t}{\Loss_t \rho_t^M}. 
\end{align}

\end{lemma}
\begin{proof}[Proof of \Cref{lem:Decomposition of tilde subdifferential}]
We prove the results for $a_t$ and $\GFmargin(\param_t)$ separately.

\myunder{The first inequality for $a_t$.} 
We have 
\begin{align*}
\langle \tilde \partial_r a_t, \tilde \param_t \rangle 
&= - \frac{ \langle\bar \partial \Loss_t, \tilde \param_t  \rangle }{\Loss_t} - {\homop_a^\prime (\rho_t)} 
= \frac{ \langle - \bar \partial \Loss_t,  \param_t  \rangle }{\Loss_t \rho_t} - {\homop_a^\prime (\rho_t)} 
\\ 
&= \frac{1}{ n \Loss_t \rho_t}\sum_{i=1}^n e^{-\bar f_i(\param_t )} \langle \nabla \bar f_i (\param_t ), \param_t \rangle - {\homop_a^\prime (\rho_t)}
\\ 
&\ge  \frac{1}{ n \Loss_t \rho_t}\sum_{i=1}^n e^{-\bar f_i(\param_t )} (M \bar f_i(\param_t) - \homop^\prime (\rho_t))  -{\homop_a^\prime (\rho_t)}
\\ 
&\ge \frac{1}{ n \Loss_t \rho_t}\sum_{i=1}^n e^{-\bar f_i(\param_t )} (M \log (1/ n \Loss_t) - \homop^\prime (\rho_t) - \homop_a^\prime (\rho_t) \rho_t)   
\\ 
&\ge \frac{1}{ n \Loss_t \rho_t}\sum_{i=1}^n e^{-\bar f_i(\param_t )} M a_t   
= \frac{M a_t}{ \rho_t}
\ge 0.
\end{align*}

\myunder{The second inequality for $a_t$.} 
We have 
\begin{align*}
\tilde \partial_\perp a_t
&=  
- \frac{ \bar \partial \Loss_t}{\Loss_t} -{\homop_a^\prime (\rho_t) \tilde \param_t} - 
\bigg \langle - \frac{ \bar \partial \Loss_t}{\Loss_t } -{\homop_a^\prime (\rho_t) \tilde \param_t}, \tilde \param_t \bigg\rangle \tilde \param_t \\ 
& = - \frac{ \bar \partial \Loss_t}{\Loss_t} + \frac{ \bar \partial_r \Loss_t}{\Loss_t } = - \frac{ \bar \partial_\perp \Loss_t}{\Loss_t}.
\end{align*}


\myunder{The first inequality for $\GFmargin(\param_t)$.} We have 
\begin{align*}
\langle \tilde \partial_r \GFmargin(\param_t), \tilde \param_t \rangle 
&= - \frac{ \langle\bar \partial \Loss_t, \tilde \param_t  \rangle }{\Loss_t \rho_t^M} - \frac{\homop_a^\prime (\rho_t)}{\rho_t^M} - \frac{Ma_t}{\rho_t^{M+1}} 
= \frac{ \langle - \bar \partial \Loss_t,  \param_t  \rangle }{\Loss_t \rho_t^{M+1}} - \frac{\homop_a^\prime (\rho_t)}{\rho_t^M} - \frac{Ma_t}{\rho_t^{M+1}} 
\\ 
&= \frac{1}{ n \Loss_t \rho_t^{M+1}}\sum_{i=1}^n e^{-\bar f_i(\param_t )} \langle \nabla \bar f_i (\param_t ), \param_t \rangle - \frac{\homop_a^\prime (\rho_t)}{\rho_t^M} - \frac{Ma_t}{\rho_t^{M+1}} 
\\ 
&\ge  \frac{1}{ n \Loss_t \rho_t^{M+1}}\sum_{i=1}^n e^{-\bar f_i(\param_t )} (M \bar f_i(\param_t) - \homop^\prime (\rho_t))  - \frac{\homop_a^\prime (\rho_t)}{\rho_t^M} - \frac{Ma_t}{\rho_t^{M+1}} 
\\ 
&\ge \frac{1}{ n \Loss_t \rho_t^{M+1}}\sum_{i=1}^n e^{-\bar f_i(\param_t )} (M \log (1/ n \Loss_t) - \homop^\prime (\rho_t) - \homop_a^\prime (\rho_t) \rho_t)   - \frac{Ma_t}{\rho_t^{M+1}} \\ 
&\ge \frac{1}{ n \Loss_t \rho_t^{M+1}}\sum_{i=1}^n e^{-\bar f_i(\param_t )} M a_t   - \frac{Ma_t}{\rho_t^{M+1}}
= 0.
\end{align*}

\myunder{The second inequality for $\GFmargin(\param_t)$.} We have 
\begin{align*}
\tilde \partial_\perp \GFmargin(\param_t) 
&=  
- \frac{ \bar \partial \Loss_t}{\Loss_t \rho_t^M} - \frac{\homop_a^\prime (\rho_t) \tilde \param_t}{\rho_t^M} - \frac{Ma_t \tilde \param_t}{\rho_t^{M+1}} - 
\bigg \langle - \frac{ \bar \partial \Loss_t}{\Loss_t \rho_t^M} - \frac{\homop_a^\prime (\rho_t) \tilde \param_t}{\rho_t^M} - \frac{Ma_t \tilde \param_t}{\rho_t^{M+1}}, \tilde \param_t \bigg\rangle \tilde \param_t \\ 
& = - \frac{ \bar \partial \Loss_t}{\Loss_t \rho_t^M} + \frac{ \bar \partial_r \Loss_t}{\Loss_t \rho_t^M} = - \frac{ \bar \partial_\perp \Loss_t}{\Loss_t \rho_t^M}.
\end{align*}
This completes the proof of \Cref{lem:Decomposition of tilde subdifferential}. 
\end{proof}

After a detailed characterization of $\tilde \partial \param_t$, we can give a decomposition of the increase of margin. This decomposition will help us construct the desired desingularizing function. 

\begin{lemma}[Decomposition of radial and spherical parts]
\label{lem:decomp-radial-spherical}
Under \Cref{asp:nearhomo,asp:initial-cond-gf}, for almost every $t\ge s$, we have  
\begin{equation}
\label{eq: decomp-gamma}
    \frac{\mathrm{d} \GFmargin(\param_t)}{ \mathrm{d} t} = \| \tilde  \partial_r \GFmargin(\param_t) \| \| \bar \partial_r \Loss_t\| + \| \tilde \partial_\perp \GFmargin(\param_t) \| \| \bar \partial_\perp \Loss_t\|,
\end{equation}
and 
\begin{equation}
\label{eq: decom-zeta}
\|\tilde \partial_\perp \GFmargin(\param_t) \|= \frac{\|\tilde \partial_\perp a_t\|}{ \rho_t^M} , \quad \text{ and } \quad \frac{\mathrm{d} \zeta_t}{\mathrm{d} t}  = \frac{\| \bar \partial_\perp \Loss_t\|}{\rho_t}.
\end{equation}
\begin{proof}[Proof of \Cref{lem:decomp-radial-spherical}]
We prove the two equations separately. 

\myunder{Proof of \eqref{eq: decomp-gamma}.} By \Cref{lem:chain-rule-clark}, we have for almost every $t \ge s$, 
\[
    \frac{\mathrm{d}  \GFmargin(\param_t)}{\mathrm{d}  t} = \langle \tilde \partial \GFmargin(\param_t), - \bar \partial \Loss_t \rangle  = \langle \tilde \partial_r \GFmargin(\param_t), - \bar \partial_r \Loss_t \rangle + \langle \tilde \partial_\perp \GFmargin(\param_t), - \bar \partial_\perp \Loss_t \rangle.  
\]
To prove \eqref{eq: decomp-gamma}, we need to show two arguments: 
\begin{itemize}
    \item $\langle \tilde \partial_r \GFmargin(\param_t), \param_t \rangle$ and $\langle -\bar  \partial_r \Loss_t, \param_t \rangle  $ share the same sign. 
    \item $\tilde \partial_\perp \GFmargin(\param_t)$ and $-\bar \partial_\perp \Loss_t$ point the same direction.  
\end{itemize}

\noindent{\bf $\bullet$ The first argument.} By \Cref{lem:Decomposition of tilde subdifferential} we have that $\langle \tilde \partial_r \GFmargin(\param_t), \tilde \param_t \rangle  \ge 0$. On the other hand, we have 
\begin{align*}
    \langle -\nabla  \Loss_t, \param_t \rangle
    &=   \frac{1}{n } \sum_{i=1}^n e^{-\bar f_i(\param_t )} \langle \nabla \bar f_i (\param_t ), \param_t \rangle\\ 
    &\ge \frac{1}{n } \sum_{i=1}^n e^{-\bar f_i(\param_t)} M \bar f_i(\param_t)  - p^\prime (\rho_t) \Loss_t &&\explain{ By \Cref{asp:nearhomo}}\\ 
    &\ge M \Loss_t \LinkFun(\Loss_t) - \Loss_t  p^\prime (\rho_t)   &&\explain{ By \Cref{lem:f-min}}\\ 
    & \ge \Loss_t M a_t  &&\explain{ By \Cref{lem:property-pa}}\\
    & \ge 0. &&\explain{ By monotonicity of $\GFmargin(\param_t)$ in \Cref{thm: Margin improving and convergence}} 
\end{align*}
Therefore $\langle \bar \partial_r \GFmargin(\param_t), \param_t \rangle$ and $\langle -\bar  \partial_r \Loss_t, \param_t \rangle  $ share the same sign. This completes the proof of the first argument. 

\noindent{\bf $\bullet$ The second argument.} This is a direct result of \Cref{lem:Decomposition of tilde subdifferential}.

\myunder{\noindent{Proof of \eqref{eq: decom-zeta}.}} The first part is also a direction result of \Cref{lem:Decomposition of tilde subdifferential}. For the second part, since $\param_t$ is an arc and $\|\param_t\| \ge \|\param_s\| >0$, it follows that $\tilde{\param}_t$ is also an arc.  Moreover, for almost every $t\ge 0$, 
\[
    \frac{\mathrm{d} \tilde{\param}_t}{\mathrm{d}  t} = \frac{1}{\rho_t} \frac{\mathrm{d}  \param_t}{\mathrm{d}  t} - \frac{1}{\rho_t} \bigg\langle\frac{\mathrm{d}  \param_t}{\mathrm{d}  t}, \tilde{\param}_t  \bigg \rangle \cdot  \tilde{\param}_t = - \frac{\bar \partial_\perp \Loss_t}{\rho_t}.
\]
Since $\tilde{\param}_t$ is an arc, $ \mathrm{d} \tilde{\param}_t / \mathrm{d}  t$ and $\| \mathrm{d} \tilde{\param}_t / \mathrm{d}  t \|$ are both integrable. By definition of the curve length, 
\[
    \zeta_t = \int_s^t  \bigg\|\frac{\mathrm{d} \tilde{\param}_t}{ \mathrm{d}  t} \bigg \| \mathrm{d} t,  
\]
and for almost every $t\ge s$, we have 
\[
    \frac{\mathrm{d} \zeta (t)}{\mathrm{d} t}  =  \bigg\|\frac{\mathrm{d} \tilde{\param}_t}{ \mathrm{d}  t} \bigg \|= \frac{\| \bar \partial_\perp \Loss_t\|}{\rho_t}.
\]
This completes the proof of \eqref{eq: decom-zeta}. Therefore, we have proved \Cref{lem:decomp-radial-spherical}.
\end{proof}

\end{lemma}

Then we need some inequalities to connect $\big( \bar \partial_r \Loss_t, \bar \partial_\perp \Loss_t \big)$ and $\big( \tilde \partial_r\GFmargin(\param_t), \tilde \partial_\perp \GFmargin(\param_t) \big)$.  The main idea is to show that their ratios are close, that is, 
\[
    \frac{\| \bar \partial_\perp \Loss_t\|}{\| \bar \partial_r \Loss_t\|} \approx \frac{\| \tilde \partial_\perp \GFmargin(\param_t)\|}{\| \tilde \partial_r \GFmargin(\param_t)\|}.
\] 
The technique is to use $a_t$ to bridge these two ratios.

\begin{lemma}
[Inequalities between $a_t$ and $\GFmargin(\param_t)$]
\label{lem: Inequalities between a and gamma}
Under \Cref{asp:nearhomo} and the condition $\Loss_s \le e^{-\homop_a(\rho_s)}/n$, for almost every $t\ge s$, we have 
\begin{equation}
\label{eq: lowerbound partial_r a}
    \| \tilde \partial_r a_t\| \ge M \GFmargin(\param_s) \rho_t^{M-1}, 
\end{equation}
and 
\begin{equation}
\label{eq: upperbound partial_r gamma^c}
    \| \tilde \partial_r \GFmargin(\param_t)\| \le \frac{M \log n + 2M \homop_a(\rho_t)}{\rho_t^{M+1}}. 
\end{equation}
Combining these two inequalities, there exists a threshold $s_1>s>0$, for almost every $t \ge s_1$,  we have 
\begin{equation}
\label{eq: partial_r a large partial_r gamma^c}
    \| \tilde \partial_r a_t\| \ge  \GFmargin(\param_s) \rho_t^{M+1 /2}  \| \tilde \partial_r \GFmargin(\param_t)\|.
\end{equation}

\end{lemma}
\begin{proof}[Proof of \Cref{lem: Inequalities between a and gamma}]
We prove the three inequalities separately. 

\noindent{\bf Proof of \eqref{eq: lowerbound partial_r a}.}  We  expand the formula. For almost every  $t\ge s$, we have
\begin{align*}
\| \tilde  \partial_r a_t\| & = \langle \tilde  \partial a_t, \tilde{\param}_t \rangle \ge \frac{M a_t}{\rho_t} &&\explain{ By \Cref{lem:Decomposition of tilde subdifferential} } \\ 
& = M \GFmargin(\param_t) \rho_t^{M-1} &&\explain{ By definition of $\GFmargin(\param_t)$}\\ 
& \ge M \GFmargin(\param_s) \rho_t^{M-1}. &&\explain{ By monotonicity  in \Cref{thm: Margin improving and convergence}}  
\end{align*}

\noindent{\bf Proof of \eqref{eq: upperbound partial_r gamma^c}.}  
For almost every $t\ge s$, we have
\begin{align}
\label{eq: sphere-gamma^c-upperbound}
\| \tilde \partial_r \GFmargin(\param_t)\| 
= \langle \tilde \partial \GFmargin(\param_t), \tilde{\param}_t \rangle = - \frac{ \langle\bar \partial \Loss_t, \tilde \param_t  \rangle }{\Loss_t \rho_t^M} - \frac{\homop_a^\prime (\rho_t)}{\rho_t^M} - \frac{Ma_t}{\rho_t^{M+1}} = \frac{\langle \tilde \partial a_t, \param_t \rangle - M a_t }{\rho_t^{M+1}}.  
\end{align}
Now we derive an upper bound 
for $\langle \tilde \partial  a_t, \param_t \rangle$.
For almost every $t\ge s$, we have 
\begin{align*}
    \langle \tilde \partial a_t, \param_t \rangle  
    & = \frac{1}{ \Loss_t}  \bigg [\frac{1}{n}\sum_{i=1}^n e^{-\bar f_i(\param_t )} \langle \nabla \bar f_i (\param_t ), \param_t \rangle - \Loss_t \homop_a^\prime \big( \rho_t \big) \rho_t  \bigg]\\  
    & \le \frac{1}{ \Loss_t}\bigg [\frac{1}{n}\sum_{i=1}^n e^{-\bar f_i(\param_t )} M\bar f_i(\param_t) + \Loss_t \rho ^\prime (\rho_t) - \Loss_t \homop_a^\prime \big( \rho_t \big) \rho_t  \bigg] &&\explain{ By \Cref{asp:nearhomo} } \\
    & \le  \frac{1}{ \Loss_t} \bigg [\frac{1}{n}\sum_{i=1}^n e^{-\bar f_i(\param_t )} M\bar f_i(\param_t) + \Loss_t \rho ^\prime (\rho_t) + \Loss_t \homop_a^\prime \big( \rho_t \big) \rho_t  \bigg]  \\ 
    & = \frac{1}{ n\Loss_t}\sum_{i=1}^n e^{-\bar f_i(\param_t )} M\bar f_i(\param_t) +  M \homop_a(\rho_t).  &&\explain{ By \Cref{lem:property-pa} }. 
\end{align*}
By \Cref{lem:log-sum-convex}, we have $ -\nabla \pi (v) v = \nabla \pi(v) (0 -v) \le \pi(0) - \pi(v)$. 
Applying this to the above inequality, we have 
\[
    \langle \tilde \partial  a_t, \param_t  \rangle \le  M  \log(1 / \Loss_t) + M \homop_a(\rho_t). 
\]
Therefore, 
\[
    \langle \tilde \partial a_t, \param_t  \rangle - M a_t \le M \log(n) + 2M \homop_a(\rho_t).  
\]
Plugging this into \eqref{eq: sphere-gamma^c-upperbound}, we have proved  
\[
    \| \tilde \partial_r \GFmargin(\param_t)\| \le \frac{M \log n + 2M \homop_a(\rho_t)}{\rho_t^{M+1}}.
\]
This completes the proof of \eqref{eq: upperbound partial_r gamma^c}. 

\noindent{\bf Proof of \eqref{eq: partial_r a large partial_r gamma^c}.}   To prove \eqref{eq: partial_r a large partial_r gamma^c}, we combine \eqref{eq: lowerbound partial_r a} and \eqref{eq: upperbound partial_r gamma^c} to get
\[
    \| \tilde \partial_r a_t\| \ge M \GFmargin(\param_t) \rho_t^{M-1} \ge  \frac{\GFmargin(\param_t) \rho_t^{2M}}{ \log n + 2p_a(\rho_t)} \| \tilde \partial_r \GFmargin(\param_t)\|.
\]
Since $\homop_a(\rho_t)$ is a polynomial of $\rho_t$ with degree $(M-1)$ and $\rho_t$ is increasing with rate $O([\log(t)]^{\frac{1}{M}}) $ (by \Cref{thm: Margin improving and convergence}), 
there exists a threhold $s_1 >s>0$ such that for almost every $t\ge s_1$, we have 
\[
    \log n + 2 \homop_a(\rho_t) \le \rho_t^{M-\frac{1}{2}}. 
\]
Therefore for almost every $t\ge s_1$, we have
\[
    \| \tilde \partial_r a_t\| \ge  \GFmargin(\param_s) \rho_t^{M+1 /2}  \| \tilde \partial_r \GFmargin(\param_t)\|.
\]
This completes the proof of \eqref{eq: partial_r a large partial_r gamma^c}. 
We have proved \Cref{lem: Inequalities between a and gamma}.
\end{proof}

We proceed to prove \Cref{lem: Existence of desingularizing function}.  The following two Kurdyka-Lojasiewicz inequalities in \citet{ji2020directional} will help us construct the desired desingularizing function. We refer the readers to Appendix B in \citet{ji2020directional} for the proofs of them. 

\begin{lemma}[Lemma 3.6 in \citet{ji2020directional}]
    \label{lem:KL1}
 Given a locally Lipschitz definable function $f$ with an open domain $D \subset\{x \mid\|x\|>1\}$, for any $c, \eta>0$, there exists $\nu>0$ and a definable desingularizing function $\Psi$ on $[0, \nu)$ such that

$$
\Psi^{\prime}(f(x))\|x\|\|\bar{\partial} f(x)\| \geq 1, \quad \text { if } f(x) \in(0, \nu) \text { and }\left\|\bar{\partial}_{\perp} f(x)\right\| \geq c\|x\|^\eta\left\|\bar{\partial}_r f(x)\right\|. 
$$
\end{lemma}

\begin{lemma}[Lemma 3.7 in \citet{ji2020directional}] 
    \label{lem:KL2} 
 Given a locally Lipschitz definable function $f$ with an open domain $D \subset\{x \mid\|x\|>1\}$, for any $\lambda>0$, there exists $\nu>0$ and a definable desingularizing function $\Psi$ on $[0, \nu)$ such that

$$
\max \left\{1, \frac{2}{\lambda}\right\} \Psi^{\prime}(f(x))\|x\|^{1+\lambda}\|\bar{\partial} f(x)\| \geq 1, \quad \text { if } f(x) \in(0, \nu). 
$$
\end{lemma}

\begin{proof}
[Proof of \Cref{lem: Existence of desingularizing function}]

Recall that we have the following decomposition by \Cref{eq: decomp-gamma}: 
\[
\frac{\mathrm{d} \GFmargin(\param_t)}{ \mathrm{d} t} = \| \tilde  \partial_r \GFmargin(\param_t) \| \| \bar \partial_r \Loss_t\| + \| \tilde \partial_\perp \GFmargin(\param_t) \| \| \bar \partial_\perp \Loss_t\|.
\]
Two cases will be considered in this proof: 
\begin{itemize}
    \item Case 1: $\| \tilde  \partial_r \GFmargin(\param_t) \| \| \bar \partial_r \Loss_t\|$ is larger, and we will apply \Cref{lem:KL2} for construction. 
    \item Case 2: $\| \tilde \partial_\perp \GFmargin(\param_t) \| \| \bar \partial_\perp \Loss_t\|$ is larger, and we will apply \Cref{lem:KL1} for construction.
\end{itemize}
The two cases will be determined by the ratio of $\| \tilde \partial_\perp \GFmargin(\param_t) \|$ and $\| \tilde \partial_r \GFmargin(\param_t) \|$. Specifically, case 1 corresponds to
\begin{equation}
\label{eq: case1}
    \| \tilde \partial_\perp \GFmargin(\param_t) \| \le \rho_t^{\frac{1}{8}} \| \tilde \partial_r \GFmargin(\param_t) \|, 
\end{equation}
and case 2 corresponds to
\begin{equation}
\label{eq: case2}
\| \tilde \partial_\perp \GFmargin(\param_t) \| \ge  \rho_t^{\frac{1}{8}} \| \tilde \partial_r \GFmargin(\param_t) \|. 
\end{equation}

\noindent{\bf Case 1.} 
By \Cref{eq: partial_r a large partial_r gamma^c}, we have 
\begin{align}
    \label{eq: ratio-a-GF}
    \| \tilde \partial_r a_t \| 
    &\ge \GFmargin(\param_s) \rho_t^{M+\frac{1}{2}} \| \tilde \partial_r \GFmargin(\param_t) \|  \notag\\
    &\ge \GFmargin(\param_s) \rho_t^{M + \frac{3}{8}} \| \tilde \partial_\perp \GFmargin(\param_t)\| &&\explain{ By \Cref{eq: case1}} \notag \\ 
    &= \GFmargin(\param_s) \rho_t^{\frac{3}{8}} \| \tilde \partial_\perp a_t\|. &&\explain{ By \Cref{lem:decomp-radial-spherical}} 
\end{align}
Now we transfer this inequality to the ratio of $\| \bar \partial_\perp \Loss_t\|$ and $\| \bar \partial_r \Loss_t\|$.  Note that 
\begin{align}
\| \tilde \partial_r a_t \| 
& = \langle \tilde \partial a_t, \param_t \rangle 
=\frac{1}{\rho_t} \langle \tilde \partial a_t, \param_t \rangle 
= -\frac{\langle \bar  \partial \Loss_t, \param_t \rangle }{\Loss_t \rho_t} - \homop_a^\prime (\rho_t)  
\notag \\ 
& \le -\frac{\langle \bar \partial \Loss_t, \param_t \rangle }{\Loss_t \rho_t} 
= -\frac{\langle \bar \partial \Loss_t, \tilde \param_t \rangle }{\Loss_t} 
= \frac{1}{\Loss_t} \| \bar \partial_r \Loss_t\|. 
\label{eq:ratio-ra-rL}
\end{align}
On the other hand, we have 
\begin{equation}
\label{eq:ratio-ta-tL}
    \| \tilde \partial_\perp a_t \| = \frac{1}{\Loss_t} \| \bar \partial_\perp \Loss_t\|.
\end{equation}
Plugging \Cref{eq:ratio-ra-rL,eq:ratio-ta-tL} into \eqref{eq: ratio-a-GF}, we have 
\begin{equation}
\label{eq: ratio-L-GF}
        \| \bar \partial_r \Loss_t\| \ge \GFmargin(\param_s)\rho_t^{\frac{3}{8}} \| \bar \partial_\perp \Loss_t\|.
\end{equation}
On the other hand, we know that there exists $s_2>s>0$ such that for almost every $t\ge s_2$, we have $\rho_t >1$.  Hence we have 
\begin{align}
\label{eq: gamma_rad-GF}
\| \tilde  \partial \GFmargin (\param_t) \| 
&\le  \| \tilde \partial_r \GFmargin(\param_t) \| + \| \tilde \partial_\perp \GFmargin(\param_t) \|
\notag\\ 
& \le (1+\rho_t^{\frac{1}{8}} ) \| \tilde \partial_r \GFmargin(\param_t) \|
&&\explain{ By \Cref{eq: case1}}  
\notag\\   
& \le 2 \rho_t^{\frac{1}{8}} \| \tilde \partial_r \GFmargin(\param_t) \|.
&&\explain{ By $\rho_t \ge 1$ } 
\end{align}
Therefore by \eqref{eq: decomp-gamma}, we have 
\begin{align}
\label{eq: psi1_key_gf}
\frac{\mathrm{d}  \GFmargin(\param_t)}{\mathrm{d}  t} 
&\ge \| \tilde \partial_r \GFmargin(\param_t) \| \| \bar \partial_r \Loss_t\| \notag
\\ 
&\ge \frac{1}{2} \rho_t^{\frac{1}{4}}  \GFmargin(\param_s) \| \tilde  \partial \GFmargin(\param_t) \| \| \bar  \partial_\perp \Loss_t\|  
&&\explain{ By \Cref{eq: ratio-L-GF} and \eqref{eq: gamma_rad-GF}} 
\notag
\\
&\ge \frac{1}{2} \rho_t^{\frac{1}{4}}  \GFmargin(\param_s) \| \bar   \partial \GFmargin(\param_t) \| \| \bar  \partial_\perp \Loss_t\|
&&\explain{ By $ \|\tilde  \partial \GFmargin(\param_t)\| \ge \|\bar   \partial \GFmargin(\param_t)\| $ } 
\notag
\\
&\ge  \rho_t^{\frac{5}{4}} \| \bar \partial \GFmargin(\param_t) \| \frac{\mathrm{d}  \zeta_t}{\mathrm{d}  t}. &&\explain{ By \Cref{lem:decomp-radial-spherical}} 
\end{align}
Now we invoke \Cref{lem:KL2} to construct the desingularizing function.  We apply it to the definable function $\gamma_* - \GFmargin(\param_t)$ with $\lambda=\frac{1}{4}$. Then there exists $\nu_1 >0$ and a definable desingularizing function $\Psi_1$ on $[0, \nu_1)$ such that 
\[
    8 \Psi_1^\prime(\gamma_* - \GFmargin(\param_t)) \rho_t^{5/4} \|\bar \partial \GFmargin(\param_t)\| \ge 1, \quad \text{ if }  \GFmargin(\param_t) \ge \gamma_* - \nu_1.
\] 
Plugging \eqref{eq: psi1_key_gf} into the above inequality, we have 
\[
    8 \Psi_1^\prime(\gamma_* - \GFmargin(\param_t))  \frac{\mathrm{d} \GFmargin(\param_t)}{\mathrm{d}  t} \ge \frac{\mathrm{d}  \zeta_t}{\mathrm{d}  t}, \quad \text{ if }  \GFmargin(\param_t) \ge \gamma_* - \nu_1.
\]
This completes the proof for case 1. 

\noindent{\bf Case 2.}
By \Cref{eq: decomp-gamma}, we have 
\begin{equation}
\label{eq: case2-key1-gf}
       \frac{\mathrm{d} \GFmargin(\param_t)}{\mathrm{d}  t} \ge  \| \tilde \partial_\perp \GFmargin(\param_t) \| \| \bar  \partial_\perp \Loss_t\| =\rho_t \|  \tilde \partial_\perp \GFmargin(\param_t) \| \frac{\mathrm{d}  \zeta_t}{\mathrm{d} t}. 
\end{equation}
For almost every $t\ge s_1 >s >0$, we have $\rho_t \ge 1$ and 
\[
    \| \tilde \partial_\perp \GFmargin(\param_t) \| \ge  \rho_t^{\frac{1}{8}} \| \tilde \partial_r \GFmargin(\param_t) \| \ge \| \tilde \partial_r \GFmargin(\param_t) \|.
\]
This leads to 
\[
    \| \tilde \partial_\perp \GFmargin(\param_t) \| \ge \frac{1}{2} ( \| \tilde \partial_\perp \GFmargin(\param_t) \| + \| \tilde \partial_r\GFmargin(\param_t) \|) \ge \frac{1}{2} \| \tilde \partial \GFmargin(\param_t)\|. 
\]
Plugging this into \eqref{eq: case2-key1-gf}, we have 
\begin{equation}
\label{eq: case2-key2}
\frac{\mathrm{d} \GFmargin(\param_t)}{\mathrm{d}  t} \ge  \frac{\rho_t}{2} \| \tilde \partial  \GFmargin(\param_t)\| \frac{\mathrm{d} \zeta_t}{\mathrm{d} t} \ge \frac{\rho_t}{2} \| \bar  \partial  \GFmargin(\param_t)\| \frac{\mathrm{d} \zeta_t}{\mathrm{d} t}.
\end{equation}
We invoke \Cref{lem:KL1} to construct the desingularizing function.  We apply it to the definable function $\gamma_* - \GFmargin(\param_t)$ with $c=1$ and $\eta=\frac{1}{8}$. Then there exists $\nu_2 >0$ and a definable desingularizing function $\Psi_2$ on $[0, \nu_2)$ such that
\[
    \Psi_2^\prime(\gamma_* - \GFmargin(\param_t)) \rho_t \|\bar \partial \GFmargin(\param_t)\| \ge 1, \quad \text{ if }  \GFmargin(\param_t) \ge \gamma_* - \nu_2.
\]
Plugging \eqref{eq: case2-key2} into the above inequality, we have 
\[
    2\Psi_2^\prime(\gamma_* - \GFmargin(\param_t)) \frac{\mathrm{d} \GFmargin(\param_t)}{\mathrm{d}  t} \ge \frac{\mathrm{d} \zeta_t}{\mathrm{d}  t}, \quad \text{ if }  \GFmargin(\param_t) \ge \gamma_* - \nu_2.
\]
This completes the proof for case 2. 

\noindent{\bf Combining the two cases.}
Since $\Psi_1^\prime  - \Psi_2^\prime$ is a definable function, it is nonnegative or nonpositive on some interval $(0, \nu)$. Let $\Psi = \max \{\Psi_1, \Psi_2\}$. Then we have for almost every large enough $t$ such that $\rho_t \ge 1$ and $\GFmargin(\param_t) \ge \gamma_* - \nu$, and $\log n + 2p_a(\rho_t) \le \rho_t^{M-\frac{1}{2}}$, it holds that 
\[
    \frac{\mathrm{d}  \GFmargin(\param_t)}{\mathrm{d} t} \ge \frac{1}{c \Psi' (\gamma_* - \GFmargin(\param_t))} \frac{\mathrm{d}  \zeta_t}{\mathrm{d}  t},
\]
for some constant $c>0$. This completes the proof of \Cref{lem: Existence of desingularizing function}.
\end{proof}

\subsection{Proof of Theorem \ref{thm: directional convergence}}
\label{sec:direct-conv:direct-proof}

Now we are ready to prove \Cref{thm: directional convergence} with the powerful tool, \Cref{lem: Existence of desingularizing function}. 
\begin{proof}[Proof of \Cref{thm: directional convergence}]
Let $t_0$ be large enough so that the condition in \Cref{lem: Existence of desingularizing function} holds. Then we have  
\[
    \lim_{t\to \infty} \zeta_t \le \zeta_{t_0} + c \Psi( \gamma_* - \GFmargin(\param_{t_0})) \le \infty.
\]
This completes the proof of \Cref{thm: directional convergence}.
\end{proof}

\subsection{KKT convergence} \label{sec:proof:KKT}

We now prove \Cref{thm: KKT convergence} by verifying the KKT conditions. Recall that the optimization problem \Cref{eq: KKT} is defined as follows: 
\[
    \min  \| \param\|_2^2 \quad \text{ s.t. } y_i \homoPredictor (\param;\xB_i) \ge 1 \text{ for all } i\in [n]. 
\]
Recall that $\bar f_{i}(\param) = y_i f(\param;\xB_i)$ and $\bar f_{\min}(\param) = \min_{i\in [n]} \bar f_{i}(\param)$.  
We use $\bar f_{\homo,i}(\param)$ to denote $y_i f_\homo(\param;\xB_i)$ and   $\bar f_{\homo,\min}$ to denote $\min_{i\in [n]} \bar f_{\homo,i}(\param)$. Recall that by \Cref{thm:homogenization}, $\bar f_{\homo,i}$ is continuously differentiable on $\Rbb^d / \{0\}$. We also assume that $f$ is continuous differentiable with respect to $\param$ on $\Rbb^d / \{0\}$ in \Cref{asp:strongerhomo}. 
For simplicity, we use $\hB_i$ to denote $\nabla \bar f_i$ and $\hB_{\homo,i}$ to denote $\nabla \bar  f_{\homo,i}$. 
We define
\[
    \hB = \frac{1}{n} \sum_{i=1}^n e^{- \bar f_{i}(\param_t)}\hB_i, \quad \hB_{\homo} = \frac{1}{n} \sum_{i=1}^n e^{- \bar f_{i}(\param_t)}\hB_{\homo,i}.
\]
Notice that $\hB$ and $\hB_{\homo}$ are weighted sum of $\hB_i$'s and $\hB_{\homo,i}$'s respectively, where the weights are the same. 
Actually, this $\hB_{\homo}$ is very tricky. Unlike $\hB = - \nabla \Loss_t$, $\hB_{\homo}$ is not the negative version of gradient of 
\[
\Loss_{\homo}(\param) \coloneqq \frac{1}{n} \sum_{i=1}^n e^{- \bar f_{\homo,i}(\param;\xB_i)}. 
\] 
It is just a weighted sum of  $\nabla \bar f_{\homo,i}$. Next, we define the KKT conditions and approximate KKT conditions for the optimization problem \Cref{eq: KKT}. 

\begin{definition}
[KKT conditions]
\label{def: KKT point of (P)}
A feasible point $\param$ of \Cref{eq: KKT} is a KKT point if there exist $\lambda_1,\ldots ,\lambda_n \ge 0$ such that 
\begin{enumerate}
    \item  $\param- \sum_{i=1}^n \lambda_i \hB_{\homo,i} =0$ for some $\hB_{\homo,1},\ldots ,\hB_{\homo,n}$ satisfying $\hB_{\homo,i}=  \nabla \bar f_{\homo,i}(\param)$ for every $i\in [n]$;
    \item For every $i\in [n]$, $\lambda_i \big(\bar f_{\homo,i}(\param) -1 \big) =0$. 
\end{enumerate}
\end{definition}

\begin{definition}
[Approximate KKT conditions]
\label{def: approx KKT point of (P)}
A feasible point $\param$ of \Cref{eq: KKT} is an $(\epsilon, \delta)$-KKT point if ther exists $\lambda_1,\ldots ,\lambda_n \ge 0$ such that 
\begin{enumerate}
    \item  $\|\param- \sum_{i=1}^n \lambda_i \hB_{\homo,i} \| \le \epsilon $ for some $\hB_{\homo,1},\ldots ,\hB_{\homo,n}$ satisfying $\hB_{\homo,i}= \nabla \bar f_{\homo,i}(\param)$ for every $i\in [n]$;
    \item For every $i\in [n]$, $\lambda_i  \big | \bar f_{\homo,i}(\param) -1 \big | \le \delta $. 
\end{enumerate}
\end{definition}
The first step is to bound the difference between $\bar f_{\min}$ and $ \bar f_{\homo,\min}$.

\begin{lemma}
[Bound between $ \bar f_{\min}$ and $ \bar f_{\homo,\min}$]
\label{lem:bound-of-two-min}
Under \Cref{asp:nearhomo,asp:initial-cond-gf}, for any $\epsilon$, there exists $s_2>s>0$ such that 
\[
\text{for almost every $t\ge s_2$,}\quad 
    \Bigg \| \frac{\param_t}{\big(\bar f_{\min}(\param_t)\big)^{1 / M}} - \frac{\param_t}{\big(\bar f_{\homo,\min}(\param_t)\big)^{1 / M}} \Bigg\| \le \epsilon.
\]
\end{lemma}
\begin{proof}[Proof of \Cref{lem:bound-of-two-min}]
By \Cref{thm: directional convergence}, we have $\tilde{\param}_t $ converges to $\param_*$.  Then we are going to show that 
$ \bar f_{i}( \param_t) / \rho_t^M$ converges for all $i$. Recall that in \Cref{thm:homogenization}, we have 
\[
    |f(\param; \xB) - f_\homo(\param_t;\xB) | \le \homop_a(\|\param_t\|). 
\]
Therefore we have 
\[
    \bigg | \frac{\bar f_{i}(\param_t )}{\rho_t^M} - \frac{\bar f_{\homo,i}(\param_t )}{\rho_t^M} \bigg | \le \frac{\homop_a(\rho_t)}{\rho_t^M} \to 0, \quad  \text{ as } t\to \infty.
\]
Hence we have 
\[
    \lim_{t\to \infty} \frac{\bar f_{i}(\param_t )}{\rho_t^M} = \lim_{t\to \infty} \frac{\bar f_{\homo,i}(\param_t )}{\rho_t^M} = \bar f_{\homo,i}(\param_*).
\]
As a consequence, we have 
\[
    \lim_{t\to \infty} \frac{\rho_t}{\big(\bar f_{i}(\param_t ) \big)^{1 / M}} = \lim_{t\to \infty} \frac{\rho_t}{\big(\bar f_{\homo,i}(\param_t ) \big)^{1 / M}} = \frac{1}{\big(\bar f_{\homo,i}(\param_*)\big)^{1 / M}}.
\]
Because we only have a finite number of $i$, we must have 
\[
    \lim_{t\to \infty} \frac{\rho_t}{\big(\bar f_{\min}(\param_t ) \big)^{1 / M}} = \lim_{t\to \infty} \frac{\rho_t}{\big(\bar f_{\homo,\min}(\param_t ) \big)^{1 / M}} = \frac{1}{\big(\bar f_{\homo,\min}(\param_*)\big)^{1 / M}}.
\]
Therefore, for every $\epsilon$, there exists $s_2>s>0$ such that for almost every $t\ge s_2$, we have 
\[
    \Bigg \| \frac{\param_t}{\big(\bar f_{\min}(\param_t)\big)^{1 / M}} - \frac{\param_t}{\big(\bar f_{\homo,\min}(\param_t)\big)^{1 / M}} \Bigg\| \le \epsilon.
\]
This completes the proof of \Cref{lem:bound-of-two-min}.
\end{proof}

Due to the above lemma, we can focus on 
\[\hat \param_t \coloneqq \param_t  / \big(\bar f_{\homo,\min}(\param_t)\big)^{1 / M}\] 
and verify it satisfies the $(\epsilon, \delta)$-KKT conditions. The key is to construct the coefficients $\lambda_i$: 
\begin{equation}
\label{eq: lambda_i}
    \lambda_i(\param_t ) = \frac{\bar f_{\homo,\min}^{1-2 / M} (\param_t) \rho_t e^{-\bar f_i(\param_t )}}{n \|\hB_{\homo}\|}. 
\end{equation}
We define the following key quantity for our proof,
\begin{equation}
\label{eq: beta}
    \beta(t) = \frac{\langle \param_t , \hB_{\homo} (\param_t ) \rangle }{\| \param_t \|\cdot \|\hB_{\homo}\|}.
\end{equation}

\begin{lemma}[Lower bound of the homogeneous margin]
\label{lem:lowerbound-homo-margin}
Under \Cref{asp:nearhomo,asp:initial-cond-gf}, we have 
\[
\frac{f_{\homo,\min} (\param_t)}{\|\param_t\|^M} \ge \GFmargin(\param_s) >0.  
\]  
Specifically, there exists $B$, such that for all $t$,
\begin{equation}
\label{eq:f_M-upperbound}
\bar f_{\homo,i}^{-1 / M} (\param_t) \cdot \|\param_t\| \le B.
\end{equation}
\end{lemma}
\begin{proof}
Note that 
\begin{align*}
f_{\homo,\min} (\param_t)
&\ge f_{\min} (\param_t) - p_a(\rho_t)
&&\explain{ by \Cref{thm:homogenization}}
\\ 
&\ge \log \frac{1}{\Loss_t n} - p_a(\rho_t). 
&&\explain{ by \Cref{lem:f-min}} 
\end{align*}
Hence, 
\[
\frac{f_{\homo,\min} (\param_t)}{\|\param_t\|^M} \ge \GFmargin(\param_t)  \ge \GFmargin(\param_s)>0.  
\]  
And, 
\[
    \bar f_{\homo,i}^{-1 / M} (\param_t) \cdot \|\param_t\| \le \big(\GFmargin(\param_s) \big)^{-\frac{1}{M}} \eqqcolon B. 
\]
This completes the proof of \Cref{lem:lowerbound-homo-margin}. 
\end{proof}

\begin{lemma}
[$\tilde \param$ satisfies the $(\epsilon, \delta)$-KKT conditions]
\label{lem: tilde theta KKT condition}
Under \Cref{asp:nearhomo,asp:initial-cond-gf}, there exists $s_2>s>0$ such that for a.e. $t\ge s_2$, $\hat \param_t $ is an $(\epsilon, \delta)$-KKT point of \Cref{eq: KKT}, where 
\[
    \epsilon = \sqrt{2}B \sqrt{1-\beta(t)} \quad  \text{ and } \quad \delta = n B^2 \frac{1+2 \homop_a(\rho_t)}{M \bar f_{\homo,\min }(\param_t)}.
\]
\end{lemma}
\begin{proof}[Proof of \Cref{lem: tilde theta KKT condition}]
We need to verify:
\begin{enumerate}
    \item $\|\hat \param- \sum_{i=1}^n \lambda_i \hB_{\homo,i} \| \le \epsilon $ for some $\hB_{\homo,1},\ldots ,\hB_{\homo,n}$ satisfying $\hB_{\homo,i}=\grad \bar f_{\homo,i}(\hat\param)$ for all $i\in [n]$.
    \item For any $i\in [n]$, $\lambda_i \big| \bar f_{\homo,i}(\hat \param) -1 \big| \le \delta $. 
\end{enumerate}
Note that we only verify one side of condition 2, since by the definition of $ \hat \param $, we have $\bar f_{\homo,i}(\hat \param) \ge 1$ for all $i$.

\vspace{1em}
\noindent{\bf Condition 1.}  We expand the left side of the inequality: 
\begin{align*}
    \hat \param_t - \sum_{i=1}^n \lambda_i \hB_{\homo,i}
    &= \bar f_{\homo,\min}^{-1 / M} (\param_t ) \rho_t\cdot \tilde{\param}_t  - \bar f_{\homo,\min}^{-1 / M}(\param_t ) \cdot \frac{1}{n} \sum_{i=1}^{n}  \frac{\bar f_{\homo,\min}^{1-1 / M}(\param_t ) \rho_t e^{-\bar f_i(\param_t )}}{ \|\hB_{\homo}\|}  \nabla \bar f_{\homo,i}(\hat \param_t )  \\ 
    & =  \bar f_{\homo,\min}^{-1 / M}(\param_t ) \rho_t\cdot \tilde{\param_t} - \bar f_{\homo,\min}^{-1 / M}(\param_t )\rho_t \cdot \frac{1}{n} \sum_{i=1}^{n}  \frac{  e^{-\bar f_i(\param_t )}}{ \|\hB_{\homo}\|}  \nabla \bar f_{\homo,i}(\param_t ) \\ 
    & = \bar f_{\homo,\min}^{-1 / M}(\param_t ) \rho_t\cdot \bigg ( \tilde{\param}_t  - \frac{\hB_{\homo}}{\|\hB_{\homo}\|} \bigg),
\end{align*}
where the second equation is by the definition of $\hat \param$ and the homogeneity of $\bar f_{\homo,i}$, and the last equation is by the definition of $\hB_{\homo}$.
Hence we have 
\[
    \bigg\| \hat \param_t - \sum_{i=1}^n \lambda_i \hB_{\homo,i} \bigg\|^2 \le B^2 \bigg \|\tilde{\param}_t  - \frac{\hB_{\homo}}{\|\hB_{\homo}\|} \bigg\|^2  = 2B^2 (1-\beta(t)).
\]
This completes the proof of Condition 1. 

\noindent{\bf Condition 2.}  For condition 2, we have shown that $\Loss_t \le e^{-\homop_a(\rho_t)}/n$. This leads to 
\[\bar f_i(\param_t)> \homop_a(\rho_t)\quad \Rightarrow\quad \bar f_{\homo,i}(\param_t) > \bar f_i(\param_t) - \homop_a(\rho_t) \ge 0.\] 
Therefore we have $\bar f_{\homo,i}(\hat \param_t ) \ge 1$ since $\hat \param_t \coloneqq \param_t  / \big(\bar f_{\homo,\min} (\param_t )\big)^{1 / M}$. It suffices to show $\sum \lambda_i \big( f_{\homo,i}(\hat \param_t ) -1 \big) \le \delta$. This is because
\begin{align}
    \label{eq: KKT cond2 verify}
    \sum_{i=1}^n \lambda_i \big( f_{\homo,i}(\hat \param_t ) -1 \big) 
    & =  \frac{\rho_t \bar f_{\homo,\min}^{-2 /M}(\param_t )}{\| \hB_{\homo}\|} \cdot \frac{1}{n}  \sum_{i=1}^n e^{-\bar f_i(\param_t )} \big(\bar f_{\homo,i}(\param_t ) - \bar f_{\homo,\min}(\param_t ) \big)  \notag\\ 
    & = \frac{\rho_t \bar f_{\homo,\min}^{-2 /M}(\param_t )}{\| \hB_{\homo}\| e^{ \bar f_{\min}(\param_t)}}\cdot \frac{1}{n}  \sum_{i=1}^n e^{-\big(\bar f_i(\param_t)- \bar f_{\min}(\param_t) \big)} \big(\bar f_{\homo,i}(\param_t) - \bar f_{\homo,\min}(\param_t) \big)  \notag\\ 
    &\le \frac{\rho_t \bar f_{\homo,\min}^{-2 /M}(\param_t )}{\| \hB_{\homo}\| e^{ \bar f_{\min}(\param_t )}}\cdot \frac{1}{n}  \sum_{i=1}^n e^{-\big(\bar f_{i}(\param_t)- \bar f_{\min}(\param_t) \big) } \big(\bar f_{i}(\param_t) - \bar f_{\min}(\param_t) + 2\homop_a(\rho_t) \big)\notag\\ 
    & \qquad \explain{ By $ | \bar f_{\homo,i} (\param_t) - \bar f_{i} (\param_t)| \le \homop_a(\rho_t)$ in  \Cref{thm:homogenization}} \notag\\ 
    &\le \frac{\rho_t \bar f_{\homo,\min}^{-2 /M}(\param_t )}{\| \hB_{\homo}\| e^{ \bar f_{\min}(\param_t)}}\cdot  \big(1+ 2\homop_a(\rho_t) \big). \\ 
    &\qquad \explain{ By $\bar f_i(\param_t) - \bar f_{\min}(\param_t) \ge 0$ and $e^{-x}x \le 1$ when $x>0$} \notag 
\end{align}
The remaining part is to lower bound $\| \hB_{\homo}\|$. We have 
\begin{align*}
\|\hB_{\homo}\| &\ge \frac{\langle \hB_{\homo}, \param_t\rangle }{\|\param_t\|} \\
&= \frac{M}{n \|\param_t\|} \sum_{i=1}^{n} e^{-\bar f_i(\param_t)} \bar f_{\homo,i} (\param_t)\\
&\ge  \frac{M e^{-\bar f_{\min}(\param_t)} \bar f_{\homo,\min }(\param_t) }{n \rho_t}, &&\explain{ by $\bar f_{\homo,i}(\param) \ge 0$ and \Cref{lem:bound-of-two-min} } 
\end{align*}
Plugging this into \eqref{eq: KKT cond2 verify}, we get 
\begin{align*}
\sum_{i=1}^n \lambda_i \big( f_{\homo,i}(\tilde \param_t) -1 \big)  
&\le \frac{n}{M}\cdot \rho_t^2 \bar f_{\homo,\min}^{-2 /M}(\param_t)\cdot \frac{ 1+ 2\homop_a(\rho_t) }{ \bar f_{\homo,\min }(\param_t)} 
&&\explain{ by \Cref{eq:f_M-upperbound} } \\ 
&\le \frac{nB^2}{M} \cdot  \frac{1+ 2 \homop_a(\rho_t)}{ \bar f_{\homo,\min }(\param_t)}. 
\end{align*}
This completes the proof of \Cref{lem: tilde theta KKT condition}.
\end{proof}

It remains to show that $\epsilon$ and $\delta$ are both close to zero.
The second is easy to verify because $f_{\homo,\min}$ is M-homogeneous.   

\begin{lemma}
[$\delta$ goes to $0$]
\label{lem: q goes to 0}
Under \Cref{asp:nearhomo,asp:initial-cond-gf}, we have  
\[
    \delta =n B^2 \frac{1+2p_a(\rho_t)}{M \bar f_{\homo,\min }(\param_t)} \to 0, \quad \text{ as } t \to \infty. 
\]
\end{lemma}
\begin{proof}[Proof of \Cref{lem: q goes to 0}]
Recall that 
\[
    \lim_{t\to \infty} \frac{\bar f_{\homo,i}(\param_t)}{\rho_t^M} = \bar f_{\homo,i}(\param_*).
\]
Denote the limit as $\gamma_*$. Then there exists $s_3 >s>0$ such that for almost every $t\ge s_3$, we have 
\[
     \frac{\bar f_{\homo,i}(\param_t)}{\rho_t^M} \ge \gamma_*/2 \implies \bar f_{\homo,i}(\param_t) \ge \gamma_* \rho_t^M/2.
\]
Then we have 
\[
    \lim_{t\to \infty}\delta \le \lim_{\rho_t \to \infty} nB^2 \frac{1+ 2\homop_a(\rho_t)}{ M \gamma_* \rho_t^M/2} = 0. 
\]
This completes the proof of \Cref{lem: q goes to 0}.
\end{proof}


We next show $\epsilon = \Ocal(\sqrt{1-\beta(t)})$ is close to zero. 
However, $\beta(t)$ might not converge as it is given by
\[
    \beta(t) = \frac{\langle \param_t , \hB_{\homo} \rangle }{\| \param_t \|\cdot \|\hB_{\homo}\|}, \quad \hB_{\homo} = \frac{1}{n} \sum_{i=1}^n e^{- \bar f_{i}(\param_t)}\nabla \bar f_{\homo,i}(\param_t),
\]
where 
$\hB_{\homo}$ might not converge in direction.
To address this issue,
our strategy is to show that there exists a subsequence $(t_k)_{k\ge 0}$ such that $\beta(t_k) \to 1$. 
To this end, we need to characterize $\beta(t)$ using the stronger \Cref{asp:strongerhomo}.
By \Cref{asp:strongerhomo} and \Cref{thm:homogenization}, there is a function $\homor(x) = o(x^{M-1})$ as $x \to \infty$, such that for almost every $\param_t$ and any $i\in [n]$, we have
\[
    \| \nabla \bar f_{i}(\param_t) - \nabla \bar f_{\homo,i}(\param_t)\| \le \homor(\|\param_t\|) = \homor(\rho_t) .
\]

\begin{lemma}
[Characterization of $\beta(t)$]
\label{lem: Characterization of beta}
Under \Cref{asp:strongerhomo,asp:initial-cond-gf}
, there exists $s_5>s>0$ such that for almost every $t_2>t_1\ge s_5$, we have 
\[
    \int_{t_1}^{t_2}  \Bigg( \frac{1 - p_1(t)}{\big(\beta(t) + p_2(t) \big)^2}-1\Bigg) \cdot \frac{\mathrm{d} }{\mathrm{d} \tau} \log \rho(\tau) \cdot \mathrm{d} \tau \le  \frac{1}{M} \log \frac{\GFmargin(\param_{t_2})}{\GFmargin(\param_{t_1})}, 
\]
where 
\[
    p_1(t) := \frac{2 \homor(\rho_t)}{ M\GFmargin(\param_s) \rho^{M-1}_t}, \quad p_2(t) := \frac{2 \homop_a(\rho_t)}{\GFmargin(\param_s) \rho^{M}_t}.
\]
\end{lemma}

\begin{proof}[Proof of \Cref{lem: Characterization of beta}]
Recall that we have \eqref{eq: gamma-c-increase}, that is,
\begin{align*}
\frac{\mathrm{d}  \log \GFmargin(\param_t)}{\mathrm{d}  t} 
&\ge \frac{\| \param_t ^\prime \|^2 \rho_t^2 - \langle  \param_t^\prime , \param_t \rangle^2 }{\rho_t^2\Loss_t (\LinkFun(\Loss_t) - \homop_a(\rho_t))} 
\\
&\ge \frac{\| \hB \|^2 \rho_t^2 - \langle  \hB , \param_t \rangle^2 }{\rho_t^2\Loss_t (\LinkFun(\Loss_t) - \homop_a(\rho_t))} 
\\
&= \frac{2\| \hB \|^2  - 2\langle  \hB , \tilde{\param}_t \rangle^2 }{\Loss_t a_t \cdot \frac{\mathrm{d}\rho_t^2}{\mathrm{d} t} / \rho_t^2 } \cdot\frac{\mathrm{d}  \log \rho_t}{\mathrm{d} t}. && \explain{ By $\frac{\mathrm{d}  \log \rho_t}{\mathrm{d} t} = \frac{\mathrm{d}  \rho_t^2}{2\rho_t^2 \mathrm{d} t}$ } 
\end{align*}
Note that for  almost every $t\ge s$, 
\begin{align}
    \label{eq: beta-bound-denom-1}
    \frac{1}{2} \frac{\mathrm{d} \rho_t^2}{\mathrm{d}  t} 
    &= \frac{1}{n} \sum_{i=1}^n e^{-\bar f_i(\param_t)} \langle \nabla \bar f_i(\param_t), \param_t \rangle \notag\\ 
    &\le  \frac{1}{n} \sum_{i=1}^n e^{-\bar f_i(\param_t)} \big[ M \bar f_i(\param_t) + \homop^\prime (\|\param\|) \big] \notag\\ 
    & \qquad \explain{ By \Cref{asp:nearhomo} } \notag\\ 
    & \le \frac{1}{n} \sum_{i=1}^n e^{-\bar f_i(\param_t)} \big[ M \bar f_{\homo,i}(\param_t) + M \homop_{a}(\|\param_t\|) +  \homop^\prime (\|\param\|) \big] \notag\\ 
    & \qquad \explain{ By $|f(\param; \xB) - f_\homo(\param_t;\xB) | \le \homop_a(\|\param_t\|)$} \notag\\
    & \le \frac{1}{n} \sum_{i=1}^n e^{-\bar f_i(\param_t)} \langle \nabla \bar f_{\homo,i}(\param_t), \param_t \rangle  + 2M\Loss_t  \homop_a(\rho_t) \notag \\ 
    &\qquad \explain{ By $M \homop_a(x) \ge \homop^\prime (x)$ for $x>0$} \notag\\
    & = \langle \hB_{\homo}, \param_t \rangle + 2M\Loss_t \homop_a(\rho_t). 
\end{align}
Similarly, we have 
\begin{align}
    \label{eq: beta-bound-denom-2}
    \Loss_t a_t &= \Loss_t \LinkFun(\Loss_t) - \homop_a(\rho_t) \Loss_t  \notag\\ 
    &\le \Loss_t \LinkFun(\Loss_t) - \homop^\prime (\rho_t) \Loss_t /M \notag\\ 
    &\qquad \explain{ By $Mp_a(x) \ge p^\prime (x)$ for $x>0$} \notag\\
    &\le \frac{1}{2M} \frac{\mathrm{d} \rho_t^2}{\mathrm{d}  t} \notag \\ 
    &\qquad \explain{ By \Cref{lem:NH-rho-increase}} \notag\\
    & \le \frac{1}{M}\langle \hB_{\homo}, \param_t \rangle + 2\Loss_t \homop_a(\rho_t).
\end{align}
The remaining part is to bound the term $\|\hB\|^2 = \| \param_t^\prime  \|^2$ in the numerator.  Recall that 
\[
    \| \nabla f_i(\param_t) - \nabla f_{\homo,i}(\param_t)\| \le\homor(\|\param_t\|)  \quad \text{ for all } i\in [n].
\]
Hence we have 
\[
    \|\hB - \hB_{\homo} \| \le \frac{1}{n} \sum_{i=1}^n e^{-\bar f_i(\param_t)} \| \hB_i - \hB_{\homo,i} \| \le \Loss_t\homor(\|\param_t\|).
\]
We get 
\[
    \|\hB\| \ge \|\hB_{\homo}\| - \| \hB - \hB_{\homo}\|  \ge \|\hB_{\homo}\| - \Loss_t\homor(\|\param_t\|).
\] 
Since $\|\param_t\| \to \infty$, for all sufficiently large $t$, we have 
\begin{align}
\| \hB_{\homo}\|
&\ge \frac{\langle \hB_{\homo}, \param_t \rangle }{\rho_t} 
\\
&= \frac{M\sum_{i=1}^{n} e^{-\bar f_i(\param_t)} \bar f_{\homo,i}(\param_t) }{n \rho_t}  \notag\\ 
& \ge  \frac{M \sum_{i=1}^{n} e^{-\bar f_i(\param_t)} \big (\bar f_i(\param_t) - \homop_a(\rho_t) \big) }{n \rho_t} 
&&\explain{ by \Cref{thm:homogenization}} 
\notag\\ 
& \ge \frac{M\Loss_t  \Big(\log \frac{1}{n \Loss_t} - p_a(\rho_t)\Big)}{\rho_t} 
&&\explain{ by \Cref{lem:f-min} } 
\\ 
&= M\Loss_t\GFmargin(\param_t) \rho_t^{M-1} 
&& \explain{By the definition of $\GFmargin$} 
\notag\\ 
&\ge M\Loss_t \GFmargin(\param_s) \rho_t^{M-1}  \label{eq: h_M lb} && \explain{by \Cref{thm: Margin improving and convergence}} \\ 
& \ge 3M\Loss_t\homor(\rho) \notag. &&\explain{ Since the degree of $r$ is $M-2$ } 
\end{align}
Therefore, there exists $s_5 > s>0$ such that for almost every $t\ge s_5$, we have 
\begin{align}
\label{eq: beta-bound-num1}
\|\hB\|_2^2 
&\ge  \|\hB_{\homo}\|_2^2 - 2\|\hB_{\homo}\| \Loss_t\homor(\|\param_t\|) + \Loss_t^2\homor(\|\param_t\|)^2 
\notag \\ 
&\ge \|\hB_{\homo}\|_2^2 - 2\|\hB_{\homo}\| \Loss_t\homor(\rho_t) \\ 
&\ge \|\hB_{\homo}\| ( 3M-2) \Loss_t\homor(\rho_t)  \ge 0. \notag
&&\explain{  by \Cref{eq: h_M lb} }
\end{align}
Moreover, we have 
\begin{align}
\label{eq: beta-bound-num2}
\langle  \hB , \param_t \rangle  
= \frac{1}{\rho_t} \frac{1}{2} \frac{\mathrm{d}  \rho_t^2}{\kd t}
\le  \langle \hB_{\homo},  \tilde \param_t
\rangle + 2M\Loss_t \homop_a(\rho_t) / \rho_t.
&&\explain{ by \Cref{eq: beta-bound-denom-2} } 
\end{align}

Combining \Cref{eq: beta-bound-denom-1,eq: beta-bound-denom-2,eq: beta-bound-num1,eq: beta-bound-num2}, when $t\ge s_5$, we have 
\begin{align*}
    \frac{2\|\hB\|^2  - 2\langle \hB, \tilde{\param}_t \rangle^2 }{\Loss_t a_t \cdot \frac{\mathrm{d}\rho_t^2}{\mathrm{d} t} / \rho_t^2 } &\ge M \frac{\|\hB_{\homo}\|_2^2 - 2\|\hB_{\homo}\| \Loss_t\homor(\rho_t) -\big(\langle \hB_{\homo}, \tilde \param_t \rangle + 2M\Loss_t \homop_a(\rho_t) / \rho_t \big)^2}{\big(\langle \hB_{\homo}, \tilde \param_t \rangle + 2M\Loss_t \homop_a(\rho_t) / \rho_t \big)^2}\\ 
    &= M \cdot \bigg( \frac{1- 2\|\hB_{\homo}\|^{-1} \Loss_t\homor(\rho_t) }{\big(\beta(t) + 2M\Loss_t \homop_a(\rho_t)/(\rho_t \|\hB_{\homo}\|)\big)^2} -1\bigg)\\ 
    & \ge M \cdot \Bigg(\frac{1 - 2r(\rho_t)/ \big(M\GFmargin(\param_s) \rho^{M-1}_t \big)}{\Big( \beta(t) + 2 \homop_a(\rho_t) / \big(\GFmargin(\param_s) \rho^{M}_t \big)  \Big)^2}-1 \Bigg) \\ 
    & = M \cdot \Bigg( \frac{1 - p_1(t)}{\big(\beta(t) + p_2(t) \big)^2}-1\Bigg), 
\end{align*}
where 
\[
    p_1(t) := \frac{2 \homor(\rho_t)}{ M\GFmargin(\param_s) \rho^{M-1}_t}, \quad p_2(t) := \frac{2 \homop_a(\rho_t)}{\GFmargin(\param_s) \rho^{M}_t}.
\]
This completes the proof of \Cref{lem: Characterization of beta}.
\end{proof}

\begin{corollary}
[$\beta$ bound]
\label{cor: beta bound}
Under \Cref{asp:nearhomo,asp:initial-cond-gf}, there exists $s_5>s>0$ such that for almost every $t_2>t_1\ge s_5$, there exists $t_*\in (t_1, t_2)$ satisfying
\[
    \frac{1 - p_1(t_*)}{\big(\beta(t_*) + p_2(t_*) \big)^2}-1 \le \frac{1}{M} \cdot \frac{\log \GFmargin(\param_{t_2}) - \log\GFmargin(\param_{t_1})}{\log \rho_{t_2} - \log \rho_{t_1}}. 
\]
\end{corollary}
\begin{proof}[Proof of \Cref{cor: beta bound}]
Denote the RHS as $C$. Assume to the contrary that 
\[
    \frac{1 - p_1(t_*)}{\big(\beta(t_*) + p_2(t_*) \big)^2}-1 > C \quad \text{ for all } t_*\in (t_1, t_2).
\]
By \Cref{lem: Characterization of beta}, we have 
\[
    \frac{1}{M} \log \frac{\GFmargin(\param_{t_2})}{\GFmargin(\param_{t_1})} > \int_{t_1}^{t_2} C \cdot \frac{\mathrm{d} }{\mathrm{d} \tau} \log \rho(\tau) \cdot \mathrm{d} \tau = C( \log \rho_{t_2} - \log \rho_{t_1}) = \frac{1}{M}\log \frac{\GFmargin(\param_{t_2})}{\GFmargin(\param_{t_1})},
\]
which leads to a contradiction. This completes the proof of \Cref{cor: beta bound}.
\end{proof}

\begin{lemma}
[$\beta$ converges to 1]
\label{lem: beta goes to 1}
Under \Cref{asp:nearhomo,asp:initial-cond-gf} and $\nabla f$ is near-$(M-1)$-homogeneous,  there exists a sequence $t_k$ such that $\lim_{k\to \infty}\beta_{t_k} \to 1$.
\end{lemma}
\begin{proof}[Proof of \Cref{lem: beta goes to 1}]
By \Cref{cor: beta bound}, there exists $s_5 >s>0$  such that for almost every $t_2>t_1\ge s_5$, there exists $t_*\in (t_1, t_2)$ satisfying 
\[
    \beta(t_*) \ge \sqrt{  \frac{1-p_1(t_*)}{1+ \frac{1}{M} \cdot \frac{\log \GFmargin(\param_{t_2}) - \log\GFmargin(\param_{t_1})}{\log \rho_{t_2} - \log \rho_{t_1}}}} - p_2(t_*).
\]
We know that as $t_* \to \infty$, $p_1(t_*), p_2(t_*) \to 0$.  Besides, we know that $\log \rho_t \to \infty$ and $\log \GFmargin(\param_t) \to \gamma_*$. For any $\epsilon_m$, there exists $t_2 >t_2 >s_5$ such that 
\[
    p_1(t_*) \le \epsilon_m/2, \log(\rho(t_2) - \log(\rho(t_1)) \ge \frac{1}{M}, \log \GFmargin(\param_{t_2}) - \log \GFmargin(\param_{t_1}) \le \epsilon_m /2, \text{ and } p_2(t_*) \le \epsilon_m/2. 
\]
Hence we have 
\[
    \beta(t_*) \ge \sqrt{\frac{1-\epsilon_m/2}{1+\epsilon _m /2}} - \epsilon_m/2 \ge 1 - \epsilon _m. 
\]
Hence, for each $\epsilon _m$ we can find a corresponding $t_m$ such that $\beta(t_m) \ge 1 - \epsilon_m$. This completes the proof of \Cref{lem: beta goes to 1}.
\end{proof}

It's worthy noting that \Cref{lem: beta goes to 1} shows that  a subsequence of $(\hB_{\homo}(\param_t), \param_t)$ aligns, which matches the result in \citet{lyu2020gradient}. However, \citet{ji2020directional} showed that $(\hB_{\homo}(\param_t), \param_t)$ aligns assuming that $f$ is homogeneous. We leave investigating gradient alignment for near-homogeneous functions as a future direction.

\begin{lemma}
    [approximate KKT point]
    \label{lem: approx kkT}
    Under \Cref{asp:nearhomo,asp:initial-cond-gf} and $\nabla f$ is $(M-1)$ near-homogeneous, there exists a sequence $t_k$ such that $\hat \param_{t_k}$ is an $(\epsilon_k, \delta_k)$-KKT point of \Cref{eq: KKT} for all $k\in \Nbb$, where $\epsilon_k \to 0$ and $\delta_k \to 0$ as $k\to \infty$.
\end{lemma}
\begin{proof}[Proof of \Cref{lem: approx kkT}]
We just apply \Cref{lem: tilde theta KKT condition}, \Cref{lem: q goes to 0}, and \Cref{lem: beta goes to 1}. This completes the proof of \Cref{lem: approx kkT}.
\end{proof}
Now we can prove the main theorem.

\subsection{Proof of Theorem~\ref{thm: KKT convergence}}

\begin{proof}[Proof of \Cref{thm: KKT convergence}] 


Appplying \Cref{lem: approx kkT}, we have a sequence $\{t_k\}$ such that $\hat \param_{t_k}$ is an $(\epsilon_k, \delta_k)$-KKT point of \Cref{eq: KKT} for all $k\in \Nbb$, where $\epsilon_k \to 0$ and $\delta_k \to 0$ as $k\to \infty$.  Since $\tilde \param_t$ and $\tilde{\param} $ share the same direction, $\tilde{\param}_{t_k}$ will converge to the same direction as one of the KKT points of \Cref{eq: KKT}.  By \Cref{thm: directional convergence}, we know $\tilde{\param}_{t_k}$ also converges to the limit $\param_*$. Hence, $\param_*/\big(\bar f_{\min}(\param_*)\big)^{1/M}$ is a KKT point of \Cref{eq: KKT}. This completes the proof of \Cref{thm: KKT convergence}. 
\end{proof}

\section{Proofs for Section \ref{sec:near-homo-nn}}  

The proofs in this section are divided into two parts. In the first part, we provide some basic results about near-homogeneous block functions. In the second part, we verify that many commonly used neural network architectures are near-homogeneous block functions.

\subsection{Near-Homogeneous Networks} \label{sec:pf-near-homo-nn}

The first thing we want to show in this section is that if we want to verify a model (block map) is near-homogeneous (dual-homogeneous), we only need to verify the first two conditions. This is because the frist condition will imply the the last condition. 
\begin{lemma}
[M-near homogeneity bound]
\label{lem: M-near homogeneity bound}
Given a locally Lipschitz function $ s(\param)$, $M\in \Zbb_+$ and $p_s(x) \in \Rbb_+[x]$ with $\deg \homop_s = M$, 
we assume that for all $\param$,
\[
  \|\langle \nabla s(\param ), \param \rangle - M \cdot s(\param)\| \le \homop_s^\prime (\|\param\|).
\]
Then, there exists $p_s^+ \in \Rbb_+[x]$ such that $\deg \homop_s^+ = M$ and
\[
    \|s(\param)\| \le \homop_s^+(\|\param\|). 
\]
\end{lemma}
We skip the proof here as it is almost the same as the proof of \Cref{lem:near-homogeneity}. Indeed, note that in the proof of \Cref{lem:near-homogeneity} we only use Assumption (A1). However, this lemma indicates that the near homogeneity will lead to the boundedness of the block function. In other words, we have 
\[
    (A1)  \implies (A3),\quad  (B1) \implies (B3). 
\]

In detail, we have the following corollary.  
\begin{corollary}
[From near homogeneity to boundedness]
\label{cor: Near homogeneity to boundedness}
Given a model $f(\param;\xB)$ satisfying (A1) and (A2) in \Cref{def:nearhomo}, then $f(\param; \xB)$ satisfies (A1), (A2) and (A3) in \Cref{def:nearhomo}. Similarly, given a block function $s(\param;\xB)$ satisfying (B1) and (B2) in \Cref{def:dual-homo}, then it satisfies (B1), (B2) and (B3) in \Cref{def:dual-homo}.
\end{corollary}

Once we have this corollary, we must only verify the first two conditions in the dual homogeneity definition. The first result we want to verify is that near-homogeneous gradients lead to near-homogeneous functions. 

\begin{lemma}[Near-homogeneous gradient leads to near-homogeneous function]
\label{lem:near-homo gradients lead to near-homo functions}
Assume that $f(\param;\xB)$ is continuously differentiable and $\nabla f$ is locally Lipschitz and satisfies \Cref{asp:strongerhomo} with parameter $M-1$ for some $M\ge 2$. Then, $f$ is near-M-homogeneous.
\end{lemma}
\begin{proof}
Since $\nabla f$ is locally Lipschitz, $\nabla^2 f$ exists almost everywhere. Here, we focus on showing (A1) in \Cref{asp:nearhomo} since (A2) is straightforward. For simplicity, we omit the $\xB$ in the function. Note (A1) is  to show there exists $\homop$ s.t. $\deg \homop \le M$ and 
\[
|\langle \nabla f(\param), \param \rangle - M f(\param) | \le \homop'(\|\param\|). 
\]
Here we assume there is a $\homop_1 \in \Rbb_+[x]$ such that $\deg \homop_1 \le M-1$ and 
\[
|\langle \nabla [\nabla f(\param)]_i, \param \rangle - (M-1) [\nabla f(\param)]_i | \le \homop_1'(\|\param\|). 
\]
Since $f$ is a continuous function, we can assume for all $\param \in \overline{B(0,1)}$, 
\[
|\langle \nabla f(\param), \param \rangle - M f(\param) | \le C. 
\]
For general $\param$, let $\tilde \param \coloneqq \param/ \|\param\|$ and  
\[
g(r) \coloneqq \langle \nabla f(r \tilde \param), r\tilde \param \rangle - M f(r \tilde \param). 
\]
We aim to show that $|g(r) - g(1)|$ is bounded by some polynomial in $r$.   Note that 
\begin{align*}
|g'(r)| 
&= \big|\langle \tilde \param^\top \nabla^2 f(r \tilde \param), r\tilde \param \rangle  +\nabla f(r \tilde \param) ^\top \tilde \param -  M \nabla f(r \tilde \param) ^\top \tilde \param \big | \\ 
&= \big|\langle \tilde \param^\top \nabla^2 f(r \tilde \param), r\tilde \param \rangle  -  (M-1) \nabla f(r \tilde \param) ^\top \tilde \param \big | \\ 
&= \Big |\tilde \param^\top \big[ \langle \tilde  \nabla^2 f(r \tilde \param), r\tilde \param \rangle  -  (M-1) \nabla f(r \tilde \param) \big] \Big | \\ 
&\le \| \langle \tilde  \nabla^2 f(r \tilde \param), r\tilde \param \rangle  -  (M-1) \nabla f(r \tilde \param) \|  \\ 
&\le \homop_1'(r). 
\end{align*}
Hence, we have 
\[
g(r) \le |g(1)| + |g(r) - g(1)| \le C + \int_1^r \homop_1'(s)\, \kd s \le \homop_1(r) + C - \homop_1(1)\le \homop_1(r) +C. 
\]
Therefore, we can set $\homop  = \int (\homop_1+C)$. This verifies (A1). Hence, $f$ is near-M-homogeneous. This completes the proof. 
\end{proof}

Now we can prove \Cref{prop: Composition of block mappings}.

\subsection{Proof of Proposition~\ref{prop: Composition of block mappings}} 
\label{sec:proof:42}
\begin{proof}[Proof of Proposition~\ref{prop: Composition of block mappings}]

For this proposition, we will only present proofs for argument 1 and 2. For argument 3, 4 and 5, they are direct consequences of argument 1 and 2. 

\mybf{Argument 1 for $s(\param,\xB) = s^1_{\param_1} \circ s^2_{\param_2}(\xB)$.} 

By \Cref{cor:Clarke Jacobian of locally Lipschitz functions}, we know the existence of the Clarke Jacobian of $s^1(\param_1; s^2(\param_2;\xB))$, since it's locally Lipschitz. Right now, for $s(\param;\xB) = s^1(\param_1; s^2(\param_2;\xB))$, we want to verify $(B1)$ and $(B2)$. Since they are some convex constraints, we only need to verify the constraints at the boundary. 

To prove this, we need to invoke \Cref{cor:subgrad-comp-block}.  Recall that the statement in \Cref{cor:subgrad-comp-block} is: 

\vspace{1em}
\hrule 
\vspace{1em}

Given two locally Lipschitz block functions 
$s^1(\param_1; \xB): \Rbb^{d_1} \times \Rbb^{d_2} \to \Rbb^{d_3}$ and $s^2(\param_2; \xB): \Rbb^{d_3} \times \Rbb^{d_4} \to \Rbb^{d_2}$, 
we have 
\begin{align*}
\partial_{\param} s^1(\param_1, s^2(\param_2;\xB)) 
&\subset \text{conv} \big\{ (\alphaB_1, \alphaB_2 \cdot\hB_1) : (\alphaB_1, \alphaB_2) \in \partial_{\param_1, \xB} s^1 (\param_1;\xB), \hB_1 \in \partial_{\param_2} s^2(\param_2;\xB)  \big\},\\   
\partial_{\xB} s^1(\param_1, s^2(\param_2;\xB)) 
&\subset \text{conv} \big\{ \alphaB_2 \cdot\hB_2 :  \alphaB_2 \in \partial_{\xB} s^1 (\param_1;\xB), \hB_2 \in \partial_{\xB} s^2(\param_2;\xB)  \big\}.  
\end{align*} 

\hrule 
\vspace{1em}

This indicates if suffices to verify $(B1)$ and $(B2)$ for all combinations of $(\alphaB_1, \alphaB_2, \hB_1, \hB_2)$, where $(\alphaB_1, \alphaB_2) \in \partial_{\param_1, \xB} s^1 (\param_1;\xB)$ and $(\hB_1,\hB_2)\in \partial_{\param_2, \xB} s^2(\param_2;\xB)$.

Recall that we already know $(\alphaB_1, \alphaB_2, \hB_1, \hB_2)$ satisfy the following conditions:
\begin{align*}
&\| \la\alphaB_1,\param \ra  - M_1 \cdot {s^1}(\param;\xB)\| \le \homop_{s^1}^\prime (\|\param\|) \homor_{s^1}(\|\xB\|),  \\  
&\| \la \alphaB_2, \xB \ra - M_2 \cdot {s^1}(\param;\xB)\| \le \homop_{s^1} (\|\param\|) \homor_{s^1}^\prime(\|\xB\|),\\ 
&\| \alphaB_1\| \le \homoq_{s^1}^\prime (\|\param\|) \homot_{s^1}(\|\xB\|),  \\  
&\| \alphaB_2 \| \le \homoq_{s^1} (\|\param\|) \homot_{s^1}^\prime(\|\xB\|),\\ 
&\|{s^1}(\param;\xB)\| \le \homoq_{s^1} (\|\param\|) \homot_{s^1}(\|\xB\|).
\end{align*}
For $s^2$, we have 
\begin{align*}
&\| \la \hB_1, \param \ra - M_3 \cdot {s^2}(\param;\xB)\| \le \homop_{s^2}^\prime (\|\param\|) \homor_{s^2}(\|\xB\|),  \\  
&\| \la \hB_2, \xB \ra - M_4 \cdot {s^2}(\param;\xB)\| \le \homop_{s^2} (\|\param\|) \homor_{s^2}^\prime(\|\xB\|), \\ 
&\| \hB_1\| \le \homoq_{s^2}^\prime (\|\param\|) \homot_{s^2}(\|\xB\|), \\  
&\| \hB_2 \| \le \homoq_{s^2} (\|\param\|) \homot_{s^2}^\prime(\|\xB\|), \\ 
&\|{s^2}(\param;\xB)\| \le \homoq_{s^2} (\|\param\|) \homot_{s^2}(\|\xB\|).
\end{align*}
Then, we will verify the two conditions (B1) and (B2). 

\myunder{Verify (B1) in \Cref{def:dual-homo}.}

As for (B1), we have 
\begin{align*}
&\langle \alphaB_1, \param_1 \rangle 
+ \langle \alphaB_2\cdot \hB_1 , \param_2 \rangle \\ 
& = \langle \alphaB_1, \param_1 \rangle 
+ \langle \alphaB_2 , M_3 s^2_{\param_2}(\xB) \rangle  + \langle \alphaB_2 , \langle \hB_1,\param_2 \rangle - M_3 s^2_{\param_2}(\xB) \rangle. 
\end{align*}

Therefore, 
\begin{align*}
& \sup_{\gB_1 \in \partial_{\param} s(\param; \xB)} \|\langle \gB_1 ,\param \rangle - (M_1 + M_2M_3) s(\param;\xB)\|\\ 
&\le  \sup_{(\alphaB_1, \alphaB_2) \in \partial s^1, (\hB_1, \hB_2) \in \partial s^2}\| \langle \alphaB_1, \param_1 \rangle - M_1 s(\param;\xB)\| + M_3 \| \langle \alphaB_2 ,  s^2_{\param_2}(\xB) \rangle - M_2 s(\param;\xB)\|\\
&\qquad + \| \langle \alphaB_2 , \langle \hB_1,\param_2 \rangle - M_3 s^2_{\param_2}(\xB) \rangle\| \\ 
&\le \homop_{s^1}^\prime (\|\param_1\|) \homor_{s^1}(\|s^2_{\param_2}(\xB)\|) + M_3 \homop_{s^1} (\|\param_1\|) \homor_{s^1}^\prime (\|s^2_{\param_2}(\xB)\|) + \homop_{s^1}(\|\param_1\|) \homor_{s^1}^\prime (\|s^2_{\param_2}(\xB)\|) \homop_{s^2}^\prime ( \| \param_2\|)  \homor_{s^2}(\|\xB\|)\\ 
&\le \homop_{s^1}^\prime (\|\param_1\|) \cdot \homor_{s^1} \circ \homoq_{s^2}(\|\param_2\|) \cdot  \homor_{s^1} \circ \homot_{s^2}(\|\xB\|) +  M_3 \homop_{s^1} (\|\param_1\|)\cdot \homor_{s^1}^\prime  \circ \homoq_{s^2}(\|\param_2\|) \cdot  \homor_{s^1}^\prime  \circ \homot_{s^2}(\|\xB\|) \\
&\qquad +  \homop_{s^1}(\|\param_1\|) \homor_{s^1}^\prime  \circ \homoq_{s^2}(\|\param_2\|) \cdot  \homor_{s^1}^\prime  \circ \homot_{s^2}(\|\xB\|)\cdot \homop_{s^2}^\prime ( \| \param_2\|)\cdot  \homor_{s^2}(\|\xB\|)\\ 
&\le \underbrace{\big [ \homop_{s^1}^\prime (\|\param_1\|) \homor_{s^1} \circ \homoq_{s^2}(\|\param_2\|) + M_3 \homop_{s^1} (\|\param_1\|) \homor_{s^1}^\prime  \circ \homoq_{s^2}(\|\param_2\|) +  \homop_{s^1}(\|\param_1\|)\cdot \homor_{s^1}^\prime  \circ \homoq_{s^2}(\|\param_2\|) \homop_{s^2}^\prime ( \| \param_2\|) \big ]}_{\text{ of order } \le M_1 + M_2M_3-1} \\ 
&\quad \times \underbrace{\big [ \homor_{s^1} \circ \homot_{s^2}(\|\xB\|)+ M_3 \homor_{s^1}^\prime  \circ \homot_{s^2}(\|\xB\|) +  \homor_{s^1}^\prime  \circ \homot_{s^2}(\|\xB\|)\cdot \homor_{s^2}(\|\xB\|) \big ]}_{\text{ of order } \le  M_2 M_4}.
\end{align*}

We can prove something similar for $ \langle \gB_2 ,\xB\rangle - M_2M_4 s(\param;\xB) $, where $\gB_2 \in \partial_{\xB} s(\param;\xB)$. Then the orders for $ \langle \gB_2 ,\xB\rangle - M_2M_4 s(\param;\xB) $ will be bound by $ (M_1 + M_2 M_3, M_2 M_4-1) $.  
Combining them, there exist $\homop_s , \homor_s \in \Rbb_+[x]$ and $\deg \homop_s \le  M_1 + M_2M_3, \deg \homor_s \le  M_2M_4$ such that 
\begin{align*}
   \sup_{\gB_1 \in \partial_{\param} s(\param;\xB)} \|\langle \gB_1 ,\param\rangle- (M_1 +M_2M_3)s(\param;\xB)\| \le \homop_s^\prime (\|\param\|) \homor_s(\|\xB\|), \\ 
   \sup_{\gB_2 \in \partial_{\xB} s(\param;\xB)}\|\langle \gB_2 ,\xB\rangle- (M_2M_4)s(\param;\xB)\| \le \homop_s (\|\param\|) \homor_s^\prime(\|\xB\|). 
\end{align*}

\myunder{Verify (B2) in \Cref{def:dual-homo}.}

Before we proceed, we have the following two inequalities: 
\begin{align*}
    \homot_{s^1}(\|s^2_{\param_2}(\xB)\|) \le \homot_{s^1}\big(\homoq_{s^2}(\|\param_2\|) \homot_{s^2}(\|\xB\|)\big)\le  [ \homot_{s^1}\circ \homoq_{s^2}(\|\param_2\|)] \cdot [\homot_{s^1}\circ \homot_{s^2}(\|\xB\|)],\\ 
    \homot^\prime _{s^1}(\|s^2_{\param_2}(\xB)\|) \le \homot^\prime _{s^1}\big(\homoq_{s^2}(\|\param_2\|) \homot_{s^2}(\|\xB\|)\big)\le  [ \homot^\prime _{s^1}\circ \homoq_{s^2}(\|\param_2\|)] \cdot [\homot^\prime _{s^1}\circ \homot_{s^2}(\|\xB\|)]. 
\end{align*}

Note that 
\begin{align*}
\sup_{\gB_1 \in \partial_{\param} s} \|\gB_1\|^2 
&\le  \sup_{(\alphaB_1, \alphaB_2) \in \partial s^1, (\hB_1, \hB_2) \in \partial s^2} \|\alphaB_1\|^2 + \|\alphaB_2 \cdot \hB_1\|^2 
\\
&\le \big [ \homoq_{s^1}^\prime (\|\param_1\|) \homot_{s^1}(\|s^2_{\param_2}(\xB)\|) \big]^2  + \big [\homoq_{s^1} (\|\param_1\|) \homot_{s^1}^\prime (\|s^2_{\param_2}(\xB)\|) \homoq_{s^2}^\prime (\|\param_2\|) \homot_{s^2} (\|\xB\|)  \big]^2  
\\ 
&\qquad \explain{ By (B2) and \Cref{lem: M-near homogeneity bound}} 
\\
&\le \big [\homoq_{s^1}^\prime (\|\param_1\|) \cdot \homot_{s^1}\circ \homoq_{s^2}(\|\param_2\|) \cdot \homot_{s^1}\circ \homot_{s^2}(\|\xB\|) \big]^2\\ 
& \quad + \big [\homoq_{s^1} (\|\param_1\|) \cdot   \homot^\prime _{s^1}\circ \homoq_{s^2}(\|\param_2\|) \cdot \homot^\prime _{s^1}\circ \homot_{s^2}(\|\xB\|)\cdot  \homoq_{s^2}^\prime (\|\param_2\|)\cdot  \homot_{s^2} (\|\xB\|)  \big]^2\\ 
&\qquad \explain{ By two inequalities above } \\ 
&\le \big [ \underbrace{\homoq_{s^1}^\prime (\|\param_1\|) \cdot \homot_{s^1}\circ \homoq_{s^2}(\|\param_2\|) + \homoq_{s^1} (\|\param_1\|) \cdot   \homot^\prime _{s^1}\circ \homoq_{s^2}(\|\param_2\|)\cdot  \homoq_{s^2}^\prime (\|\param_2\|)}_{\text{ of order } \le  M_1 + M_2M_3-1} \big]^2 \\ 
&\quad \times \big [\underbrace{ \homot_{s^1}\circ \homot_{s^2}(\|\xB\|) +  \homot^\prime _{s^1}\circ \homot_{s^2}(\|\xB\|) \cdot \homot_{s^2} (\|\xB\|)}_{\text{ of order } \le M_2M_4} \big]^2. 
\end{align*}
Therefore, there exist $\homoq_s , \homot_s \in \Rbb_+[x]$ and $\deg \homoq_s = M_1 + M_2M_3, \deg \homot_s = M_2M_4$ such that  \[\sup_{\gB_1 \in \partial_{\param} s} \|\gB_1\|^2 \le \homoq_s^\prime (\|\param\|) \homot_s(\|\xB\|).\]
On the other hand, 
\begin{align*}
    \sup_{\gB_2 \in \partial_{\xB} s}\|\gB_2\| &\le  \sup_{(\alphaB_1, \alphaB_2) \in \partial s^1, (\hB_1, \hB_2) \in \partial s^2} \|\alphaB_2 \cdot \hB_2\|\\ 
    &\le \homoq_{s^1} (\|\param_1\|) \homot_{s^1}^\prime (\|s^2_{\param_2}(\xB)\|) \homoq_{s^2} (\|\param_2\|) \homot_{s^2}^\prime (\|\xB\|)\\
    & \le \homoq_{s^1} (\|\param_1\|)\cdot \homot_{s^1}^\prime \circ \homoq_{s^2}(\|\param_2\|)\cdot \homot_{s^1}^\prime \circ \homot_{s^2}(\|\xB\|)\cdot \homoq_{s^2} (\|\param_2\|) \cdot \homot_{s^2}^\prime (\|\xB\|)\\
    & = \underbrace{\big [\homoq_{s^1} (\|\param_1\|)\cdot \homot_{s^1}^\prime \circ \homoq_{s^2}(\|\param_2\|) \cdot \homoq_{s^2} (\|\param_2\|) \big]}_{\text{ of order } \le M_1 + M_2M_3} \times \underbrace{ \big [\homot_{s^1}^\prime \circ \homot_{s^2}(\|\xB\|)\cdot \homot_{s^2}^\prime (\|\xB\|) \big] }_{\text{ of order } \le  M_2M_4-1}. 
\end{align*}
Therefore, there exist $\homoq_s , \homot_s \in \Rbb_+[x]$ and $\deg \homoq_s \le M_1 + M_2M_3, \deg \homot_s \le  M_2M_4$ such that  
\begin{align*}
    \sup_{\gB_1 \in \partial_{\param} s}\|\gB_1\| \le \homoq_s^\prime (\|\param\|) \homot_s(\|\xB\|),\\ 
    \sup_{\gB_2 \in \partial_{\xB} s}\|\gB_2\| \le \homoq_s (\|\param\|) \homot_s^\prime (\|\xB\|). 
\end{align*}
We have shown that $s$ satisfies (B1) and (B2) with parameter $(M_1 + M_2M_3, M_2M_4)$.  Since (B1) and (B2) lead to (B3), This indicates that $s(\param;\xB)$ is near-$(M_1 + M_2M_3, M_2M_4)$-homogeneous.

\mybf{Argument 2 for $s(\param,\xB) = s^1_{\param_1} \otimes s^2_{\param_2}(\xB)$.} 
Without loss of generality, we can focus on just one entry of the output like: 
\[
\big[ s(\param,\xB) \big]_{i,j} =  \big[ s^1_{\param_1}(\xB) \big]_i \cdot \big[ s^2_{\param_2}(\xB) \big]_j.
\] 
To simplify the notation, we can assume $s^1, s^2$ are both scalar functions. Then, by \Cref{lem:chain-rule-clark}, we have 
\[
\partial (s^1 \cdot s^2) \subseteq \text{conv} \{  (s^2\alphaB_1 , s^1 \hB_1, s^2 \alphaB_2 + s^1 \hB_2)   : (\alphaB_1, \alphaB_2) \in \partial s^1, (\hB_1, \hB_2) \in \partial s^2 \}.  
\] 
We will also use $ (\gB_1, \gB_2) $ to denote the Jacobian of $ s $. Then, we can verify $(B1)$ and $(B2)$. 

\myunder{Verify (B1) in \Cref{def:dual-homo}.}
As for $ \gB_1  \in \partial_{\param} s$, we have 
\begin{align*}
&\sup_{\gB_1 \in \partial_{\param} s} \big \| \langle \gB_1, \param \rangle  - (M_1 +M_3) s \big \|
\\ 
&\le \sup_{\alphaB_1 \in \partial_{\param} s^1, \hB_1 \in \partial_{\param} s^2} \big \| \langle s^2 \alphaB_1, \param_1 \rangle  + \langle s^1 \hB_1, \param_2 \rangle - (M_1 + M_3) s \| 
\\ 
&\le \sup_{\alphaB_1 \in \partial_{\param} s^1} \|s^2\| \cdot \| \langle \alphaB_1, \param_1\rangle  - M_1 s^1 \| + \sup_{\hB_1 \in \partial_{\param} s^2} \|s^1\| \cdot \| \langle \hB_1, \param_2\rangle  - M_3 s^2 \| 
\\ 
&\le \homoq_{s^2}(\|\param\|) \homot_{s^2}(\|\xB\|) \homop_{s^1}^\prime (\|\param\|) \homor_{s^1}(\|\xB\|) + \homoq_{s^1}(\|\param\|) \homot_{s^1}(\|\xB\|) \homop_{s^2}^\prime (\|\param\|) \homor_{s^2}(\|\xB\|) \\ 
&= \underbrace{\homoq_{s^2}(\|\param\|)  \homop_{s^1}^\prime (\|\param\|)}_{ \text{of order } (M_1 +M_3-1)} \underbrace{\homot_{s^2}(\|\xB\|)\homor_{s^1}(\|\xB\|)}_{ \text{of order } \le (M_2 +M_4)} + \underbrace{\homoq_{s^1}(\|\param\|) \homop_{s^2}^\prime (\|\param\|)}_{ \text{of order } \le (M_1 +M_3-1)}  \underbrace{\homot_{s^1}(\|\xB\|)\homor_{s^2}(\|\xB\|)}_{ \text{of order } \le (M_2 +M_4)}.
\end{align*}

Similarly, we can get that 
\begin{align*}
    &\sup_{\gB_2 \in \partial_{\xB} s} \big \| \langle \gB_2, \xB \rangle  - (M_2 +M_4) s \big \|
    \\ 
    &\le \sup_{\alphaB_2 \in \partial_{\xB} s^1, \hB_2 \in \partial_{\xB} s^2} \big \| \langle s^2 \alphaB_2, \xB \rangle  + \langle s^1 \hB_2, \xB \rangle - (M_1 + M_3) s \| 
    \\ 
    &\le \sup_{\alphaB_2 \in \partial_{\xB} s^1} \|s^2\| \cdot \| \langle \alphaB_2, \xB\rangle  - M_1 s^1 \| + \sup_{\hB_2 \in \partial_{\xB} s^2} \|s^1\| \cdot \| \langle \hB_2, \xB\rangle  - M_3 s^2 \| 
    \\ 
    &\le \homoq_{s^2}(\|\param\|) \homot_{s^2}(\|\xB\|) \homop_{s^1} (\|\param\|) \homor_{s^1}^\prime(\|\xB\|) + \homoq_{s^1}(\|\param\|) \homot_{s^1}(\|\xB\|) \homop_{s^2} (\|\param\|) \homor_{s^2}^\prime(\|\xB\|) \\ 
    &= \underbrace{\homoq_{s^2}(\|\param\|)  \homop_{s^1} (\|\param\|)}_{ \text{of order } (M_1 +M_3)} \underbrace{\homot_{s^2}(\|\xB\|)\homor_{s^1}^\prime(\|\xB\|)}_{ \text{of order } \le (M_2 +M_4-1)} + \underbrace{\homoq_{s^1}(\|\param\|) \homop_{s^2} (\|\param\|)}_{ \text{of order } \le (M_1 +M_3)}  \underbrace{\homot_{s^1}(\|\xB\|)\homor_{s^2}^\prime(\|\xB\|)}_{ \text{of order } \le (M_2 +M_4-1)}.
    \end{align*}

Combining these two inequalities, we can see that there exists $\homoq_s,\homor_s \in \Rbb_+[x]$ such that $\deg \homoq_s \le (M_1 +M_3)$, $\deg \homor_s \le (M_2+M_4)$, and 
\begin{align*}
    \sup_{\gB_1 \in \partial_{\param} s(\param;\xB)} \|\langle \gB_1 ,\param\rangle- (M_1 +M_3)s(\param;\xB)\| \le \homop_s^\prime (\|\param\|) \homor_s(\|\xB\|), \\ 
    \sup_{\gB_2 \in \partial_{\xB} s(\param;\xB)}\|\langle \gB_2 ,\xB\rangle- (M_2+M_4)s(\param;\xB)\| \le \homop_s (\|\param\|) \homor_s^\prime(\|\xB\|). 
 \end{align*}

\myunder{Verify (B2) in \Cref{def:dual-homo}.}

As for $ \gB_1 $, we have 
\begin{align*}
\sup_{\gB_1 \in \partial_{\param}s} \|\gB_1\|^2  
&\le \sup_{\alphaB_1 \in \partial_{\param}s^1, \hB_1 \in \partial_{\param} s^2 } \|s^2 \alphaB_1 \|^2 + \|s^1 \hB_1\|^2 
\\ 
&\le \big[\homoq_{s^2}(\|\param\|) \homot_{s^2}(\|\xB\|) \homoq_{s^1}^\prime (\|\param\|) \homot_{s^1}(\|\xB\|)\big]^2 + \big[\homoq_{s^1}(\|\param\|) \homot_{s^1}(\|\xB\|) \homoq_{s^2}^\prime (\|\param\|) \homot_{s^2}(\|\xB\|)\big]^2 \\ 
&= \big[\underbrace{\homoq_{s^2}(\|\param\|)  \homoq_{s^1}^\prime (\|\param\|)}_{ \text{  of order } \le (M_1 +M_3-1)} \cdot \underbrace{ \homot_{s^1}(\|\xB\|)\homot_{s^2}(\|\xB\|)}_{\text{ of order }\le (M_2 + M_4)}\big]^2 + \big[\underbrace{\homoq_{s^1}(\|\param\|)  \homoq_{s^2}^\prime (\|\param\|)}_{ \text{  of order } \le (M_1 +M_3-1)} \cdot \underbrace{ \homot_{s^2}(\|\xB\|)\homot_{s^1}(\|\xB\|)}_{\text{ of order }\le (M_2 + M_4)}\big]^2. 
\end{align*}
Similarly, we have 
\begin{align*}
\sup_{\gB_2 \in \partial_{\xB}s} \|\gB_2\|^2  
&\le \sup_{\alphaB_2 \in \partial_{\xB}s^1, \hB_2 \in \partial_{\xB} s^2 } \|s^2 \alphaB_2 + s^1 \hB_2\|^2
\\
&\le \sup_{\alphaB_2 \in \partial_{\xB}s^1, \hB_2 \in \partial_{\xB} s^2 } 2\|s^2 \alphaB_2 \|^2 + 2\|s^1 \hB_2\|^2 
\\ 
&\le \big[\homoq_{s^2}(\|\param\|) \homot_{s^2}(\|\xB\|) \homoq_{s^1} (\|\param\|) \homot_{s^1}^\prime(\|\xB\|)\big]^2 + \big[\homoq_{s^1}(\|\param\|) \homot_{s^1}(\|\xB\|) \homoq_{s^2} (\|\param\|) \homot^\prime_{s^2}(\|\xB\|)\big]^2 \\ 
&= \big[\underbrace{\homoq_{s^2}(\|\param\|)  \homoq_{s^1} (\|\param\|)}_{ \text{  of order } \le (M_1 +M_3)} \cdot \underbrace{ \homot_{s^1}^\prime(\|\xB\|)\homot_{s^2}(\|\xB\|)}_{\text{ of order }\le (M_2 + M_4-1)}\big]^2 + \big[\underbrace{\homoq_{s^1}(\|\param\|)  \homoq_{s^2} (\|\param\|)}_{ \text{  of order } \le (M_1 +M_3)} \cdot \underbrace{ \homot_{s^2}^\prime(\|\xB\|)\homot_{s^1}(\|\xB\|)}_{\text{ of order }\le (M_2 + M_4-1)}\big]^2. 
\end{align*}

Therefore, there exist $\homoq_s , \homot_s \in \Rbb_+[x]$ and $\deg \homoq_s \le M_1 + M_2M_3, \deg \homot_s \le  M_2M_4$ such that  
\begin{align*}
    \sup_{\gB_1 \in \partial_{\param} s}\|\gB_1\| \le \homoq_s^\prime (\|\param\|) \homot_s(\|\xB\|),\\ 
    \sup_{\gB_2 \in \partial_{\xB} s}\|\gB_2\| \le \homoq_s (\|\param\|) \homot_s^\prime (\|\xB\|). 
\end{align*}
We have shown that $s$ satisfies (B1) and (B2) with parameter $(M_1 + M_3, M_2+ M_4)$.  Since (B1) and (B2) lead to (B3), this leads to that $s(\param;\xB)$ is near-$(M_1 + M_3, M_2+ M_4)$-homogeneous.

This completes the proof of \Cref{prop: Composition of block mappings}. 
\end{proof}

\subsection{Proof of Corollary~\ref{cor: Near homogeneity order of networks}}
\label{sec:proof-43}
\begin{proof}[Proof of Corollary~\ref{cor: Near homogeneity order of networks}]
This is a direct consequence of \Cref{prop: Composition of block mappings} and induction on the number of layers.  To prove this, we need a stronger result. For 
\[
    f(\param;\xB) = s^1_{\param_1} \circ s^2_{\param_2} \circ \cdots \circ s^L_{\param_L}(\xB),
\]
it's 
\[
    \bigg( \sum_{j=1}^L M_1^j \prod_{i=1}^{j-1} M_2^i,\; \prod_{i=1}^{L} M_2^i \bigg)
\]
dual-homogeneous. 
When $L=1$, the result is trivial. We assume it holds for $L-1$ layers. Then we have $f^{L-1}(\param;\xB) = s^1_{\param_1} \circ s^2_{\param_2} \circ \cdots \circ s^{L-1}_{\param_{L-1}}(\xB)$  is dual-homogeneous with parameters
\[
    \bigg( \sum_{j=1}^{L-1} M_1^j \prod_{i=1}^{j-1} M_2^i,\; \prod_{i=1}^{L-1} M_2^i \bigg).
\]  
Then we apply \Cref{prop: Composition of block mappings} to $f(\param;\xB) = f^{L-1}(\param;\xB) \circ s^L_{\param_L}(\xB)$, we have $f$ is dual-homogeneous with parameters
\[
    \bigg( \sum_{j=1}^{L-1} M_1^j \prod_{i=1}^{j-1} M_2^i + M_1^{L} \prod_{i=1}^{L-1} M_2^i,\; \prod_{i=1}^{L} M_2^i \bigg) = \bigg( \sum_{j=1}^{L} M_1^j \prod_{i=1}^{j-1} M_2^i,\; \prod_{i=1}^{L} M_2^i \bigg).  
\]
This completes the proof of \Cref{cor: Near homogeneity order of networks}.
\end{proof}

\subsection{Dual-Homogeneous Examples} \label{sec:proof-dual-homo-eg}
In this subsection, we will try to provide some examples of dual-homogeneous blocks. The frist example we want to verify is the linear layers in neural networks.

\begin{example}
[Linear transform]
\label{eg:linear-trans}
Let the linear transform block function $s(\param;\xB) = A \xB + b$. Then, $s$ is near-homogeneous of parameter $(1,1)$. We assume that $\xB\in \Rbb^{d_1}, A \in \Rbb^{d_2 \times d_1}$ and $b \in \Rbb^{d_2}$. 
\end{example}
\begin{proof}
Note that 
\[
    \nabla_A s = I_{d_2\times d_2} \otimes \xB, \quad \nabla_b s = I_{d_2 \times d_2}, \quad \nabla_{\xB} s = A. 
\]
Therefore, we have 
\[
    \langle \nabla_{\xB} s, \xB \rangle -s = b, \quad \langle \nabla_{\param} s, \param \rangle - s= 0.   
\]
Furthermore, if we define the norm on $A$ as the $\| \text{vec} (A)\|_2$, then we have 
\[
    \|\nabla_{\param} s\|_2 \le \sqrt{d_2}\| \xB \|_2 + \sqrt{d_1}, \quad \|\nabla_{\xB} s\|_2 \le \|A\|_2.  
\]
Then we can choose 
\[
   \begin{cases}
    \homop(x) = x \\ 
    \homor(x) =  x \\ 
    \homoq(x) =  x \\ 
    \homot(x) = \sqrt{d_2} x + \sqrt{d_1}.
   \end{cases}
\]
Therefore, $s$ satisfies (B1), (B2) and (B3) in \Cref{def:dual-homo} with parameter $(1,1)$. 
\end{proof}

The second kind of block functions is the activation functions $\phi(\xB)$. Note that the activation function doesn't have parameter $\param$. Therefore, we can reduce the dual homogeneity assumptions (B1), (B2) and (B3) to:

\begin{definition}
[Near-homogeneous activation function]
\label{def:near-homo-activation}
{M-homogeneous activation assumptions.} Given a definable and locally Lipschitz activation function $\phi(\xB)$, we assume that there exists $M\in \Zbb_+$ and $\homop, \homoq \in \Rbb_+[x]$ with $\deg \homor, \deg \homot  \le  M$, such that for all $\xB\in (\xB)_{i=1}^n$ and all $ \hB \in \partial \phi(\xB) $: 
\begin{itemize}
    \item [(C1)] \textbf{$M$-Near homogeneity.} 
    \[
        \|\langle \hB, \xB \rangle - M \phi(\xB)\| \le \homor^\prime (\|\xB\|).
    \]
    \item [(C2)] \textbf{$M$-bounded gradient.}
    \[
        \| \hB\| \le \homot^\prime (\|\xB\|).
    \]
    \item [(C3)]  \textbf{$M$-bounded value.}
    \[
        \|\phi(\xB)\| \le \homot(\|\xB\|).  
    \]
\end{itemize}
In this case, we say that $\phi(\xB)$ satisfies (C1), (C2) and (C3) with parameter $M$.
\end{definition}

A direct property of this definition is that: 

\begin{corollary}
[Near-homo activation as dual homo block]
\label{cor:Near-homo activation as dual homo block}
Given a near-homogeneous activation function $\phi(x)$ with parameter $M$, then it's dual homogeneous with parameters $(0,M)$. 
\end{corollary}

We can easily verify  the following results: 

\begin{example}
[Activation functions]
\label{eg: Activation functions}
The following activation functions are near-homogenneous activations with parameter $1$: ReLU, Leaky ReLU, Softplus, Huberized ReLU, Swish, GELU, SiLU. 
\end{example}

Another important property of the activation function is that: 
\begin{lemma}
    [Power of activation functions]
    \label{lem:power-activation}
    Given an activation function $s$ of parameter $(0,M)$, $s^k$ is of parameter $(0,kM)$. 
\end{lemma}

\begin{proof}
    Similar to \Cref{lem: M-near homogeneity bound}, we can show that (C1) implies (C3). Therefore, it suffices to prove (C1) and (C2) for $s^k$ with $M$ replaced by $kM$. By the definition of Clarke subdifferential, we have 
    \[
        \partial s^k (\xB) = \{ s^{k-1} \cdot \alphaB : \alphaB \in \partial s(\xB) \}.
    \] 
    By direct calculation, we obtain for any $ \hB \in \partial s^k(x) $,
    \begin{align*}
        \left\| \langle \hB, \xB \rangle - kM s^k (\xB) \right\| = \, & \left\vert k s^{k-1} (\xB) \right\vert \left\| \langle \nabla s (\xB), \xB \rangle - M s (\xB) \right\| \\
        \le \, & k \underbrace{\homot^{k-1} (\| \xB \|) \homor'(\| \xB \|)}_{\text{ of order } \le (kM-1)},
    \end{align*}
    which implies (C1). For (C2), note that
\[
 \left\| \hB \right\| = \, \left\vert k s^{k-1} (\xB) \right\vert \left\| \alphaB \right\| \le \underbrace{k\homot^{k-1} (\| \xB \|) \homot'(\| \xB \|)}_{\text{ of order }\le (kM-1)} ,
\]

    completing the proof.
\end{proof}

With this lemma, we can understand the near homogeneity of the polynomial activation functions. Then we have the following result for a layer of the neural network, i.e., 
\[
    s(\param; \xB) = \phi(A \xB + b). 
\]

\begin{corollary}
    [Near homogeneity layer] 
    \label{cor: Near homogeneity layer}
    Given a layer of the neural network $s(\param; \xB) = \phi(A \xB + b)$, where $\phi$ is a near-homogeneous activation with parameter $M$, then $s$ is dual-homogeneous with parameter $(1,M)$.
\end{corollary}

\begin{proof}
    Apply \Cref{prop: Composition of block mappings} to $s_{\param_1}^1 (\xB) = \phi(\xB)$ and $s_{\param_2}^2 (\xB) = A \xB + b$.
\end{proof}

With all the tools we build here, we can now understand the near homogeneity of many commonly used network architectures.

\subsection{Proof of Example~\ref{eg:Examples of dual homogeneous blocks}}
\begin{proof}[Proof of Example~\ref{eg:Examples of dual homogeneous blocks}]
There are multiple arguments in this example. We will prove them one by one.

\myunder{Argument 1.} This is the result of \Cref{eg:linear-trans}. 

\myunder{Argument 2.} This is the result of \Cref{cor: Near homogeneity layer}. 

\myunder{Argument 3.} Max pooling layer, convolution layer, max layer, average layer are 1-homogeneous operations. Hence, they are dual-homogeneous with parameter $(0,1)$. For the residual connection, we just keep the input as the output. Therefore, it's dual-homogeneous with parameter $(0,1)$.

\myunder{Argument 4.} Recall that the SwiGLU is defined as: 
\[
\text{SwiGLU}(x, W, V, b,c,\beta) = \text{Swish}_{\beta} (xW+ b) \otimes (xV+c), 
\]
where $\beta$ is one hyperparameter and $W,V,b,c$ are trainable. We can see that SwiGLU is a tensor product of two near-homogeneous functions. In terms of \Cref{cor: Near homogeneity layer}, $\text{Swish}_{\beta} (xW+ b)$ is dual-homogeneous with parameter $(1,1)$. $(xV+c)$ is also dual-homogeneous with parameter $(1,1)$. Then by \Cref{prop: Composition of block mappings}, we have $\text{SwiGLU}$ is dual-homogeneous with parameter $(2,2)$.  

\myunder{Argument 5.} The self-linear attention \citep{zhang2024trained} is defined as: 
\[
 f(\param;\HB)  =  \HB + \WB^{PV} \HB \cdot \frac{\HB^\top \WB^{KQ} \HB}{\rho}, 
\]
where $\HB$ is the input token matrix and $\param = (\WB^{PV}, \WB^{KQ})$.  We can observe the dominating term in the linear attention layer is $\WB^{PV} \HB \cdot \frac{\HB^\top \WB^{KQ} \HB}{\rho}$. Therefore, the self-linear attention layer is dual-homogeneous with parameter $(2,3)$.

The ReLU attention \citep{wortsman2023replacing} is defined as: 
\[
f(\param;\HB)  =  \HB + \WB^{P} \WB^{V} \HB \cdot \relu  \bigg(\frac{\HB^\top \WB^{K} \WB^{Q} \HB}{\sqrt{d}L} \bigg), 
\]
where $\HB$ is the input token matrix, $L$ is the contextual length,  and $\param = (\WB^{P}, \WB^{V}, \WB^{K}, \WB^{Q})$. The dominating term here is 
\[
\WB^{P} \WB^{V} \HB \cdot \relu  \bigg(\frac{\HB^\top \WB^{K} \WB^{Q} \HB}{\sqrt{d}L} \bigg),
\]
which is $(4,3)$-dual-homogeneous. Hence, the ReLU attention layer is dual-homogeneous with parameter $(4,3)$.
This completes the proof of Example~\ref{eg:Examples of dual homogeneous blocks}.
\end{proof}



\subsection{Proof of Example~\ref{eg: Near homogeneous networks}}
\begin{proof}[Proof of Example~\ref{eg: Near homogeneous networks}]
We will prove the following arguments one by one.

\myunder{Argument 1.}

This is a direct result of \Cref{cor: Near homogeneity order of networks} and \Cref{eg:Examples of dual homogeneous blocks}. Now we have 
\[
    f(\param;\xB) = s^1_{\param_1} \circ s^2_{\param_2} \circ \cdots \circ s^L_{\param_L}(\xB),
\]
and 
\[
(M_{1}^i, M_{2}^i) = (1,1), \quad \forall i \in [L].
\]
Hence, by \Cref{cor: Near homogeneity order of networks}, we have $f(\param;\xB)$ is near-homogeneous with parameter 
\[
M = \sum_{j=1}^L M_1^j \cdot \bigg [  \prod_{i=1}^{j-1} M^i_2 \bigg] = L.
\]

\myunder{Argument 2.}

For this case, we have 
\[
(M_{1}^i, M_{2}^i) = (1,k), \quad \forall i \in [L].
\]
Hence, by \Cref{cor: Near homogeneity order of networks}, we have $f(\param;\xB)$ is near-homogeneous with parameter 
\begin{align*}
M = \sum_{j=1}^L M_1^j \cdot \bigg [  \prod_{i=1}^{j-1} M^i_2 \bigg]
=  \sum_{j=1}^L k^{j-1} 
= \frac{k^L-1}{k-1}. 
\end{align*}

\myunder{Argument 3.}

For VGG networks, all the layers are average pooling, max pooling, and convolutional layers followed by a ReLU activation. Since convolutional layers are (1,1)-dual-homogeneous and ReLU activation is (0,1)-dual-homogeneous, we can apply \Cref{cor: Near homogeneity order of networks} to VGG networks and get VGG-i is near-homogeneous with parameter $i$.

\myunder{Argument 4.}

For ResNet and DenseNet without batch normalization, all the layers are average pooling, max pooling, convolutional layers followed by a ReLU, and residual connections. 
By argument C in \Cref{prop: Composition of block mappings}, residual connections won't influence the homogeneous degree of the network. 
Since convolutional layers are (1,1)-dual-homogeneous and ReLU activation is (0,1)-dual-homogeneous, we can apply \Cref{cor: Near homogeneity order of networks} to ResNet and DenseNet and get ResNet-$L$ and DenseNet-$L$ are near-homogeneous with parameter $L$.
\end{proof}

\section{Proofs for Section \ref{sec:example}}  
\label{sec:proof-eg}

The proofs in this section also has two parts. In Part I, we introduce the homogeneization of models and show the existence of parmeter which satisfies the initial condition. In Part II, we focus on a toy example and show that the initial condition can be reached in the GF. 

\subsection{Homogenization} \label{sec:homogenization}

For those Homogenization results, the most important idea is to analysis: 

\begin{equation*}
g(r) \coloneqq \frac{f(r \param; \xB)}{r^M}.
\end{equation*}
Since $f$ is locally Lipschitz and definable, we know that $g$ is also locally Lipschitz and definable in any open subset of $\Rbb \backslash \{ 0 \}$. Hence, $g$ is differentiable almost everywhere. Actually, the best thing is that $g$ is $M$-homogeneous. 

\subsection{Proof of Proposition \ref{thm:homogenization}}
\begin{proof}[Proof of \Cref{thm:homogenization}]
By direct calculation, we get (see also the proof of \Cref{lem:near-homogeneity}) that for a.e. $r > 0$:
\begin{align*}
g'(r) = \, \frac{r^M \langle \param, \nabla_{\param} f(r \param; \xB) \rangle - Mr^{M-1} f(r \param; \xB)}{r^{2M}} = \, \frac{\langle r \param, \nabla_{\param} f(r \param; \xB) \rangle - M f(r \param; \xB)}{r^{M+1}},
\end{align*}
which leads to the following estimate
\begin{equation*}
\left\vert g'(r) \right\vert \le \, \frac{\homop'(r \| \param \| )}{r^{M+1}} =  \sum_{i=0}^{M-1} a_{i+1} (i+1) r^{-(M+1-i)} \|\param\|^{i},
\end{equation*}
where the equality follows that $\homop (x) = \sum_{i=0}^M a_i x^i$. This further implies that for any $0 < r_1 < r_2$:
\begin{equation*}
\vert g(r_1) - g(r_2) \vert \le \, \int_{r_1}^{r_2} \vert g'(r) \vert \mathrm{d} r \le \sum_{i=0}^{M-1} \frac{a_{i+1} (i+1)}{M-i} {r_1}^{-(M-i)} \|\param\|^{i} \le \frac{\homop_a(r_1 \|\param\|)}{r_1^M}.
\end{equation*}
Note that $\deg \homop_a = M-1$, and this leads to $\homop_a(r_1\|\param\|) / r_1^M \to 0$ as $r_1 \to \infty$. As a consequence, for any sequence $\{ r_k \}$ such that $r_k \to +\infty$ as $k \to \infty$, $\{ g(r_k) \}$ is a Cauchy sequence. Therefore, $\lim_{r \to +\infty} g(r)$ exists, and $\homoPredictor (\param; \xB)$ is well-defined. From its definition, it is easily seen that $\homoPredictor (\param; \xB)$ is $M$-homogeneous.

We next show that $\homoPredictor (\param; \xB)$ is continuous and a.e. differentiable in $\param$ for fixed $\xB$. To this end, it suffices to show that $\homoPredictor (\param; \xB) \vert_{\param \in \mathbb{S}^{d-1}}$ is Lipschitz continuous. For any $r > 0$ and $\param_1, \param_2 \in \mathbb{S}^{d-1}$, we have the following estimate:
\begin{align*}
\left\vert \frac{f(r \param_1; \xB)}{r^M} - \frac{f(r \param_2; \xB)}{r^M} \right\vert = \, & \frac{1}{r^M} \left\vert f(r \param_1; \xB) - f(r \param_2; \xB) \right\vert \\
\stackrel{(i)}{\le} \, & \frac{\homop'(r)}{r^M} \left\| r \param_1 - r \param_2 \right\| \le C_{M, \homop} \left\| \param_1 - \param_2 \right\|,
\end{align*}
where $(i)$ follows from Assumption (A2) in \Cref{def:nearhomo}, and $C_{M, \homop}$ is a constant depending on $M$ and $\homop$. Taking $r \to \infty$ yields
\begin{equation*}
\left\vert \homoPredictor (\param_1; \xB) - \homoPredictor (\param_2; \xB) \right\vert \le \, C_{M, \homop} \left\| \param_1 - \param_2 \right\|,
\end{equation*}
i.e., $\homoPredictor (\param; \xB)$ is locally Lipschitz continuous. The error bound follows directly from our previous calculation:
\[
\left\vert f(\param; \xB) - \homoPredictor (\param; \xB) \right\vert = \, \left\vert g(1) - g(+\infty) \right\vert \le \homop_a( \left\| \param \right\| ).
\]
At last, we show that if $f$ is differentiable and $\nabla f$ satisfies \Cref{asp:strongerhomo}, $f_{\homo}$ is differentiable.  This is equivalent to showing that the following limit exists for any $\theta$ and $j\in [d]$: 
\[
\lim_{t\to 0} \frac{f_\homo(\param + t e_j;\xB) - f_\homo(\param;\xB)}{t}. 
\]
For simplicity, we omit the $\xB$ in the function. Note that 
\[
\lim_{t\to 0} \frac{f_\homo(\param + t e_j) - f_\homo(\param)}{t} = \lim_{t\to 0} \lim_{r_n \to \infty} \frac{f(r_n\param + r_n t e_j) - f(r_n\param)}{r_n^M t},
\]
due to the definition of $f_{\homo}$.  We let 
\[
    \tilde f_n(t) \coloneqq \frac{f(r_n\param + r_n t e_j) - f(r_n\param)}{r_n^M t}, \quad \tilde f_{\homo}(t) \coloneqq \frac{f_\homo(\param + t e_j) - f_\homo(\param)}{t}. 
\]
And we want to invoke \Cref{lem:exchange-limit} to finish this. 
To invoke \Cref{lem:exchange-limit}, we need to verify two conditions.  
\begin{itemize}
    \item Condition 1: There exists an open interval $(-a, a)$ such that $(\tilde f_n(t))_{n=1}^\infty$ converges uniformly to $\tilde f_\homo (t)$.
    \item Condition 2: Dor any fixed $n$, as $t \to 0$, $\tilde f_n(t)$ converges to $[\nabla f]_j (r_n \param) / r_n^{M-1}$. 
\end{itemize}
Once these two conditions are verified, we can exchange the limit and get
\[
\lim_{n\to \infty} \frac{[\nabla f]_j (r_n \param)}{r_n^{M-1}} =\lim_{n\to \infty} \lim_{t\to 0} \tilde f_n(t) = \lim_{t\to 0} \lim_{n\to \infty}  \tilde f_n(t) = \lim_{t\to 0} \tilde f_{\homo} (t). 
\]
By \Cref{asp:strongerhomo}, this will lead to: 
\[
[(\nabla f)_\homo ]_j (\param) = \lim_{t\to 0} \tilde f_{\homo} (t) = \lim_{t\to 0} \frac{f_\homo(\param + t e_j) - f_\homo(\param)}{t} = [\nabla (f_\homo)]_j (\param). 
\]
This indicates that $\nabla (f_\homo)$ exists and is equal to $(\nabla f)_\homo$. Now we focus on verifying the two conditions.

\myunder{Verify condition 1.} 
Here we fix some $\param \in \Rbb^d / \{0\}$ and pick some special $a$ such that 
\[
\argmax_{t \in [-a,a], j \in [d]} \| \param + t e_j \| \le \bar{r}_*, \quad \argmin_{t \in [-a,a], j \in [d]} \| \param + t e_j \| \ge \underline{r}_*,
\]
for some $\bar{r}_*,\underline{r}_* > 0$. Then we will show that for any $t\in[a, -a]$ $\tilde f_n(t)$ converges uniformly to 
\[
\int_0^1  \langle (\nabla f)_\homo (\param + s \cdot t e_j), t e_j  \rangle\, \mathrm{d}s. 
\] 
The first thing to do is to show this intergral is well-defined. 
Note that $\nabla f$ is measurable. Hence, the limit $(\nabla f)_\homo$ is also measurable. Furthermore, we have $\|\nabla f \| \le \homoq'(x)$ where $\deg \homoq \le M$. We must have 
\[
    (\nabla f)_\homo (\param) \coloneqq \lim_{r\to \infty} \frac{\nabla f(r\param)}{r^{M-1}} \le \lim_{r\to \infty} \frac{\homoq'(r\|\param\|)}{r^{M-1}} \le C \|\param\|^{M-1},
\]
Since $\arg \max_{s \in [0,1],t\in[-a, a]} \|\param + s \cdot t e_j\|$ is bounded, the integral is well-defined.

Now, we will show the uniform convergence. By \Cref{asp:strongerhomo}, we know for any $\epsilon$:
\[
\lim_{x\to \infty} \frac{\homor(x)}{ x^{M-1}} =0 \implies  \exists M, \text{ if } x \ge M, \homor(x) \le \frac{\epsilon}{(\bar r_*)^{M-1}} \cdot x^{M-1}. 
\]
Furthermore, since $r_n \to \infty$, there exists an $M'$ such that for any $n>M'$: 
\[
r_n \ge M/\underline{r}_* \implies \argmin_{t \in [-a,a], j \in [d]} r_n \| \param + t e_j \| \ge r_n  \underline{r}_* \ge M.
\]
Therefore, for any $t\in [-a, a]$ and $n>M'$,
\begin{align*}
&\tilde f_n(t)  - \int_0^1  \langle (\nabla f)_\homo (\param + s \cdot t e_j), t e_j  \rangle\, \mathrm{d}s \\ 
&= 
\frac{f(r_n\param + r_n t e_j) - f(r_n\param)}{r_n^M t}  - \int_0^1  \langle (\nabla f)_\homo (\param + s \cdot t e_j), t e_j  \rangle\, \mathrm{d}s
\\ 
&= \frac{1}{r_n^M t}  \int_{0}^1 \langle \nabla  f(r_n\param + s\cdot  r_n t e_j),  r_n t e_j \rangle \, \mathrm{d}s  - \int_0^1  \langle (\nabla f)_\homo (\param + s \cdot t e_j), t e_j  \rangle\, \mathrm{d}s
\\ 
&=   \int_{0}^1 \bigg\langle   \frac{\nabla f(r_n\param + s\cdot  r_n t e_j)}{r_n^{M-1}} - (\nabla f)_\homo (\param + s \cdot t e_j),   e_j \bigg\rangle \, \mathrm{d}s
\\ 
&=   \int_{0}^1 \bigg\langle  \frac{\nabla  f(r_n\param + s\cdot  r_n t e_j)-  (\nabla f)_\homo (r_n\param + s \cdot r_n t e_j)}{r_n^{M-1}} ,   e_j \bigg\rangle \, \mathrm{d}s
\\ 
&\le   \int_{0}^1   \frac{ \homor ( \|r_n\param + s\cdot  r_n t e_j\|)}{r_n^{M-1}}  \, \mathrm{d}s
\\
&\le \int_{0}^1   \frac{\epsilon}{(\bar r_*)^{M-1}} \cdot \frac{  \|r_n\param + s\cdot  r_n t e_j\|^{M-1}}{r_n^{M-1}}  \, \mathrm{d}s 
\\ 
&\qquad  \explain{ By $\argmin_{t \in [-a,a], j \in [d]} r_n \| \param + t e_j \|  \ge M$ }
\\
&= \int_{0}^1    \frac{\epsilon}{(\bar r_*)^{M-1}}  \|\param + s\cdot   t e_j\|^{M-1}  \, \mathrm{d}s
\\
&\le \frac{\epsilon}{(\bar r_*)^{M-1}} (\bar r_*)^{M-1} = \epsilon.
\end{align*}

At last, since we already know that $\tilde f_n(t)$ pointwise converges to $\tilde f_H(t)$ and $\tilde f_n(t)$ uniformly converges to $\int_0^1  \langle (\nabla f)_\homo (\param + s \cdot t e_j), t e_j  \rangle\, \mathrm{d}s$, we must have $\tilde f_h(t)$ uniformly converges to $\tilde f_H(t)$ and 
\[
\tilde f_H(t) = \int_0^1  \langle (\nabla f)_\homo (\param + s \cdot t e_j), t e_j  \rangle\, \mathrm{d}s. 
\]

\myunder{Verify condition 2.} 
Recall that $f$ is differentiable. Hence, for any fixed $n$, as $t \to 0$, we have 
\[
    \lim_{t \to 0} \tilde f_n(t) = \lim_{t \to 0} \frac{f(r_n\param + r_n t e_j) - f(r_n\param)}{r_n^{M-1} \cdot ( r_n t )} = \frac{1}{r_n^{M-1}} [\nabla f]_j (r_n \param). 
\]
Therefore, $f_{\homo}$ is continuously differentiable on $\Rbb^d / \{0\}$.
 Note that this further shows that 
\[
    \nabla (f_\homo) = (\nabla f)_\homo. 
\]
This completes the proof of \Cref{thm:homogenization}.
\end{proof}

\subsection{Proof of Proposition~\ref{thm:block_compo_homogenization}}
\begin{proof}[Proof of \Cref{thm:block_compo_homogenization}]
This proof mainly has two parts. 

\begin{itemize}
    \item Part I: We show that 
    $$
    \lim_{r_1, r_2 \to \infty} \frac{f(r_1 \param; r_2 \xB)}{r_1^{M_1} r_2^{M_2}}
    $$
    is well-defined, continuous, almost everywhere differentiable and $(M_1,M_2)$-homogeneous.
    \item Part II: We show that 
    \begin{equation}
    \label{eq:homo_decomp_f}
        f_\homo(\param;\xB) = s_{\homo, \param_1}^1 \circ s_{\homo, \param_2}^2 \circ \cdots \circ s_{\homo, \param_L}^L (\xB).
    \end{equation}
\end{itemize}

\mybf{Part I.}
Similar to the proof of \Cref{thm:homogenization}, we define for fixed $\param$ and $\xB$:
\begin{equation*}
g(r_1, r_2) = \frac{s(r_1 \param, r_2 \xB)}{r_1^{M_1} r_2^{M_2}}.
\end{equation*}
By Assumption (B1), we obtain the following estimates on the partial derivatives of $g$:
\begin{align*}
\partial_{r_1} g(r_1, r_2) = \, & \frac{1}{r_1^{M_1 + 1} r_2^{M_2}} \left\| \nabla_{\param} s(r_1 \param; r_2 \xB) r_1 \param - M_1 s(r_1 \param; r_2 \xB) \right\| \le \frac{\homop_s' (r_1 \| \param \|) \homor_s (r_2 \| \xB \|)}{r_1^{M_1 + 1} r_2^{M_2}}, \\
\partial_{r_2} g(r_1, r_2) = \, & \frac{1}{r_1^{M_1} r_2^{M_2 + 1}} \left\| \nabla_{\xB} s(r_1 \param; r_2 \xB) r_2 \xB - M_2 s(r_1 \param; r_2 \xB) \right\| \le \frac{\homop_s (r_1 \| \param \|) \homor_s' (r_2 \| \xB \|)}{r_1^{M_1} r_2^{M_2 + 1}}.
\end{align*}
As a consequence, for $r_3 > r_1$ and $r_4 > r_2$, we have
\begin{align*}
\left\vert g(r_3, r_4) - g(r_1, r_2) \right\vert \le \, & \left\vert g(r_3, r_4) - g(r_3, r_2) \right\vert + \left\vert g(r_3, r_2) - g(r_1, r_2) \right\vert \\
\le \, & \frac{\homop_s (r_3 \| \param \|)}{r_3^{M_1}} \int_{r_2}^{r_4} \frac{\homor_s' (r \| \xB \|)}{r^{M_2 + 1}} d r + \frac{\homor_s (r_2 \| \xB \|)}{r_2^{M_2}} \int_{r_1}^{r_3} \frac{\homop_s' (r \| \param \|)}{r^{M_1 + 1}} \mathrm{d} r \\
\stackrel{(i)}{\le} \, & \frac{\homop_s (r_1 \| \param \|)}{r_1^{M_1}} \int_{r_2}^{+\infty} \frac{\homor_s' (r \| \xB \|)}{r^{M_2 + 1}} d r + \frac{\homor_s (r_2 \| \xB \|)}{r_2^{M_2}} \int_{r_1}^{+\infty} \frac{\homop_s' (r \| \param \|)}{r^{M_1 + 1}} \mathrm{d} r,
\end{align*}
where $(i)$ follows from the fact that the coefficients of $\homop_s$ (resp. $\homor_s$) are all non-negative, so $r \mapsto \homop_s (r \| \param \|) / r^{M_1}$ (resp. $r \mapsto \homor_s (r \| \param \|) / r^{M_2}$) is non-increasing. Further, we can check that the right hand side goes to $0$ as $r_1, r_2 \to \infty$. Using a similar argument as in the proof of \Cref{thm:homogenization}, we deduce that $s_\homo (\param; \xB) = \lim_{r_1, r_2 \to \infty} g(r_1, r_2)$ exists, and is $M_1$-homogeneous in $\param$ and $M_2$-homogeneous in $\xB$. This completes the proof of well-definedness. The proof of local Lipschitz is completely the same as that in the proof of \Cref{thm:homogenization}.

\mybf{Part II.}
We prove via induction on $L$. For $L = 1$, \Cref{eq:homo_decomp_f} holds automatically since $f(\param; \xB) = s^1 (\param_1; \xB)$ and $\param = \param_1$. Now assume the conclusion holds for $L-1$, then we can write 
\begin{equation*}
f(\param; \xB) = s^1 (\param_1; f^{L-1}(\param^{L-1}; \xB)),
\end{equation*}
where $\param^{L-1} = (\param_2, \cdots, \param_L)$, and
\begin{equation*}
f^{L-1}(\param^{L-1}; \xB) = s_{\param_2}^2 \circ \cdots \circ s_{\param_L}^L (\xB).
\end{equation*}
Therefore,
\begin{equation*}
\homoPredictor (\param; \xB) = \, \lim_{r_1, r_2 \to \infty} \frac{f(r_1 \param; r_2 \xB)}{r_1^{M_1} r_2^{M_2}} = \lim_{r_1, r_2 \to \infty} \frac{s^1 (r_1 \param_1; f^{L-1}(r_1 \param^{L-1}; r_2 \xB))}{r_1^{M_1} r_2^{M_2}}.
\end{equation*}
Define $M_1' = \sum_{j=2}^{L} M_1^j \cdot \prod_{i=2}^{j-1} M_2^i$ and $M_2' = \prod_{j=2}^{L} M_2^j$. By our induction assumption,
\begin{align*}
f^{L-1}(r_1 \param^{L-1}; r_2 \xB) 
= \, & r_1^{M_1'} r_2^{M_2'} \left( f_\homo ^{L-1}(\param^{L-1}; \xB) + o(1) \right) \\
= \, & r_1^{M_1'} r_2^{M_2'} \left( s_{\homo, \param_2}^2 \circ \cdots \circ s_{\homo, \param_L}^L (\xB) + o(1) \right).
\end{align*}
Use a similar estimate as in the proof of \Cref{thm:homogenization}, we obtain that
\begin{align*}
& \left\vert s^1 (r_1 \param_1; f^{L-1}(r_1 \param^{L-1}; r_2 \xB)) - s^1 \left( r_1 \param_1; r_1^{M_1'} r_2^{M_2'} f_{\homo}^{L-1}(\param^{L-1}; \xB) \right) \right\vert \\
\stackrel{(i)}{\le} \, & C (p, q, \param, \xB) r_1^{M_1^1} (r_1^{M_1'} r_2^{M_2'})^{M_2^1 - 1} o(r_1^{M_1'} r_2^{M_2'}) \stackrel{(ii)}{=} o(r_1^{M_1} r_2^{M_2}),
\end{align*}
where $(i)$ follows from Assumption (B2), and the constant $C (p, q, \param, \xB)$ might depend on the polynomials $(p_s, \homor_s)$ associated with $s_{\param_i}^i$ and $(\param, \xB)$, but not $r_1$ and $r_2$, and in $(ii)$ we use the identity $M_1 = M_1^1 + M_1' M_2^1$, $M_2 = M_2' M_2^1$. This implies that
\begin{align*}
    \homoPredictor (\param; \xB) = \, & \lim_{r_1, r_2 \to \infty} \frac{s^1 \left( r_1 \param_1; r_1^{M_1'} r_2^{M_2'} f_\homo ^{L-1}(\param^{L-1}; \xB) \right)}{r_1^{M_1} r_2^{M_2}} \\
    = \, & \lim_{r_1, r_2 \to \infty} \frac{s^1 \left( r_1 \param_1; r_1^{M_1'} r_2^{M_2'} f_\homo^{L-1}(\param^{L-1}; \xB) \right)}{r_1^{M_1^1} \cdot (r_1^{M_1'} r_2^{M_2'})^{M_2^1}} \\
    = \, & s_M^1 \left( \param_1; f_\homo^{L-1}(\param^{L-1}; \xB) \right) \\
    = \, & s_{\homo, \param_1}^1 \circ s_{\homo, \param_2}^2 \circ \cdots \circ s_{\homo, \param_L}^L (\xB),
\end{align*}
completing the proof of \Cref{thm:block_compo_homogenization}.
\end{proof}

\subsection{Proof of Proposition~\ref{cor:Initial condition via homogeneization}}
\begin{proof}[Proof of \Cref{cor:Initial condition via homogeneization}]
First, we let 
\[
\gamma = \arg\min_{i\in [n]} y_i f_{\homo}(\param^\prime ;\xB_i) >0.   
\] 
If we replace $\param$ with $c \param^\prime $, then we have 
\[
\arg\min_{i\in [n]} y_i f_{\homo}(c \param^\prime ;\xB_i) = \gamma c^M. 
\]
Therefore, 
\[
\arg\min_{i\in [n]} y_i f(c \param^\prime ;\xB_i) \ge \arg\min_{i\in [n]} y_i f_{\homo}(c \param^\prime ;\xB_i) - \homop_a(c \|\param^\prime\|) = \gamma c^M - \homop_a(c \|\param^\prime\|). 
\]
Hence, 
\[
\Loss(c\param') \le \exp( - \gamma c^M + \homop_a(c \|\param^\prime\|) ) . 
\]
To satisfy the initial condition in \Cref{asp:initial-cond-gf}, we need to find $c$ such that
\[
- \gamma c^M + \homop_a(c \|\param^\prime\|) < -\homop_a(c \|\param^\prime\|) - \log n \Leftrightarrow \gamma c^M > 2  \homop_a(c \|\param^\prime\|) + \log n. 
\]
Since $\deg p_a \le M-1$, we can find $c$ such that the above inequality holds. Similar argument can be applied to the case in \Cref{asp:initial-cond-gd}. This completes the proof of \Cref{cor:Initial condition via homogeneization}.
\end{proof}


\subsection{Two-Layer Example}
The following subsections are devoted to the proof of the two-layer example in \Cref{sec:example}. We first show that the GD dynamics is symmetric and the limit solution of the GD dynamics is the solution to the ODE \eqref{eq: dynamic of w and a}. Then, we show that the loss is of rate $O(1/t)$ and the parameter norm $\|\param_t\|$ is of rate $O(\sqrt{\log t})$. Combining these two rates, we have 
\[
\frac{\Loss(\param_t)}{\|\param_t\|} = \Omega(\sqrt{\log t}) \to  \infty.  
\]

The first thing we are going to show here is that due to symmetry, the limit of GD can be reduced to a GF with the following dynamics. 

\begin{lemma}
[Symmetry of the parameters]
\label{lem:symmetry-toy}
We assume the dataset satisfies \Cref{asp: symmetric}, and the model is \eqref{eq: toy-resnet} with initial condition  $\wB_{1,0} = \wB_{2,0} = \boldsymbol 0$ and $a_{1,0} = a_{2,0} = 0$. Then, $\wB_{1,t} = \wB_{2,t}= \wB_t$ and $a_{1,t} = a_{2,t} = a_t$ is a solution to the gradient flow dynamics \eqref{eq:GF}, where $\wB_t$ and $a_t$ satifies the following ODE:  
\begin{equation}
\label{eq: dynamic of w and a}
\begin{cases}
\dot \wB_t = \frac{1}{n} \sum_{i=1}^n e^{-y_i f(\param_t;\xB_i)} (1+ c_L a)y_i \xB_i\\ 
\dot a_t = \frac{1}{n} \sum_{i=1}^n e^{-y_i f(\param_t;\xB_i)}  c_L y_i \xB_i^\top \wB_t
\end{cases}
\end{equation}
initialized at $(\wB_0, a_0) = (\boldsymbol{0}, 0)$, with $c_L \coloneqq (1+\alpha_L)/2$. Furthermore, $(\wB_t, a_t)$ is the limit of the GD dynamics \eqref{eq:GD} with respect to \Cref{eq: toy-resnet}, as the step size goes to $0$: 
\[
(\wB_t, a_t) = \, \lim_{\eta \to 0} ( \wB_{\eta t}, a_{\eta t} ). 
\]
\end{lemma}
\begin{proof}[Proof of Lemma~\ref{lem:symmetry-toy}]

\mybf{Step 1.} We show that the GD dynamics \eqref{eq:GD} is symmetric.  

\noindent We can prove this by induction. We use the following notation for GD: $\param_t = (a_{1,t}, a_{2,t},\wB_{1,t}, \wB_{2,t})$. For $t=0$, we have $\wB_{1,0} = \wB_{2,0} = \vec 0$ and $a_{1,0} = a_{2,0} = 0$. Assume that $\wB_{1,t} = \wB_{2,t}$ and $a_{1,t} = a_{2,t}$ for some $t$. Under these induction assumptions, 
\begin{align*}
    y_i f(\param_t;\xB_i) &= y_i \wB_{1,t}^\top \xB_i + y_i \wB_{2,t}^\top \xB_i + y_i a_{1,t} \varphi(\wB_{1,t}^\top \xB_i) - y_i a_{2,t} \varphi(-\wB_{2,t}^\top \xB_i) \\ 
    &= 2y_i \wB_{1,t}^\top \xB_i + y_i (1+ \alpha_L)a_{1,t} \varphi(\wB_{1,t}^\top \xB_i)  \qquad \explain{ Since $\varphi(x) - \varphi(-x) = (1+\alpha _L) x$ } \\ 
    & = y_{i+n/2} f(\param_t;\xB_{i+n/2}), \qquad \explain{ By \cref{asp: symmetric} } 
\end{align*}
for $i=1,\ldots ,n/2$. Then, we have 
\begin{align*}
    \nabla_{\wB_1} \Loss (\param_t) &= \frac{1}{n} \sum_{i=1}^n e^{-y_if_i(\param_t)} (1+ c_L a_{1,t} \varphi'(\wB_{1,t}^\top \xB_i))y_i \xB_i \\ 
    &= \frac{1}{n} \sum_{i=1}^n e^{-f_i(\param_{t+n /2})} (1+ c_L a_{2,t} \varphi'(-\wB_{2,t}^\top \xB_{i+n/2}))y_{i+n/2} \xB_{i+n/2} \\ 
    & \qquad \explain{ By the symmetry of the dataset and the paremeter }  \\
    & = \nabla_{\wB_2} \Loss (\param_t). 
\end{align*}
Therefore, we have $\wB_{1,t+1} = \wB_{2,t+1}$ for all $t$. Similarly, we can show that $a_{1,t+1} = a_{2,t+1}$ for all $t$.

\textbf{Step 2.} We show that the limit solution of the GD dynamics is the solution to \eqref{eq: dynamic of w and a}. 

\noindent Since the GD dynamics is symmetric, we can treat $(a_{1,t}, \wB_{1,t})$  as the result of Euler method with time step $\eta$ on the ODE \eqref{eq: dynamic of w and a}. Since the ODE is locally Lipschitz continuous, there exists a unique solution $(a_t,\wB_t)$ to the ODE \eqref{eq: dynamic of w and a}, and the limit of results of Euler method as $\eta$ goes to zero is the solution. At last, it's easy to verify that $(a_t,a_t, \wB_t, \wB_t)$ satifies the \eqref{eq:GF} for the model in \eqref{eq: toy-resnet} with respect to the zero initialization.

This completes the proof of \Cref{lem:symmetry-toy}.
\end{proof}
\subsection{Loss Convergence}  \label{sec:proof:eg-loss}

Therefore, we can focus on the ODE \eqref{eq: dynamic of w and a}. We will show that the loss is of rate $O(1/t)$ and the parameter norm $\|\param_t\|$ is of rate $O(\sqrt{\log t})$. The main trick we use here is the balancing of the parameters.

\begin{lemma}
[Parameter balancing]
\label{lem: Parameter balancing} 
Consider the dynamic of the parameters in \eqref{eq: dynamic of w and a} with initialization $\wB_0 = \vec 0$ and $a_0=0$. Then we have
\begin{equation}
\label{eq: balancing}
    \bigg(a_t + \frac{1}{c_L} \bigg)^2 = \| \wB_t\|^2 + \frac{1}{c_L^2}. 
\end{equation}

\end{lemma}

\begin{proof}
Note that 

\begin{align*}
    \frac{d}{dt} \bigg (a_t + \frac{1}{c_L}\bigg)^2 
    & = 2 \dot a_t \bigg(a_t + \frac{1}{c_L}\bigg) \\
    & = \frac{1}{n} \sum_{i=1}^n e^{-y_i f(\param_t;\xB_i)} y_i \xB_i^\top \wB_t\big(2c_La_t + 2\big)\\ 
    & = 2 \langle \dot \wB_t, \wB_t \rangle = \frac{d}{dt} \|\wB_t\|_2^2. 
\end{align*}
Therefore, we have 
\[
    \bigg(a_t + \frac{1}{c_L} \bigg)^2 - \| \wB_t\|^2  = \bigg(a_0 + \frac{1}{c_L} \bigg)^2 - \| \wB_0\|^2  = \frac{1}{c_L^2}. 
\]
This completes the proof of \Cref{lem: Parameter balancing}.
\end{proof}

Then, we can show that $a_t>0$ for all $t$.

\begin{lemma}
[Positiveness of $a_t$]
\label{lem: positiveness of a_t }
Consider the dynamic of the parameters in \eqref{eq: dynamic of w and a} with initialization $\wB_0 = \vec 0$ and $a_0=0$. We have $a_t\ge 0$ for all $t$. 
\end{lemma}
\begin{proof}
Note that 
\[
    \bigg(a_t + \frac{1}{c_L} \bigg)^2 = \| \wB_t\|^2 + \frac{1}{c_L^2} \ge \frac{1}{c_L^2} \implies a_t \le -\frac{2}{c_L} \text{ or } a_t \ge 0.  
\]
if there exits $T$ such that $a_t< - 2/ c_L$. Due to the continuity of $a_t$ and $a_0=0$, there exists $0<T^\prime <T$, such that $a(T^\prime) = -1/c_L$. Then, it contradicts with the condition that $ a_t\le -2/c_L \text{ or } a_t \ge 0$. This completes the proof of \Cref{lem: positiveness of a_t }.
\end{proof}

Now we can already write $a_t$ as a function of $\wB_t$, i.e., 
\[
    a_t = \sqrt{\| \wB_t\|^2 + \frac{1}{c_L^2}} - \frac{1}{c_L}. 
\]
We will show that the the loss is of rate $O(1/t)$.

\begin{lemma}
[Loss upper bound]  
\label{lem: toy model loss upper bound}   
Given the dataset satisfying \Cref{asp: symmetric} and the model in \eqref{eq: toy-resnet} with initial condition  $\wB_{1,0} = \wB_{2,0} = \vec 0$ and $a_{1,0} = a_{2,0} = 0$. Then we must have 
\[
    \Loss (\param_T) \le \frac{1+ \log^2 (T) / 4 \gamma_*^2}{T}. 
\]
\end{lemma}
\begin{proof}
Let $\wB^c = \beta \wB_*$. Note that 
\begin{align*}
    &\lefteqn{\frac{\mathrm{d} }{\mathrm{d}  t} \| \wB_t - \wB^c\|^2 
    =  2 \langle \dot \wB_t, \wB_t - \wB^c \rangle }\\ 
    &=  \frac{2}{n} \sum_{i=1}^n e^{-y_i f(\param_t;\xB_i)} (1+ c_L a)y_i \xB_i^\top (\wB_t - \wB^c) \\ 
    & = \frac{2}{n} \sum_{i=1}^n e^{-y_i f(\param_t;\xB_i)} \bigg( \frac{1}{2} y_i f(\param;\xB_i) -  (1+c_L a) \beta y_i \xB_i^\top \wB_* \bigg) &&\explain{ By $f(\param;\xB) = (2+ 2c_L a) \wB^\top \xB$}\\ 
    & = \frac{1}{n}\sum_{i=1}^n e^{-y_i f(\param_t;\xB_i)} \bigg(  y_i f(\param;\xB_i) -  2(1+c_L a) \beta y_i \xB_i^\top \wB_* \bigg) \\ 
    & \le \frac{1}{n}\sum_{i=1}^n e^{-y_i f(\param_t;\xB_i)} \big(  y_i f(\param;\xB_i) -  2\beta \gamma_* \big) &&\explain{ by $y_i\xB_i^\top \wB_*\ge \gamma_*$ and $a\ge 0$ } \\ 
    & = \frac{1}{n}\sum_{i=1}^n -e^{-y_i f(\param_t;\xB_i)} \big(   2\beta \gamma_* - y_i f(\param;\xB_i)  \big) \\ 
    & \le e^{-2\beta \gamma_*} - \Loss (\param_t).  &&\explain{ By convexity of  $e^{-x}$}
\end{align*}
This implies that 
\[
    \|\wB_{T} - \wB^c\|^2 + \int_0^T \Loss (\param_t)\, \mathrm{d}  t \le T e^{-2\beta \gamma_*}  + \|\wB_0 - \wB^c\|^2.  
\]
Note that $\Loss (\param_t)$ is decreasing, $\wB_0 = \vec 0$ and we can set $\beta = \log (T) / (2\gamma^*)$. We have 
\[
    T \Loss(\param_t) \le 1  + \log^2 (T) / 4 \gamma_*^2.
\]
This completes the proof of \Cref{lem: toy model loss upper bound}.
\end{proof}
\subsection{Parameter Norm Bounds} \label{sec:proof:eg-parameter}
Now we can use the loss bound to achieve a bound for the parameter norm. 

\begin{lemma}
[Parameter norm bound]
\label{lem: Parameter norm bound}
Under the same assumptions of \Cref{lem: toy model loss upper bound}, when \[t \ge \max \{ 16 , ( 4/ \gamma_* \log (4 / \gamma_* ))^4  \},\] we have  
\begin{equation}
\label{eq: parameter norm bound}
    \|\wB_t\|^2 c_L^2 + 1 \ge \frac{ \log t}{16}. 
\end{equation}
\end{lemma}
\begin{proof}
Applying the balancing equation \cref{eq: balancing}, we have 
\begin{align*}
    \Loss (\param_t) 
    &= \frac{1}{n} \sum_{i=1}^n e^{-y_i f(\param_t;\xB_i)} \\ 
    &= \frac{1}{n} \sum_{i=1}^n e^{-y_i \xB_i^\top  \wB_t \cdot 2 \sqrt{\|\wB_t\|^2 c_L^2 +1} } \\ 
    &\ge e^{- 2 \|\wB_t\| \sqrt{\|\wB_t\|^2 c_L^2 +1} } &&\explain{ By $\|\xB_i\|\le 1$ }. 
\end{align*}
Combining this with the Loss upper bound in \Cref{lem: toy model loss upper bound}, we have
\begin{align}
    \label{eq: log mess}
    2 \|\wB_t\| \sqrt{\|\wB_t\|^2 c_L^2 +1} &\ge  - \log (\Loss (\param_t) ) \notag\\ 
    &\ge  \log (t) - \log \bigg( 1+ \frac{\log^2t}{4 \gamma_*^2} \bigg) .
\end{align}
Since $t \ge \max \{ 16 , ( 4/ \gamma_* \log (4 / \gamma_* ))^4  \}$, we have $1 \le \log^2 t/4 \gamma_*^2$. Therefore, we have 
\begin{align*}
    \log \bigg( 1+ \frac{\log^2t}{4 \gamma_*^2} \bigg) 
    &\le \log \bigg(  \frac{2\log^2t}{4 \gamma_*^2} \bigg) \le \log 2 + 2 \log \bigg(  \frac{\log t}{2 \gamma_*} \bigg) \\ 
    &\le \log2 + 2 \log (t^{\frac{1}{4}}) &&\explain{ By \Cref{lem: log over linear}}\\
    & = \log 2 + \frac{1}{2} \log t \le  \frac{3}{4}\log t. &&\explain{ By $t \ge 16$}
\end{align*}
Plugging in this into \Cref{eq: log mess}, we have 
\[
    \frac{1}{4} \log t \le  2 \|\wB_t\| \sqrt{\|\wB_t\|^2 c_L^2 +1} \le  \frac{2}{c_L} \big(\| \wB_t\|^2 c_L^2 + 1 \big). 
\]
At last, we use that $c_L >1/2 $ and we have 
\[
    \|\wB_t\|^2 c_L^2 + 1 \ge \frac{1}{16} \log t.
\]
This completes the proof of \Cref{lem: Parameter norm bound}.
\end{proof}

Now we turn back to the analysis of the loss and try to extract the upper bound of the norm. Now we use $t_0$ to denote 
$ \max \{ 16 , ( 4/ \gamma_* \log (4 / \gamma_* ))^4  \}$.

\begin{lemma}
[Parameter norm upper bound]
\label{lem: Parameter norm upper bound}
Under the same assumptions of \Cref{lem: toy model loss upper bound}, let $t_0 \coloneqq \max \{ 16 , ( 4/ \gamma_* \log (4 / \gamma_* ))^4  \}$. Then we have  
\begin{equation}
\label{eq: parameter norm upper bound}
    \|\wB_t \|_2 \le  \bigg(\frac{4}{\gamma_*}+1 \bigg) \sqrt{\log t} + \|\wB_{t_0}\|_2, \quad \forall t\ge t_0. 
\end{equation}
\end{lemma}
\begin{proof}
Just like the proof of \Cref{lem: toy model loss upper bound},
we let $\wB^c = \beta \wB_*$. Note that 
\begin{align*}
    &\frac{d}{dt} \| \wB_t - \wB^c\|^2 =  2 \langle \dot \wB_t, \wB_t - \wB^c \rangle \\ 
    = \, & \frac{2}{n} \sum_{i=1}^n e^{-y_i f(\param_t;\xB_i)} (1+ c_L a)y_i \xB_i^\top (\wB_t - \wB^c) \\ 
    = \, & \frac{2}{n} \sum_{i=1}^n e^{-y_i f(\param_t;\xB_i)} \bigg( \frac{1}{2} y_i f(\param;\xB_i) -  (1+c_L a) \beta y_i \xB_i^\top \wB_* \bigg) &&\explain{ By $f(\param;\xB) = (2+ 2c_L a) \wB^\top \xB$}\\ 
    = \, & \frac{1}{n}\sum_{i=1}^n e^{-y_i f(\param_t;\xB_i)} \bigg(  y_i f(\param;\xB_i) -  2 \sqrt{\|\wB_t\|^2 c_L^2 + 1}  \beta y_i \xB_i^\top \wB_* \bigg) &&\explain{ By \Cref{eq: balancing}}  \\ 
    \le \, & \frac{1}{n}\sum_{i=1}^n e^{-y_i f(\param_t;\xB_i)} \bigg(  y_i f(\param;\xB_i) -  \frac{1}{2} \sqrt{\log t}  \beta \gamma_* \bigg) &&\explain{ By \Cref{eq: parameter norm bound}  }\\ 
    \le \, & e^{- \beta \gamma_* \sqrt{\log t} / 2} - \Loss (\param_t).  &&\explain{ By convexity of  $e^{-x}$}\\ 
    \le \, & e^{- \beta \gamma_* \sqrt{\log t} / 2}. 
\end{align*}
 We integral the above inequality from $t_0$ to $t$ and we have
\[
    \|\wB_T- \wB^c\|^2 \le \|\wB_{t_0} - \wB^c\|^2  + \int_{t_0}^t e^{- \beta \gamma_* \sqrt{\log t} / 2} dt.
\] 
Here we set $\beta = 2\sqrt{ \log T}/\gamma_*$, then we have  
\begin{align*}
        \|\wB_{T}- \wB^c\|^2  
        &\le \|\wB_{t_0} - \wB^c\|^2  + \int_{t_0}^T e^{- \sqrt{\log T \log t} } dt\\ 
        &\le \|\wB_{t_0} - \wB^c\|^2  + \int_{t_0}^T  \frac{1}{t} dt &&\explain{ By $\log T\log t \ge \log^2 t$ } \\ 
        &\le \|\wB_{t_0} - \wB^c\|^2 + \log T. 
\end{align*}
Hence, 
\begin{align*}
        \|\wB_t\|
        &\le \|\wB_T- \wB^c\|  + \|\wB^c\| \\ 
        &\le \|\wB_{t_0} - \wB^c\| + \sqrt{ \log T}  + \|\wB^c\| \\ 
        & \le \|\wB_{t_0}\| + 2 \|\wB^c\| + \sqrt{ \log T}\\ 
        & = \|\wB_{t_0}\| + \frac{4}{\gamma_*} \sqrt{\log T}  + \sqrt{ \log T}.  
\end{align*}
This completes the proof of \Cref{lem: Parameter norm upper bound}.
\end{proof}
Now we can give a theorem about the margin. 
\subsection{Proof of Theorem~\ref{thm: Near-two homogeneous example achieves init bound}}
\begin{proof}[Proof of Theorem~\ref{thm: Near-two homogeneous example achieves init bound}]
By \Cref{lem: toy model loss upper bound} and \Cref{lem: Parameter norm upper bound}, we have 
\[
    \frac{\log 1/\Loss (\param_t)}{\|\param_t\|_2} \ge \frac{\log t - \log \bigg(1 + \frac{\log^2 t}{4 \gamma_*^2} \bigg)}{\|\wB_{t_0} \| + (1+ 4/ \gamma_*) \sqrt{\log t} } \to \infty, \text{ as } t\to \infty.
    \]
    Therefore, there exists $T$ such that the initial bound is achieved. Now we complete the proof of \Cref{thm: Near-two homogeneous example achieves init bound}.
\end{proof}

\section{Proofs for Section \ref{sec:margin-direct-gd}}  
In this section, we will provide the proofs omitted in Section \ref{sec:margin-direct-gd}. Recall the gradient descent dynamics:
\begin{equation}
\label{eq: gd}
        \param_{t+1} = \param_t - \eta \nabla \Loss(\param_t), \quad \text{ for } t\ge 0.
\end{equation}
In order to analyze the convergence of the GD, we need to focus on the modified loss. We define 
\begin{equation}\label{eq:modified-loss}
    \ModifiedLoss(\param_t) \coloneqq e^{\homop_a( \| \param_t\|)}\Loss(\param_t),
\end{equation}
where we recall the definition of $\homop_a$ from Eq.~\eqref{eq:def-pa-gd}. Note that we assume $\deg (\homop_a) \ge 1$ if $M \ge 2$. For notational simplicity, we denote 
\begin{equation*}
    \Loss_t \coloneqq \Loss (\param_t), \, \nabla \Loss_t \coloneqq \nabla_{\param} \Loss (\param_t), \quad \ModifiedLoss_t \coloneqq \ModifiedLoss(\param_t), \, \nabla \ModifiedLoss_t \coloneqq \nabla_{\param}  \ModifiedLoss(\param_t). 
\end{equation*}
We recall the definitions of the link functions
\begin{equation}
\label{eq:link-gd}
    \phi (x) = \log \frac{1}{n x}, \quad \Phi(x)= \log \phi(x) - \frac{2}{\phi(x)},
\end{equation}
and the modified margin: 
\begin{equation}\label{eq:modified_margin_gd}
    \GDmargin(\param) \coloneqq \frac{e^{\Phi(\ModifiedLoss(\param))}}{\|\param\|^M}. 
\end{equation}
We also recall $\GFmargin$, the modified margin defined for gradient flow:
\begin{equation}\label{eq:redefine_modified_margin}
    \GFmargin \big( \param \big)  \coloneqq \frac{1}{\| \param \|^M} \left( \log \frac{1}{n \Loss ( \param )} - \homop_a \big( \| \param \| \big) \right) = \frac{\phi(\ModifiedLoss (\param) )}{\| \param \|^M}.
\end{equation}
Finally, unless otherwise stated, the constants $B$ and $B_i$'s in this section only depend on $(M, \homop, \homoq, A)$ and $(\homor, \homos)$ in \Cref{asp:strongerhomo,asp:initial-cond-gd}.

\subsection{Preliminary Results} \label{sec:proof:gd:prelim}
We collect below some important properties of $\homop_a$ that will be used in this section.
\begin{lemma}\label{lem:modified_p_a}
    For all $M \ge 1$, the function $\homop_a: [0, +\infty) \to [0, +\infty)$ is increasing, and satisfies the followings:
    \begin{itemize}
        \item [(i)] $x \homop_a' (x) + \homop' (x) \le M \homop_a (x)$.
        \item [(ii)] If $M \ge 2$, there exists a constant $B_1$, such that for all $x \ge 0$:
        \begin{align}
            & \homop_a (x) \le B_1 \left( x^{M-1} + 1 \right), \quad \homop_a' (x) \le B_1 \cdot \left(x I \{ x \le 1 \} + x^{M-2} I \{ x > 1 \} \right), \\
            & \homop_a'' (x) \le B_1 \cdot \left( I \{ x \le 1 \} + x^{M-3} I \{ x > 1 \} \right).
        \end{align}
    \end{itemize}
\end{lemma}
\begin{proof}
    The conclusion of part $(i)$ for $M = 1$ can be directly verified, since $\homop_a (x) = a_1 / (M - 1/2)$. From now on, we assume $M \ge 2$. For part $(i)$, recall that $\homop(x) = \sum_{i=0}^{M}a_i x^i$ and 
\begin{align*}
 \homop_a (x) := 
     \begin{cases}
    \sum_{i=1}^{M-1} \frac{(i+1)a_{i+1}}{M-i} x^{i} + \frac{a_1}{M-1/2}, & x \ge 1, \\
    \sum_{i=2}^{M-1} \frac{(i+1)a_{i+1}}{M-i} x^{i}  + \frac{2 a_2}{M - 1} \frac{x^2 + 1}{2} + \frac{a_1}{M-1/2}, & 0 \le x < 1.
\end{cases}
\end{align*}
Hence, 
\begin{align*}
&\homop_a^\prime (x) + \homop^\prime (x) - M \homop_a(x) \\ 
&= \begin{dcases}
\sum_{i=1}^{M-1}  \big[\frac{i(i+1)a_{i+1}}{M-i} x^{i} + (i+1)a_{i+1} x^i - \frac{M a_{i+1}}{M-i}x^i  \big] + a_1 -  \frac{Ma_1}{M-1/2}  , & x \ge 1 \\
\begin{aligned}
& \sum_{i=2}^{M-1} \big[\frac{i(i+1)a_{i+1}}{M-i} x^{i} + (i+1)a_{i+1} x^i - \frac{M a_{i+1}}{M-i}x^i  \big] \\
& \quad + \frac{2 a_2}{M - 1} x^2 + 2a_2 x - \frac{2Ma_2}{M-1} \frac{x^2 +1 }{2} + a_1 -  \frac{Ma_1}{M-1/2},
\end{aligned} & 0 \le x < 1
\end{dcases} \\ 
&= \begin{dcases}
 -\frac{a_1}{M-1/2}  , & x \ge 1 \\
 - 2a_2 \frac{(x-1)^2 }{2} -  \frac{a_1}{M-1/2}, & 0 \le x < 1
\end{dcases} \\ 
&\le -  \frac{a_1}{M-1/2} \le 0. 
\end{align*}

Part $(ii)$ can be verified via direct calculation.
\end{proof}
The next lemma characterizes the derivatives of functions related to $\phi(x)$.

\begin{lemma}
[Derivatives of functions related to $\phi(x)$]
\label{lem: Derivative of phi}
Given the function $\phi(x) = \log (1/(nx))$, we have 
\[
 \big(\log \phi (x)\big)^\prime  = - \frac{1}{x \phi(x)}, \quad \bigg( \frac{1}{\phi(x)}\bigg)^\prime  = \frac{1}{\phi(x)^2 x}, \quad \text{ and } \quad \bigg( \frac{1}{\phi(x)^2}\bigg)^\prime  = \frac{2}{\phi(x)^3 x}.
\] 
\end{lemma}

\begin{lemma}
[Convexity of $\Phi$]
\label{lem: Convexity of Phi}
For $x \in (0, 1/ne^2)$, $\Phi (x)$ is convex. 
\end{lemma}

\begin{proof}

    By direct calculation, we get
    \begin{equation*}
        \Phi^\prime  (x) = \left( \frac{1}{\phi(x)} + \frac{2}{\phi(x)^2} \right) \phi^\prime (x),
    \end{equation*}
    and
    \begin{align*}
        \Phi^{\prime \prime}  (x) = \, & \left( \frac{1}{\phi(x)} + \frac{2}{\phi(x)^2}  \right) \phi''(x) + \left( - \frac{1}{\phi(x)^2} - \frac{4}{\phi(x)^3} \right) \phi'(x)^2 \\
        \stackrel{(i)}{=} \, & \frac{1}{x^2} \left( \frac{1}{\phi(x)} + \frac{1}{\phi(x)^2} - \frac{4}{\phi(x)^3} \right) \\
        \stackrel{(ii)}{=} \, & \frac{1}{x^2} \left( \frac{1}{\phi(x)} - \frac{1}{\phi(x)^2} + 2 \left( \frac{1}{\phi(x)^2} - \frac{2}{\phi(x)^3} \right)  \right)  > 0,
    \end{align*}
    where $(i)$ follows from the fact that $\phi''(x) = 1/x^2 = \phi'(x)^2$, $(ii)$ is due to $\phi(x) > 2$ whenever $x < 1/ne^2$. This completes the proof.
\end{proof}

\begin{lemma}\label{lem:relation_modified_margins}
    For any $\param$ satisfying $\ModifiedLoss(\param) < 1/n$, we have $\GFmargin (\param) > \GDmargin (\param)$.
\end{lemma}
\begin{proof}
    Recall that
    \begin{equation*}
        \GDmargin (\param) = \frac{\exp(\Phi(\ModifiedLoss(\param)))}{\| \param \|^M}, \quad \GFmargin (\param) = \frac{\phi(\ModifiedLoss(\param))}{\| \param \|^M}.
    \end{equation*}
    The result follows directly from the fact that $\Phi < \log \phi$ when $n \ModifiedLoss(\param) < 1$.
\end{proof}

We next give some a priori estimates on $f$ and $\homop_a$.
\begin{lemma}
[Homogeneous constant]
\label{lem: Homogeneous constant}
Under \Cref{asp:initial-cond-gd}, there exists some constant $B_2$, such that for any $\param$:
\begin{equation}
\label{eq: homogeneous-constant}
    \|\nabla f(\param ) \| \le B_2 \left( \|\param \|^{M-1} + 1 \right), \quad \| \nabla^2 f(\param )\| \le B_2 \left( \| \param \|^{(M-2)_+} + 1 \right).  
\end{equation}
\end{lemma}

\begin{lemma}
[Gradient and Hessian bound for $\Loss$]
\label{lem:gradient_hessian_bound_L}
Under \Cref{asp:initial-cond-gd}, there exists some constant $B_3$:
\begin{equation}
\label{eq: bound of gradient and hessian 2}
   \|\nabla \Loss (\param) \| \le B_3 \Loss (\param ) \left( \| \thetaB \|^{M-1} + 1 \right), \quad \| \nabla^2 \Loss (\param) \| \le B_3 \Loss ( \thetaB) \left( \| \thetaB \|^{2M - 2} + 1 \right).
\end{equation}
\end{lemma}
\begin{proof}
According to \Cref{lem: Homogeneous constant}, we obtain that
\begin{align*}
    \| \nabla \Loss (\param) \| = \, & \left\| \frac{1}{n} \sum_{i=1}^n e^{-\bar f_i(\param )} \nabla \bar f_i(\param ) \right\| \le B_2 \Loss (\param ) \left( \|\param \|^{M-1} + 1 \right).\\ 
    \| \nabla^2 \Loss (\param) \| = \, & \left\| \frac{1}{n} \sum_{i=1}^n e^{-\bar f_i(\param )} (\nabla  \bar f_i(\param ) \nabla \bar f_i(\param )^\top -\nabla^2 \bar f_i(\param ) \right\| \\
    \le \, & \Loss (\param ) \left( B_2^2 (\|\param \|^{M-1} + 1)^2 + B_2 (\|\param \|^{(M-2)_+} + 1) \right).
\end{align*}
Using AM-GM inequality, the desired result follows immediately.
\end{proof}
\begin{lemma}[Gradient and Hessian bound for $\ModifiedLoss$]\label{lem:gradient_hessian_bound}
Under \Cref{asp:initial-cond-gd}, there exists a constant $B_4$ such that for any $\param$:
\begin{equation}
\label{eq:gradient_hessian_bound}
   \|\nabla \ModifiedLoss (\param) \| \le B_4 \ModifiedLoss(\param ) \left( \| \param \|^{M-1} + 1 \right), \quad \| \nabla^2 \ModifiedLoss (\param) \| \le B_4 \ModifiedLoss(\param ) \left( \| \param \|^{2M-2} + 1 \right). 
\end{equation}
\end{lemma}

\begin{proof}
We first upper bound $\| \nabla \ModifiedLoss(\param )\|$. By \Cref{lem: Homogeneous constant} and \Cref{lem:gradient_hessian_bound_L}, we have
\begin{align*}
    \| \nabla \ModifiedLoss(\param )\| = \, & \Big\|e^{ \homop_a(\|\param \|)} \nabla \Loss (\param ) + \Loss(\param ) e^{\homop_a(\|\param \|)} \homop_a^\prime(\|\param \|) \tilde \param  \Big\| \\
    \le \, & e^{ \homop_a(\|\param \|)} \left\| \nabla \Loss (\param ) \right\| + \ModifiedLoss(\param) \homop_a^\prime (\|\param \|) \\
    \le \, & B_3 \ModifiedLoss (\param) \left( \| \param \|^{M-1} + 1 \right) + B_1 \ModifiedLoss (\param) \left( \| \param \|^{(M-2)_+} + 1 \right) \\
    \le \, & B_4 \ModifiedLoss(\param ) \left( \| \param \|^{M-1} + 1 \right),
\end{align*}
where $B_4$ only depends on $B_3$ and $B_1$. For $\| \nabla^2 \ModifiedLoss(\param )\|$, note that
\begin{align*}
    \| \nabla^2 \ModifiedLoss(\param ) \| 
    = \, & \Big \| e^{\homop_a(\|\param \|)} \nabla^2 \Loss(\param) + 2 e^{\homop_a(\|\param \|)} \homop_a' (\| \param \|) \nabla \Loss(\param) \tilde{\param}^\top \\
    & + e^{\homop_a(\|\param \|)} \Loss (\param) \left( \homop_a'' (\| \param \|) + \homop_a' (\| \param \|)^2 \right) \tilde{\param} \tilde{\param}^\top \\
    & + e^{\homop_a(\|\param \|)} \Loss(\param) \homop_a' (\| \param \|) \left( I - \tilde{\param} \tilde{\param}^\top \right) / \| \param \| \Big\| \\
    \le \, & e^{\homop_a(\|\param \|)} \left\| \nabla^2 \Loss(\param) \right\| + 2 e^{\homop_a(\|\param \|)} \homop_a' (\| \param \|) \left\| \nabla \Loss(\param) \right\| \\
    & + \ModifiedLoss (\param) \left( \homop_a'' (\| \param \|) + \homop_a' (\| \param \|)^2 \right) + \ModifiedLoss(\param) \frac{\homop_a' (\| \param \|)}{\| \param \|} \\
    \le \, & B_3 \ModifiedLoss (\param) \left( \| \thetaB \|^{2M - 2} + 1 \right) + 2 B_1 B_3 \ModifiedLoss (\param) \left( \| \thetaB \|^{M-1} + 1 \right) \left( \| \thetaB \|^{(M-2)_+} + 1 \right) \\
    & + \ModifiedLoss (\param) \left( B_1 \left( \|\param \|^{(M-3)_+} + 1 \right) + B_1^2 \left( \|\param \|^{(M-2)_+} + 1 \right)^2 \right) + B_1 \ModifiedLoss(\param) \left( \|\param \|^{(M-3)_+} + 1 \right).
\end{align*}
Hence, there exists a constant $B_4$ only depending on $B_1$ and $B_3$, such that
\begin{equation}
    \| \nabla^2 \ModifiedLoss(\param ) \| \le \, B_4 \ModifiedLoss (\param) \left( \| \thetaB \|^{2M - 2} + 1 \right).
\end{equation}
This completes the proof.
\end{proof}

We introduce a crucial quantity, $v_t$, defined as the inner product of the gradient and the negative weight vector: 
\begin{equation}
\label{eq: vt}
    v_t \coloneqq  \langle \nabla \Loss(\param_t), - \param_t \rangle.  
\end{equation}
We have the following bound for $v_t$, recall that $\rho_t = \| \param_t \|$.

\begin{lemma}
[Bound of $v_t$]
\label{lem: Bound of v_t} Under \Cref{asp:initial-cond-gd}, for any $\param_t$ satisfying $\Loss (\param_t) < (1/n) e^{-\homop_a(\rho_t)}$, we have 
\[
   M \Loss (\param_t) \log \frac{1}{\Loss (\param_t)} + \homop'(\rho_t) \Loss (\param_t) \ge  v_t \ge M \Loss (\param_t) \log \frac{1}{n\Loss (\param_t)} - \homop^{\prime}(\rho_t)\Loss (\param_t)>0.
\]
\end{lemma}

\begin{proof}
We first prove the first ``$\ge$". Note that
\begin{align*}
    v_t = \, & \langle -\nabla \Loss (\param_t),\param_t \rangle = \frac{1}{n}\sum_{i=1}^n e^{-\bar f_i(\param_t)} \langle \nabla \bar f_i(\param_t), \param_t \rangle &&\explain{ Since $-\nabla \Loss_t = \sum_{i=1}^n e^{-y_if_i(t)} y_i \nabla f_i(t)$ } \\
    \le \, & \frac{1}{n}\sum_{i=1}^n e^{-\bar f_i(\param_t)} \left( M \bar f_i(\param_t) + \homop'(\rho_t) \right) && \explain{By \Cref{asp:nearhomo} } \\
    = \, & M \frac{1}{n} \sum_{i=1}^{n} e^{-\bar f_i(\param_t)} \bar f_i(\param_t) + \homop'(\rho_t) \Loss (\param_t) \\
    \le \, & - M \Loss (\param_t) \log \Loss (\param_t) + \homop' (\rho_t) \Loss (\param_t). && \explain{By Jensen's inequality}
\end{align*}
For the second ``$\ge$", by \Cref{lem:f-min}, we have 
\begin{equation}
\label{eq:f_min_lbbb}
        \bar f_i(\param_t) \ge \bar f_{\min}(\param_t) \ge \log \frac{1}{n\Loss (\param_t)}. 
\end{equation}
Furthermore, 
\begin{align*}
    v_t
    & = \langle -\nabla \Loss (\param_t),\param_t \rangle = \frac{1}{n}\sum_{i=1}^n e^{-\bar f_i(\param_t)} y_i  \langle \nabla f_i(\param_t), \param_t \rangle &&\explain{ Since $-\nabla \Loss_t = \sum_{i=1}^n e^{-y_if_i(t)} y_i \nabla f_i(t)$ } \\
    & \ge M \cdot \frac{1}{n}\sum_{i=1}^n e^{-\bar f_i(\param_t)} \bar f_i(\param_t) - \frac{1}{n}\sum_{i=1}^n e^{-\bar f_i(\param_t)}\homop^{\prime}(\rho_t) &&\explain{ By \Cref{asp:nearhomo} } \\
    & = M \cdot\frac{1}{n}\sum_{i=1}^n e^{-\bar f_i(\param_t)} \bar f_i(\param_t) - \Loss (\param_t) \homop^{\prime}(\rho_t) \\ 
    & \ge M \cdot\frac{1}{n}\sum_{i=1}^n e^{-\bar f_i(\param_t)} \log \frac{1}{n \Loss (\param_t)}- \Loss (\param_t) \homop^{\prime}(\rho_t) &&\explain{ By Eq.~\Cref{eq:f_min_lbbb}} \\ 
    & =   M \Loss (\param_t) \log \frac{1}{n\Loss (\param_t)} - \homop^{\prime}(\rho_t)\Loss (\param_t). 
\end{align*}
Note that since $\rho_t \ge 0$, we have 
\begin{align*}
    \Loss (\param_t) < e^{-\homop_a(\rho_t)}/n \Longleftrightarrow M\log \frac{1}{n \Loss (\param_t)} & >   M \homop_a \big(\rho_t \big).
\end{align*}
By \Cref{lem:modified_p_a}, we have 
\[
    M \homop_a(x) - \homop' (x) \ge x \homop_a^\prime(x) > 0. 
\]
As a consequence, 
\[
     M\log \frac{1}{n \Loss (\param_t)}  >  M \homop_a (\rho_t) \ge \homop^\prime\big(\rho_t \big) \implies v_t >0. 
\]
This completes the proof of \Cref{lem: Bound of v_t}. 
\end{proof}

\subsection{Margin Improvement: Proof of \texorpdfstring{\Cref{thm: Margin improving and convergence-gd}}{Theorem 6.1}}\label{sec:margin_improve_gd}
This section is devoted to the proof of the margin improvement part of \Cref{thm: Margin improving and convergence-gd}. We first establish the desired result under another set of assumptions in the following theorem, and then show that \Cref{asp:initial-cond-gd} implies this set of assumptions.
\begin{theorem}\label{thm:combined_argument}
    Let $f(\param;\xB)$ be a twice-differentiable network (with respect to $\param$) satisfying \Cref{asp:nearhomo} with $(M,\homop, \homoq)$, and that
    \begin{equation*}
        \left\| \nabla_{\param}^2 f(\param; \xB) \right\| \le A \left( \| \param \|^{(M - 2)_+} + 1 \right)
    \end{equation*}
    for some constant $A > 0$. Further, assume that
    \begin{equation}
    \label{eq:init-GD-bound}
        \ModifiedLoss_s < \frac{1}{n e^2} \Longleftrightarrow \Loss_s < \frac{1}{n e^2} \exp \left( - \homop_a (\rho_s) \right).
    \end{equation}
    Recall $B_3$ and $B_4$ from \Cref{lem:gradient_hessian_bound_L} and \Cref{lem:gradient_hessian_bound}, we define for $\eta, r > 0$:
    \begin{equation*}
        B_5 (\eta, r) \coloneqq \, \left( r + \eta \ModifiedLoss_s B_3 e^{-\homop_a(r)} \left( r^{M-1} + 1 \right) \right)^{2M - 2} + 1,
    \end{equation*}
    and for $t \ge s$:
    \begin{align*}
        R_1 (t) \coloneqq B_3^2 \Loss_t (\rho_t^{M - 1} + 1)^2, \quad R_2 (t) \coloneqq B_4 B_5 (\eta, \rho_t) \Loss_t \phi(\Loss_t).
    \end{align*}
    Further, assume that
    \begin{equation*}
        \eta \sup_{\rho_t \ge \rho_s, \ModifiedLoss_t \le \ModifiedLoss_s} R_1 (t) \le \frac{1}{2}, \quad \eta \sup_{\rho_t \ge \rho_s, \ModifiedLoss_t \le \ModifiedLoss_s} R_2 (t)  \le \frac{1}{2}.
    \end{equation*}
    For all $t \ge s$, $t \in \mathbb{N}$, we interpolate between $\param_t$ and $\param_{t+1}$ by defining
    \begin{equation*}
        \param_{t+\alpha} = \param_t + \alpha \left( \param_{t+1} - \param_t \right) = \param_t - \alpha \eta \nabla \Loss (\param_t)
    \end{equation*}
    for $\alpha \in [0, 1]$. Then, we have for all $t \ge s$ and $\alpha \in [0, 1]$, $v_t > 0$, and
    \begin{itemize}
        \item[(1)] $ \GFmargin (\param_{t+\alpha}) > \GDmargin (\param_{s}) . $
        \item[(2)] $ 2 \eta \alpha v_t \le \rho_{t+\alpha}^2 - \rho_t^2 \le 2 \eta \alpha v_t \Big ( 1+ \frac{\eta \alpha R_1 (t) }{2M \phi(\ModifiedLoss_t)} \Big)$. 
        \item[(3)] $ \ModifiedLoss_{t+\alpha} - \ModifiedLoss_t \le \, - \alpha \eta \left( 1- \alpha \eta R_2 (t) \right) M \ModifiedLoss_t \phi (\ModifiedLoss_t) v_t^{-1} \|\nabla \Loss_t\|^2$.
        \item[(4)] $ \log \GDmargin(\param_{t+\alpha}) - \log \GDmargin(\param_t) \ge \frac{M\rho_t^2\| \partial_\perp \Loss_t \|_2^2}{v_t^2 } \log \frac{\rho_{t+\alpha}}{\rho_t} . $
\end{itemize}
\end{theorem}

To prove \Cref{thm:combined_argument}, we use a similar strategy as the proof of Lemma E.8 in \cite{lyu2020gradient}. To this end, it suffices to prove the following lemma (analogous to Lemma E.9 in \cite{lyu2020gradient}):

\begin{lemma}
    Fix an integer $T \geq s$. Suppose that (1), (2), (3), (4) hold for any $t+\alpha \leq T$. Then if (1) holds for $(t, \alpha) \in\{T\} \times[0, A)$ for some $A \in(0,1]$, then all of (1), (2), (3), (4) hold for $(t, \alpha) \in\{T\} \times[0, A]$.
\end{lemma}

\begin{proof}
    By our assumption, we know that $\ModifiedLoss_t \le \ModifiedLoss_s < 1/ne^2$ for all $t \le T$, hence \Cref{lem: Bound of v_t} implies $v_t > 0$. Further, we have $\rho_t \ge \rho_s$ and $\GFmargin (\param_t) \ge \GFmargin (\param_s)$ for all $t \le T$.

    Now we fix $t = T$. Since $\GFmargin (\param_{t+\alpha}) > \GDmargin (\param_{s})$ for all $\alpha \in [0, A)$, by continuity of $\GFmargin$ we know that $\GFmargin (\param_{t+\alpha}) \ge \GDmargin (\param_{s})$ for $\alpha = A$.

    \paragraph{Proof of (2) for $\alpha = A$.} By definition, we have 
    \begin{align*}
        \rho_{t + \alpha}^2 - \rho_t^2 &= 2 \eta \alpha \langle \nabla \Loss_t, - \param_t \rangle  + \eta^2 \alpha^2 \|\nabla \Loss_t\|^2 \\
        & = 2 \eta \alpha v_t + \eta^2 \alpha^2 \|\nabla \Loss_t\|^2 \ge 2 \eta \alpha v_t > 0,
    \end{align*}
    where the last inequality is due to \Cref{lem: Bound of v_t}. According to \Cref{lem:gradient_hessian_bound_L}, we have
    \begin{align*}
        \| \nabla \Loss (\param) \| \le \Loss (\param) \cdot B_3 \left( \| \param \|^{M - 1} + 1 \right),
    \end{align*}
    namely
    \begin{equation*}
        \| \nabla \Loss_t \| \le B_3 \Loss_t \left( \rho_t^{M-1} + 1 \right).
    \end{equation*}
    Further, applying \Cref{lem: Bound of v_t} leads to
    \begin{equation*}
        v_t \ge M \Loss_t \phi(\ModifiedLoss_t).
    \end{equation*}
    We finally obtain that
    \begin{align*}
        \rho_{t+\alpha}^2 - \rho_t^2 
        & = 2 \eta \alpha v_t \bigg( 1+ \frac{\eta \alpha \|\nabla \Loss_t\|^2}{2 v_t} \bigg) \\ 
        & \le 2 \eta \alpha v_t \bigg ( 1+ \frac{\eta \alpha B_3^2 \Loss_t^2 (\rho_t^{M - 1} + 1)^2 }{2M \Loss_t \phi(\ModifiedLoss_t)} \bigg) \\ 
        & = 2 \eta \alpha v_t \bigg ( 1+ \frac{\eta \alpha B_3^2 \Loss_t (\rho_t^{M - 1} + 1)^2 }{2M \phi(\ModifiedLoss_t)} \bigg) \\
        & = \, 2 \eta \alpha v_t \bigg ( 1+ \frac{\eta \alpha R_1 (t) }{2M \phi(\ModifiedLoss_t)} \bigg),
    \end{align*}
    where we denote
    \begin{equation*}
        R_1 (t) = B_3^2 \Loss_t (\rho_t^{M - 1} + 1)^2.
    \end{equation*}
    This completes the proof of (2) for $\alpha = A$.

    \paragraph{Proof of (3) for $\alpha = A$.} Using Taylor expansion, we know that there exists $\epsilon \in (0, \alpha)$ such that
    \begin{align*}
       \ModifiedLoss_{t+\alpha} = \, & \ModifiedLoss_t +  \nabla \ModifiedLoss_t^\top \big(\param_{t+\alpha} - \param_t\big)  + \frac{1 }{2} \big(\param_{t+\alpha} - \param_t\big)^\top \nabla^2 \ModifiedLoss_{t+\epsilon} \big(\param_{t+\alpha} - \param_t\big) \\ 
       \le \, & \ModifiedLoss_t - \alpha \eta e^{\homop_a(\rho_t)} \|\nabla \Loss_t\|^2 + \eta \alpha \ModifiedLoss_t \homop_a^\prime (\rho_t) \frac{\langle \param_t, - \nabla \Loss_t \rangle }{\rho_t} +   \frac{\alpha^2 \eta^2 }{2} \|\nabla^2 \ModifiedLoss_{t+\epsilon}\| \cdot \|\nabla \Loss_t\|^2\\ 
       = \, & \ModifiedLoss_t - \alpha \eta e^{\homop_a(\rho_t)} \|\nabla \Loss_t\|^2 + \eta \alpha \ModifiedLoss_t \homop_a^\prime (\rho_t) \frac{v_t }{\rho_t} +   \frac{\alpha^2 \eta^2 }{2} \|\nabla^2 \ModifiedLoss_{t+\epsilon}\| \cdot \|\nabla \Loss_t\|^2\\ 
       \le \, & \ModifiedLoss_t - \alpha \eta e^{\homop_a(\rho_t)} \|\nabla \Loss_t\|^2 + \eta \alpha \ModifiedLoss_t \homop_a^\prime (\rho_t) \frac{\rho_t}{v_t} \|\nabla \Loss_t\|^2 \\
       & + \frac{\alpha^2 \eta^2 B_4}{2} \ModifiedLoss_{t+\epsilon} (\rho_{t+\epsilon}^{2M - 2} + 1) \cdot \|\nabla \Loss_t\|^2\\ 
       & \qquad \explain{By Cauchy-Schwarz and \Cref{lem:gradient_hessian_bound}} \\
       = \, & \ModifiedLoss_t - \alpha \eta \left[ e^{\homop_a(\rho_t)} v_t - \ModifiedLoss_t \homop_a^\prime (\rho_t) \rho_t - \frac{\alpha  \eta}{2}  B_4 \ModifiedLoss_{t+\epsilon} (\rho_{t+\epsilon}^{2M - 2} + 1) v_t  \right] \frac{\|\nabla \Loss_t\|^2}{v_t}.
    \end{align*}
    To estimate the right hand side of the above equation, we first note that 
    \begin{align*}
        \rho_{t+\epsilon}^{2M - 2} + 1 \le \, & \left( \rho_t + \eta \| \nabla \Loss_t \| \right)^{2M - 2} + 1 \le \left( \rho_t + \eta B_3 \Loss_t \left( \rho_t^{M-1} + 1 \right) \right)^{2M - 2} + 1 \\
        \le \, & \left( \rho_t + \eta \ModifiedLoss_t B_3 e^{-\homop_a(\rho_t)} \left( \rho_t^{M-1} + 1 \right) \right)^{2M - 2} + 1 \\
        \le \, &  \left( \rho_t + \eta \ModifiedLoss_s B_3 e^{-\homop_a(\rho_t)} \left( \rho_t^{M-1} + 1 \right) \right)^{2M - 2} + 1 := B_5 (\eta, \rho_t).
    \end{align*}
    Next we show that $\ModifiedLoss_{t+\epsilon} < \ModifiedLoss_t$ for all $\epsilon \in (0, \alpha]$. If this is true, then we get
    \begin{equation*}
        \ModifiedLoss_{t+\epsilon} \left( \rho_{t+\epsilon}^{2M - 2} + 1 \right) \le \ModifiedLoss_t B_5 (\eta, \rho_t),
    \end{equation*}
    and consequently
    \begin{equation*}
        \ModifiedLoss_{t+\alpha} \le \, \ModifiedLoss_t - \alpha \eta \left[ e^{\homop_a(\rho_t)} v_t - \ModifiedLoss_t \homop_a^\prime (\rho_t) \rho_t - \frac{\alpha  \eta}{2}  B_4 \ModifiedLoss_t B_5 (\eta, \rho_t) v_t  \right] \frac{\|\nabla \Loss_t\|^2}{v_t}.
    \end{equation*}
    By \Cref{lem: Bound of v_t}, we have
    \begin{align*}
        & e^{\homop_a(\rho_t)} v_t - \ModifiedLoss_t \homop_a^\prime (\rho_t) \rho_t - \alpha  \eta  B_4 \ModifiedLoss_t B_5 (\eta, \rho_t) v_t /2 \\
        \ge \, & \left( e^{\homop_a(\rho_t)} - \alpha \eta B_4 \ModifiedLoss_t B_5 (\eta, \rho_t) /2 \right) \left( M \Loss_t \LinkFun(\Loss_t) -\homop'(\rho_t) \Loss_t \right) - \ModifiedLoss_t \homop_a^\prime (\rho_t) \rho_t \\
        \ge \, & \left( e^{\homop_a(\rho_t)} - \alpha \eta B_4 \ModifiedLoss_t B_5 (\eta, \rho_t) /2 \right) M \Loss_t \LinkFun(\Loss_t) - \ModifiedLoss_t \left( \homop'(\rho_t) + \homop_a^\prime (\rho_t) \rho_t \right) \\
        = \, & \left( e^{\homop_a(\rho_t)} - \alpha \eta B_4 \ModifiedLoss_t B_5 (\eta, \rho_t) /2 \right) M \Loss_t \LinkFun(\Loss_t) - M \ModifiedLoss_t \homop_a (\rho_t) \\
        = \, & M \ModifiedLoss_t \phi (\ModifiedLoss_t) - M \ModifiedLoss_t \cdot \frac{\alpha \eta}{2} B_4 \Loss_t B_5 (\eta, \rho_t) \LinkFun(\Loss_t) \\
        \stackrel{(i)}{=} \, & M \ModifiedLoss_t \phi (\ModifiedLoss_t) - M \ModifiedLoss_t \cdot \frac{\alpha \eta}{2} R_2 (t) \stackrel{(ii)}{\ge} \, M \ModifiedLoss_t \phi (\ModifiedLoss_t) \left( 1- \alpha \eta R_2 (t) \right),
    \end{align*}
    where in $(i)$ we define
    \begin{equation*}
        R_2 (t) = B_4 B_5 (\eta, \rho_t) \Loss_t \phi(\Loss_t),
    \end{equation*}
    and $(ii)$ is because $\ModifiedLoss_t \le \ModifiedLoss_s < 1/ne^2$ implies $\phi(\ModifiedLoss_t) \ge 2$.
    It finally follows that
    \begin{align*}
        \ModifiedLoss_{t+\alpha} - \ModifiedLoss_t \le \, - \alpha \eta \left( 1- \alpha \eta R_2 (t) \right) M \ModifiedLoss_t \phi (\ModifiedLoss_t) v_t^{-1} \|\nabla \Loss_t\|^2.
    \end{align*}
    Now it suffices to show that $\ModifiedLoss_{t+\epsilon} < \ModifiedLoss_t$ for all $\epsilon \in (0, \alpha]$. To this end, assume that $\epsilon_0 = \inf \{ \epsilon \in (0, \alpha]: \ModifiedLoss_{t+\epsilon} \ge \ModifiedLoss_t \}$ exists. Note that
    \begin{align*}
        \frac{\mathrm{d} \ModifiedLoss_{t+\epsilon}}{\mathrm{d} \epsilon } \Big\vert_{\epsilon = 0} &=  \nabla G_{t}^\top (- \eta \nabla \Loss_t) \\ 
        &= -\eta e^{\homop_a(\rho_t)} \|\nabla \Loss_t\|^2 + \eta \ModifiedLoss_t \homop_a^\prime (\rho_t) v_t / \rho_t \\ 
        &\le -\eta e^{\homop_a(\rho_t)} v_t^2 /\rho_t^2 + \eta \ModifiedLoss_t \homop_a^\prime (\rho_t) v_t / \rho_t \\ 
        &= -v_t /\rho_t^2 (\eta e^{\homop_a(\rho_t)} v_t - \eta \ModifiedLoss_t \homop_a^\prime (\rho_t)  \rho_t)\\ 
        &\le - (v_t /\rho_t^2) \eta M \ModifiedLoss_t \phi(\ModifiedLoss_t) <0,
    \end{align*}
    which implies $\epsilon_0 > 0$. Therefore, $\ModifiedLoss_{t+\epsilon} < \ModifiedLoss_t$ for all $\epsilon \in (0, \epsilon_0)$, thus leading to
    \begin{equation*}
        \ModifiedLoss_{t + \epsilon_0} \le \ModifiedLoss_t - \epsilon_0 \eta \left( 1- \epsilon_0 \eta R_2 (t) \right) M \ModifiedLoss_t \phi (\ModifiedLoss_t) v_t^{-1} \|\nabla \Loss_t\|^2 < \ModifiedLoss_t,
    \end{equation*}
    which contradicts $\ModifiedLoss_{t + \epsilon_0} \ge \ModifiedLoss_t$. Hence, we must have $\ModifiedLoss_{t+\epsilon} < \ModifiedLoss_t$ for all $\epsilon \in (0, \alpha]$. This concludes the proof of (3).

    \paragraph{Proof of (4) for $\alpha = A$.} We only need to show that (2) and (3) together imply $(4)$. 
    Note that 
    \[
        \ModifiedLoss_{t+\alpha} - \ModifiedLoss_t \le \, - \alpha \eta \left( 1- \alpha \eta R_2 (t) \right) M \ModifiedLoss_t \phi (\ModifiedLoss_t) v_t^{-1} \|\nabla \Loss_t\|^2. 
    \] 
    Hence,  
    \begin{align*}
        \frac{\ModifiedLoss_{t+\alpha}  - \ModifiedLoss_t}{(1- \alpha \eta R_2 (t)) \ModifiedLoss_t \phi(\ModifiedLoss_t)} \le \, - \eta \alpha M v_t^{-1}  \|\nabla \Loss_t\|^2. 
    \end{align*}
    Multiplying $1+ \frac{\eta \alpha R_1 (t) }{2M \phi(\ModifiedLoss_t)}$ on both sides, we have
    \begin{align*}
        & \frac{1+ \eta \alpha R_1 (t) / 2M \phi(\ModifiedLoss_t)}{(1- \alpha \eta R_2 (t) ) \ModifiedLoss_t \phi(\ModifiedLoss_t)} (\ModifiedLoss_{t+\alpha}  - \ModifiedLoss_t) \\
        \le \, & - \eta \alpha M  v_t^{-1}  \|\nabla \Loss_t\|^2 \bigg ( 1+ \frac{\eta \alpha R_1 (t) }{2M \phi(\ModifiedLoss_t)} \bigg) \\ 
        \stackrel{(i)}{\le} \, & - M v_t^{-2}  \|\nabla \Loss_t\|^2 (\rho_{t+\alpha}^2 - \rho_t^2) /2 \\ 
        \stackrel{(ii)}{=} \, & - \frac{1}{\rho_t^2}(\rho_{t+\alpha}^2 - \rho_t^2)   \bigg( \frac{M}{2}+  \frac{M\rho_t^2\| \bar \partial_\perp \Loss_t\|_2^2}{2v_t^2} \bigg),
    \end{align*}
    where $(i)$ follows from part (2), and $(ii)$ is because we know that: 
    \[
        \|\nabla \Loss_t \|^2 = \| \bar \partial_r \Loss_t\|^2 + \| \bar \partial_\perp \Loss_t\|^2  =  \frac{v_t^2}{\rho_t^2} + \| \bar \partial_\perp \Loss_t\|^2. 
    \] 
    Note that since $\rho_t \ge \rho_s$, our assumption implies that
    \begin{align*}
        \alpha \eta R_1 (t) \le \, \eta R_1 (t) \le \frac{1}{2}, \quad \alpha \eta R_2 (t) \le \, \eta R_2 (t) \le \frac{1}{2}.
    \end{align*}
    Therefore,
    \begin{align*}
        \frac{1+ \eta \alpha R_1 (t) / 2M \phi(\ModifiedLoss_t)}{(1- \alpha \eta R_2 (t)) \ModifiedLoss_t \phi(\ModifiedLoss_t)} \le \, & \frac{1+ \eta R_1 (t) / 2M \phi(\ModifiedLoss_t)}{(1- \eta R_2 (t)) \ModifiedLoss_t \phi(\ModifiedLoss_t)} \\
        \le \, & \frac{1 + \eta R_1 (t) / 2M \phi(\ModifiedLoss_t) + 2 \eta R_2 (t) }{\ModifiedLoss_t \phi(\ModifiedLoss_t)} \\
        \le \, & \frac{1}{\ModifiedLoss_t \phi(\ModifiedLoss_t)} \left( 1 + \frac{1}{ M \phi(\ModifiedLoss_t)} + \frac{1}{\phi(\ModifiedLoss_t)} \right) \\
        \le \, & \frac{1}{\ModifiedLoss_t \phi(\ModifiedLoss_t)} + \frac{2}{\ModifiedLoss_t \phi(\ModifiedLoss_t)^2} = - \Phi'(\ModifiedLoss_t),
    \end{align*}
    where we recall that
    \begin{equation*}
        \Phi(x) = \log \phi(x) - \frac{2}{\phi(x)}.
    \end{equation*}
    By convexity of $\Phi$, we get that
    \begin{align*}
        &\Phi(\ModifiedLoss_{t+\alpha}) - \Phi(\ModifiedLoss_t)  - (\log \rho_{t+\alpha} - \log \rho_t) \bigg( M+  \frac{M\rho_t^2\|\bar \partial_\perp  \Loss_t\|_2^2}{v_t^2} \bigg)\\ 
        & \ge  \Phi^\prime (\ModifiedLoss_t) (\ModifiedLoss_{t+\alpha} - \ModifiedLoss_t)  - \frac{1}{\rho_t^2}(\rho_{t+\alpha}^2 - \rho_t^2)   \bigg( \frac{M}{2}+  \frac{M\rho_t^2\|\bar \partial_\perp \Loss_t\|_2^2}{2v_t^2} \bigg) \ge 0. 
    \end{align*}
    This leads to 
    \begin{align*}
        \log \GDmargin(\param_{t+\alpha}) - \log \GDmargin(\param_t) = \, & \Phi (\ModifiedLoss_{t+\alpha}) - \Phi(\ModifiedLoss_t)- M (\log \rho_{t+\alpha} - \log \rho_t) \\
        \ge \, & \frac{M \rho_t^2 \| \bar \partial_\perp \Loss_t \|_2^2}{v_t^2 } \log \frac{\rho_{t+\alpha}}{\rho_t},
    \end{align*}
    completing the proof of (4).

    \paragraph{Proof of (1) for $\alpha = A$.} By \Cref{lem:relation_modified_margins} and (4), we have
    \begin{equation*}
        \GFmargin (\param_{t + \alpha}) > \GDmargin (\param_{t+\alpha}) \ge \GDmargin (\param_t) \ge \GDmargin (\param_s),
    \end{equation*}
    which proves (1) for $\alpha = A$. This completes the proof of \Cref{thm:combined_argument}. 
\end{proof}

We next show that \Cref{asp:initial-cond-gd} verifies the conditions in \Cref{thm:combined_argument}.

\begin{lemma}\label{residual_estimates}
    There exists a constant $B$ that only depends on $(M, \homop, \homoq, A)$ and $(\homor, \homos)$ in \Cref{asp:strongerhomo,asp:initial-cond-gd}, such that for any step size $\eta > 0$,
    \begin{equation*}
        \ModifiedLoss_s < \min \left\{ \frac{1}{n e^2}, \, \frac{1}{B \eta} \right\}
    \end{equation*}
    implies
    \begin{equation*}
        \eta \sup_{\rho_t \ge \rho_s, \ModifiedLoss_t \le \ModifiedLoss_s} R_1 (t) \le \frac{1}{2}, \quad \eta \sup_{\rho_t \ge \rho_s, \ModifiedLoss_t \le \ModifiedLoss_s} R_2 (t)  \le \frac{1}{2}.
    \end{equation*}
    As a consequence, the conclusions of \Cref{thm:combined_argument} hold.
\end{lemma}

\begin{proof}
    In this proof, we will use $B$ to denote a generic constant that depends on $(M, \homop, \homoq, \homor, \homos)$, while the meaning of $B$ can change from line to line. We begin with estimating $R_1 (t)$. Since $\deg (\homop_a) \ge 1$ when $M \ge 2$, we know that there exists a constant $B$ such that
    \begin{equation*}
        r^{M - 1} + 1 \le B \exp \left( \frac{1}{2} \homop_a (r) \right), \quad \forall r \ge 0.
    \end{equation*}
    Therefore, if $\rho_t \ge \rho_s, \ModifiedLoss_t \le \ModifiedLoss_s$, we have
    \begin{equation*}
        \eta R_1 (t) = \, \eta B_3^2 \Loss_t (\rho_t^{M - 1} + 1)^2 \le \eta B_3^2 B^2 \Loss_t \exp (\homop_a (\rho_t)) = \eta B \ModifiedLoss_t \le \eta B \ModifiedLoss_s.
    \end{equation*}
    For estimating $R_2 (t)$, we first note that $\forall r > 0$:
    \begin{align*}
        B_5 (\eta, r) = \, & \left( r + \eta \ModifiedLoss_s B_3 e^{-\homop_a(r)} \left( r^{M-1} + 1 \right) \right)^{2M - 2} + 1 \\
        \le \, & \left( r + \eta \ModifiedLoss_s B_3 B e^{-\homop_a(r) / 2} \right)^{2M - 2} + 1 \\
        \le \, & 2^{2M - 3} \left( r^{2 M - 2} + \eta^{2M - 2} \ModifiedLoss_s^{2M - 2} B^{2M - 2} e^{- (2M - 2) \homop_a(r) / 2} \right) + 1 \\
        \le \, & B \left( r^{2 M - 2} + \eta^{2M - 2} \ModifiedLoss_s^{2M - 2} e^{- \homop_a(r)} + 1 \right).
    \end{align*}
    Second, since $\ModifiedLoss_t \le \ModifiedLoss_s$, we have
    \begin{align*}
        \Loss_t \phi(\Loss_t) = \, & \exp(- \homop_a (\rho_t)) \ModifiedLoss_t \left( \phi(\ModifiedLoss_t) + \homop_a (\rho_t) \right) \\
        \stackrel{(i)}{\le} \, & \exp(- \homop_a (\rho_t)) \ModifiedLoss_s  \phi(\ModifiedLoss_s) + \homop_a (\rho_t) \exp(- \homop_a (\rho_t)) \ModifiedLoss_s \\
        \stackrel{(ii)}{\le} \, & B \exp \left(- \frac{1}{2} \homop_a (\rho_t) \right) \ModifiedLoss_s  \phi(\ModifiedLoss_s),
    \end{align*}
    where $(i)$ follows from the fact that $x \mapsto x \phi(x)$ is increasing, and $(ii)$ is because that there exists a constant $B$ such that $\homop_a (\rho_t) \le B \exp(\homop_a (\rho_t) / 2)$ and $\phi(\ModifiedLoss_s) \ge 2$. We thus obtain that
    \begin{align*}
        R_2 (t) = \, & B_4 B_5 (\eta, \rho_t) \Loss_t \phi(\Loss_t) \le B \left( \rho_t^{2 M - 2} + \eta^{2M - 2} \ModifiedLoss_s^{2M - 2} e^{- \homop_a(\rho_t)} + 1 \right) \Loss_t \phi(\Loss_t) \\
        \le \, & B \exp \left(- \frac{1}{2} \homop_a (\rho_t) \right) \left( \rho_t^{2 M - 2} + \eta^{2M - 2} \ModifiedLoss_s^{2M - 2} e^{- \homop_a(\rho_t)} + 1 \right) \ModifiedLoss_s  \phi(\ModifiedLoss_s) \\
        \le \, & B \exp \left(- \frac{1}{4} \homop_a (\rho_t) \right) \left( \eta^{2M - 2} \ModifiedLoss_s^{2M - 2} + 1 \right) \ModifiedLoss_s  \phi(\ModifiedLoss_s).
    \end{align*}
    To further estimate this term, we use the following naive bound on $\ModifiedLoss_s$:
    \begin{equation*}
        \ModifiedLoss_s \ge \, \exp \left( \homop_a (\rho_s) - \homoq (\rho_s) \right),
    \end{equation*}
    which further implies
    \begin{equation*}
        \phi(\ModifiedLoss_s) = \, - \log n - \log \ModifiedLoss_s \le \homoq (\rho_s) \le B \exp \left( \frac{1}{4} \homop_a (\rho_s) \right) \le B \exp \left( \frac{1}{4} \homop_a (\rho_t) \right)
    \end{equation*}
    for some constant $B$, since the degree of $\homoq$ is at most $M - 1$. In the above display, the last step is due to $\rho_s \le \rho_t$. It finally follows that
    \begin{align*}
        \eta R_2 (t) \le \, & B \eta \ModifiedLoss_s \left( \eta^{2M - 2} \ModifiedLoss_s^{2M - 2} + 1 \right).
    \end{align*}
    Therefore, the requirements $\eta R_1 (t) \le 1/2$ and $\eta R_2 (t) \le 1/2$ are satisfied as long as
    \begin{equation*}
        \eta \ModifiedLoss_s \left( \eta^{2M - 2} \ModifiedLoss_s^{2M - 2} + 1 \right) \le \, \frac{1}{2 B},
    \end{equation*}
    which is equivalent to
    \begin{equation*}
        \eta \ModifiedLoss_s \le \frac{1}{B}
    \end{equation*}
    for some constant $B$ only depending on $(M, \homop, \homoq, \homor, \homos)$. This completes the proof.
\end{proof}
We have completed the proof of margin improvement. The proof of 
\begin{equation*}
    \GDmargin\big(\param_t\big) \le \gamma\big(\param_t\big) \le \big(1+\epsilon_t\big) \GDmargin\big(\param_t\big), \quad \epsilon_t \to 0
\end{equation*}
is the same as the gradient flow case, so we omit it here for simplicity.

\subsection{Convergence Rates: Proof of Theorem \ref{thm: Margin improving and convergence-gd}}\label{sec:converge_rate_gd}
In this section, we give refined bounds on the decrease rate of $\Loss_t = \Loss (\param_t )$ and the increase rate of $\rho_t = \| \param_t \|$, completing the proof of the convergence rates part of \Cref{thm: Margin improving and convergence-gd}. In a similar spirit to \Cref{residual_estimates}, we can show that the conditions of the following \Cref{lem: decrease of L} are satisfied as long as
\begin{equation*}
    \ModifiedLoss_s < \min \left\{ \frac{1}{n e^2}, \, \frac{1}{B \eta} \right\},
\end{equation*}
which holds under \Cref{asp:initial-cond-gd}.

\begin{lemma}
[Decrease of the loss function]
\label{lem: decrease of L}
Under the same assumptions as \Cref{thm:combined_argument}, and further assume that
\begin{equation*}
    \eta \sup_{\rho_t \ge \rho_s, \ModifiedLoss_t \le \ModifiedLoss_s} \left\{ B_3 \Loss_t \left( \rho_t^{2M - 2} + 1 \right) \right\} \le \frac{1}{2}.
\end{equation*}
Then, we have for all $t \ge s$:
\begin{equation}
\label{eq: decrease of L}
- C_2 \eta \Loss_t^2 \bigg(\LinkFun(\Loss_t)\bigg)^{2-2 / M} \le \Loss_{t+1} - \Loss_t \le  - C_1 \eta \Loss_t^2 \bigg(\LinkFun(\Loss_t)\bigg)^{2-2 / M},  
\end{equation}
where $C_1, C_2$ are some constants depending on $\GDmargin(\param_s)$, $\rho_s$, $A$ and $(M, \homop, \homoq, \homor, \homos)$.
\end{lemma}
\begin{proof}
Using Taylor expansion, we get that for some $\epsilon \in [0, 1]$:
\[
\Loss_{t+1} - \Loss_t  = -\eta \|\nabla \Loss_t\|^2 +  \frac{\eta^2}{2} \nabla \Loss_t^\top \nabla^2 \Loss_{t+\epsilon} \nabla \Loss_t.  
\] 
Note that by our assumption,
\begin{align*}
    \frac{\eta^2}{2} \nabla \Loss_t^\top \nabla^2 \Loss_{t+\epsilon} \nabla \Loss_t \le \, & \frac{\eta^2}{2} \left\| \nabla^2 \Loss_{t+\epsilon} \right\| \left\| \nabla \Loss_t \right\|^2 \\
    \stackrel{(i)}{\le} \, & \frac{\eta^2}{2} \left\| \nabla \Loss_t \right\|^2 \cdot B_3 \Loss_{t+\epsilon} \left( \rho_{t+\epsilon}^{2M-2} + 1 \right) \\
    \stackrel{(ii)}{\le} \, & \frac{\eta}{4} \left\| \nabla \Loss_t \right\|^2,
\end{align*}
where $(i)$ follows from \Cref{lem:gradient_hessian_bound_L}, and $(ii)$ is due to our assumption. This implies that
\begin{equation*}
    - \frac{5}{4} \eta \|\nabla \Loss_t\|^2 \le \, \Loss_{t+1} - \Loss_t \le \, - \frac{3}{4} \eta \|\nabla \Loss_t\|^2.
\end{equation*}
We now establish lower and upper bounds on $\|\nabla \Loss_t\|$. On the one hand, we have
\begin{align*}
    \|\nabla \Loss_t \| \le \, & B_3 \Loss_t \left( \rho_t^{M-1} + 1 \right) \le \, C_2 \Loss_t \rho_t^{M-1} \le C_2 \Loss_t \left( \frac{\phi(\ModifiedLoss_t)}{\GFmargin (\param_t)} \right)^{1 - 1/M} \\
    \le \, & C_2 \Loss_t \left( \frac{\phi(\Loss_t)}{\GDmargin (\param_s)} \right)^{1 - 1/M} \le C_2 \Loss_t \phi(\Loss_t)^{1 - 1/M},
\end{align*}
where $C_2$ depends on $\GDmargin(\param_s)$, $\rho_s$, and $(M, \homop, \homoq, \homor, \homos)$. On the other hand,
\begin{align*}
    \| \nabla \Loss_t \| \ge \, & \frac{v_t}{\rho_t} \ge \frac{M\Loss_t \phi(\Loss_t) - \homop^\prime (\rho_t) \Loss_t}{\rho_t} &&\explain{ By \Cref{lem: Bound of v_t}} \\
    \ge \, & \frac{\GFmargin (\param_t)^{1/M} \Loss_t (M \phi(\Loss_t) - \homop^\prime (\rho_t))}{(\phi(\Loss_t) - \homop_a(\rho_t))^{1/M}} &&\explain{ By definition of $\GFmargin$ } \\ 
    \ge \, & M^{1/M} \frac{\GFmargin (\param_t)^{1/M} \Loss_t (M \phi(\Loss_t) - \homop^\prime (\rho_t))}{(M\phi(\Loss_t) - \homop^\prime (\rho_t))^{1/M}} &&\explain{ By \Cref{lem:property-pa} } \\ 
    = \, & M^{1/M} \GFmargin (\param_t)^{1/M} \Loss_t (M \phi(\Loss_t) - \homop^\prime (\rho_t))^{1-1/M}\\ 
    \ge \, & M^{1/M} \GFmargin (\param_t)^{1/M} \Loss_t \bigg(\LinkFun(\Loss_t)\bigg)^{1 - 1/M} &&\explain{ By Eq.~\Cref{eq:v_t log lb} } \\ 
    \ge \, & M^{1/M} \GDmargin (\param_s)^{1/M} \Loss_t \bigg(\LinkFun(\Loss_t)\bigg)^{1 - 1/M} &&\explain{ By monotonicity of $\GFmargin$ }\\ 
    := \, & C_1 \Loss_t \bigg(\LinkFun(\Loss_t)\bigg)^{1 - 1/M}. 
\end{align*}
This completes the proof of \Cref{eq: decrease of L}.
\end{proof}

We first show that $\Loss_t$ converges to zero and derive its convergence rate.

\begin{lemma}
[Loss convergence rate for GD]
\label{lem: Loss convergence GD}
Under the same assumptions as \Cref{lem: decrease of L}, we have
\begin{equation}
\label{eq: loss convergence GD}
        \Loss_t =  \Theta \bigg( \frac{1}{\eta t (\log  \eta t)^{2-2/M}} \bigg) \to 0, 
\end{equation}
as $t \to \infty$. Here, the constants hided by $\Theta$ only depend on $\GDmargin(\param_s)$, $\rho_s$, and $(M, \homop, \homoq, \homor, \homos)$.
\end{lemma}
\begin{proof}
With \Cref{lem: decrease of L}, we know that
\[
\frac{ \Loss_{t+1} - \Loss_t}{\Loss_t^2 \Big(\LinkFun(\Loss_t)\Big)^{2-2 / M}} \le - C_1 \eta. 
\] 
Now we consider the function, 
\[
S(x) \coloneqq \int_x^{\Loss_s} \frac{1}{u^2 \Big(\log \frac{1}{n u}\Big)^{2-2 / M}} \, \mathrm{d} u, \quad S^\prime (x) = -\frac{1}{x^2 \Big(\log \frac{1}{n x}\Big)^{2-2 / M}}.
\] 
When $x\le \Loss_s < \frac{1}{ne^2}$, we have $S^\prime (x)$ is increasing and $S(x)$ is convex. Hence, 
\[
S(\Loss_{t+1}) - S(\Loss_t)\ge S^\prime (\Loss_t) (\Loss_{t+1} - \Loss_t) \ge C_1 \eta.  
\] 
This leads to 
\[
S(\Loss_t) - S(\Loss_s) \ge C_1 \eta(t-s) \Rightarrow S(\Loss_t) = \Omega( \eta t).
\] 
Hence, we must have $S(\Loss_t) \to +\infty$ as $t \to \infty$. This leads to $\Loss_t \to 0$. Hence, there exists $t > s$ such that $\Loss_t < \Loss_s / 2$. To get an upper bound on $S(\Loss_t)$, we show that $\Loss_{t+1} \ge \Loss_t / 2$ for $t \ge s$. Note that
\begin{align*}
    \left\vert \Loss_{t+1} - \Loss_t \right\vert \le \, & \left\| \nabla \Loss_{t+\epsilon} \right\| \left\| \param_{t+1} - \param_t \right\| = \eta \left\| \nabla \Loss_t \right\| \left\| \nabla \Loss_{t+\epsilon} \right\| \\
    \stackrel{(i)}{\le} \, & \eta B_3^2 \Loss_t \Loss_{t+\epsilon} \left( \rho_t^{M - 1} + 1 \right) \left( \rho_{t+\epsilon}^{M - 1} + 1 \right) \\
    \stackrel{(ii)}{\le} \, & \eta \Loss_t B_3^2 \Loss_{t+\epsilon} \left( \rho_{t+\epsilon}^{M - 1} + 1 \right)^2 = \eta R_1 (t+\epsilon) \Loss_t \le \frac{1}{2} \Loss_t,
\end{align*}
where $(i)$ follows from \Cref{lem:gradient_hessian_bound_L}, $(ii)$ is because of $\rho_{t+\epsilon} \ge \rho_t$, and the last inequality is due to our assumption.
Hence, we have (note that $\Loss_{t+1} \le \Loss_t$)
\[
    \Loss_t^2 \Big(\LinkFun(\Loss_t)\Big)^{2-2 / M} \le 4 \Loss_{t+1}^2 \Big(\log \frac{1}{n \Loss_{t+1}}\Big)^{2-2 / M}. 
\] 
Therefore, 
\[
    \frac{ \Loss_{t+1} - \Loss_t}{\Loss_{t+1}^2 \Big(\log \frac{1}{n \Loss_{t+1}}\Big)^{2-2 / M}} \ge - 4 C_2 \eta. 
\] 
Now since $\Loss_{t+1} \le \Loss_{t+\epsilon} \le \Loss_t$, we have
\[
S(\Loss_{t+1}) - S(\Loss_t)  = S^\prime (\Loss_{t+\epsilon}) (\Loss_{t+1} - \Loss_t)  \le  S^\prime (\Loss_{t+1}) (\Loss_{t+1} - \Loss_t) \le 4 C_2 \eta. 
\] 
This leads to 
\[
S(\Loss_t) - S(\Loss_s) \le 4 C_2 \eta(t-s) \Rightarrow S(\Loss_t) = O( \eta t).
\] 
Therefore, $S(\Loss_t) = \Theta( \eta t)$, where $\Theta$ only hides a constant depending on $C_1$ and $C_2$.
By \Cref{lem:int-log-2}, we know that 
\[
    S^{-1}(x) = \Theta \bigg( \frac{1}{x (\log(nx))^{2-2/M}} \bigg). 
\]
Hence, 
\[
\Loss_t =  \Theta \bigg( \frac{1}{\eta t (\log  \eta t)^{2-2/M}} \bigg) \to 0.
\] 
This completes the proof of \Cref{lem: Loss convergence GD}. 
\end{proof}

\begin{lemma}
[Parameter convergence rate for GD]
\label{lem: Parameter convergence rate for GD}
Under the same assumptions as \Cref{lem: decrease of L}, we have
\begin{equation}
\label{eq:param-convergence-rate-gd}
       \rho_t = \Theta  ( \log \eta t)^{\frac{1}{M}},  
\end{equation}
as $t \to \infty$. Here, the constants hided by $\Theta$ only depend on $\GDmargin(\param_s)$, $\rho_s$, and $(M, \homop, \homoq, \homor, \homos)$.
\end{lemma}
\begin{proof}
    Note that by definition
    \[
    \rho_t^M = \frac{\LinkFun(\ModifiedLoss_t)}{\GFmargin(\param_t)} \le \frac{\LinkFun(\Loss_t)}{\GDmargin(\param_s)} = O(\log \eta t). 
    \] 
On the other hand, by the near-homogeneity of $f$, we have 
\[
    B \rho_t^M \ge \bar f_{\min} (\param_t) \ge \log \frac{1}{n \Loss_t} \Rightarrow \rho_t^M = \Omega( \log \eta t).
\] 
This completes the proof of \Cref{lem: Parameter convergence rate for GD}.
\end{proof}

Combining \Cref{lem: Loss convergence GD,lem: Parameter convergence rate for GD} concludes the proof of convergence rates in \Cref{thm: Margin improving and convergence-gd}.

\subsection{Directional Convergence: Proof of Theorem \ref{thm: directional convergence-gd} (Part 1)}\label{sec:direct_converge_gd}

Similar to the proof in \Cref{sec:proof:direct}, the main idea is to construct the desingularizing function $\Phi$. For GD, we use $\GDmargin (\param_t)$ to denote $\GDmargin(\param_t)$ and $\zeta_t$ to denote $\|\tilde{\param}_{t+1} - \tilde{\param}_t \|$. Our goal is to show that $\sum_{t\ge s} \zeta_t < \infty$. This leads to that the curve swept by $\tilde{\param}_t$ has finite length, and that $\tilde{\param}_t$ converges to a limit, thus proving \Cref{thm: directional convergence-gd}.

The rest of this section will be devoted to the proof of $\sum_{t\ge s} \zeta_t < \infty$. Since $\{ \GDmargin(\param_t) \}_{t \ge s}$ is increasing and bounded from above, we know that
\begin{equation*}
    \gamma_* = \, \lim_{t \to \infty} \GDmargin(\param_t)
\end{equation*}
exists. To show that $\sum_{t\ge s} \zeta_t < \infty$, we prove the following lemma. 

\begin{lemma}
[Desingularizing function for GD]
\label{lem: Desingularizing function for GD}
There exist $R>0, \nu>0$ and a definable desingularizing function $\Psi$ on $[0, \nu)$, such that for all large enough $t$ with $\left\|\param_t\right\|>R$ and $\GDmargin(\param_t)>\gamma_*-\nu$, and all $\alpha \in [0, 1]$, it holds that
$$
\zeta_t \leq c \Psi^\prime \left(\gamma_*-\GDmargin (\param_{t+\alpha}) \right) (\GDmargin (\param_{t+1}) - \GDmargin (\param_t)). 
$$
for some constant $c>0$. As a consequence, we have
\begin{equation*}
    \zeta_t \leq c \left( \Psi \left(\gamma_*-\GDmargin (\param_t) \right) - \Psi \left(\gamma_*-\GDmargin (\param_{t+1}) \right) \right).
\end{equation*}
\end{lemma}

Once we have this lemma, we immediately know that 
\[
\sum_{t\ge s} \zeta_t \le c \sum_{t\ge s} \left( \Psi \left(\gamma_*-\GDmargin (\param_t) \right) - \Psi \left(\gamma_*-\GDmargin (\param_{t+1}) \right) \right) < + \infty.   
\] 
The remainder of this section will be devoted to the proof of \Cref{lem: Desingularizing function for GD}.

Recall that we defined the spherical parts and radial parts of $\partial f(\param) = \nabla f(\param)$ for a differentiable function $f$.  In this section, we also use $a_t$ to denote $e^{\Phi(\ModifiedLoss_t)}$, where we recall that
\begin{equation*}
    \Phi(x) = \log \phi(x) - \frac{2}{\phi(x)}, \quad \Phi'(x) = - \frac{1}{x \phi(x)} - \frac{2}{x \phi(x)^2}.
\end{equation*}
We will first show that

\begin{lemma}
 [Decomposition of radial and spherical parts for GD]
\label{lem: Decomposition of raidal and spherical parts for GD}
Under \Cref{asp:initial-cond-gd}, we have for all sufficiently large $t \ge s$, 
\begin{equation}
\label{eq: decomp-gamma-GD}
    \GDmargin (\param_{t+1}) - \GDmargin (\param_t) \ge c \eta \Big(\| \partial_r \GDmargin (\param_t) \| \| \partial_r \Loss_t\| + \| \partial_\perp \GDmargin (\param_t) \| \| \partial_\perp \Loss_t\| ]\Big).
\end{equation}
and 
\begin{equation}
\label{eq: decomp-zeta-GD}
\|\partial_\perp \GDmargin (\param_t) \|= \frac{\|\partial_\perp a_t \|}{\rho_t^M} , \quad \text{ and } \quad \zeta_t \le  \eta \frac{\| \partial_\perp \Loss_t\|}{\rho_t}.
\end{equation}
\end{lemma}

\begin{proof}
    Recall from our definition that
    \begin{align*}
        & \param_{t + \alpha} = \param_t - \alpha \eta \nabla \Loss_t, \quad \ModifiedLoss_{t+\alpha} = \ModifiedLoss (\param_{t+\alpha}), \\
        & a_{t+\alpha} = \exp (\Phi(\ModifiedLoss_{t+\alpha})), \quad \rho_{t+\alpha} = \| \param_{t+\alpha} \|, \quad \GDmargin (\param_{t + \alpha}) = \frac{a_{t+\alpha}}{\rho_{t + \alpha}^M}.
    \end{align*}
    We first show that
    \begin{equation*}
        \frac{\mathrm{d} \GDmargin(\param_{t+\alpha})}{\mathrm{d} \alpha}  \Big\vert_{\alpha = 0} = \frac{\kd \GDmargin(\param_t)}{\kd t} = \eta \Big(\| \partial_r \GDmargin (\param_t) \| \| \partial_r \Loss_t\| + \| \partial_\perp \GDmargin (\param_t) \| \| \partial_\perp \Loss_t\| ]\Big).
    \end{equation*}
    To this end, we follow the proof scheme of \Cref{lem:decomp-radial-spherical}. Note that
    \begin{equation*}
        \frac{\kd \GDmargin(\param_t)}{\kd t} = \langle \partial \GDmargin (\param_t), - \eta \nabla \Loss_t \rangle = \eta \left( \langle \partial_r \GDmargin(\param_t), - \partial_r \Loss_t \rangle + \langle \partial_\perp \GDmargin(\param_t), - \partial_\perp \Loss_t \rangle \right),
    \end{equation*}
    it suffices to prove the following two statements:
    \begin{itemize}
        \item [(a)] $\langle \partial_r \GDmargin(\param_t), \param_t \rangle$ and $\langle -\partial_r \Loss_t, \param_t \rangle $ have the same sign. 
        \item [(b)] $\partial_\perp \GDmargin(\param_t)$ and $-\partial_\perp \Loss_t$ point to the same direction.  
\end{itemize}
\paragraph{Proof of (a).} It suffices to show that $\langle \nabla \GDmargin(\param_t), \param_t \rangle$ and $- \langle \nabla \Loss_t, \param_t \rangle$ have the same sign. By \Cref{lem: Bound of v_t}, we have $- \langle \nabla \Loss_t, \param_t \rangle = v_t \ge 0$. Next, we show that $\langle \nabla \GDmargin(\param_t), \param_t \rangle \ge 0$. Note that
\begin{equation*}
    \nabla \GDmargin (\param_t) = \frac{\nabla a_t}{\rho_t^M} - \frac{M a_t \tilde{\param}_t}{\rho_t^{M+1}},
\end{equation*}
which implies
\begin{equation*}
    \langle \nabla \GDmargin(\param_t), \param_t \rangle = \frac{\langle \nabla a_t, \param_t \rangle}{\rho_t^M} - \frac{M a_t}{\rho_t^{M}}.
\end{equation*}
Through direct calculation, we obtain that
\begin{align*}
    & \langle \nabla a_t, \param_t \rangle = a_t \Phi' (\ModifiedLoss_t) \langle \nabla \ModifiedLoss_t, \param_t \rangle \\
    = \, & a_t \Phi' (\ModifiedLoss_t) \left\langle e^{\homop_a (\rho_t)} \nabla \Loss_t + \ModifiedLoss_t \homop_a' (\rho_t) \tilde{\param}_t, \param_t \right\rangle \\
    = \, & a_t \Phi' (\ModifiedLoss_t) \left( - e^{\homop_a (\rho_t)} v_t + \ModifiedLoss_t \homop_a' (\rho_t) \rho_t \right).
\end{align*}
From the proof of \Cref{lem: Convexity of Phi}, we know that
\begin{equation*}
     \Phi'(x) < - \frac{1}{x \phi(x)} < 0, \quad 0 < x < \frac{1}{n e^2}.
\end{equation*}
Further since (see, e.g., the proof of \Cref{thm:combined_argument})
\begin{equation*}
    e^{\homop_a (\rho_t)} v_t - \ModifiedLoss_t \homop_a' (\rho_t) \rho_t \ge M \ModifiedLoss_t \phi(\ModifiedLoss_t),
\end{equation*}
we obtain that
\begin{align*}
    \langle \nabla a_t, \param_t \rangle = \, & a_t \Phi' (\ModifiedLoss_t) \left( - e^{\homop_a (\rho_t)} v_t + \ModifiedLoss_t \homop_a' (\rho_t) \rho_t \right) \\
    \ge \, & a_t \frac{1}{\ModifiedLoss_t \phi(\ModifiedLoss_t)} \cdot M \ModifiedLoss_t \phi(\ModifiedLoss_t) = M a_t,
\end{align*}
which leads to $\langle \nabla \GDmargin(\param_t), \param_t \rangle \ge 0$. This proves part (a).

\paragraph{Proof of (b).} Straightforward calculation leads to
\begin{align*}
    \partial_\perp \GDmargin(\param_t) = \, & \nabla \GDmargin (\param_t) - \langle \nabla \GDmargin(\param_t), \tilde{\param}_t \rangle \tilde{\param}_t = \frac{\nabla a_t - \langle \nabla a_t, \tilde{\param}_t \rangle \tilde{\param}_t}{\rho_t^M} \\
    = \, & a_t \Phi'(\ModifiedLoss_t) e^{\homop_a (\rho_t)} \left( \nabla \Loss_t - \langle \nabla \Loss_t, \tilde{\param}_t \rangle \tilde{\param}_t \right) \\
    = \, & a_t \Phi'(\ModifiedLoss_t) e^{\homop_a (\rho_t)} \partial_{\perp} \Loss_t,
\end{align*}
which points to the same direction as $- \partial_{\perp} \Loss_t$ since $\Phi'(\ModifiedLoss_t) < 0$.

To complete the proof, we only need to show that
\begin{equation*}
    \GDmargin (\param_{t+1}) - \GDmargin (\param_t) \ge c \frac{\kd \GDmargin(\param_t)}{\kd t}
\end{equation*}
for some constant $c$. We first compute
\begin{equation*}
    \frac{\kd \GDmargin(\param_t)}{\kd t} = \frac{a_t^{\prime}}{\rho_t^M} - \frac{M a_t \rho_t'}{\rho_t^{M+1}}.
\end{equation*}
Further since $x \mapsto e^{\Phi (x)}$ and $\alpha \mapsto \rho_{t + \alpha}^{-M} = \| \param_{t} - \alpha \nabla \Loss_t \|^{-M}$ are all convex, it follows that
\begin{align*}
    & \GDmargin (\param_{t+1}) - \GDmargin (\param_t) = \, \frac{a_{t+1}}{\rho_{t+1}^M} - \frac{a_t}{\rho_t^M} \\
    = \, & \frac{a_{t+1} - a_t}{\rho_{t+1}^M} + a_t \left( \rho_{t+1}^{-M} - \rho_t^{-M} \right) \\
    \ge \, & \frac{a_t \Phi'(\ModifiedLoss_t) (\ModifiedLoss_{t+1} - \ModifiedLoss_t)}{\rho_{t+1}^M} - \frac{M a_t \rho_t'}{\rho_t^{M+1}} \\
    \stackrel{(i)}{\ge} \, & \frac{a_t \Phi'(\ModifiedLoss_t) \ModifiedLoss_t'}{\rho_{t+1}^M} - \frac{M a_t \rho_t'}{\rho_t^{M+1}} - \frac{a_t \vert \Phi'(\ModifiedLoss_t) \vert \vert \ModifiedLoss_{t + \epsilon}'' \vert }{2 \rho_{t+1}^M} \\
    \stackrel{(ii)}{\ge} \, & \frac{a_t \Phi'(\ModifiedLoss_t) \ModifiedLoss_t'}{\rho_t^M} - \frac{M a_t \rho_t'}{\rho_t^{M+1}} - \frac{M a_t \Phi'(\ModifiedLoss_t) \ModifiedLoss_t' \rho_t'}{\rho_t^{M+1}} - \frac{a_t \vert \Phi'(\ModifiedLoss_t) \vert \vert \ModifiedLoss_{t + \epsilon}'' \vert }{2 \rho_{t+1}^M} \\
    = \, & \frac{\kd \GDmargin(\param_t)}{\kd t} - \frac{M a_t \Phi'(\ModifiedLoss_t) \ModifiedLoss_t' \rho_t'}{\rho_t^{M+1}} - \frac{a_t \vert \Phi'(\ModifiedLoss_t) \vert \vert \ModifiedLoss_{t + \epsilon}'' \vert }{2 \rho_{t+1}^M},
\end{align*}
where $(i)$ follows from Taylor expansion: $\epsilon \in [0, 1]$, and $(ii)$ follows again from the fact that $\alpha \mapsto \rho_{t + \alpha}^{-M}$ is convex. To proceed, we first give a lower bound on $\frac{\kd \GDmargin(\param_t)}{\kd t}$. By definition,
\begin{align*}
    \frac{\kd \GDmargin(\param_t)}{\kd t} = \, & \frac{1}{\rho_t^{M+1}} \left( a_t^\prime \rho_t - M a_t \rho_t' \right) = \frac{\eta a_t}{\rho_t^{M}} \left( \Phi'(\ModifiedLoss_t) \left\langle \nabla \ModifiedLoss_t, - \nabla \Loss_t \right\rangle - M  \frac{v_t}{\rho_t^2} \right) \\
    = \, & \frac{\eta a_t}{\rho_t^{M}} \left( - \Phi'(\ModifiedLoss_t) \left( e^{\homop_a (\rho_t)} \| \nabla \Loss_t \|^2 - \ModifiedLoss_t \homop_a' (\rho_t) \frac{v_t}{\rho_t} \right) - M  \frac{v_t}{\rho_t^2} \right).
\end{align*}
To deal with this term, we note that
\begin{align*}
    & - \Phi'(\ModifiedLoss_t) \left( e^{\homop_a (\rho_t)} \| \nabla \Loss_t \|^2 - \ModifiedLoss_t \homop_a' (\rho_t) \frac{v_t}{\rho_t} \right) - M  \frac{v_t}{\rho_t^2} \\
    \stackrel{(i)}{=} \, & \frac{1}{\ModifiedLoss_t \phi (\ModifiedLoss_t)} \left( e^{\homop_a (\rho_t)} \| \nabla \Loss_t \|^2 - \ModifiedLoss_t \homop_a' (\rho_t) \frac{v_t}{\rho_t} \right) - M  \frac{v_t}{\rho_t^2} \\
    & + \frac{2}{\ModifiedLoss_t \phi (\ModifiedLoss_t)^2} \left( e^{\homop_a (\rho_t)} \| \nabla \Loss_t \|^2 - \ModifiedLoss_t \homop_a' (\rho_t) \frac{v_t}{\rho_t} \right) \\
    \stackrel{(ii)}{\ge} \, & \frac{2}{\ModifiedLoss_t \phi (\ModifiedLoss_t)^2} \left( e^{\homop_a (\rho_t)} \| \nabla \Loss_t \|^2 - \ModifiedLoss_t \homop_a' (\rho_t) \frac{v_t}{\rho_t} \right),
\end{align*}
where $(i)$ is due to $\Phi'(x) = - 1 / x \phi(x) - 2 / x \phi(x)^2$, and $(ii)$ follows from the fact that
\begin{align*}
    e^{\homop_a (\rho_t)} \| \nabla \Loss_t \|^2 - \ModifiedLoss_t \homop_a' (\rho_t) \frac{v_t}{\rho_t} \ge \, & e^{\homop_a (\rho_t)} \frac{v_t^2}{\rho_t^2} - \ModifiedLoss_t \homop_a' (\rho_t) \frac{v_t}{\rho_t} && \explain{By Cauchy-Schwarz} \\ 
    = \, & \frac{v_t}{\rho_t^2} \left( e^{\homop_a (\rho_t)} v_t - \ModifiedLoss_t \homop_a' (\rho_t) \rho_t \right) \\
    \ge \, & \frac{v_t}{\rho_t^2} \left( M \ModifiedLoss_t \phi(\Loss_t) - \homop'(\rho_t) \ModifiedLoss_t - \ModifiedLoss_t \homop_a' (\rho_t) \rho_t \right) && \explain{By \Cref{lem: Bound of v_t}} \\
    \ge \, & \frac{v_t}{\rho_t^2} \left( M \ModifiedLoss_t \phi(\Loss_t) - M \homop_a (\rho_t) \ModifiedLoss_t \right) && \explain{By \Cref{lem:modified_p_a}} \\
    = \, &  M \ModifiedLoss_t \phi (\ModifiedLoss_t) \frac{v_t}{\rho_t^2}.
\end{align*}
We thus obtain that
\begin{align}\label{eq:initial_lowerbd_diff_GDmargin}
\frac{\kd \GDmargin(\param_t)}{\kd t} 
\ge \, & \frac{\eta a_t}{\rho_t^{M}} \frac{2}{\ModifiedLoss_t \phi(\ModifiedLoss_t)^2} \left( e^{\homop_a (\rho_t)} \| \nabla \Loss_t \|^2 - \ModifiedLoss_t \homop_a' (\rho_t) \frac{v_t}{\rho_t} \right) \\
\ge \, & \frac{\eta a_t}{\rho_t^{M}} \frac{2}{\ModifiedLoss_t \phi(\ModifiedLoss_t)^2} \left( e^{\homop_a (\rho_t)} \| \nabla \Loss_t \|^2 - \ModifiedLoss_t \homop_a' (\rho_t) \rho_t \frac{\| \nabla \Loss_t \|^2}{v_t} \right)  &&\explain{ By Cauchy-Schwarz } \notag \\
= \, & \frac{\eta a_t}{\rho_t^{M}} \frac{2}{\ModifiedLoss_t \phi(\ModifiedLoss_t)^2} \frac{\| \nabla \Loss_t \|^2}{v_t} \left( e^{\homop_a (\rho_t)} v_t - \ModifiedLoss_t \homop_a' (\rho_t) \rho_t \right) \notag  \\
\ge \, & \frac{\eta a_t}{\rho_t^{M}} \frac{2}{\ModifiedLoss_t \phi(\ModifiedLoss_t)^2} \frac{\| \nabla \Loss_t \|^2}{v_t} M \ModifiedLoss_t \phi(\ModifiedLoss_t)  &&\explain{ Similar to the proof of $(ii)$ above } \notag \\ 
= \,& \frac{\eta M a_t}{\rho_t^{M}} \frac{2}{\phi(\ModifiedLoss_t)} \frac{\| \nabla \Loss_t \|^2}{v_t} \notag .
\end{align}
We next prove two useful claims:
\begin{itemize}
    \item [(c)] There exists a constant $c_1 < 1/2$, such that
    \begin{equation*}
        \frac{a_t \vert \Phi'(\ModifiedLoss_t) \vert \vert \ModifiedLoss_{t + \epsilon}'' \vert }{2 \rho_{t+1}^M} \le c_1 \frac{\eta M a_t}{\rho_t^{M}} \frac{2}{\phi(\ModifiedLoss_t)} \frac{\| \nabla \Loss_t \|^2}{v_t} \le c_1 \frac{\kd \GDmargin(\param_t)}{\kd t}.
    \end{equation*}
    \item [(d)] There exists a constant $c_2 < 1/2$, such that
    \begin{equation*}
        \frac{M a_t \Phi'(\ModifiedLoss_t) \ModifiedLoss_t' \rho_t'}{\rho_t^{M+1}} \le c_2 \frac{a_t \vert \ModifiedLoss_t' \vert}{\rho_t^{M}} \frac{2}{\ModifiedLoss_t \phi(\ModifiedLoss_t)^2} \le c_2 \frac{\kd \GDmargin(\param_t)}{\kd t}.
    \end{equation*}
\end{itemize}

\paragraph{Proof of (c).} 
According to \Cref{lem:gradient_hessian_bound_L} and \Cref{lem:gradient_hessian_bound}, we have
\begin{equation*}
    \vert \ModifiedLoss_{t + \epsilon}'' \vert \le \, \eta^2 \| \nabla^2 \ModifiedLoss_{t+\epsilon} \| \| \nabla \Loss_t \|^2 \le B_4 \eta^2 \ModifiedLoss_{t+\epsilon} (\rho_{t+\epsilon}^{2M-2} + 1) \| \nabla \Loss_t \|^2
\end{equation*}
for sufficiently large $t$. This implies
\begin{align*}
    \frac{a_t \vert \Phi'(\ModifiedLoss_t) \vert \vert \ModifiedLoss_{t + \epsilon}'' \vert }{2 \rho_{t+1}^M} \le \, & B_4 \eta^2 \| \nabla \Loss_t \|^2 \frac{a_t \vert \Phi'(\ModifiedLoss_t) \vert \ModifiedLoss_t (\rho_{t+1}^{M-2} + 1)}{2} \\
    \le \, & C \eta^2 \| \nabla \Loss_t \|^2 a_t \vert \Phi'(\ModifiedLoss_t) \vert \ModifiedLoss_t (\rho_t^{M-2}+1) \\
    \le \, & C \eta^2 \| \nabla \Loss_t \|^2 a_t \frac{1}{\phi(\ModifiedLoss_t)} \rho_t^{M-2},
\end{align*}
where the last line follows from the fact
\begin{equation*}
    \left\vert \Phi'(\ModifiedLoss_t) \right\vert = \, \frac{1}{\ModifiedLoss_t \phi(\ModifiedLoss_t)} + \frac{2}{\ModifiedLoss_t \phi(\ModifiedLoss_t)^2} \le \frac{C}{\ModifiedLoss_t \phi(\ModifiedLoss_t)}
\end{equation*}
when $\ModifiedLoss_t < 1/ne^2$.
Hence, it suffices to show that
\begin{align*}
    C \eta^2 \| \nabla \Loss_t \|^2 a_t \frac{1}{\phi(\ModifiedLoss_t)} \rho_t^{M-2} \le c_1 \frac{\eta M a_t}{\rho_t^{M}} \frac{2}{\phi(\ModifiedLoss_t)} \frac{\| \nabla \Loss_t \|^2}{v_t}
    \Longleftrightarrow \, C \eta \rho_t^{2M-2} v_t \le c_1,
\end{align*}
where $C$ depends on $\ModifiedLoss_s$ and $M$. Using \Cref{lem:gradient_hessian_bound_L}, we know that for sufficiently large $t$,
\begin{equation*}
    \| \nabla \Loss_t \| \le B_3 \Loss_t (\rho_t^{M-1} + 1) \le C \Loss_t \rho_t^{M-1},
\end{equation*}
thus leading to
\begin{align*}
    C \eta \rho_t^{2M-2} v_t \le \, C \eta \rho_t^{2M-1} \| \nabla \Loss_t \| \le C \eta \rho_t^{3M-2} \Loss_t,
\end{align*}
which can be arbitrarily small as $t \to \infty$ due to our result on convergence rates of $\rho_t$ and $\Loss_t$. This proves part (c). Note that our proof works for any $\eta$ and sufficiently large $t$, and the constant $c_1$ can be arbitrarily small.

\paragraph{Proof of (d).} In Eq.~\Cref{eq:initial_lowerbd_diff_GDmargin}, we already proved that
\begin{align*}
    \frac{\kd \GDmargin(\param_t)}{\kd t} \ge \, & \frac{\eta a_t}{\rho_t^{M}} \frac{2}{\ModifiedLoss_t \phi(\ModifiedLoss_t)^2} \left( e^{\homop_a (\rho_t)} \| \nabla \Loss_t \|^2 - \ModifiedLoss_t \homop_a' (\rho_t) \frac{v_t}{\rho_t} \right) \\
    = \, & \frac{a_t \vert \ModifiedLoss_t' \vert}{\rho_t^{M}} \frac{2}{\ModifiedLoss_t \phi(\ModifiedLoss_t)^2}.
\end{align*}
Therefore, it suffices to show that
\begin{equation*}
    \frac{M a_t \vert \Phi'(\ModifiedLoss_t) \vert \vert \ModifiedLoss_t' \vert \rho_t'}{\rho_t^{M+1}} \le c_2 \frac{a_t \vert \ModifiedLoss_t' \vert}{\rho_t^{M}} \frac{2}{\ModifiedLoss_t \phi(\ModifiedLoss_t)^2},
\end{equation*}
which reduces to
\begin{equation*}
    \frac{M \vert \Phi'(\ModifiedLoss_t) \vert \rho_t'}{\rho_t} \le c_2 \frac{2}{\ModifiedLoss_t \phi(\ModifiedLoss_t)^2}.
\end{equation*}
Since
\begin{equation*}
    \vert \Phi'(\ModifiedLoss_t) \vert \le \frac{C}{\ModifiedLoss_t \phi(\ModifiedLoss_t)}, \quad \rho_t' = \eta \frac{v_t}{\rho_t},
\end{equation*}
we deduce that
\begin{align*}
    \frac{M \vert \Phi'(\ModifiedLoss_t) \vert \rho_t'}{\rho_t} \le \eta \frac{C M}{\ModifiedLoss_t \phi(\ModifiedLoss_t)} \frac{v_t}{\rho_t^2}.
\end{align*}
It remains to show
\begin{equation*}
    C \eta v_t \le c_2 \frac{\rho_t^2}{\phi(\ModifiedLoss_t)}.
\end{equation*}
Using our estimates on $v_t$ from the proof of part (c), we know that
\begin{equation*}
    C \eta v_t \le C \eta \rho_t^{M} \Loss_t.
\end{equation*}
To complete the proof, we need an upper bound on $\phi(\ModifiedLoss_t)$. By definition,
\begin{equation*}
    \phi (\ModifiedLoss_t) = \log \frac{1}{n e^{\homop_a(\rho_t)} \Loss_t} \le \log \frac{1}{\Loss_t}.
\end{equation*}
We thus obtain that
\begin{equation*}
    C \eta v_t \phi(\ModifiedLoss_t) \le C \eta \rho_t^{M} \Loss_t \log \frac{1}{\Loss_t} \ll \rho_t^2
\end{equation*}
as $t \to \infty$, which is easily seen from the converges rates of $\Loss_t$ and $\rho_t$.  Similarly, $c_2$ can be arbitrarily small and our proof works for any $\eta > 0$.

Combining parts (c) and (d), it follows immediately that
\begin{equation*}
    \GDmargin (\param_{t+1}) - \GDmargin (\param_t) \ge (1 - c_1 - c_2) (\GDmargin)' (t),
\end{equation*}
completing the proof of \Cref{eq: decomp-gamma-GD}. The first part of \Cref{eq: decomp-zeta-GD} can be shown by direct calculation. For the second part, we have
\begin{align*}
    \zeta_t = \, \| \tilde{\param}_{t+1} - \tilde{\param}_t \| = \left\| \int_{0}^{1} \frac{\mathrm{d} \tilde{\param}_{t+z} }{\mathrm{d} z } \mathrm{d} z \right\| \le \int_{0}^{1} \left\| \frac{\mathrm{d} \tilde{\param}_{t+z} }{\mathrm{d} z } \right\| \mathrm{d} z.
\end{align*}
We will show that $\| \dot{\tilde{\param}}_{t+z} \| := \left\| \frac{\mathrm{d} \tilde{\param}_{t+z} }{\mathrm{d} z } \right\|$ is a non-increasing function of $z$, so that
\begin{equation*}
    \zeta_t \le \, \| \dot{\tilde{\param}}_t \| = \eta \frac{\| \partial_\perp \Loss_t\|}{\rho_t},
\end{equation*}
where the last equality follows from the proof of \cref{lem:decomp-radial-spherical}. Note that by direct calculation,
\begin{align*}
    \left\| \dot{\tilde{\param}}_{t+z} \right\| = \, & \left\| \frac{1}{\rho_{t+z}} \frac{\mathrm{d}  \param_{t + z}}{\mathrm{d} z} - \frac{1}{\rho_{t+z}} \bigg\langle\frac{\mathrm{d} \param_{t + z}}{\mathrm{d} z}, \tilde{\param}_{t+z} \bigg \rangle \cdot \tilde{\param}_{t+z} \right\| \\
    \le \, & \frac{\eta}{\rho_t} \left\| \nabla \Loss (\param_t) - \bigg\langle \nabla \Loss (\param_t), \tilde{\param}_{t+z} \bigg \rangle \cdot \tilde{\param}_{t+z} \right\| \\
    \le \, & \frac{\eta}{\rho_t} \left\| \nabla \Loss (\param_t) - \bigg\langle \nabla \Loss (\param_t), \tilde{\param}_t \bigg \rangle \cdot \tilde{\param}_t \right\| = \left\| \dot{\tilde{\param}}_{t} \right\|,
\end{align*}
since $\rho_{t+z} \ge \rho_t$ and $\langle - \nabla \Loss (\param_t), \tilde{\param}_t \rangle \ge 0$. This proves \Cref{eq: decomp-zeta-GD}.
\end{proof}

Then, we need some inequalities to connect $\big( \partial_r \Loss_t, \partial_\perp \Loss_t \big)$ and $\big( \partial_r\GDmargin (\param_t), \partial_\perp \GDmargin (\param_t) \big)$.  The main idea is to show that their ratios are close. i.e., 
\[
    \frac{\| \partial_\perp \Loss_t\|}{\| \partial_r \Loss_t\|} \approx \frac{\| \partial_\perp \GDmargin (\param_t)\|}{\| \partial_r \GDmargin (\param_t)\|}.
\] 
The technique is to use $a_t$ to bridge these two ratios. To this end, we need the following auxiliary result:

\begin{lemma}\label{lem:gd_interpolation_bound}
    There exists a constant $C > 0$ such that for all sufficiently large $t \ge s$ and $\alpha \in [0, 1]$,
    \begin{equation}\label{eq:gd_interpolation_bound}
        \| \param_{t+\alpha} \| \le C \| \param_t \|, \quad \| \partial \GDmargin (\param_{t+\alpha}) \| \le C \| \partial \GDmargin (\param_t) \|.
    \end{equation}
\end{lemma}

\begin{proof}
    We first prove $\| \param_{t+\alpha} \| \le C \| \param_t \|$. By definition,
    \begin{equation*}
        \| \param_{t+\alpha} \| \le \, \| \param_t \| + \eta \| \nabla \Loss (\param_t) \|.
    \end{equation*}
    From the proof of \Cref{lem: Decomposition of raidal and spherical parts for GD}, we get the following estimate:
    \begin{equation*}
        \| \nabla \Loss (\param_t) \| \le \, B_1 \Loss_t \rho_t^{M-1} \le \frac{B_1}{n} \exp(- \homop_a(\rho_t)) \rho_t^{M-1} \le \rho_t
    \end{equation*}
    for large enough $t$. Hence, $\| \param_{t+\alpha} \| \le C \| \param_t \|$. For the second part of \Cref{eq:gd_interpolation_bound}, we have for any $\param$,
    \begin{align*}
        \partial \GDmargin (\param) = \, & \frac{1}{\| \param \|^M} \left( \nabla a(\param) - M a(\param) \frac{\param}{\| \param \|^2} \right),
    \end{align*} 
    where $a(\param) = \exp(\Phi(\ModifiedLoss(\param)))$. It suffices to show that
    \begin{equation*}
        \left\| \nabla a(\param_{t+\alpha}) - M a(\param_{t+\alpha}) \frac{\param_{t+\alpha}}{\| \param_{t+\alpha} \|^2} \right\| \le \, C \left\| \nabla a(\param_t) - M a(\param_t) \frac{\param_t}{\| \param_t \|^2} \right\|.
    \end{equation*}
    To this end, we first give a lower bound on the right hand side. By Cauchy-Schwarz inequality,
    \begin{equation}\label{eq:lowerbd_A_theta}
    \begin{split}
        & \left\| \nabla a(\param_t) - M a(\param_t) \frac{\param_t}{\| \param_t \|^2} \right\| \ge \, \frac{1}{\rho_t} \left\vert \langle \nabla a_t, \param_t \rangle - M a_t \right\vert \\
        \ge \, & \frac{1}{\rho_t} \left\vert a_t \Phi' (\ModifiedLoss_t) \left( - e^{\homop_a (\rho_t)} v_t + \ModifiedLoss_t \homop_a' (\rho_t) \rho_t \right) - M a_t \right\vert \\
        \stackrel{(i)}{\ge} \, & \frac{1}{\rho_t} a_t \frac{2}{\ModifiedLoss_t \phi(\ModifiedLoss_t)^2} \cdot M \ModifiedLoss_t \phi(\ModifiedLoss_t) = \frac{M a_t}{\rho_t} \cdot \frac{2}{\phi(\ModifiedLoss_t)}, 
    \end{split}
    \end{equation}
    where $(i)$ follows similarly as in the proof of \Cref{lem: Decomposition of raidal and spherical parts for GD}. Denoting
    \begin{equation*}
        A(\param) = \, \nabla a(\param) - M a(\param) \frac{\param}{\| \param \|^2},
    \end{equation*}
    then we have for some $\epsilon \in [0, \alpha]$,
    \begin{align*}
        \left\Vert A(\param_{t+\alpha}) - A(\param_t) \right\Vert \le \, & \| \nabla A(\param_{t+\epsilon}) \| \| \param_{t+\alpha} - \param_t \| \\
        = \, & \eta \| \nabla \Loss (\param_t) \| \| \nabla A(\param_{t+\epsilon}) \|.
    \end{align*}
    We now upper bound $\| \nabla A(\param_{t+\epsilon}) \|$ in the above display. For any $\param$,
    \begin{align*}
        \| \nabla A(\param) \| = \, & \left\| \nabla^2 a(\param) - M \frac{\param \nabla a (\param)^\top}{\| \param \|^2} - M a(\param) \left( \frac{\| \param \|^2 I - \param \param^\top}{\| \param \|^4} \right) \right\| \\
        \le \, & \Vert \nabla^2 a(\param) \Vert + M \frac{\| \nabla a(\param) \|}{\| \param \|} + M a(\param) \frac{1}{\| \param \|^2}.
    \end{align*}
    We estimate these terms respectively. For $\| \param \|$ large enough:
    \begin{align*}
        \Vert \nabla a(\param) \Vert = \, & a(\param) \vert \Phi'(\ModifiedLoss(\param)) \vert \Vert \nabla \ModifiedLoss(\param) \Vert \le C_1 a(\param) \vert \Phi'(\ModifiedLoss(\param)) \vert \ModifiedLoss(\param) \left( \| \param \|^{M-1} + 1 \right) \\
        \le \, & C_1 a(\param) \frac{\| \param \|^{M-1} + 1}{\phi(\ModifiedLoss(\param))} \le C a(\param) \frac{\| \param \|^{M-1}}{\phi(\ModifiedLoss(\param))}, \\
        \Vert \nabla^2 a(\param) \Vert \le \, & \vert \Phi'(\ModifiedLoss(\param)) \vert \Vert \nabla \ModifiedLoss(\param) \Vert \Vert \nabla a(\param) \Vert + a(\param) \vert \Phi'' (\ModifiedLoss(\param)) \vert \Vert \nabla \ModifiedLoss(\param) \Vert^2 + a(\param) \vert \Phi' (\ModifiedLoss(\param)) \vert \Vert \nabla^2 \ModifiedLoss(\param) \Vert \\
        \le \, & C_1^2 a(\param) \frac{\left( \| \param \|^{M-1} + 1 \right)^2}{\phi(\ModifiedLoss(\param))^2} + C_1 a(\param) \frac{\left( \| \param \|^{M-1} + 1 \right)^2}{\phi(\ModifiedLoss(\param))} + C_2 a(\param) \frac{\| \param \|^{2M-2} + 1}{\phi(\ModifiedLoss(\param))} \\
        \le \, & C a(\param) \frac{\left( \| \param \|^{M-1} + 1 \right)^2}{\phi(\ModifiedLoss(\param))} \le C a(\param) \frac{\| \param \|^{2M - 2} }{\phi(\ModifiedLoss(\param))}.
    \end{align*}
    Here, the upper bounds on $\| \nabla \ModifiedLoss(\param) \|$ and $\| \nabla^2 \ModifiedLoss(\param) \|$ are due to \Cref{lem:gradient_hessian_bound}. We thus obtain that
    \begin{align*}
        \| \nabla A(\param) \| \le \, C \left( a(\param) \frac{ \| \param \|^{2 M - 2}}{\phi(\ModifiedLoss(\param))} + M a(\param) \frac{ \| \param \|^{M - 2}}{\phi(\ModifiedLoss(\param))} + M \frac{a(\param)}{\| \param \|^2} \right) \le C a(\param) \| \param \|^{2 M - 2}
    \end{align*}
    for sufficiently large $\| \param \|$.
    Further since $\vert \ModifiedLoss (\param_{t+\epsilon} ) - \ModifiedLoss (\param_t) \vert = o (\vert \ModifiedLoss (\param_t) \vert)$ as $t \to \infty$, we know that for sufficiently large $t$, $a(\param_{t+\epsilon}) \le 2 a(\param_t)$,
    thus leading to
    \begin{equation*}
        \| \nabla A(\param_{t+\epsilon}) \| \le \, C a(\param_{t+\epsilon}) \rho_{t+\epsilon}^{2M - 2} \le C a(\param_t) \rho_t^{2M - 2}.
    \end{equation*}
    It finally follows that
    \begin{align*}
        \left\Vert A(\param_{t+\alpha}) - A(\param_t) \right\Vert \le \, & \eta \| \nabla \Loss (\param_t) \| \| \nabla A(\param_{t+\epsilon}) \| \le C \eta \| \nabla \Loss_t \| a_t \rho_t^{2M - 2} \\
        \le \, & C \eta \Loss_t a_t \rho_t^{3M - 3} \le C \eta \ModifiedLoss_t a_t \exp(- \homop_a(\rho_t)) \rho_t^{3M - 3} \\
        \le \, & \frac{M a_t}{\rho_t} \cdot \frac{2}{\phi(\ModifiedLoss_t)} \le \Vert A(\param_t) \Vert
    \end{align*}
    for sufficiently large $t$,
    where the last inequality is just \Cref{eq:lowerbd_A_theta}, and the second-to-last one is because of $\ModifiedLoss_t \phi(\ModifiedLoss_t) \le C$ for any $\ModifiedLoss_t < 1 / ne^2$.
    This finally leads to
    \begin{equation*}
        \left\| \nabla a(\param_{t+\alpha}) - M a(\param_{t+\alpha}) \frac{\param_{t+\alpha}}{\| \param_{t+\alpha} \|^2} \right\| \le \, 2 \left\| \nabla a(\param_t) - M a(\param_t) \frac{\param_t}{\| \param_t \|^2} \right\|,
    \end{equation*}
    completing the proof of \Cref{lem:gd_interpolation_bound}.
\end{proof}

\begin{lemma}
    [Inequalities between $a_t$ and $\GDmargin (\param_t)$]
    \label{lem: Inequalities between a and gamma GD}
    Under \Cref{asp:initial-cond-gd}, for all sufficiently large $t\ge s$,
    \begin{equation}
    \label{eq: lowerbound partial_r a GD}
        \| \partial_r a_t\| \ge M \GDmargin(\param_s) \rho_t^{M-1}, 
    \end{equation}
    and 
    \begin{equation}
    \label{eq: upperbound partial_r gamma^c GD}
        \| \partial_r \GDmargin (\param_t)\| \le \frac{C_1 M \log n + C_2 \homop'(\rho_t)}{\rho_t^{M+1}}. 
    \end{equation}
    Combining these two inequalities, there exists a threhold $s_1>s>0$, for all sufficiently large $t \ge s_1$, we have 
    \begin{equation}
    \label{eq: partial_r a large partial_r gamma^c GD}
        \| \partial_r a_t\| \ge  \GDmargin(\param_s) \rho_t^{M+1 /2}  \| \partial_r \GDmargin (\param_t)\|.
    \end{equation}
\end{lemma}
\begin{proof}
    First, note that from the proof of \Cref{lem: Decomposition of raidal and spherical parts for GD}, we have
    \begin{equation}
        \| \partial_r a_t \| = \frac{1}{\rho_t} \langle \nabla a_t, \param_t \rangle \ge \frac{M a_t}{\rho_t} = M \GDmargin (\param_t) \rho_t^{M - 1} \ge M \GDmargin (\param_s) \rho_t^{M - 1},
    \end{equation}
    which proves \cref{eq: lowerbound partial_r a GD}.

    Second, again from the proof of \Cref{lem: Decomposition of raidal and spherical parts for GD}, we get
    \begin{align*}
        \| \partial_r \GDmargin (\param_t)\| = \, & \frac{1}{\rho_t^{M+1}} \left( \langle \nabla a_t, \param_t \rangle - M a_t \right) \\
        = \, & \frac{1}{\rho_t^{M+1}} \left( a_t \vert \Phi' (\ModifiedLoss_t) \vert \left( e^{\homop_a (\rho_t)} v_t - \ModifiedLoss_t \homop_a' (\rho_t) \rho_t \right) - M a_t \right) \\
        \stackrel{(i)}{\le} \, & \frac{1}{\rho_t^{M+1}} \left( a_t \vert \Phi' (\ModifiedLoss_t) \vert \ModifiedLoss_t \left( M \log \frac{1}{\Loss_t} + \homop'(\rho_t) - \homop_a' (\rho_t) \rho_t \right) - M a_t \right) \\
        = \, & \frac{1}{\rho_t^{M+1}} \left( a_t \vert \Phi' (\ModifiedLoss_t) \vert \ModifiedLoss_t \left( M \log n + M \phi(\ModifiedLoss_t) + 2 \homop'(\rho_t) \right) - M a_t \right) \\
        = \, & \frac{1}{\rho_t^{M+1}} \left( a_t \vert \Phi' (\ModifiedLoss_t) \vert \ModifiedLoss_t \left( M \log n + 2 \homop'(\rho_t) \right) + M a_t \left( \vert \Phi' (\ModifiedLoss_t) \vert \ModifiedLoss_t \phi(\ModifiedLoss_t) - 1 \right) \right)  \\
        \le \, & \frac{1}{\rho_t^{M+1}} \left( C_1 \left( M \log n + 2 \homop'(\rho_t) \right) + C_2 M \right) \\
        \le \, & \frac{1}{\rho_t^{M+1}} \left( C_1' M \log n + C_2' \homop'(\rho_t) \right),
    \end{align*}
    where
    \begin{equation*}
        C_1 = \sup_{0 < x < 1/ne^3} e^{\Phi(x)} \vert \Phi'(x) \vert x, \quad C_2 = \sup_{0 < x < 1/ne^3} e^{\Phi(x)} \left( \vert \Phi' (x) \vert x \phi(x) - 1 \right)
    \end{equation*}
    are positive constants (easily seen from the definition of $\Phi$), and $(i)$ follows from \cref{lem: Bound of v_t}.

    Finally, the proof of \cref{eq: partial_r a large partial_r gamma^c GD} follows the same way as \cref{eq: partial_r a large partial_r gamma^c}.
\end{proof}

Now we will use  \Cref{lem:KL1,lem:KL2} to prove \Cref{lem: Desingularizing function for GD}.

\begin{proof}
    [Proof of \Cref{lem: Desingularizing function for GD}]
    
    Recall that we have the following decompisiton in \Cref{eq: decomp-gamma-GD}: 
    \[
        \GDmargin (\param_{t+1}) - \GDmargin (\param_t) \ge c\eta \Big(\| \partial_r \GDmargin(\param_t) \| \| \partial_r \Loss_t\| + \| \partial_\perp \GDmargin (\param_t) \| \| \partial_\perp \Loss_t\| ]\Big).
    \]
    Two cases will be considered in this proof: 
    \begin{itemize}
        \item Case 1: $\| \partial_r \GDmargin (\param_t) \| \| \partial_r \Loss_t\|$ is larger, and we will apply \Cref{lem:KL2} for construction. 
        \item Case 2: $\| \partial_\perp \GDmargin (\param_t) \| \| \partial_\perp \Loss_t\|$ is larger, and we will apply \Cref{lem:KL1} for construction.
    \end{itemize}
    The two cases will be determined by the ratio of $\| \partial_\perp \GDmargin(\param_t) \|$ and $\| \partial_r \GDmargin(\param_t) \|$. In case 1, we assume that: 
    \begin{equation}
    \label{eq: case1-GD}
        \| \partial_\perp \GDmargin (\param_t) \| \le \rho_t^{\frac{1}{8}} \| \partial_r \GDmargin (\param_t) \|. 
    \end{equation}
    For case 2, the condition is: 
    \begin{equation}
    \label{eq: case2-GD}
    \| \partial_\perp \GDmargin (\param_t) \| \ge  \rho_t^{\frac{1}{8}} \| \partial_r \GDmargin (\param_t) \|. 
    \end{equation}
    
    \noindent{\bf Case 1.} By \Cref{eq: partial_r a large partial_r gamma^c GD}, we have 
    \begin{align}
        \label{eq: ratio-a}
        \| \partial_r a_t \| 
        &\ge \GDmargin(\param_s) \rho_t^{M+\frac{1}{2}} \| \partial_r \GDmargin (\param_t) \|  \notag\\
        &\ge \GDmargin(\param_s) \rho_t^{M + \frac{3}{8}} \| \partial_\perp \GDmargin (\param_t)\| &&\explain{ By \Cref{eq: case1}} \notag \\ 
        &\ge \GDmargin(\param_s) \rho_t^{\frac{3}{8}} \| \partial_\perp a_t\|. &&\explain{ By \Cref{lem: Decomposition of raidal and spherical parts for GD}}. 
    \end{align}
    
    Now we need to transfer this inequality to the ratio of $\| \partial_\perp \Loss_t\|$ and $\| \partial_r \Loss_t\|$.  Note that 
    \begin{align*}
        \| \partial_r a_t \| &=\frac{1}{\rho_t} \langle \nabla a_t, \param_t  \rangle = -\frac{\langle \nabla \Loss_t, \param_t \rangle }{\Loss_t \rho_t} - \homop_a^\prime (\rho_t)  \\ 
        & \le -\frac{\langle \nabla \Loss_t, \param_t \rangle }{\Loss_t \rho_t} = \frac{1}{\Loss_t} \| \partial_r \Loss_t\|.
    \end{align*}
    On the other hand, we have 
    \[
        \| \partial_\perp a_t \| = \frac{1}{\Loss_t} \| \partial_\perp \Loss_t\|.
    \]
    Combining these two inequalities and pluging them into \eqref{eq: ratio-a}, we have 
    \begin{equation}
    \label{eq: ratio-L}
            \| \partial_r \Loss_t\| \ge \GDmargin(\param_s)\rho_t^{\frac{3}{8}} \| \partial_\perp \Loss_t\|.
    \end{equation}
    On the other hand, we know that there exists $s_2>s>0$ such that for a.e. $t\ge s_2$, we have $\rho_t >1$. Hence, we have 
    \begin{equation}
    \label{eq: gamma_rad}
        \| \partial \GDmargin (\param_t) \| \le  \| \partial_r \GDmargin (\param_t) \| + \| \partial_\perp \GDmargin (\param_t) \| \le 2 \rho_t^{\frac{1}{8}} \| \partial_r \GDmargin (\param_t) \|.
    \end{equation}
    Therefore, in terms of \eqref{eq: decomp-gamma-GD}, we have 
    \begin{align}
        \label{eq: psi1_key}
        \GDmargin (\param_{t+1}) - \GDmargin (\param_t)
        &\ge  c \eta \| \partial_r \GDmargin (\param_t) \| \| \partial_r \Loss_t\| \notag\\ 
        &\ge  c \eta \frac{\GDmargin(\param_s)}{2} \rho_t^{\frac{1}{4}} \| \partial  \GDmargin (\param_t) \| \| \partial_\perp \Loss_t\|  &&\explain{ By \Cref{eq: ratio-L} and \eqref{eq: gamma_rad}} \notag\\
        &\ge  c\rho_t^{\frac{5}{4}} \| \partial  \GDmargin (\param_t) \| \zeta_t. &&\explain{ By \Cref{lem: Decomposition of raidal and spherical parts for GD}} 
    \end{align}
    
    Now we invoke \Cref{lem:KL2} to construct the desingularizing function.  We apply it to the definable function $\gamma_* - \GDmargin(\param)$ with $\lambda=\frac{1}{4}$. Then there exists $\nu_1 >0$ and a definable desingularizing function $\Psi_1$ on $[0, \nu_1)$ such that 
    \[
        8 \Psi_1^\prime(\gamma_* - \GDmargin(\param)) \| \param \|^{5/4} \|\partial \GDmargin(\param)\| \ge 1, \quad \text{ if }  \GDmargin(\param) \ge \gamma_* - \nu_1.
    \] 
    Further, \Cref{lem:gd_interpolation_bound} implies that there exists a constant $c>0$, such that for all $\alpha \in [0, 1]$,
    \begin{equation}
        8 \Psi_1^\prime(\gamma_* - \GDmargin(t+\alpha)) \rho_t^{5/4} \|\partial \GDmargin(\param_t) \| \ge c, \quad \text{ if }  \GDmargin(\param_t) \ge \gamma_* - \nu_1.
    \end{equation}
    Plugging \eqref{eq: psi1_key} into the above inequality, we have 
    \[
        8 \Psi_1^\prime(\gamma_* - \GDmargin (\param_{t+\alpha})) \left( \GDmargin (\param_{t+1}) - \GDmargin (\param_t) \right) \ge c \zeta_t, \quad \text{ if }  \GDmargin (\param_t) \ge \gamma_* - \nu_1.
    \]
    This completes the proof for case 1. 
    
    \noindent{\bf Case 2.} By \cref{lem: Decomposition of raidal and spherical parts for GD}, we have 
    \begin{equation}
    \label{eq: case2-key1-gd}
           \GDmargin (\param_{t+1}) - \GDmargin (\param_t) \ge \, c \eta \| \partial_\perp \GDmargin(\param_t) \| \| \partial_\perp \Loss_t\| \ge c \rho_t \|  \partial_\perp \GDmargin(\param_t) \| \zeta_t. 
    \end{equation}
    For a.e. $t\ge s_1 >s >0$, we have 
    \[
        \| \partial_\perp \GDmargin (\param_t) \| \ge  \rho_t^{\frac{1}{8}} \| \partial_r \GDmargin (\param_t) \| \ge \| \partial_r \GDmargin (\param_t) \|.
    \]
    This leads to 
    \[
        \| \partial_\perp \GDmargin (\param_t) \| \ge \frac{1}{2} \| \partial \GDmargin (\param_t)\|. 
    \]
    Plugging this into \eqref{eq: case2-key1-gd}, we have 
    \begin{equation}
    \label{eq: case2-key2-gd}
    \GDmargin (\param_{t+1}) - \GDmargin (\param_t) \ge \, \frac{c \rho_t}{2} \| \partial  \GDmargin (\param_t)\| \zeta_t.
    \end{equation}
    
    We invoke \Cref{lem:KL1} to construct the desingularizing function.  We apply it to the definable function $\gamma_* - \GDmargin(\param)$ with $c=1$ and $\eta=\frac{1}{8}$. Then there exists $\nu_2 >0$ and a definable desingularizing function $\Psi_2$ on $[0, \nu_2)$ such that
    \[
        \Psi_2^\prime(\gamma_* - \GDmargin(\param)) \| \param \| \|\partial \GDmargin(\param)\| \ge 1, \quad \text{ if }  \GDmargin(\param) \ge \gamma_* - \nu_2.
    \]
    Further, \Cref{lem:gd_interpolation_bound} implies that there exists a constant $c>0$, such that for all $\alpha \in [0, 1]$,
    \begin{equation}
        \Psi_2^\prime(\gamma_* - \GDmargin(t+\alpha)) \rho_t \|\partial \GDmargin(\param_t) \| \ge c, \quad \text{ if }  \GDmargin(\param_t) \ge \gamma_* - \nu_2.
    \end{equation}
    Plugging \eqref{eq: case2-key2-gd} into the above inequality, we have 
    \[
        2\Psi_2^\prime(\gamma_* - \GDmargin (\param_{t+\alpha})) \left( \GDmargin (\param_{t+1}) - \GDmargin (\param_t) \right) \ge c \zeta_t, \quad \text{ if }  \GDmargin (\param_t) \ge \gamma_* - \nu_2.
    \]
    This completes the proof for case 2. 
    
    \noindent{\bf Combining the two cases.}
    Since $\Psi_1^\prime  - \Psi_2^\prime$ is a definable function, it's nonnegative or nonpositive on some interval $(0, \nu)$. Let $\Psi = \max \{\Psi_1, \Psi_2\}$. Then, we have for a.e. large enought $t$ such that $\rho_t \ge 1$ and $\GDmargin (\param_t) \ge \gamma_* - \nu$, and $\log n + 2p_a(\rho_t) \le \rho_t^{M-\frac{1}{2}}$, it holds that 
    \[
        c \Psi' (\gamma_* - \GDmargin (\param_{t+\alpha})) \left( \GDmargin (\param_{t+1}) - \GDmargin (\param_t) \right) \ge \, \zeta_t,
    \]
    for some constant $c>0$. The final conclusion follows directly from the Lagrange mean value theorem. This completes the proof of \Cref{lem: Desingularizing function for GD}.
\end{proof}

\subsection{KKT Convergence: Proof of Theorem \ref{thm: directional convergence-gd} (Part 2)}\label{sec:KKT_convergence_gd}
The main idea is to verify the KKT conditions. Recall that the optimization problem \Cref{eq: KKT} is defined as follows: 
\[
    \min  \| \param\|_2^2, \quad \text{ s.t. } y_i f_\homo (\param;\xB_i) \ge 1 \text{ for all } i\in [n]. 
\]
Following the notations in \Cref{sec:proof:KKT}, we have 
\begin{align*}
    &\bar f_i (\param ) = y_i f(\param ;\xB_i), \quad \bar f_{\min} (\param ) = \min_{i\in [n]} \bar f_i (\param ),\\ 
    &\bar f_{\homo,i} (\param ) = y_i f_\homo(\param ;\xB_i), \quad \bar f_{\homo,\min} (\param ) = \min_{i\in [n]} \bar f_{\homo,i} (\param ),\\ 
    &\hB_i(\param ) = \partial \bar{f}_i(\param ), \quad \hB (\param )= \frac{1}{n} \sum_{i=1}^{n} e^{-\bar f_i(\param )} \hB_i(\param ), \\  
    &\hB_{\homo,i}(\param ) = \partial \bar{f}_{\homo,i}(\param ), \quad \hB_{\homo} (\param )= \frac{1}{n} \sum_{i=1}^{n} e^{-\bar f_i(\param )} \hB_{\homo,i}(\param ). 
\end{align*}

We are going to verify that $\hat \param_t \coloneqq \param_t / (\bar f_{\homo, \min} (\param_t))^{1/M}$ satisfies the two conditions in \Cref{def: approx KKT point of (P)}, i.e., 
\begin{itemize}
    \item [1.] $\|\param - \sum_{i=1}^n \lambda_i \hB_{\homo,i} (\param) \| \le \epsilon $, where $\hB_{\homo,i} (\param) = \partial \bar{f}_{\homo,i}(\param)$ for all $i\in [n]$;
    \item [2.]For any $i\in [n]$, $\lambda_i \big( \bar f_{\homo,i}(\param ) -1 \big) \le \delta $. 
\end{itemize}
Recall that $\lambda_i$ and $\beta$ are constructed as follows:
\begin{align*}
    \lambda_i(\param ) \coloneqq \frac{\bar f_{\homo,\min}^{1-2 / M} \rho_t e^{-\bar f_i(\param )}}{n \| \hB_{\homo} (\param) \|} , \quad 
    \beta_t \coloneqq \frac{\langle \param_t, \hB_{\homo} (\param_t)  \rangle }{ \| \param_t \| \|\hB_{\homo}(\param_t)\|}.
\end{align*}
Since our model assumptions are the same as those of the GF case, \Cref{lem:bound-of-two-min}, \Cref{lem: tilde theta KKT condition} and \Cref{lem: q goes to 0} hold here. Hence, the second condition is satisfied. All we need to show is that 
\[
    \beta_t = \frac{\langle \param_t, \hB_{\homo} (\param_t)  \rangle }{ \| \param_t \| \|\hB_{\homo}(\param_t)\|} \to 1. 
\] 
Similarly, we will only show a subsequence of $\beta$ goes to $0$. Note that by \Cref{asp:strongerhomo} and \Cref{thm:homogenization}, there is a function $\homor(x) = o(x^{M-1})$ as $x \to \infty$, such that for almost every $\param_t$ and any $i\in [n]$, we have
\[
\| \nabla \bar f_{i}(\param_t) - \nabla \bar f_{\homo,i}(\param_t)\| \le \homor(\|\param_t\|) = \homor(\rho_t) .
\]

\begin{lemma}
[Bound of $\beta$ in GD]
\label{lem: Bound of beta in GD}
Under \Cref{asp:initial-cond-gd,asp:strongerhomo}, we have for any $t_2>t_1 >s$,
\begin{equation}
\label{eq: bound of beta in GD}
    \sum_{t=t_1}^{t_2} \bigg[ \frac{1-p_1(t)}{(p_2(t)+\beta_t)^2} -1\bigg] \cdot \log \frac{\rho_{t+1}}{\rho_t} \le \frac{1}{M} \log \frac{\GDmargin(\param_{t_2})}{\GDmargin(\param_{t_1})},
\end{equation}
where 
\[
p_1(t) = \frac{2\homor (\rho_t)}{M\GDmargin(\param_s) \rho^{M-1}_t}, \quad p_2(t) = \frac{ \homop_a(\rho_t)}{M\GDmargin(\param_s) \rho^{M}_t}.  
\] 
\end{lemma}
\begin{proof}

Note that in \Cref{thm:combined_argument}, we have shown: 
\[
    \log \GDmargin(\param_{t+1}) - \log \GDmargin(\param_t) \ge  \log \frac{\rho_{t+1}}{\rho_t} \frac{M\rho_t^2\|\partial_\perp  \Loss_t\|_2^2}{v_t^2 } . 
\] 
We will give a lower bound of $\frac{M\rho_t^2\|\partial_\perp  \Loss_t\|_2^2}{v_t^2 }$. For the denominator, we have 
\begin{equation}
\label{eq:bound-denom-kkt}
    v_t  = \langle -\nabla \Loss_t, \param_t \rangle = \frac{1}{n}\sum_{i=1}^n e^{-\bar f_i(\param_t)} \langle  \nabla \bar f_i(\param_t), \param_t \rangle \le  \langle \hB_{\homo}(\param_t), \param_t \rangle  + 2M \Loss_t \homop_a(\rho_t). 
\end{equation}
For the numerator, we have 
\begin{equation}
\label{eq:nabla-Lt-square}
        \|\partial_\perp \Loss_t\|_2^2 = \|\nabla \Loss_t\|_2^2 - \langle \nabla \Loss_t, \tilde{\param}_t \rangle ^2 = \| \hB (\param_t) \|^2 - \langle \hB (\param_t), \tilde{\param}_t \rangle ^2. 
\end{equation}
Similarly, we have 
\[
    \|\hB (\param_t) - \hB_{\homo} (\param_t) \| \le \frac{1}{n} \sum_{i=1}^n e^{-\bar f_i(\param_t )} \| \hB_i (\param_t) - \hB_{\homo,i} (\param_t) \| \le \Loss_t\homor (\rho_t).
\]
And we can get 
\[
    \| \hB (\param_t) \| \ge \| \hB_{\homo} (\param_t) \| - \| \hB (\param_t) - \hB_{\homo} (\param_t) \|  \ge \| \hB_{\homo} (\param_t) \| - \Loss_t\homor (\rho_t).
\] 
In what follows, we will use $\hB$ and $\hB_{\homo}$ as shorthands for $\hB (\param_t)$ and $\hB_{\homo} (\param_t)$, repsectively. As $\rho_t \to \infty$, 
\begin{align}
     \| \hB_{\homo}\| &\ge \frac{\langle \hB_{\homo}, \param_t  \rangle }{\rho_t} = \frac{M\sum_{i=1}^{n} e^{-\bar f_i(\param_t )} \bar f_{\homo,i}(\param_t ) }{n \rho_t}  \notag\\ 
     & \ge  \frac{M\sum_{i=1}^{n} e^{-\bar f_i(\param_t )} \big (\bar f_i(\param_t ) - \homop_a(\rho_t) \big) }{n \rho_t} \notag\\ 
     & \ge \frac{M\Loss_t \phi(\ModifiedLoss_t)}{\rho_t} \notag \ge M\Loss_t\GDmargin(\param_t) \rho_t^{M-1} \notag \\
     & \ge M\Loss_t \GDmargin(\param_s) \rho_t^{M-1} \ge 2\Loss_t\homor (\rho_t) \label{eq: h_M lb gd}
\end{align}
for sufficiently large $\rho_t$, since $\homor (\rho_t)= o(\rho_t^{M-1})$ as $t \to \infty$.
Therefore, there exists $s_5 > s>0$ such that for a.e. $t\ge s_5$, we have 
\begin{align}
     \|\hB\|^2 & = \|\hB_{\homo}\|^2 + 2 \langle \hB_{\homo}, \hB - \hB_{\homo} \rangle + \|\hB -\hB_{\homo}\|^2 \notag   \\
     &\ge \|\hB_{\homo}\|^2 + 2 \langle \hB_{\homo}, \hB - \hB_{\homo} \rangle \notag\\ 
     &\ge \|\hB_{\homo}\|^2(1 - 2\|\hB_{\homo}\|^{-1} \Loss_t\homor (\rho_t)) \notag\\ 
     &\ge  \|\hB_{\homo}\|^2 \bigg(1 -\frac{ 2 \homor (\rho_t)}{M\GDmargin(\param_s) \rho_t^{M-1}} \bigg) .\label{eq: beta-bound-num1-GD}
\end{align}
At last, we have 
\begin{equation}
\label{eq: beta-bound-num2-GD}
      \langle \hB, \tilde \param_t  \rangle  = v_t/\rho_t \le  \langle \hB_{\homo}, \tilde{\param}_t  \rangle + \Loss_t \homop_a(\rho_t) / \rho_t.
\end{equation}
Plugging  \Cref{eq: beta-bound-num1-GD,eq: beta-bound-num2-GD} into \Cref{eq:nabla-Lt-square}, we can give a lower bound for $\|\partial_\perp \Loss_t\|_2^2$, 
\begin{equation}
\label{eq:bound-nabla-Lt}
\|\partial_\perp \Loss_t\|_2^2 \ge \|\hB_{\homo}\|_2^2 \bigg(1 -\frac{ 2 \homor (\rho_t)}{M\GDmargin(\param_s) \rho_t^{M-1}} \bigg)  - \big( \langle \hB_{\homo}, \tilde{\param}_t  \rangle + \Loss_t \homop_a(\rho_t) / \rho_t \big)^2. 
\end{equation}

Combining \Cref{eq:bound-nabla-Lt,eq:bound-denom-kkt}, we have 
\begin{align*}
    \frac{\rho_t^2\|\partial_\perp \Loss_t\|_2^2}{v_t^2 } &\ge  \frac{\|\hB_{\homo}\|_2^2 \Big\{1 -  2 \homor(\rho_t)/ \big[M\GDmargin(\param_s) \rho_t^{M-1}\big] \Big\}- \big( \langle \hB_{\homo}, \tilde{\param}_t  \rangle + \Loss_t \homop_a(\rho_t) / \rho_t \big)^2 }{ \big(\langle \hB_{\homo}, \tilde{\param}_t  \rangle + \Loss_t \homop_a(\rho_t) / \rho_t\big)^2}\\
    & = \frac{\|\hB_{\homo}\|_2^2 \Big\{1 -  2 \homor(\rho_t)/ \big[M\GDmargin(\param_s) \rho_t^{M-1}\big] \Big\}  }{ \big(\langle \hB_{\homo}, \tilde{\param}_t  \rangle + \Loss_t \homop_a(\rho_t) / \rho_t\big)^2} - 1\\ 
    & = \frac{1 -  2 \homor(\rho_t)/ \big[M\GDmargin(\param_s) \rho_t^{M-1}\big]  }{\big( \beta_t + \Loss_t \homop_a(\rho_t) / \rho_t \|\hB_{\homo}\|^{-1} \big)^2} - 1\\ 
    &\ge  \frac{1- 2 \homor(\rho_t)/(M\GDmargin(\param_s) \rho^{M-1}_t)}{\big( \beta_t +  \homop_a(\rho_t) / (M\GDmargin(\param_s) \rho^{M}_t) \big)^2} - 1\\ 
    & = \frac{1-p_1(t)}{\big(\beta_t + p_2(t)\big)^2}-1 . 
\end{align*}
This completes the proof of \Cref{lem: Bound of beta in GD}.
\end{proof}
Once we have this, we can show the following analogous results. We omit their proofs since they are completely similar to those of \Cref{lem: beta goes to 1} and \Cref{lem: approx kkT}.

\begin{lemma}
    [$\beta$ converges to 1]
    \label{lem: beta goes to 1gd}
    Under \Cref{asp:strongerhomo,asp:initial-cond-gd},  there exists a sequence $t_k$ such that $\lim_{k\to \infty}\beta_{t_k} \to 1$.
\end{lemma}
\begin{lemma}
    [Approximate KKT point]
    \label{lem: approx kkTgd}
    Under \Cref{asp:strongerhomo,asp:initial-cond-gd}, there exists a sequence $t_k$ such that $\tilde \param_{t_k}$ is an $(\epsilon_k, \delta_k)$-KKT point of \Cref{eq: KKT} for all $k\in \Nbb$, where $\epsilon_k \to 0$ and $\delta_k \to 0$ as $k\to \infty$.
\end{lemma}
We are now in position to prove the KKT convergence. Appplying \Cref{lem: approx kkTgd}, we have a sequence $\{t_k\}$ such that $\tilde \param_{t_k}$ is an $(\epsilon_k, \delta_k)$-KKT point of \Cref{eq: KKT} for all $k\in \Nbb$, where $\epsilon_k \to 0$ and $\delta_k \to 0$ as $k\to \infty$. Since $\tilde{\param}_{t_k}$ converges in direction, $\tilde{\param}_{t_k}$ will converge to the same direction as one of the KKT points of \Cref{eq: KKT}.  By \Cref{thm: directional convergence-gd}, we know $\tilde{\param}_{t_k}$ also converges to the limit $\param_*$. Hence, $\param_*/\big(\bar f_{\homo,\min}(\param _*)\big)^{1/M}$ is a KKT point of \Cref{eq: KKT}. This completes the proof of \Cref{thm: directional convergence-gd}. 

\section{Additional Lemmas} \label{sec:proof:additional}
\begin{lemma}
\label{lem:log-sum-convex}
    The function $\pi: \Rbb^n \to \Rbb$: 
    \[
        \pi(v) = \log \bigg ( \frac{1}{n} \sum_{i=1}^{n} \exp(-v_i) \bigg)
    \]
    is convex. 
\end{lemma}
\begin{proof}
Let $u_i = e^{-x_i}, v_i = e^{-y_i}$. So 
\[
    \pi\big(\theta x_i + (1-\theta) y_i\big) = \log \bigg( \frac{1}{n} \sum_{i=1}^{n} \exp(-\theta x_i - (1-\theta) y_i) \bigg) = \log \bigg( \frac{1}{n} \sum_{i=1}^{n} u_i^{\theta } v_i^{1-\theta } \bigg). 
\]
From H\"older's inequality, we have 
\[
    \sum_{i=1}x_i y_i \le  \bigg( \sum_{i=1} x_i^{\frac{1}{\theta }} \bigg)^{\theta} \bigg( \sum_{i=1} y_i^{\frac{1}{1-\theta }} \bigg)^{(1-\theta)}.
\]
Therefore, 
\[
  \log \bigg( \frac{1}{n} \sum_{i=1}^{n} u_i^{\theta } v_i^{1-\theta } \bigg) \le \log \Bigg[ \bigg(\frac{1}{n}\sum_{i=1}^n u_i^{\theta \cdot \frac{1}{\theta }} \bigg)^\theta  \cdot \bigg( \frac{1}{n} \sum_{i=1}^n v_i^{(1-\theta) \cdot \frac{1}{1-\theta}} \bigg)^{1-\theta} \Bigg] . 
\]
The right formula can be reduced to: 
\[
    \theta \log \bigg( \frac{1}{n} \sum_{i=1}^{n} u_i \bigg) + (1-\theta) \log \bigg( \frac{1}{n} \sum_{i=1}^{n} v_i \bigg).
\]
Therefore, we get: 
\[
    \pi(\theta x + (1-\theta) y) \le \theta \pi(x) + (1-\theta) \pi(y).
\]
This completes the proof of \Cref{lem:log-sum-convex}.
\end{proof}
\begin{lemma} 
\label{lem:int-log}
Let $s>1$ and $S(x) = \int_s^x \frac{1}{(\log t)^{2-2/M}} dt$. Then, 
\[
    S(x) = \Theta \bigg( \frac{x}{ (\log x)^{2-2/M}} \bigg), \quad S^{-1}(y) = \Theta \Big( y (\log y)^{2-2/M} \Big). 
\]
\end{lemma}
\begin{proof}

First, we bound the rate of $S(x)$.  Note that 
\[
    S(x)  = \int_s^x \frac{1}{(\log t)^{2-2/M}} dt  = \int_s^x \frac{(\log t)^{2/M}}{(\log t)^{2}} dt \le (\log x)^{2/M} \int_s^x \frac{1}{(\log t)^{2}} dt  = \Ocal  \bigg(  \frac{x}{ (\log x)^{2-2/M}}\bigg). 
\]
On the other hand, we have 
\[
    S(x) \ge \int_{\sqrt{x}}^x \frac{1}{(\log t)^{2-2/M}} dt \ge (  ( \log x)/2)^{2/M} \int_{\sqrt{x}}^x \frac{1}{(\log t)^{2}} dt = \Omega \bigg(  \frac{x}{ (\log x)^{2-2/M}}\bigg).
\]
Now we bound the rate of $S^{-1}(y)$. Let $x = S^{-1}(y)$ for $y\ge 0$. By the previous results, we have $x\to \infty$ as $y\to \infty$. Besides, we know that $y = S(x) = \Theta \Big( \frac{x}{ (\log x)^{2-2/M}} \Big)$. Taking logarithm on both sides, we have $\log y = \Theta ( \log x)$. Therefore, 
\[
    y (\log y)^{2-2/M} = \Theta ( y (\log x)^{2-2/M}) = \Theta (x). 
\] 
This implies that $x = \Theta (y (\log y)^{2-2/M})$. This completes the proof of \Cref{lem:int-log}.
\end{proof}

\begin{lemma} 
\label{lem:int-log-2}
Given two integers $m,n \in \Zbb_+$ and a constant $s<\frac{1}{n e^2}$, let $S(x) \coloneqq \int_x^s \frac{1}{t^2(\log (n t))^{2-2/M}} \mathrm{d} t$ for $x<\frac{s}{2}$. Then, 
\[
    S(x) = \Theta \bigg( \frac{1}{ x(\log (n x))^{2-2/M}} \bigg), \quad S^{-1}(y) = \Theta \bigg( \frac{1}{ y(\log (n y))^{2-2/M}} \bigg). 
\]
\end{lemma}
\begin{proof}
We prove the LHS rate first. 
 Now we can bound the rate of $S(x)$. Note that
\begin{align*}
    S(x) &= \int_x^s \frac{1}{t^2(\log (n t))^{2-2/M}} \mathrm{d} t 
    \ge   \frac{1}{(\log (n x))^{2-2/M}} \int_x^s \frac{1}{t^2} \mathrm{d} t \\ 
    & = \frac{1}{(\log (n x))^{2-2/M}} \bigg( \frac{1}{x} - \frac{1}{s} \bigg) = \Omega \bigg( \frac{1}{x(\log (n x))^{2-2/M}} \bigg).
\end{align*}
Besides, we know that 
\begin{align*}
    S(x) &= \int_x^s \frac{1}{t^2(\log (n t))^{2-2/M}} \mathrm{d} t \\ 
    &\le (\log(nx))^{2/M} \int_x^s \frac{1}{t^2 (\log (n t))^{2}} \mathrm{d} t \\
    &\le  (\log(nx))^{2/M} \bigg( -\frac{3}{t(\log (n t))^2} \bigg|_x^s \bigg) \\
    &\le  \frac{3(\log(nx))^{2/M}}{x(\log (n x))^2} = \Ocal \bigg( \frac{1}{x(\log (n x))^{2-2/M}} \bigg).
\end{align*}
Combine them, we get that 
\[
S(x) = \Theta \bigg( \frac{1}{x(\log (n x))^{2-2/M}} \bigg).
\] 
Let $y = S(x)$. Then, we have $ \log(x) = \Theta(-\log(y))$. Hence, we have 
\[
x = \frac{1}{y (\log (n x))^{2-2/M}}  = \Theta \bigg( \frac{1}{y(\log (n y))^{2-2/M}} \bigg).
\] 
This completes the proof of \Cref{lem:int-log-2}.

\end{proof}

\begin{lemma}
\label{lem: log over linear}
Given $0<\gamma <1$, when $t \ge \big(4/\gamma ( \log( 4/\gamma))\big)^4$, we have 
\[
    \frac{\log t}{2 \gamma} \le  t^{\frac{1}{4}}. 
\]
\end{lemma}
\begin{proof}

Let $z = t^{\frac{1}{4}}$. Then, we want to show: 
\[
   g(z) \coloneqq  \gamma z  - 2 \log z \ge 0. 
\]
Note that $g^\prime (z) = \gamma - \frac{2}{z}$. When $z \ge 2/\gamma$, $g(z)$ is increasing. Furthermore, we have 
\[
    g(4/\gamma \log( 4/\gamma))  = 4 \log(4/\gamma)  -  2 \log(4/ \gamma)   - 2 \log \log(4/\gamma) = 2\log \frac{4/\gamma}{\log(4/\gamma)}>0. 
\]
Since $4/\gamma \log( 4/\gamma)) > 2/\gamma$, we have finished the proof of \Cref{lem: log over linear}.
\end{proof}

\begin{lemma}[Moore-Osgood Theorem]
\label{lem:exchange-limit}
Assume that a series of functions $(f_n(x))_{n=1}^\infty$ converge uniformly to $f(x)$ in $(-a,a)$ for some $a>0$, and for any $n$, $\lim_{x\to 0} f_n(x) = L_n$ exists. Then, both $\lim_{x \to 0} f(x)$ and $\lim_{n \to \infty} L_n$ exist and are equal, namely
\[ 
\lim_{n\to\infty} L_n = \lim_{n\to \infty} \lim_{x\to 0} f_n(x) = \lim_{x\to 0} \lim_{n\to \infty} f_n(x) = \lim_{x\to 0} f(x). 
\]

\end{lemma}

\begin{proof}
Due to uniform convergence, for any $\epsilon>0$ there exist $N(\epsilon) \in \mathbb{N}$, such that:
for all $x \in (-a, a) \backslash \{ 0 \}$, $n, m>N$ implies $\left|f_n(x)-f_m(x)\right|<\frac{\epsilon}{3}$.
As $x \rightarrow 0$, we have $\left|L_n-L_m\right|<\frac{\epsilon}{3}$, which means that $L_n$ is a Cauchy sequence which converges to a limit $L$. In addition, as $m \rightarrow \infty$, we have $\left|L_n-L\right|<\frac{\epsilon}{3}$.
On the other hand, if we take $m \rightarrow \infty$ first, we have $\left|f_n(x)-f(x)\right|<\frac{\epsilon}{3}$.
By the existence of pointwise limit, for any $\epsilon>0$ and $n>N$, there exist $\delta(\epsilon, n)>0$, such that $0<|x|<\delta$ implies $\left|f_n(x)-L_n\right|<\frac{\epsilon}{3}$.
Then for that fixed $n, 0<|x|<\delta$ implies $|f(x)-L| \leq\left|f(x)-f_n(x)\right|+\left|f_n(x)-L_n\right|+\left|L_n-L\right| \leq \epsilon$.
This proves that $\lim _{x \rightarrow 0} f(x)=L=\lim _{n \rightarrow \infty} L_n$.
\end{proof}

